\documentclass[backref=page]{amsart}

\newcommand{\druckVersion}{0} %
\newcommand{\specialStyle}{1} %
\usepackage[english]{babel}
\usepackage{amssymb}
\usepackage{mathtools}
\usepackage{amsmath}
\usepackage{bbm}
\usepackage{graphicx}
\usepackage{caption}
\usepackage{subcaption}%

\usepackage{listings}
\usepackage{color}
\usepackage{transparent}
\usepackage{xcolor}
\usepackage{textcomp}
\usepackage[numbers]{natbib}
\usepackage[shortlabels]{enumitem} 
\usepackage[autostyle=false, style=english]{csquotes}
\MakeOuterQuote{"}
\usepackage{mathrsfs}
\usepackage{eurosym}
\usepackage{textcomp}
\usepackage{amsthm}
\usepackage{mdframed}
\usepackage{xpatch}
\usepackage{xifthen}
\usepackage{pdfpages}
\usepackage[percent]{overpic}
\usepackage{braket}
\usepackage[pdfusetitle]{hyperref}
\usepackage[nameinlink]{cleveref}
\usepackage{tablefootnote}
\usepackage[usestackEOL]{stackengine}
\usepackage{graphicx}
\usepackage{etoolbox}

\usepackage{crossreftools}

\definecolor{meinCiteGruen}{rgb}{0.0, 0.7, 0.05}
\definecolor{meineLinkFarbe}{rgb}{0.0, 0.0, 0.61}
\definecolor{orange}{rgb}{1, 0.5, 0.00}

\ifthenelse{\equal{\specialStyle}{1}}{
    \ifthenelse{\equal{\druckVersion}{1}}{
        \newcommand*{\meinGreyscale}{0.5}
	    \newcommand*{\meineTranzparenz}{0.5} %
	}{
	    \newcommand*{\meinGreyscale}{0.6}
	    \newcommand*{\meineTranzparenz}{0.4}
	}
}{
    \newcommand*{\meinGreyscale}{0}
	\newcommand*{\meineTranzparenz}{1}
}
\definecolor{hellgrau}{rgb}{\meinGreyscale, \meinGreyscale, \meinGreyscale}
\ifthenelse{\equal{\druckVersion}{1}}{
	\hypersetup{hidelinks}
}{
	\hypersetup{
		colorlinks   = true, %
		urlcolor     = blue, %
		linkcolor    = meineLinkFarbe, %
		citecolor   = meinCiteGruen, %
	}
}

\hypersetup{bookmarksdepth=3, %
	bookmarksopenlevel=2,
	pdfkeywords={early stopping} {implicit regularization} {machine learning} {neural networks} {spline} {regression} {gradient descent} {artificial intelligence} {spline interpolation} {L2 regularization}} %
	
\makeatletter 
\AfterEndEnvironment{mdframed}{%
	\tfn@tablefootnoteprintout%
	\gdef\tfn@fnt{0}%
}
\makeatother

\makeatletter
\patchcmd{\@makefnmark}{\fontsize}{\check@mathfonts\fontsize}{}{}
\makeatother

\let\oldnormalcolor\normalcolor%
\renewcommand{\normalcolor}{\oldnormalcolor\transparent{1}}

\newcommand{\crefrangeconjunction}{--}
\crefalias{subequation}{equation}
\mathtoolsset{centercolon}

\newcommand{\mydot}{\,{.}}  %
\newcommand{\mycomma}{\,{,}}  %
\newcommand\logeq{\mathrel{\vcentcolon\Leftrightarrow}}
\newcommand*{\meinfullref}[1]{\hyperref[{#1}]{\Cref*{#1} \nameref*{#1}}}
\newcommand*{\unameref}[2][{}]{\hyperref[{#2}]{ \textcolor{black}{\nameref*{#2}{#1}}}}
\newcommand{\meqref}[1]{\labelcref{#1}}

\newcommand*{\meinref}[2][{}]{\hyperref[{#2}]{{#1}~\labelcref*{#2}}}

\newcommand*\colvec[3][]{
	\begin{pmatrix}\ifx\relax#1\relax\else#1\\\fi#2\\#3\end{pmatrix}
}
\newcommand*\colvecs[3][]{
	\begin{psmallmatrix}\ifx\relax#1\relax\else#1\\\fi#2\\#3\end{psmallmatrix}
}

\ifthenelse{\equal{\specialStyle}{1}}{
    \newcommand{\unimportantStart}{\transparent{\meineTranzparenz}}
    \newcommand{\mediumimportantStart}{\transparent{\meineTranzparenz}}
}{
    \ifthenelse{\equal{\specialStyle}{0.5}}{
        \newcommand{\unimportantStart}{\scriptsize}
        \newcommand{\mediumimportantStart}{}
    }{
        \newcommand{\unimportantStart}{}
        \newcommand{\mediumimportantStart}{}
    }
}

    \newcommand{\unimportant}[1]{\begingroup\unimportantStart{#1}\endgroup}

\newcommand{\veryunimportant}[1]{}

\newcommand{\R}{\ensuremath{\mathbb{R}}}
\newcommand{\C}{\ensuremath{\mathcal{C}}}
\newcommand{\Rp}{\ensuremath{\R_{>0}}}
\newcommand{\Rpz}{\ensuremath{\R_{\geq 0}}}
\newcommand{\N}{\ensuremath{\mathbb{N}}}
\newcommand{\Z}{\ensuremath{\mathbb{Z}}}
\newcommand{\B}{\ensuremath{\mathfrak{B}}}
\newcommand{\E}[1][]{\ensuremath{\mathbb{E}\ifthenelse{\equal{#1}{}}{}{\left[ {#1}\right] }}}
\newcommand{\Eco}[2]{\ensuremath{\E\left[{#1}\middle|{#2}\right]}} 
\newcommand{\PP}[1][]{\ensuremath{\mathbb{P}\ifthenelse{\equal{#1}{}}{}{\left[ {#1}\right] }}}
\newcommand{\PPco}[2]{\ensuremath{\PP\left[{#1}\middle|{#2}\right]}} 
\newcommand{\X}{\ensuremath{\mathcal{X}}}
\newcommand{\Y}{\ensuremath{\mathcal{Y}}}
\newcommand{\eps}{\ensuremath{\varepsilon}}
\newcommand{\om}{\ensuremath{{\color{hellgrau}\omega}}}
\newcommand{\Om}{\ensuremath{\Omega}}
\newcommand{\omb}{\ensuremath{{\color{hellgrau}(\om)}}}
\newcommand{\Wkp}[3][K]{\ensuremath{{W^{{#2},{#3}}({#1})}}}
\newcommand{\Woi}{\ensuremath{\hyperref[rem:sobnormoi]{\Wkp[K]{1}{\infty}}}}
\newcommand{\Wop}[1][{K,\nu}]{\Wkp[#1]{1}{p}}
\newcommand{\WT}{\ensuremath{\hyperref[re:setting]W^{2}}}

\newcommand{\WkpShort}[2]{\ensuremath{\href{https://en.wikipedia.org/w/index.php?title=Sobolev_space&oldid=910223537}{W^{{#1},{#2}}}}}
\newcommand{\iid}{\href{https://en.wikipedia.org/w/index.php?title=Independent_and_identically_distributed_random_variables&oldid=910267759}{\oldnormalcolor i.i.d.}}

\newcommand{\fae}[1][]{\ensuremath{\forall\epsilon_{#1}\in\Rp}}
\newcommand{\fao}{\ensuremath{\forall\omega\in\Om}}
\newcommand{\faog}{\ensuremath{{\color{hellgrau}\fao}}}
\newcommand{\fax}{\ensuremath{\forall x\in\R}}
\newcommand{\faxg}{\ensuremath{{\color{hellgrau}\fax}}}
\newcommand{\faxd}{\ensuremath{\forall x\in\R^d}}
\newcommand{\faxdg}{\ensuremath{{\color{hellgrau}\faxd}}}
\newcommand{\fromto}[2][1]{\ensuremath{\left\lbrace {#1},\dots,{#2}\right\rbrace }} %
\newcommand{\allIndi}[3][1]{\ensuremath{\forall {#3}\in\fromto[{#1}]{#2}}} %

\newcommand{\sgnb}[1]{\ensuremath{\sgn\left( {#1}\right) }}
\newcommand{\relu}[1][\cdot]{\ensuremath{\max\left( 0,{#1}\right) }}
\newcommand{\ind}{\ensuremath{\mathbbm{1}}}
\newcommand{\sobnorm}[1][\cdot]{\ensuremath{\left\| {#1}\right\|_{\Woi}}}
\newcommand{\sobnormop}[1][\cdot]{\ensuremath{\left\| {#1}\right\|_{\Wop}}}
\newcommand{\supnorm}[1][\cdot]{\ensuremath{\left\| {#1}\right\|_{L^\infty(K)}}}
\newcommand{\twonorm}[1][\cdot]{\ensuremath{\left\| {#1}\right\|_2}}
\newcommand{\Ltwonorm}[1][\cdot]{\ensuremath{\left\| {#1}\right\|_{L^2(K)}}}

\newcommand{\plim}[1][n\to\infty]{\ensuremath{\Plim_{#1}}}
\newcommand{\pBigO}[2][n\to\infty]{\ensuremath{\PBigO_{#1}\left( {#2}\right)}}
\newcommand{\tlim}{\ensuremath{\lim_{T\to\infty}}}

\newcommand{\T}{\ensuremath{\hyperref[def:adaptedSplineReg]{\mathcal{T}}}}

\newcommand{\Tgpgm}{\ensuremath{\hyperlink{eq:Tgpgm}{\mathcal{T}_{g_+,g_-}}}}
\newcommand{\gxi}{\ensuremath{\hyperref[def:kinkPosDens]{g_{\xi}}}}
\newcommand{\gxip}{\ensuremath{\hyperref[def:conKinkPosDens]{g_{\xi}^+}}}
\newcommand{\gxim}{\ensuremath{\hyperref[def:conKinkPosDens]{g_{\xi}^-}}}
\newcommand{\Cgl}[1][g]{\ensuremath{\hyperref[rem:compactSupp]{C_{#1}^\ell}}}
\newcommand{\Cgu}[1][g]{\ensuremath{\hyperref[rem:compactSupp]{C_{#1}^u}}}
\newcommand{\cm}{\ensuremath{\hyperref[rem:compactSupp]{c_{-}}}}
\newcommand{\cp}{\ensuremath{\hyperref[rem:compactSupp]{c_{+}}}}
\newcommand{\kapx}{\ensuremath{\hyperref[def:smoothRSNappr]{\kappa_x}}}
\newcommand{\xtr}{\ensuremath{x^{\text{\tiny train}}}}
\newcommand{\ytr}{\ensuremath{y^{\text{\tiny train}}}}
\newcommand{\Xt}{\ensuremath{X^{\top}}}
\newcommand{\lw}{\ensuremath{\hyperref[thm:ridgeToSpline]{\tilde{\lambda}}}}
\newcommand{\lwnl}{\ensuremath{\tilde{\lambda}}}
\newcommand{\Ltr}{\ensuremath{\hyperref[as:generalloss]{L}}}
\newcommand{\Ltrs}{\ensuremath{\hyperref[eq:Ltr_squaredloss]{L}}}
\newcommand{\ltr}{\ensuremath{\hyperref[eq:ltr]{l}}}
\newcommand{\ltri}[1][i]{\ensuremath{\hyperref[ex:lossfunctionalsum]{l_{#1}}}}
\newcommand{\ltrib}[2][i]{\ensuremath{\hyperref[ex:lossfunctionalsum]{l_{#1}}\left( {#2}\right)}}

\newcommand{\Ltrb}[1]{\ensuremath{\Ltr\left( {#1}\right) }}
\newcommand{\Ltrsb}[1]{\ensuremath{\Ltrs\left( {#1}\right) }}
\newcommand{\Pfunc}{\ensuremath{\hyperref[rem:Pfunc]{P}}}
\newcommand{\Pg}[1][g]{\ensuremath{\hyperref[def:splineReg]{P^{\color{hellgrau}{#1}}}}}
\newcommand{\Pgpmm}[1][g]{\ensuremath{\hyperlink{eq:Pgpmm}{P_{\pm}^{{\color{hellgrau} {#1}}}}}}
\newcommand{\Pgpm}[1][g]{\ensuremath{\hyperlink{eq:Pgpm}{P_{+-}^{{\color{hellgrau} {#1}}}}}}
\newcommand{\Pgpmb}[2][g]{\ensuremath{\Pgpm[{#1}]\left( {#2}\right) }}

\newcommand{\ggreg}[1][g]{{\color{hellgrau}{#1},\text{reg}}}
\newcommand{\PgpmmaffineReg}[2][g]{\ensuremath{\hyperlink{eq:Pgpmaffine}{\acute{P}_{\pm}^{{\color{hellgrau} {#1},{#2}}}}}}
\newcommand{\Pgpmmaffine}[1][g]{\ensuremath{\hyperlink{eq:Pgpmaffine}{\acute{P}_{\pm}^{{\color{hellgrau} {#1},\text{reg}}}}}}
\newcommand{\Fn}[1][n]{\ensuremath{\hyperref[def:ridgeNet]{F_{#1}^{{\color{hellgrau}\lw}}}}} %
\newcommand{\Fnb}[2][n]{\ensuremath{\Fn[{#1}]\left( {#2}\right) }}
\newcommand{\Fl}[1][]{\ensuremath{\hyperref[def:splineReg]{F^{{\color{hellgrau}\lambda}\ifthenelse{\equal{#1}{}}{}{,} {#1}}}}} %
\newcommand{\Flb}[2][]{\ensuremath{\Fl[{#1}]\left( {#2}\right) }} %
\newcommand{\Flpmm}[1][g]{\ensuremath{\hyperref[def:adaptedSplineReg]{F_{\pm}^{{\color{hellgrau}\lambda\ifthenelse{\equal{#1}{}}{}{,} {#1}}}}}}
\newcommand{\Flpmmb}[2][g]{\ensuremath{\Flpmm[{#1}]\left( {#2}\right) }}
\newcommand{\Flpm}[1][g]{\ensuremath{\hyperref[def:adaptedSplineRegTuple]{F_{+-}^{{\color{hellgrau}\lambda\ifthenelse{\equal{#1}{}}{}{,} {#1}}}}}}
\newcommand{\Flpmb}[2][g]{\ensuremath{\Flpm[{#1}]\left( {#2}\right) }}
\newcommand{\Flpmasym}[1][g_+,g_-]{\ensuremath{\hyperref[def:asymadaptedSplineReg]{F_{+-}^{{\color{hellgrau}\lambda\ifthenelse{\equal{#1}{}}{}{,} {#1}}}}}}
\newcommand{\Flpmasymb}[2][g_+,g_-]{\ensuremath{\Flpmasym[{#1}]\left( {#2}\right) }}
\newcommand{\kp}{\ensuremath{\hyperref[subeq:kp]{\mathfrak{K}^{+}}}} %
\newcommand{\km}{\ensuremath{\hyperref[subeq:km]{\mathfrak{K}^{-}}}}
\newcommand{\wR}{\ensuremath{\hyperref[def:ridgeNet]{w^{*{\color{hellgrau},n,\lwnl}}}}}
\newcommand{\wRk}[1][k]{\ensuremath{\hyperref[def:ridgeNet]{w_{#1}^{*{\color{hellgrau},n,\lwnl}}}}}
\newcommand{\wRl}[1][\lwnl]{\ensuremath{\hyperref[def:ridgeNet]{w^{*{\color{hellgrau},n,{#1}}}}}}
\newcommand{\wRo}{\ensuremath{\wR\omb}}
\newcommand{\wRp}{\ensuremath{w^{*+{\color{hellgrau},n,\lwnl}}}}
\newcommand{\wRm}{\ensuremath{w^{*-{\color{hellgrau},n,\lwnl}}}}
\newcommand{\ww}{\ensuremath{\hyperref[eq:ww]{\tilde{w}^{\color{hellgrau}n}}}}
\newcommand{\wwo}{\ensuremath{\ww\omb}}
\newcommand{\wwp}{\ensuremath{\hyperref[par:wwp]{\ww^+}}}
\newcommand{\wwm}{\ensuremath{\hyperref[par:wwp]{\ww^-}}}

\newcommand{\wt}[1][T]{\ensuremath{\hyperref[def:GDsolution]{w^{#1{\color{hellgrau},n}}}}}
\newcommand{\wto}[1][T]{\ensuremath{\hyperref[def:GDsolution]{w^{{#1{\color{hellgrau},n}}}}\omb}}
\newcommand{\wth}[1][T]{\ensuremath{\hyperref[def:GDsolution]{\check{w}^{#1{\color{hellgrau},n}}}}}
\newcommand{\wtho}[1][T]{\ensuremath{\hyperref[def:GDsolution]{\check{w}^{{#1{\color{hellgrau},n}}}}\omb}}
\newcommand{\wdag}{\ensuremath{\hyperref[def:minNormSol]{w^{*,n,0+}}}} %
\newcommand{\wdago}{\ensuremath{\wdag\omb}}
\newcommand{\waffine}{\ensuremath{\hyperref[subsec:linearSkip]{\acute{w}}}}
\newcommand{\waffinenl}{\ensuremath{{\acute{w}}}}
\newcommand{\wRaffine}{\ensuremath{\hyperref[eq:RidgeProblem With SKipConnections]{\waffinenl^{*{\color{hellgrau},n,\lwnl,\text{reg}}}}}}
\newcommand{\wRkaffine}[1][k]{\ensuremath{%
w_{#1}^{*{\color{hellgrau},n,\lwnl,\text{reg}}}}}%
\newcommand{\wthaffine}[1][T]{\ensuremath{\hyperref[eq:GDescentSkip]{\check{\waffinenl}^{#1{\color{hellgrau},n}}}}}
\newcommand{\hq}{\ensuremath{\hyperref[def:estKinkDist]{\bar{h}}}}
\newcommand{\hqp}{\ensuremath{\hyperref[subeq:hqp]{\bar{h}^+}}}
\newcommand{\hqm}{\ensuremath{\hyperref[subeq:hqm]{\bar{h}^-}}}
\newcommand{\NN}{\ensuremath{\hyperref[def:NN]{\mathcal{N\!N}}}}
\newcommand{\RN}{\ensuremath{\hyperref[def:RSNN]{\mathcal{R\!N}}}}
\newcommand{\RNnl}{\ensuremath{\mathcal{R\!N}}}
\newcommand{\RNw}[1][w]{\ensuremath{\RN_{#1}}}
\newcommand{\RNwo}[1][w]{\ensuremath{\RN_{{#1}{\color{hellgrau},\om}}}}
\newcommand{\RNR}[1][{\lwnl}]{\ensuremath{\hyperref[def:ridgeNet]{\RNnl^{*{\color{hellgrau},n,{#1}}}}}}
\newcommand{\RNRo}[1][{\lwnl}]{\ensuremath{\hyperref[def:ridgeNet]{\RNnl_{\om}^{*{\color{hellgrau},n,{#1}}}}}}
\newcommand{\RNRz}{\ensuremath{\hyperref[def:minNormSol]{\RNnl^{*{\color{hellgrau},n,0+}}}}}
\newcommand{\RNRp}{\ensuremath{\hyperref[eq:RNRp]{\RNnl^{*{\color{hellgrau},n,\lwnl}}_+}}} 
\newcommand{\RNRm}{\ensuremath{\hyperref[eq:RNRp]{\RNnl^{*{\color{hellgrau},n,\lwnl}}_-}}}  
\newcommand{\sRNw}{\ensuremath{\hyperref[def:splineApproximatingRSN]{\RNnl_{\ww}}}}
\newcommand{\sRNwo}{\ensuremath{\hyperref[def:splineApproximatingRSN]{\RNnl_{\wwo{\color{hellgrau},\om}}}}}
\newcommand{\sRNwp}{\ensuremath{\hyperref[eq:splineApprRSNsplit]{\RNnl_{\wwp}^+}}}
\newcommand{\sRNwpo}{\ensuremath{\hyperref[eq:splineApprRSNsplit]{\RNnl_{\wwp\omb{\color{hellgrau},\om}}^+}}}
\newcommand{\sRNwm}{\ensuremath{\hyperref[eq:splineApprRSNsplit]{\RNnl_{\wwm}^-}}}
\newcommand{\sRNwmo}{\ensuremath{\hyperref[eq:splineApprRSNsplit]{\RNnl_{\wwm\omb{\color{hellgrau},\om}}^-}}}
\newcommand{\RNaffinenl}{\ensuremath{\acute{\RNnl}}}
\newcommand{\RNaffine}{\ensuremath{\hyperref[subsec:linearSkip]{\RNaffinenl}}}
\newcommand{\RNRaffine}[1][{\lwnl}]{\ensuremath{\hyperref[subsec:linearSkip]{\RNaffine_{\wRaffine}}}}
\newcommand{\RNRpaffine}{\ensuremath{\hyperref[eq:RNRpmaffine]{\RNaffinenl^{*{\color{hellgrau},n,\lwnl,\text{reg}}}_+%
}}}
\newcommand{\RNRmaffine}{\ensuremath{\hyperref[eq:RNRpmaffine]{\RNaffinenl^{*{\color{hellgrau},n,\lwnl,\text{reg}}}_-%
}}} 
\newcommand{\ftrue}{\ensuremath{f_{\text{True}}}}
\newcommand{\fl}[1][\lambda]{\ensuremath{\hyperref[def:splineReg]{f^{*{\color{hellgrau},{#1}}}}}}
\newcommand{\flg}{\ensuremath{\hyperref[eq:weighted_spline_regression]{f_g^{*{\color{hellgrau},\lambda}}}}}
\newcommand{\flgg}[1][g]{\ensuremath{\hyperref[eq:weighted_spline_regression]{f_{#1}^{*{\color{hellgrau},\lambda}}}}}
\newcommand{\fz}{\ensuremath{\hyperref[def:splineInterpolation]{f^{*,0+}}}}

\newcommand{\flp}[1][g]{\ensuremath{\hyperref[def:adaptedSplineReg]{f^{*{\color{hellgrau},\lambda}}_{{#1},+}}}}
\newcommand{\flm}[1][g]{\ensuremath{\hyperref[def:adaptedSplineReg]{f^{*{\color{hellgrau},\lambda}}_{{#1},-}}}}
\newcommand{\flpm}[1][g]{\ensuremath{\hyperref[def:adaptedSplineReg]{f^{*{\color{hellgrau},\lambda}}_{{#1},\pm}}}}
\newcommand{\flpmt}[1][g]{\ensuremath{\left( \flp[{#1}] , \flm[{#1}]\right)}}
\newcommand{\fzpm}{\ensuremath{\hyperref[def:adaptedSplineInterpolation]{f_{g,\pm}^{*,0+}}}}
\newcommand{\fLpm}[1][\lambda]{\ensuremath{\hyperref[def:adaptedSplineReg]{f^{*{\color{hellgrau},{#1}}}_{g,\pm}}}}
\newcommand{\flpasym}[1][g_+,g_-]{\ensuremath{\hyperref[def:asymadaptedSplineReg]{f^{*{\color{hellgrau},\lambda}}_{{#1},+}}}}
\newcommand{\flmasym}[1][g_+,g_-]{\ensuremath{\hyperref[def:asymadaptedSplineReg]{f^{*{\color{hellgrau},\lambda}}_{{#1},-}}}}
\newcommand{\flpmasym}[1][g_+,g_-]{\ensuremath{\hyperref[def:asymadaptedSplineReg]{f^{*{\color{hellgrau},\lambda}}_{{#1},\pm}}}}
\newcommand{\gammaAsym}[1][g_+,g_-]{\ensuremath{\hyperref[def:asymadaptedSplineReg]{\gamma^{*{\color{hellgrau},\lambda}}_{{#1}}}}}
\newcommand{\flpmtasym}[1][g_+,g_-]{\ensuremath{\left( \flpasym[{#1}] , \flmasym[{#1}], \gammaAsym\right)}}
\newcommand{\fnp}{\ensuremath{f_+^n}}
\newcommand{\fnm}{\ensuremath{f_-^n}}
\newcommand{\fnpmt}{\ensuremath{\left(\fnp,\fnm\right)}}
\newcommand{\fonepmt}{\ensuremath{\left(f_+^1,f_-^1\right)}}
\newcommand{\ftwopmt}{\ensuremath{\left(f_+^2,f_-^2\right)}}
\newcommand{\up}{\ensuremath{u_+}}
\newcommand{\um}{\ensuremath{u_-}}
\newcommand{\upmt}{\ensuremath{\left(\up,\um\right)}}
\newcommand{\unp}{\ensuremath{u_+^n}}
\newcommand{\unm}{\ensuremath{u_-^n}}
\newcommand{\unpmt}{\ensuremath{\left(\unp,\unm\right)}}
\newcommand{\fwR}{\ensuremath{\hyperref[def:smoothRSNappr]{f^{\wR}}}}
\newcommand{\fwRo}{\ensuremath{\hyperref[def:smoothRSNappr]{f^{\wRo}}}}
\newcommand{\fwRp}{\ensuremath{\hyperref[def:smoothRSNappr]{f^{\wR}_+}}}
\newcommand{\fwRm}{\ensuremath{\hyperref[def:smoothRSNappr]{f^{\wR}_-}}}
\newcommand{\fwRpmt}{\ensuremath{\left( \fwRp , \fwRm\right)}}
\newcommand{\ClPrime}{\ensuremath{\hyperref[eq:l_prime_bound]{C_{l^{{'}}}}}}
\newcommand{\LLip}{\ensuremath{\hyperref[eq:l_prime_bound]{C_{L}}}}
\newcommand{\nua}{\ensuremath{\nu_{\text{a}}}}
\newcommand{\nuc}{\ensuremath{\nu_{\text{c}}}}
\newcommand{\lip}[1][\cdot]{\ensuremath{\hyperref[def:partition]{\lfloor {#1}\rfloor_{\mathcal{P}}}}}
\newcommand{\uip}[1][\cdot]{\ensuremath{\hyperref[def:partition]{\lceil {#1}\rceil_{\mathcal{P}}}}}
\newcommand{\Ip}[1][\cdot]{\ensuremath{\hyperref[def:partition]{\mathcal{I}_\mathcal{P}}\left({#1}\right)}}
\newcommand{\Iwp}[1][\cdot]{\ensuremath{\hyperref[eq:def:Iwp]{\tilde{\mathcal{I}}_\mathcal{P}}\left({#1}\right)}}
\newcommand{\sScale}{s_\text{scale}}

\newcommand{\mfootref}[1]{\text{{\upshape \textsuperscript{\ref{#1}}}}}

\newcommand{\regSpl}{\hyperref[def:splineReg]{regression spline}}
\newcommand{\regSplf}{\hyperref[def:splineReg]{regression spline~\ensuremath{\fl}}}
\newcommand{\SplReg}{\hyperref[def:splineReg]{smooting spline regression}}
\newcommand{\wregSpl}{\hyperref[def:splineReg]{weighted regression spline}}
\newcommand{\wregSplf}{\hyperref[def:splineReg]{weighted regression spline~\ensuremath{\flg}}}
\newcommand{\aregSpl}{\hyperref[def:adaptedSplineReg]{adapted regression spline}}
\newcommand{\asymaregSpl}{\hyperref[def:asymadaptedSplineReg]{asymmetric adapted regression spline}}
\newcommand{\aregSplf}{\hyperref[def:adaptedSplineReg]{adapted regression spline~\ensuremath{\flpm}}}
\newcommand{\ReLU}{\hyperref[itm:as:ReLU]{\oldnormalcolor ReLU}}
\newcommand{\RReLU}{\hyperref[itm:as:ReLU]{\oldnormalcolor\textbf{R}eLU}}
\newcommand{\RSN}{\hyperref[def:RSNN]{RSN}}
\newcommand{\RSNlong}{\hyperref[def:RSNN]{\textbf{r}andomized \textbf{s}hallow neural \textbf{n}etwork}}
\newcommand{\RSNlongAndShort}{\hyperref[def:RSNN]{\textbf{r}andomized \textbf{s}hallow neural \textbf{n}etwork (\RSN)}}
\newcommand{\wRRSN}{wR\RSN}

\newcommand{\wRRSNlongAndShort}{\textbf{w}ide \textbf{R}eLU \hyperref[def:RSNN]{\textbf{r}andomized \textbf{s}hallow neural \textbf{n}etwork} (\wRRSN)}
\newcommand{\wlargeRRSNlongAndShort}{\textbf{w}ide \unimportant{(large number of neurons~$n$)} \textbf{R}eLU \hyperref[def:RSNN]{\textbf{r}andomized \textbf{s}hallow neural \textbf{n}etwork} (\wRRSN)}
\newcommand{\ridgeRSN}{\hyperref[def:ridgeNet]{ridge regularized RSN}}
\newcommand{\ridgeRSNRN}{\hyperref[def:ridgeNet]{ridge regularized RSN~\RNR}}

\newcommand{\RIDGERSN}{\hyperref[def:ridgeNet]{Ridge Regularized RSN}}
\newcommand{\proofInSec}[2]{{\unimportantStart
		\begin{proof}
			The \hyperlink{proof:#1}{proof of \Cref*{#1}} is formulated in \Cref{#2}.
\end{proof}}}

\mdfdefinestyle
{meinTheoremstyle}{
	linewidth=0pt,
	backgroundcolor =  hellgrau,
	innerleftmargin=0pt,
	innerrightmargin=0pt,
	innertopmargin=-5pt,
	innerbottommargin=0pt}
\theoremstyle{plain}
\newtheorem{theorem}{\meinPrefix Theorem}[section] %
\newtheorem{lemma}[theorem]{\meinPrefix Lemma}
\newtheorem{corollary}[theorem]{\meinPrefix Corollary}
\newtheorem{proposition}[theorem]{\meinPrefix Proposition}
\crefname{proposition}{Proposition}{Propositions}
\theoremstyle{remark}
\newtheorem{remark}[theorem]{\meinPrefix Remark}
\theoremstyle{definition}
\newtheorem{definition}[theorem]{\meinPrefix Definition}
\newtheorem{example}[theorem]{\meinPrefix Example}

\newtheorem{assumption}{\meinPrefix Assumption}
\crefname{assumption}{Assumption}{Assumptions}
\newtheorem{paradox}[]{\meinPrefix Paradox}
\crefname{paradox}{Paradox}{Paradoxes}
\newcommand{\meinPrefix}{}
\DeclareMathOperator*{\argmin}{\hyperref[eq:argminDef]{arg\,min}}
\DeclareMathOperator*{\argmax}{arg\,max}
\DeclareMathOperator*{\supp}{supp}
\DeclareMathOperator*{\ConvexHull}{ConvexHull}
\DeclareMathOperator*{\sgn}{sgn}

\DeclareMathOperator*{\Plim}{\PP-\lim}
\DeclareMathOperator*{\BigO}{\mathcal{O}}
\DeclareMathOperator*{\PBigO}{\PP-\mathcal{O}}

\newcommand\scaledinset[6]{%
	\setbox0=\hbox{#6}%
	\stackinset{#1}{#2\wd0}{#3}{#4\ht0}{#5}{#6}%
} %

\calclayout

\begin{document}

\title[\texorpdfstring{\resizebox{4.7in}{!}{How (Implicit) Regularization of ReLU Neural Networks Characterizes the Learned Function --- Part I}}{How (Implicit) Regularization of ReLU Neural Networks Characterizes the Learned Function - Part I: the 1-D Case of Two Layers with Random First Layer}]{How (Implicit) Regularization of ReLU Neural Networks Characterizes the Learned Function \\ Part I: \\ the 1-D Case of Two Layers with Random First Layer}
\author{Jakob Heiss, Josef Teichmann and Hanna Wutte}
\address{ETH Z\"urich, D-Math, R\"amistrasse 101, CH-8092 Z\"urich, Switzerland}
\email{jakob.heiss@math.ethz.ch, jteichma@math.ethz.ch, hanna.wutte@math.ethz.ch}
\thanks{The authors gratefully acknowledge the support from ETH-foundation. We are very thankful for numerous helpful discussions, feedback, corrections and proof reading---especially to: Lukas Fertl, Peter M\"uhlbacher, Martin \v{S}tef\'anik, Alexis Stockinger, Teresa Heiss and Jakob Weissteiner.}
\curraddr{}
	\begin{abstract}
	In this paper, we consider one dimensional (shallow) ReLU neural networks in which weights are chosen randomly and only the terminal layer is trained.
	First, we mathematically show that for such networks $\ell_2$-regularized regression corresponds in function space to regularizing the estimate's second derivative for fairly general loss functionals.
	For least squares regression, we show that the trained network converges to the smooth spline interpolation of the training data as the number of hidden nodes tends to infinity. 
	Moreover, we derive a novel correspondence between the early stopped gradient descent (without any explicit regularization of the weights) and the smoothing spline regression. 
	\end{abstract}
\subjclass[]{}

\maketitle

	\section{Introduction}\label{se:Introduction}
	Even though neural networks are becoming increasingly popular in supervised learning, their theoretical understanding is still very limited.
	The most important open questions in the mathematical theory of neural networks nowadays include the following:\footnote{The literature agrees with questions~\labelcref{item:Generalization,item:Gradient,item:Expressive} to be central \cite{PoggioGernalizationDeepNN2018arXiv180611379P}. Question~\ref{item:ProsAndCons} motivates the importance of questions~\labelcref{item:Generalization,item:Gradient,item:Expressive} by summarizing them and concluding their implications.}
	\begin{enumerate}[I., ref=\Roman*]
		\item\label{item:Generalization} \textbf{Generalization}:
		Why and under which conditions can neural networks make good predictions of the output for new unseen input data even though they have only been trained on finitely many data points? How does the trained function behave out of sample? How can one get control of over-fitting?
		\item\label{item:Gradient} \textbf{Gradient Descent}: When training neural networks, a typically very high-dimensional non-convex optimization problem is claimed to be solved by (stochastic) gradient descent quite fast. 
		What happens if the algorithm is stopped early after a realistic number of steps depending on a certain starting point?
		\item\label{item:Expressive} \textbf{Expressiveness:} How expressive are neural networks (with a finite number of nodes)? \cite{shaham2018provable,bianchini2014complexity,ITOApproxNNWithoutScale1991817,leshno1993multilayer}
		\item\label{item:ProsAndCons}\textbf{Summary}: What are the advantages and disadvantages of different architectures? What are the advantages and disadvantages of considering neural networks in approximation/prediction tasks compared to other methods such as Random Forests or Kernel-based Gaussian processes? In both theory and applications, it is of great interest to gain a precise understanding of \ref{item:ProsAndCons}, much of which could be achieved by answering \ref{item:Generalization}\crefrangeconjunction\ref{item:Expressive}.
	\end{enumerate}

	The goal of this work is to contribute to answering these questions by rigorously proving \Cref{thm:ridgeToSpline,thm:GDridge}
	that almost completely resolve question~\ref{item:Gradient} (cp. \cref{eq:conclusionApprox}) for the class of wide \RSNlong{}s (\RSN{}s) with \ReLU\ activation (i.e., \wRRSN s). These answers together with the intuition acquired from \Cref{sec:Regression,sec:Paradoxon} give quite extensive insights to ~\ref{item:Generalization} and thus ~\ref{item:ProsAndCons}.\footnote{We also contribute to answering question~\ref{item:Expressive} within the results marked with a \enquote{\mfootref{footnote:star}}: \Cref{rem:furtherNotationMeasureTheory}, \Cref{cor:universal_in_prob}, \Cref{le:almost_sure_interpolation} and \Cref{rem:AlmostSureInterpolation} in \Cref{sec:RSNN}. These results form an independent storyline.}
	
	The result of this work can be seen in analogy to mean field theory in thermodynamics:
	like we are understanding the collision behavior of each particle, we understand the training behavior of each neuron%
	\unimportant{\footnote{\unimportant{In this work, only \emph{artificial} neural networks are considered. Thus, terms such as 'neurons' and 'neural networks' do not refer to actual biological neurons but rather to their artificial counterparts. The term \enquote{node} will be used interchangeable with the term \enquote{neuron}.}}}\ifthenelse{\equal{\specialStyle}{0}}{\addtocounter{footnote}{1}}{\textsuperscript{\unimportant{,} \!}\unskip%
	\footnote{Notation remark: To improve readability of the paper, we use \ifthenelse{\equal{\specialStyle}{1}}{partially transparent {\transparent{\meineTranzparenz}(grey)}}{
	smaller} fonts to encourage the reader to skip these details.}}. However, due to the extensive number of interactions between particles/neurons the complexity increases in a way that the individual behavior of a particle/neuron does no longer give direct insight into the overall system's behavior. In both cases, taking the limit to infinity allows to precisely derive the system's behavior in terms of interpretable macroscopic laws/theorems (see \Cref{thm:ridgeToSpline}\footnote{\Cref{thm:ridgeToSpline} results from letting the number of neurons~\unimportant{$n$} tend to infinity. In thermodynamics, Brownian motion particle movements or heat equations result from taking the limit of the number of particles to infinity.}).
	
	\subsection{The Regression Problem in Machine Learning}\label{sec:Regression}

	Let $\mathcal{X}$ and $\mathcal{Y}$ be an input and output space, respectively. Assume further, we observe a finite number~$N\in\N$ of \iid\ samples $\left( \xtr_i,\ytr_i\right)\in\X\times\Y$ with $i\in \fromto{N}$ from an \emph{unknown} probability distribution $\PP_D$ on $\X\times\Y$. Given an additional realization $(X,Y)\omb$ of $(X,Y)\sim\PP_D$, for which we can only observe $X\omb$ but not $Y\omb$, the goal is to make a suitable prediction~$\hat{f}(X\omb)$ of $Y\omb$. Thus, for a given cost function~$C:\Y\times\Y\to\R$, we are interested in an estimator~$\hat{f}:\X\to\Y$ with low risk, i.e., for which the expected cost $\E[C\left(\hat{f}(X),Y\right)]$ is minimal. However, since $\PP_D$ is unknown, this risk cannot be calculated. In supervised machine learning, one hence tries to learn an estimator $\hat{f}$ based on the given training data~$\left( \xtr_i,\ytr_i\right)_{i\in \fromto{N}}$. 
	A common heuristic
	is to minimize a suitable training loss 
	    $\Ltr(f)$ over a suitable class of functions $\mathcal{H}$, i.e., 
	\[\min_{f\in\mathcal{H}}\Ltr(f).\]

\begin{remark}[Setting]\label{re:setting}
    Throughout this work, we consider $\X=\R^d$ with input dimension $d\in\N$ and $\Y=\R$. %
    The concepts of this paper apply to supervised learning in general, but within this paper we focus primarily on regression.
    Moreover, this paper's main contribution \Cref{thm:ridgeToSpline} (linking regularization of parameters to regularization on function space) holds for fairly general (non-negative) loss functionals $\Ltr:L^\infty_\text{loc}\to\Rpz$ that need not even depend on training data (see \Cref{as:generalloss}).
    Note however, that \Cref{thm:GDridge} linking gradient descent to regularization of the parameters is derived for least squares loss~$\Ltrs$ as in \meqref{eq:Ltr_squaredloss} \unimportant{(\Cref{as:squaredloss}; see also \Cref{rem:generalizeLoss})}. For ease of exposition, we consider this squared loss throughout the paper with the exception of \Cref{sec:RidgeToSpline}.
	 Finally, we denote by $\WT$ the set of twice weakly differentiable functions. 
\end{remark}

	Historically, linear regression~\cite{gauss1809LeastSquares1,gauss1823theoriaLeastSquares2,legendre1805nouvellesLeastSquares} was among the first methods used within supervised learning. Here, one restricts oneself to a tiny subspace of all functions: the space of \unimportant{(affine-)}linear functions. This choice indeed favors parsimony: if the number of samples~$N$ is larger than the input dimension~$d\unimportant{(+1)}$ there exists a unique\footnote{The solution of a least square linear regression is unique, if there are $d$ linearly independent training data points $\xtr_i$ \unimportant{(or $d+1$ affine independent input points $\xtr_i$ if an intercept is used)}. If the training data points are drawn as \iid\ samples from a distribution that is absolutely continuous with respect to the $d$-dimensional Lebesgue measure, this is almost surely the case, if $d\unimportant{(+1)}\leq N$.} function $\hat{f}$ that fits through the training data best, i.e. minimizes the training loss
	\begin{equation}\label{eq:Ltr_squaredloss}
	\Ltrsb{\hat{f}}:=\sum_{i=1}^{N}\left( \hat{f}(\xtr_i)-\ytr_i\right)^2 .
	\end{equation}
	Although this approach is still extensively used in real-world applications, the space of linear functions often is not sufficient, as true relations between input and output are mostly more involved if not highly non-linear. Ideally, the class $\mathcal{H}$ would hence be chosen to be more expressive, so as to be able to approximate well these underlying maps from input $X$ onto output $Y$.
	
	As a consequence, the challenge nowadays is to choose the \enquote{most desirable} function~$\hat{f}$ out of the infinitely many functions with equal training loss~$\Ltrsb{\hat{f}}$. This opens the question to what the mathematical meaning of \enquote{most desirable} could be.
	At least intuitively, engineers have quite specific convictions (also known as \href{https://en.wikipedia.org/w/index.php?title=Inductive_bias&oldid=901756495}{\emph{\oldnormalcolor inductive bias}}) which functions are not desirable (see \Cref{fig:crazyOscillation,fig:SplineReg}).
	\begin{figure}[htp] 
		\centering
		\includegraphics[width=0.8\textwidth]{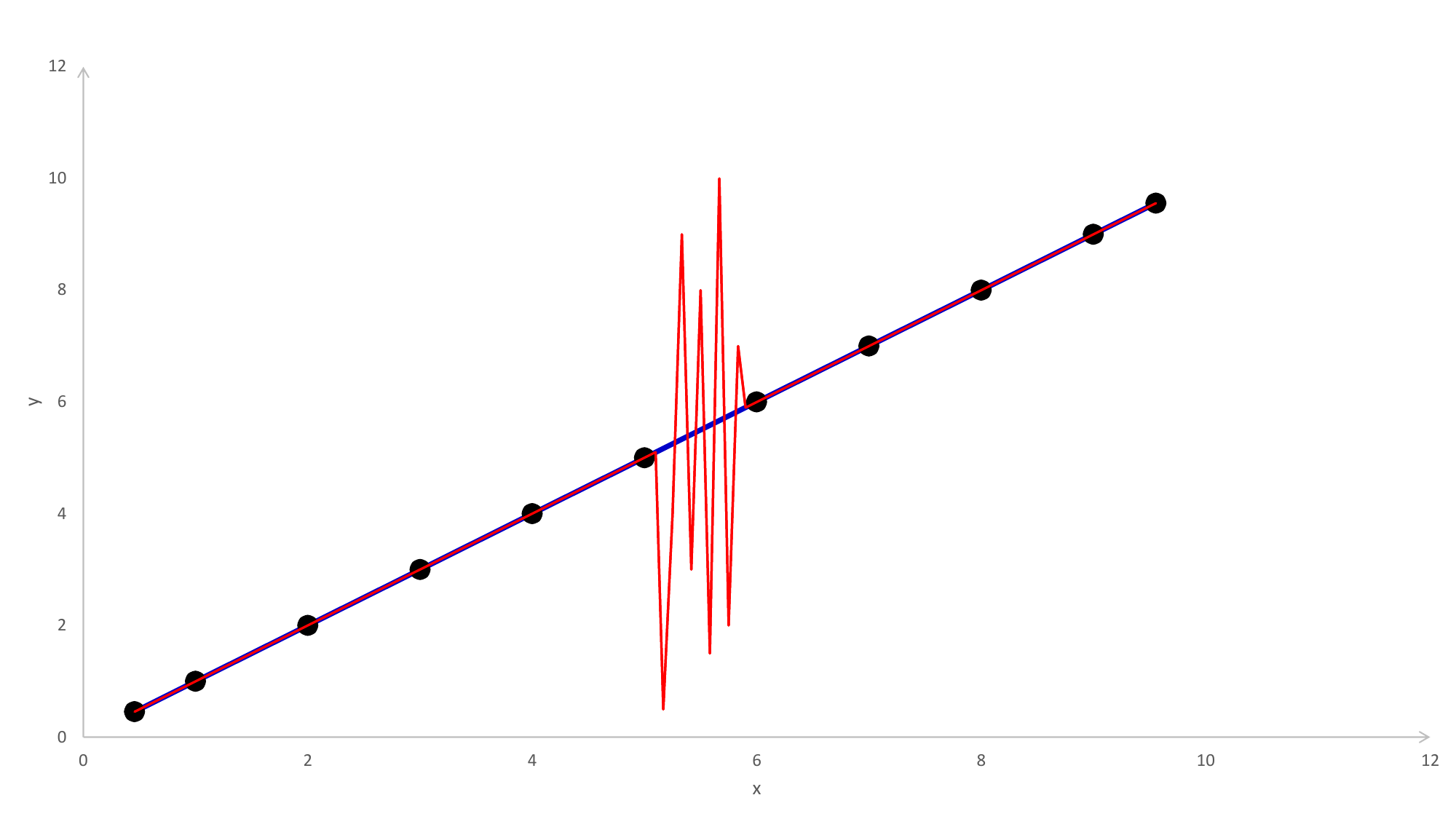}
		\caption{Example: Given these $N=11$ training data points~$\left( \xtr_i,\ytr_i\right)$ (black dots) there are infinitely many functions~$f$ that perfectly fit through the training data and therefore have training loss~$\Ltrsb{f}=0$. Intuition often tells us that one should prefer the straight blue line over the oscillating red line, even though both functions have zero training loss~$\Ltrsb{f}=0$.}
		\label{fig:crazyOscillation}
	\end{figure}
	\begin{figure}[htp]
		\centering
		\includegraphics[width=0.8\textwidth]{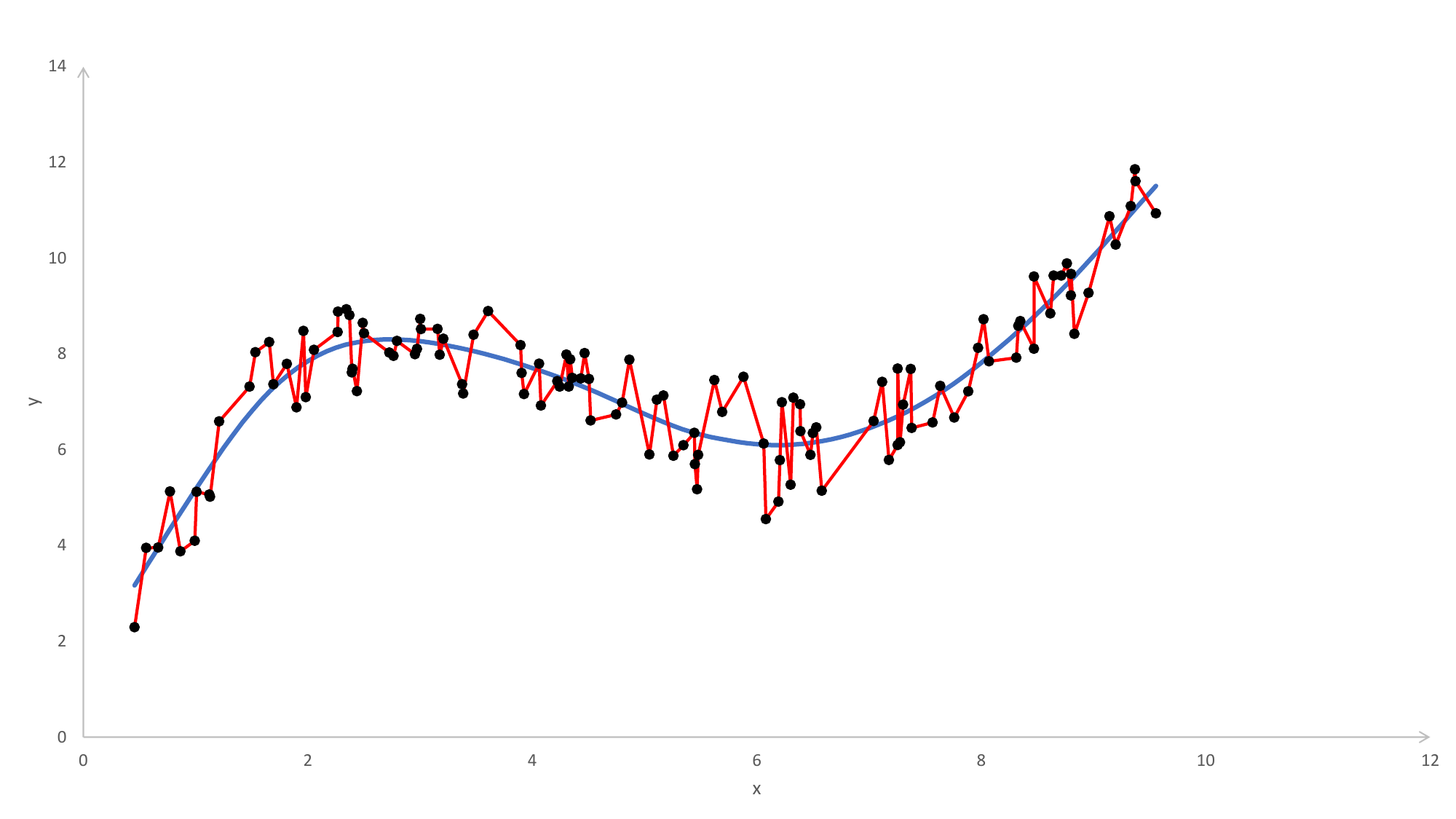}
		\caption{Example: Given these $N=120$ training data points~$\left( \xtr_i,\ytr_i\right)$ (black dots) there are infinitely many functions~$f$ that perfectly fit through the training data and therefore have training loss~$\Ltrsb{f}=0$. For many applications our intuition tells us that we should prefer the smooth blue line~\fl\ over the oscillating red line, even though the smooth function~\fl\ results in training loss~$\Ltrsb{\fl}>0$.}
		\label{fig:SplineReg}
	\end{figure}
	This intuition could be formalized mathematically as a Bayesian prior knowledge\interfootnotelinepenalty=10000\footnote{\label{footnote:longBayesian}One could theoretically formulate this prior knowledge regarding the unknown distribution of $(X,Y)$ on $\X\times\Y$ as a (probability)-measure on the space of all probability measures on $\X\times\Y$. 
If the prior measure is a probability measure, one can work perfectly rigorously in the framework of classical Bayes law. If the prior measure is not a probability measure, we speak of an improper prior, which can also lead to good results in applications. Consider for instance the very restrictive improper prior measure that assigns measure $0$ to the set of all non-linear functions and weights all linear functions the same.
		(Since this measure assigns $\infty$ to the subspace of all linear functions, it is an improper prior.) This improper prior leads to the least square linear regression in the case of \iid\ normally distributed noise. The simple intuitive prior knowledge \enquote{I am absolutely sure that \ftrue\ is linear, but I consider all linear functions as equally likely.} is captured quite well by this improper prior and the solution of the corresponding Bayesian problem can be computed quite fast (linear regression). But for most real-world applications, a more realistic intuitive prior knowledge such as \enquote{I cannot exclude any function for sure, but I have some vague feeling that \ftrue\ is more likely to be a \enquote{simpler}, \enquote{smoother} function than a \enquote{heavily oscillating} function.} is harder to mathematically formalize and calculating the solution of such Bayesian problems is often not tractable. Still, Bayesian theory can be considered a very powerful and general abstract theoretical framework without explicitly solving Bayesian problems and even without explicitly writing down priors.} \cite[e.g.\ page 22]{bishop2006patternML}.

	One approach to capture a common intuition about the prior knowledge is to directly regularize the second derivative of $\hat{f}$. 
	Therefore, in the case of input dimension~$d=1$, the \nameref{def:splineReg} \cite{ReinschSpline1967,CravenSpline1978,KimeldorfSplineBayes10.2307/2239347} is frequently considered in order to choose the function~$\hat{f}$ which minimizes a weighted combination of the integrated square of the second derivative and the training loss~\Ltrs.
	\begin{definition}[spline regression]\label{def:splineReg}
		Let $\allIndi{N}{i}: \xtr_i, \ytr_i \in \R$ and $\lambda \in \Rp$. Then the \textit{(smoothing\footnote{In the literature, the \nameref{def:splineReg} is often called \textit{(natural) (cubic) smoothing spline}, but in this text \fl\ will simply be called \regSpl.}) \regSpl}~$\fl: \R\to\R$ is defined\footnote{\label{footnote:uniquefl}We use the notation $a:\in \{ s\}$ to define $a$ as the unique element $s$ of the set $\{ s\}$ (i.e. $a:=s$). \unimportant{So strictly speaking the set after \enquote{$:\in$} should be a singleton---we are using footnotes to indicate under which assumptions uniqueness can be guaranteed.} The (weighted) \regSpl~\flg\ is uniquely defined \unimportant{(i.e. $\argmin_f \left(\Ltrsb{f} + \Pg(f)\vphantom{\flg} \right)=\left\{\flg\right\}$)} if $\exists(i,j)\in\fromto{N}^2:\xtr_i\neq \xtr_j$ \unimportant{and $g(0)\neq 0$ in the case of the \wregSpl}. The \enquote{$\argmin$} is defined as the set of all minimizers:
		\begin{equation}\label{eq:argminDef}
		    \argmin_{s\in S}F(s):=\Set{s\in S | F(s)=\min_{\tilde{s}\in S} F(\tilde{s})}    \unimportant{=\Set{s\in S|\forall \tilde{s}\in S: F(s)\leq F(\tilde{s})}}
		    .\tag{arg\,min}
		\end{equation}%
		\vspace*{-3ex}
		} as
		\begin{equation}\label{eq:spline_regression}
		\fl \overset{\mfootref{footnote:uniquefl}}{:\in} \argmin_{f\in \C^2(\R)}\underbrace{\left( \overbrace{\sum_{i=1}^{N}\left( f(\xtr_i)-\ytr_i\right)^2}^{\Ltrsb{f}=} + \lambda \overbrace{\int_{-\infty}^{\infty} \left( f''(x) \right) ^2 dx}^{\Pg[1](f):=} \right) }_{=:\Flb{f}}
		\end{equation}
		and for a given function $g:\R\to\Rpz$ the \textit{\wregSplf} is defined\mfootref{footnote:uniquefl} as
		\begin{equation}\label{eq:weighted_spline_regression}
		\flg \overset{\mfootref{footnote:uniquefl}}{:\in} \argmin_{\substack{f\in \C^2(\R)\\ \supp (f'')\subseteq \supp (g)}}\underbrace{\left( \overbrace{\sum_{i=1}^{N}\left( f(\xtr_i)-\ytr_i\right)^2}^{\Ltrsb{f}=}+\lambda \overbrace{g(0) \int_{\supp (g)} \dfrac{\left( f''(x) \right)^2}{g(x)} dx}^{\Pg(f):=} \right) }_{=:\Flb[g]{f}}.  
		\end{equation}
	\end{definition}
	
	The hyperparameter~$\lambda$ controls the trade-off between low training loss and low squared second derivative.
	See~\fl\ in \Cref{fig:SplineReg} for an example of the regression spline (which corresponds to the weighted regression spline~\flg\, with constant weight $g\equiv c>0$).
		
	Letting the regularization parameter $\lambda$ tend to zero in \meqref{eq:spline_regression}, one obtains the smooth \nameref{def:splineInterpolation}, i.e. the \enquote{smoothest} $\C^2$-function interpolating the observed data.
	
	\begin{definition}[spline interpolation]\label{def:splineInterpolation}
		Let $\allIndi{N}{i}: \xtr_i, \ytr_i \in \R$, with $\xtr_i$ distinct, and $\lambda \in \Rp$. Then the \textit{(smooth) spline interpolation}~$\fz: \R\to\R$ is defined\footnote{\label{footnote:uniquefz}Analogous to \cref{footnote:uniquefl}, the spline interpolation~\fz\ is uniquely defined and the unique solution to the right-hand side optimization problem in \cref{eq:SplineInterpolation} if $N>2$, since $\xtr_i$ are distinct.%
		} as:
		\begin{equation}\label{eq:SplineInterpolation}
		\fz :=  \lim_{\lambda\to 0+} \fl \overset{\mfootref{footnote:uniquefz}}{\in} \argmin_{\substack{f\in \C^2(\R), \\ f(\xtr_i)=\ytr_i \ \allIndi{N}{i} }}\left( \int_{-\infty}^{\infty} \left( f''(x) \right) ^2 dx \right)
		.\end{equation}
	\end{definition}
	
	The \Cref{def:splineReg,def:splineInterpolation} can also be seen as solutions to Bayesian problems~\cite{KimeldorfSplineBayes10.2307/2239347}\footnote{\label{footnote:SplineImproperPrior}More precisely, \Cref{def:splineReg,def:splineInterpolation}  can be seen as limits of Bayesian problems \cite[p.\ 502]{KimeldorfSplineBayes10.2307/2239347}. \Cref{def:splineReg,def:splineInterpolation} cannot be solutions of a classical Bayesian problem with a \emph{proper} prior (cp.\ \cref{footnote:longBayesian} on \cpageref{footnote:longBayesian}, \cite[eq.~(4.1) on p.\ 501]{KimeldorfSplineBayes10.2307/2239347} and \cite{wahba1978improper}).}.
	
	\subsection{A paradox of neural networks}\label{sec:Paradoxon}
	As argued above, within a regression problem one might have an intuition about certain attributes of solution functions $\hat{f}$ that are particularly \enquote{desirable}. Moreover, these ideas of suitability could be incorporated directly by including certain regularization terms to the learning problem, such as seen in the popular example of the \nameref{def:splineReg}~\flg. Surprisingly, however, standard algorithms applied to train neural networks \unimportant{(i.e. gradient descent applied to the training loss \Ltrs)} 
	are able to find \enquote{desirable} functions~$\hat{f}$ \emph{without explicit regularization}. This paradox shall be discussed throughout the present section. In particular, we will demonstrate two severe misassumptions typically made in the classical approach to explain supervised learning using neural networks.
	
	The paradox can be observed for deep \cite{Goodfellow-et-al-DeepLearning-2016} as well as for shallow\footnote{\label{footnote:shallowDeep}In the literature \emph{shallow neural networks} are also referred to as \enquote{simple deep neural networks} or \enquote{two-layer (deep) neural networks} \cite[Section~1.1 p.\ 3]{GidelImplicitDiscreteRegularizationDeepLinearNN2019arXiv190413262G}. These three terms all are reasonable, since such a network consists of three layers of neurons (input$\to$hidden$\to$output), therefore it has two layers of weights and biases ($(v,b)\to(w,c)$) and thus one hidden layer of neurons.%
	}
	neural networks. This paper resolves the phenomenon rigorously only in the context of wide \RSNlong{}s (\RSN{}s) with \ReLU\ activation (i.e.\ \wRRSN{}, cp.\ \Cref{sec:theorem}). For simplicity, we outline the paradox for shallow neural networks defined below. %
	
	\begin{definition}[shallow neural network\mfootref{footnote:shallowDeep}]\label{def:NN}
		Let the activation function $\sigma:\mathbb{R}\to\mathbb{R}$ be a non-constant Lipschitz function. Then, a \textit{shallow neural network} is defined as $\NN_\theta:\,\R^d\to \mathbb{R}$ s.t.
		\begin{equation*}%
		\NN_\theta(x):=\sum_{k=1}^{n}w_k\,\sigma\left({b_k}+\sum_{j=1}^{d}{v_{k,j}}x_j\right) + c \quad \faxdg,
		\end{equation*}
		with
		\begin{itemize}
			\item number of neurons~$n\in\N$ and input dimension~$d\in \mathbb{N}$,
			\item weights~$v_k\in\R^d$, $w_k\in\R$, $k=1,\dots, n$ and
			\item biases $c\in\R$, $b_k\in\R$, $k=1,\dots, n$.
		\end{itemize}
		Weights and biases are collected in $$\theta :=(w,b,v,c)\in\Theta:=\R^n\times\R^n\times\R^{n\times d}\times\R.$$
			
	\end{definition}
	\begin{paradox}\label{paradox}
	The paradox of how the training of neural networks leads to solution functions that are surprisingly sensible from a Bayesian perspective (summarized in \Cref{fig:paradox}) consists of two parts:
	\begin{enumerate}[1., ref=\arabic{*}]
		\item\label{itm:pardoxModel} In the literature it is often claimed that the goal of training a neural network is to find parameters
		\begin{equation}\label{eq:thetastar}
		\theta^*\in\argmin_{\theta\in\Theta}\Ltrsb{\NN_\theta},
		\end{equation}
		such that the corresponding neural network~$\hat{f}:=\NN_{\theta^*}$ fits through the training data as good as possible (where goodness of fit is characterized by the loss \Ltrs).
		
		However, such an optimal neural network~$\NN_{\theta^*}$ might have bad generalization properties.
		First, if the number of hidden neurons $n\geq N$ is larger or equal than the number of training data points $N$, there are infinitely many \meqref{eq:thetastar}-optimizing shallow neural networks~$\NN_{\theta^*}$ that generalize arbitrarily badly%
		\footnote{For \ReLU\ activation functions, one can prove, that for every training data $\left( \xtr_i,\ytr_i\right)_{i\in\fromto{N}}$ there exist infinitely many $\NN_{\theta^*}$ such that the $d$-dimensional Lebesgue-measure of the set~$\Set{x\in \left[ -1, 1 \right]^d | {\left|\NN_{\theta^*}(x)\right| >9999} }$ is larger than $99\%\cdot 2^d$ and $\Ltrsb{\NN_{\theta^*}}=0$, if $n\geq N$ and $n\geq1$. %
		This implies that there exist different global optima~$\NN_{\theta^*}$ of $\Ltrs$ that are arbitrarily far from each other in any $L^p$-norm. 
		}, even if there is zero noise~$\eps_i=0$ on the training data.\newline
		Second, if $n\leq N-2$, then $\NN_{\theta^*}$ can be unique, but $\NN_{\theta^*}$ might still overfit to the noise on the training data (see \Cref{fig:exampleparadox}).
		As a consequence of the universal approximation theorems~\cite{CybenkoUniversalApprox1989,HornikUniversalApprox1991251}, we have that large neural networks~$\NN_{\theta^*}$ (or any other universally approximating class of functions) can potentially behave arbitrarily badly (as, for instance, in \Cref{fig:crazyOscillation}) in-between the training data~$\xtr_i$ while keeping the training loss arbitrarily low, i.e. $\Ltrsb{\NN_{\theta^*}}\leq\epsilon$, exactly because of their universal approximation properties.
		(If a very small number of neurons $n\ll\frac{N}{d}$ were chosen, over-fitting of $\NN_{\theta^*}$ would not pose such a severe problem, however, in that case, neural networks would lose their universal approximation property (which is one of their main selling points) and therefore $\NN_{\theta^*}$ could not achieve a low loss~$\Ltrsb{\NN_{\theta^*}}$.)
		
		Paradoxically, however, extremely large (trained) neural networks~$\NN_{\theta}$ typically generalize very well in practice. Indeed, \Cref{thm:ridgeToSpline,thm:GDridge} will demonstrate how well neural networks~$\NN_{\theta}$ with an infinite number of neurons behave in between the data.	
		\item\label{itm:paradox:GDAlgorithm} The objective function in optimization problem~\meqref{eq:thetastar} (in the case of typical activation functions) is a \veryunimportant{Lebesgue-}almost everywhere differentiable function on the finite dimensional \veryunimportant{\R-}vector space~$\Theta$. Thus, for solving \meqref{eq:thetastar}, it seems evident \veryunimportant{not only to most engineers }to use a (stochastic\footnote{\label{footnote:stochasticGradient}Stochastic gradient descent poses immense computational advantages in the case of a very large number~$N$ of training observations (computing an approximate gradient based on small batch is computationally much cheaper than calculating the exact gradient on the entire data set). %
		Within the present work, stochastic gradient descent can be treated equivalently to ordinary gradient descent as we are considering the regime of constant $\gamma/\tau\equiv T$ with diminishing learning rate~$\gamma\to 0$ and $N\in\N$ fixed. In this regime the two methods become equivalent to the continuous gradient flow.}) gradient descent algorithm (where the gradient can be calculated via backpropagation algorithm in the case of neural networks). %

		However, there are no known guarantees that these algorithms converge to a global optimum for a general, typically non-convex optimization problem. Moreover, numerical experiments show that if the algorithm continues for a reasonable time, the solution function obtained is still quite far from being optimal (w.r.t.\ the target function \Ltrs, that 
		the algorithm claims to try to optimize.) (e.g.\ \Cref{fig:exampleparadox}). 
	\end{enumerate}
	\end{paradox}
	\begin{figure}[htp] 
		\centering
		\begin{overpic}[width=0.8\textwidth]{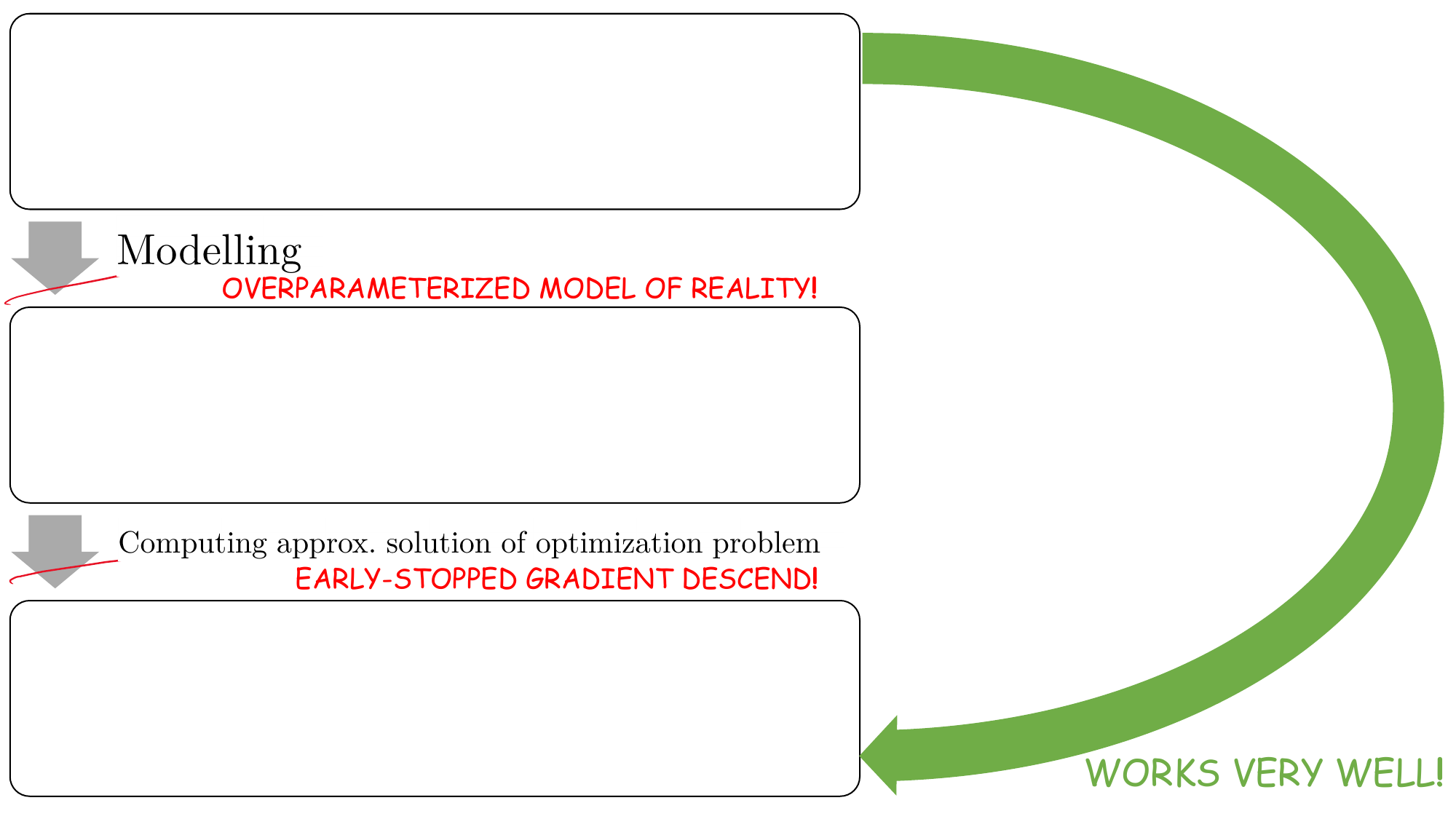}
			\put (3,48) {%
			\begin{tabular}{@{}l@{}}\hspace{-1mm}True Problem in Application: $\hat{f}=\ ?$\\ {\transparent{\meineTranzparenz}\hspace{-1mm}Bayesian Problem with realistic prior} \end{tabular}
			}
			\put (-2.5,35.5) {%
			\begin{minipage}[t]{0.5\textwidth}\small\begin{align*}%
					\hspace{-1mm}\theta^*\span\in\argmin_{\theta\in\Theta}\quad  \underbrace{\Ltrsb{\NN_\theta} }_{\mathclap{\scriptscriptstyle\sum_{i=1}^{N}\left( \NN_\theta (\xtr_i)-\ytr_i\right)^2}}, &\hspace{-1mm} \hat{f}&:=\NN_{\theta^*} \\[8.5ex]%
					\begin{split}
					\theta^{t+\gamma}&=\theta^t-\gamma\nabla_\theta \Ltrsb{\NN_{\theta^t}},\\
					\transparent{\meineTranzparenz}\theta^0&\transparent{\meineTranzparenz}\approx 0,
					\end{split} &\hspace{-1mm} \hat{f}&:=\NN_{\theta^T}
					\end{align*} \end{minipage}
			}
			\put (1.75,38) { %
			\small\ref{itm:pardoxModel}. }
			\put (1.75,18) { \small\ref{itm:paradox:GDAlgorithm}. }
		\end{overpic}
		\caption{\Cref{paradox}:
			\ref{itm:pardoxModel}.\ It would not be desirable for neural networks to solely minimize the training loss~$\Ltrs$.
			\ref{itm:paradox:GDAlgorithm}.\ The (stochastic) gradient descent algorithm (also known as backpropagation algorithm) typically does not succeed in finding a global optimum.
			Nevertheless, the algorithm results in functions~$\hat{f}=\NN_{\theta^T}$ that are surprisingly useful for a wide range of practical applications.}
		\label{fig:paradox}
	\end{figure}

	\begin{figure}[htp]
		\centering
		\scaledinset{l}{.191}{b}{.725}{\resizebox{0.262\hsize}{!}{ \begin{tabular}{@{}l@{}}$\left( \xtr_i,\ytr_i\right)$\\[0.4ex]
				$\ftrue=0$\\
    $\NN_{\theta^*}$ (minimizing $\Ltr$)\\
			$\NN_{\theta^T}$ \text{(obtained by SGD)}\end{tabular}}}{\includegraphics[width=0.8\textwidth,trim={0 0.7cm 0 1.1cm},clip]{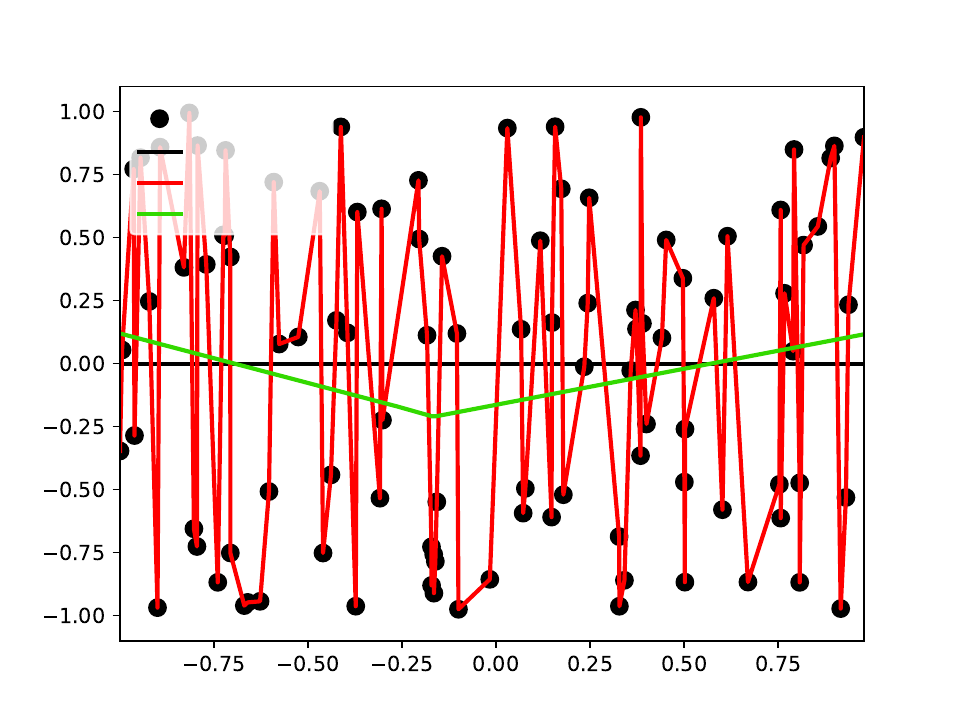}}		\caption{Example: Let $N=100$ training samples~$(\xtr_i,\ytr_i)$ be scattered uniformly around the true function~$\ftrue=0$ and consider a shallow neural network~\NN\ with $n=N=100$ hidden nodes. After 10\,000 training epochs of Adam SGD~\cite{ADAM2014arXiv1412.6980K} the neural network does not converge to the global optimum~$\NN_{\theta^*}$ (red line) with~$\Ltrsb{\NN_{\theta^*}}=0$, but to a more regular function~$\NN_{\theta^T}$ (green line) which is closer to the true function~\ftrue.}  
		\label{fig:exampleparadox}
	\end{figure}
	
	\subsection{Resolving \cref{paradox}: Implicit Regularization}\label{sec:ResolveParadoxon}
	In the following, we like to resolve the paradox described above. Moreover, at the end of this section, a short overview will be given, showing how this work contributes to a better understanding of the aforementioned phenomenon.
	
	Points \ref{itm:pardoxModel}, \ref{itm:paradox:GDAlgorithm} and the observation that neural networks are very useful in practice can be true at the same time:
	
	As discussed above, an \enquote{optimal} network~$\NN_{\theta^*}$ would typically perform quite poorly in practice (cp.\ \ref{itm:pardoxModel}). However, such a network is hardly obtained as a solution from a generic training process involving a gradient descent based algorithm. The reason being that, fortunately, the backpropagation algorithm which was designed to yield trained networks close to $\NN_{\theta^*}$ by minimizing the training loss~\Ltrs\ does not achieve{\transparent{\meineTranzparenz}\footnote{In the limit of infinite training time~$T\to\infty$, the gradient descent method can converge to a global optimum. As we will see in the sequel, even though there typically are infinitely many global optima this limit will be a very specific representative (cp.\ \Cref{def:minNormSol,def:adaptedSplineInterpolation}, \Cref{thm:GDridge,thm:ridgeToSpline,eq:conclusion}). Nonetheless, the training process is typically stopped after a few epochs (with training time $T\ll\infty$). The corresponding solution $\NN_{\theta^T}$ typically satisfies $\Ltrsb{\NN_{\theta^T}}\gg\Ltrsb{\NN_{\theta^*}}$ and is much more desirable (cp.\ \Cref{def:adaptedSplineReg,eq:conclusionApprox}).}} this goal (cp.\ \ref{itm:paradox:GDAlgorithm}, i.e. typically $\Ltrsb{\NN_{\theta^T}}\gg\Ltrsb{\NN_{\theta^*}}$). Instead, it surprisingly succeeds in reaching a much more desirable objective by not only minimizing the training loss~\Ltrs\, but also \emph{implicitly}\footnote{\enquote{\emph{Implicitly}} means that one uses exactly the same algorithm (gradient descent on the training loss~\Ltrs\, cp.\ \Cref{fig:paradox}) that one would use, if one did not care about regularization, but running the algorithm surprisingly results in a very regular solution function $\NN_{\theta^T}$.} regularizing the problem.
	Hence, the typically bad generalization property~\ref{itm:pardoxModel} of $\NN_{\theta^*}$ does not contradict the great out-of-sample performance of $\NN_{\theta^T}$, which is observed to be the much more regular (see \Cref{fig:exampleparadox}).
	
	This phenomenon is known in the literature as \enquote{implicit regularization} \cite{NeyshaburImplicitReg2014arXiv1412.6614N,NeyshaburImplicitRegPhD2017arXiv170901953N,YuanzhiImplicitRegPseudoSmoothing2018arXiv180801204L,KuboImplicitReg2019arXiv,SoudryImplicitBiasSeparableData2017arXiv171010345S,PoggioGernalizationDeepNN2018arXiv180611379P,GidelImplicitDiscreteRegularizationDeepLinearNN2019arXiv190413262G} (also known as \enquote{implicit bias}\cite{SoudryImplicitBiasSeparableData2017arXiv171010345S}). It demonstrates that questions~\ref{item:Generalization} and ~\ref{item:Gradient}, i.e. the generalization properties of neural networks and the use of gradient descent-based methods in their training are strongly linked in practice.
	
	In applications, the phenomenon of implicit regularization is frequently observed \cite{hastie_09_elements-of.statistical-learning,MaennelDradDecentPicewiseLinear2018arXiv180308367M,NeyshaburImplicitReg2014arXiv1412.6614N,NeyshaburImplicitRegPhD2017arXiv170901953N,YuanzhiImplicitRegPseudoSmoothing2018arXiv180801204L,KuboImplicitReg2019arXiv,PoggioGernalizationDeepNN2018arXiv180611379P}. Nonetheless, the theory behind it is still largely unexplored \cite{YuanzhiImplicitRegPseudoSmoothing2018arXiv180801204L,KuboImplicitReg2019arXiv,PoggioGernalizationDeepNN2018arXiv180611379P,MaennelDradDecentPicewiseLinear2018arXiv180308367M}. The contribution of this work (summarized in \Cref{fig:paradoxSolution}) is proving very precisely in which manner the implicit regularization effects occur when training a so-called \RSNlongAndShort\ (a specific type of neural network with one hidden layer and randomly chosen first-layer parameters\unimportant{---\Cref{def:RSNN}}) with a large number of hidden nodes \unimportant{$n\to \infty$} and \ReLU\ activation (i.e., a \wRRSN) using a gradient descent method. \RSN{}s (also known as \enquote{extreme learning machines}) are successfully applied in various settings \citep{HUANG2006489,OptimalStoppingRNarxiv.2104.13669}. As we shall see in the following, for such a network (as a function from $\X$ to $\Y$) the second derivative is implicitly regularized during training. More precisely, we will characterize the solution function obtained in infinite training time for wide networks with a large number of hidden nodes (cp.\ \Cref{def:adaptedSplineReg} and \Cref{thm:ridgeToSpline,thm:GDridge}). In a typical setting, this limit is very close to a \regSplf, whose theory is highly understood  \cite{ReinschSpline1967,CravenSpline1978,KimeldorfSplineBayes10.2307/2239347}.
	
	\begin{remark}[\Pfunc-Functional]\label{rem:Pfunc}In supervised learning, \Pfunc-regularized loss minimization models, i.e.,
	\begin{equation*}
	   f^{*,\Pfunc,\lambda}\in\argmin_f \Ltrb{f}+\lambda\Pfunc(f),
	\end{equation*}
	are typically quite easy to interpret and have nice theoretical properties (e.g. \Cref{def:splineReg}). Each of these models is fully characterized by its regularizing functional~$\Pfunc\unimportant{:\Y^{\X}\to\bar{\R}}$ (e.g. $\Pfunc=\Pg$ in the case of weighted \SplReg\,\flg).\footnote{The letter \enquote{\Pfunc} can be motivated by the fact that the \Pfunc-functional \emph{\textbf{p}enalizes} less regular functions~$f\in\Y^{\X}$, assigning to them a large value of the \emph{\textbf{p}enalty} $\Pfunc(f)$. Moreover, it expresses a certain \emph{\textbf{p}rior belief}\mfootref{footnote:PnotPrior} of which types of functions should be preferred in the supervised learning task.
	\unimportant{Metaphorically speaking the \Pfunc-functional could in some sense be seen as the \enquote{\textbf{p}syche} of a particular type of neural network.
	(I.e. the \Pfunc-functional enables us to easily conclude how the experiences $(\xtr_i,\ytr_i)$ a neural network~\NN\ encounters during training, effect its future behaviour~$\hat{f}(x)=\NN_{\theta^T}(x)$ for any future situation $x\in\X$. This would be a typical question asked in \emph{\textbf{p}sychology} in the case of biological neural networks. Note that different architectures (e.g. different activation functions or different number of layers) can lead to a different \emph{\textbf{p}syche}/character within this analogy.)}
	    }\unimportant{\textsuperscript{, \!}\unskip
	    \footnote{\unimportant{Instead of restricting the definition of \Pfunc\ and the optimization problem $\min_f \Ltrb{f}+\lambda\Pfunc(f)$ to a certain subspace (e.g. $\C^2$) one can also define $\Pfunc (f):=\infty$ for all functions outside the subspace.}
	    }}
	Our key finding is that other supervised learning algorithms (such as standard neural network algorithms) that are typically not considered as \Pfunc-regularized loss minimization, in fact are equivalent to \Pfunc-regularized loss minimization with a specific \Pfunc-functional \unimportant{(i.e. $\NN_{\theta^{\text{result}}}\approx f^{*,\Pfunc,\lambda}$)}. We believe that the framework of \Pfunc-regularized loss minimization could be very well suited to understand and compare the behavior of many different standard methods in supervised learning (in particular neural networks).
	Whether or not a certain \Pfunc-functional (or an equivalent algorithm) leads to functions~$f^{*,\Pfunc,\lambda}$ that generalize well, depends on one's prior\footnote{\label{footnote:PnotPrior}\Pfunc\ should not be directly interpreted as the prior distribution on the function space. However, some \Pfunc-functionals have the property that $f^{*,\Pfunc,\lambda}\in\argmin_f \Ltrb{f}+\lambda\Pfunc(f)$ is equal to the Bayesian a posteriori mean with respect to some Bayesian prior distribution (see e.g.\ \cite{KimeldorfSplineBayes10.2307/2239347} or \cref{footnote:SplineImproperPrior}).} belief. The goal of this work will not be to determine how \emph{well} certain types of neural networks generalize in general situations \unimportant{(this is not possible without assumptions on the data generating process---i.e. $\PP_D$)}. Instead, the main \Cref{thm:ridgeToSpline} expresses \emph{how} a certain neural network~\unimportant{$\RNR\approx\RNw[\wt]$} behaves, by showing its equivalence to a certain \Pfunc-regularized loss minimizer~\unimportant{$\flpm\approx\fl$} characterized by a certain \Pfunc-functional~\unimportant{\Pgpmm\ (see \Cref{def:adaptedSplineReg})}. The long-term goal of this line of research is to describe the learning-behavior of every neural network configuration with its own \Pfunc-functional \unimportant{(see \Cref{fig:paradoxSolution})}, such that one can choose a suitable configuration based on one's prior belief.

	\end{remark}

	\begin{figure}[htp] 
		\centering
		\begin{overpic}[width=0.8\textwidth]{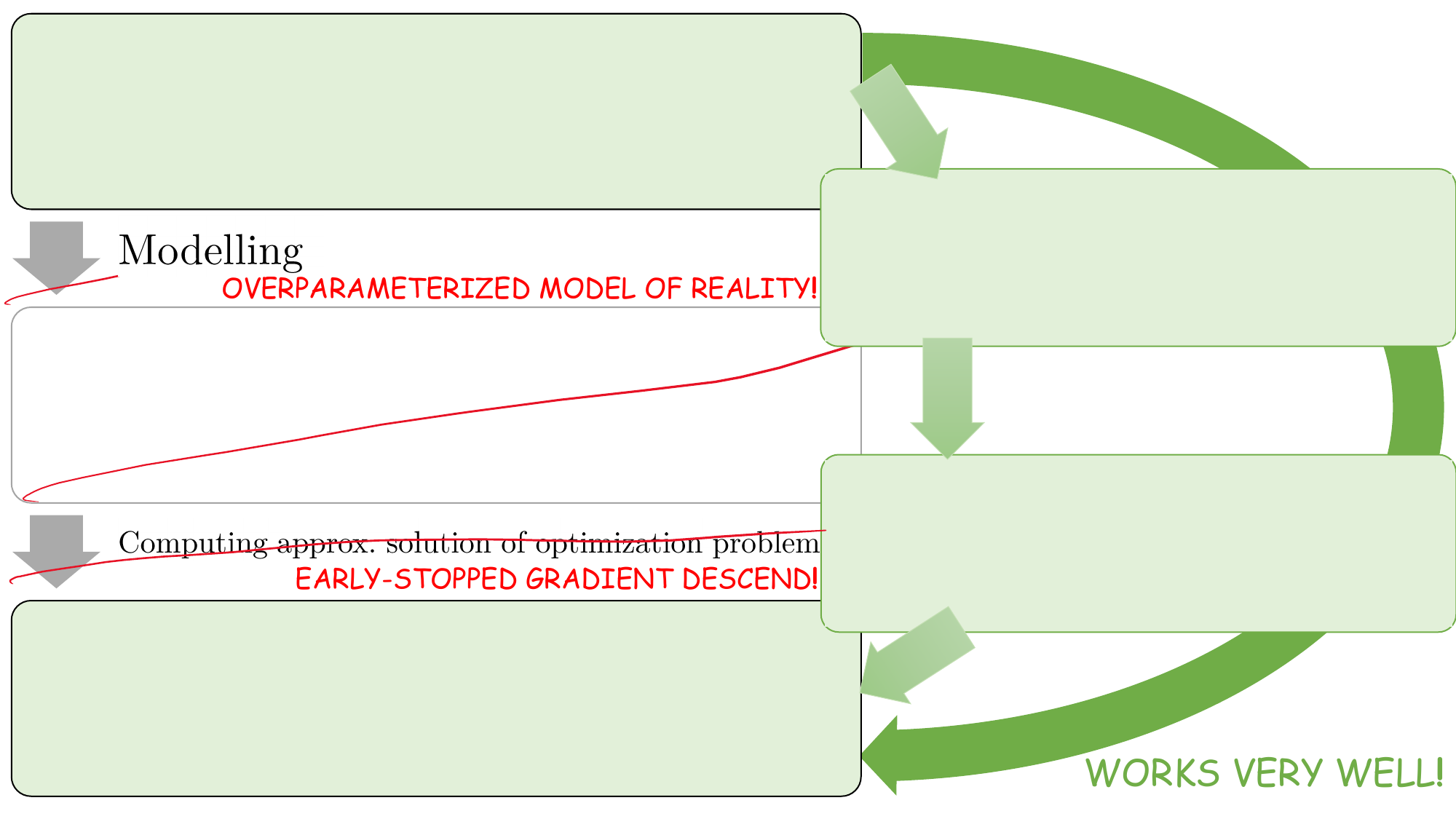}
			\put (3,48) {%
			\begin{tabular}{@{}l@{}}\hspace{-1mm}True Problem in Application: $\hat{f}=\ ?$\\ {\transparent{\meineTranzparenz}\hspace{-1mm}Bayesian Problem with realistic prior} \end{tabular}
			}
			\put (-2.5,35.5) {%
			\begin{minipage}[t]{0.5\textwidth}\small\begin{align*}%
					\transparent{\meineTranzparenz}\hspace{-1mm}\theta^*\span\transparent{\meineTranzparenz}\in\argmin_{\theta\in\Theta}\quad  \underbrace{\Ltrsb{\NN_\theta} }_{\mathclap{\scriptscriptstyle\sum_{i=1}^{N}\left( \NN_\theta (\xtr_i)-\ytr_i\right)^2}}, &\hspace{-1mm} \transparent{\meineTranzparenz}\hat{f}&\transparent{\meineTranzparenz}:=\NN_{\theta^*}\\[8.5ex]%
					\begin{split}
					\theta^{t+\gamma}&=\theta^t-\gamma\nabla_\theta \Ltrsb{\NN_{\theta^t}},\\
					\theta^0&\approx 0,
					\end{split} &\hspace{-1mm} \hat{f}&:=\NN_{\theta^T}
					\end{align*} \end{minipage}
			}
			\put (47,43.75) {%
			\begin{minipage}[t]{0.5\textwidth}\small\begin{align*}%
					&\hat{f}:\in\argmin_{f:\X\to\Y} \Ltrsb{f}+\lambda \Pfunc(f) \\[8.5ex]%
					&\theta_{\lw}:\in\argmin_{\theta\in\Theta}\Ltrsb{\NN_\theta} + \lw\twonorm[\theta]^2 ,\\
					&\hat{f}:=\NN_{\theta_{\lw}}
					\end{align*} \end{minipage}
			}
			\put (1.75,38) { \small\ref{itm:pardoxModel}. }
			\put (1.75,18) { \small\ref{itm:paradox:GDAlgorithm}. }
			\put (67,28) { \small\Cref{thm:ridgeToSpline} }
			\put (64,9) { \small\Cref{thm:GDridge} }
			\put (64,47) { \small\Cref{rem:Pfunc} }

		\end{overpic}
		\caption{Solution of \Cref{paradox}:
		The (early-stopped) (stochastic) gradient descent algorithm on \Ltrs\ (w.r.t.\ the trainable terminal-layer weights) does not solely minimize \Ltrs---instead it minimizes a regularized optimization problem much more accurately, when initialized close to zero $\theta^0\approx0$ (\Cref{thm:GDridge}). This line of research describes this regularization macroscopically on the function space in terms of a \Pfunc-functional (see \Cref{rem:Pfunc}). (\Cref{thm:ridgeToSpline} reveals the very easy to interpret \Pfunc-functional~\Pgpmm\ (\Cref{def:adaptedSplineReg}) in the case of a 1-dimensional \wRRSN~\RN\ (\Cref{def:RSNN}). Other types of neural networks correspond to different \Pfunc-functionals that will be shown in future work.).}
		\label{fig:paradoxSolution}
	\end{figure}

	Within this paper, we state two main theorems that jointly lead to the desired characterization of the solution function obtained in the limit.
	\begin{itemize}
		\item \Cref{thm:GDridge} connects the \RSNlongAndShort\ obtained by performing ordinary gradient descent \veryunimportant{initialized close to zero to train the parameters }without any explicit regularization to the \RSN{} obtained from an explicit ridge regularization of the weights. (This theorem builds on very similar results that are well known in the literature \cite{bishop1995regularizationEarlyStopping,friedman2003gradient,PoggioGernalizationDeepNN2018arXiv180611379P,GidelImplicitDiscreteRegularizationDeepLinearNN2019arXiv190413262G,pmlr-v89-ali19a}.)\footnote{We like to highlight \Cref{sec:EarlyStopping} and in particular \meqref{eq:lambdaWelleT}, which gives a novel approximate relation between the training time $T$ of the GD algorithm and the corresponding ridge regularization parameter $\lw$. In an 'early-stopped' setting, \meqref{eq:lambdaWelleT} is a more refined mapping than the relation $\lw=1/T$. This result is highly interesting on its own as it contributes to the open question posed in e.g.\ \cite{pmlr-v89-ali19a}.}
		\item \Cref{thm:ridgeToSpline} shows how the training of the \wRRSN's weights via ridge regularization results in the (slightly adapted) spline regularization of the learned network function
		if the number of neurons $n\to\infty$. This theorem is the main contribution of this work.
	\end{itemize}

\subsection{Related Work}
Understanding the training of neural networks and, in particular, their frequently astonishing generalization properties has been at the center of interest in many recent works. Without aiming to be exhaustive, we give a brief overview of existing results most related to the the present paper.
	\begin{itemize}
		\item There are a number of works that discuss implicit regularization on the weight space (comparable to \Cref{thm:GDridge}) \cite{bishop1995regularizationEarlyStopping,SoudryImplicitBiasSeparableData2017arXiv171010345S,PoggioGernalizationDeepNN2018arXiv180611379P,GidelImplicitDiscreteRegularizationDeepLinearNN2019arXiv190413262G}\footnote{\cite{SoudryImplicitBiasSeparableData2017arXiv171010345S,PoggioGernalizationDeepNN2018arXiv180611379P} focus on classification (exponential loss) and in
			\cite{bishop1995regularizationEarlyStopping,GidelImplicitDiscreteRegularizationDeepLinearNN2019arXiv190413262G} regression problems (with least square training loss~\Ltrs) are considered.}.
		However, within these works it is mostly not explained how these effects translate to implicit regularization on the function space. As an exception within the framework of classification, \cite{SoudryImplicitBiasSeparableData2017arXiv171010345S,PoggioGernalizationDeepNN2018arXiv180611379P} give insight about the margins between the classes, which is a property of the learned function. These papers provide a precise and quite complete mathematical understanding of linear neural networks without any hidden layers. The theorems in these papers that deal with neural networks with one (or more) hidden layers serve as a basis for arguments why an implicit regularization effect can exist on a qualitative level, but not on a precise quantitative level (especially when non-linear activation functions~$\sigma$ are considered).
		\item Contrary to the above, this paper's main contribution, \Cref{thm:ridgeToSpline}, explains the implicit regularization effects on the function space. In that regard, the results presented in \cite{MaennelDradDecentPicewiseLinear2018arXiv180308367M,KuboImplicitReg2019arXiv,YuanzhiImplicitRegPseudoSmoothing2018arXiv180801204L} are more closely related.
		\begin{itemize}
			\item in \cite{YuanzhiImplicitRegPseudoSmoothing2018arXiv180801204L}, the implicit regularization effects that happen when fully training a shallow neural network~$\NN$ with non-linear \ReLU\ activation function $\sigma=\relu$ are studied on a qualitative level in the context of classification (cross entropy loss over the softmax as a training loss). In said work, the notion \enquote{pseudo-smooth} \cite[e.g.\ p.\ 4]{YuanzhiImplicitRegPseudoSmoothing2018arXiv180801204L} is used, but a quantitative mathematical analysis of the pseudo-smoothness is missing.
			\hypertarget{{subitem:MaennelPiecewiseLinear}}{\item} Similarly in \cite{MaennelDradDecentPicewiseLinear2018arXiv180308367M}, the implicit regularization for a fully trained shallow neural network~$\NN$ with non-linear \ReLU\ activation functions $\sigma=\relu$ is discussed. In the context of regression (using an arbitrary differentiable loss function) the main goal of \cite{MaennelDradDecentPicewiseLinear2018arXiv180308367M} is to explain the macroscopic behavior of the learned neural network function~$\NN_{\theta^T}$, i.e. its generalization properties in between the training data. Within this work, a very rich qualitative understanding of $\NN_{\theta^T}$ as well as very helpful visualizations are provided, however, there is no mention of a precise quantitative formula. Hence, a complete macroscopic characterization of the learned function is not given. In contrast, we provide a precise quantitative macroscopic formula (\Cref{def:adaptedSplineReg}) that characterizes trained \wRRSN s~\unimportant{\RN}. Thus, we provide a quite complete understanding of \wRRSN s~\unimportant{\RN}.
			\item The implicit regularization effects in the training of deep neural networks with non-linear \ReLU\ activation functions $\sigma=\relu$ are studied in \cite{KuboImplicitReg2019arXiv}. Therein, it is stated that the learned function interpolates \enquote{almost linearly} between samples. This behavior is related to a low (in the case of \ReLU s distributional) second derivative which corresponds to the notion of \enquote{gradient gaps} introduced in \cite{KuboImplicitReg2019arXiv}.
		\end{itemize}
		\item In a concurrent line of work, infinitely-wide, ReLU-activated one-layer networks with bounded $\ell_2$-norm of weights are analyzed in function space. Similarly to our work, \cite{ongie2019function,savarese2019infinite,williams2019gradient} intend to link these networks to solutions of regularized optimization problems in function space. %
		These works show convergence of regularized loss functionals, thereafter arguing that minimizers coincide. By contrast, we directly establish the convergence of minimizers. %
  The interchange of limit is highly non-trivial.
		Further differences and similarities to the presented results are the following.
		\begin{itemize}
		    \item \cite{savarese2019infinite} derive a function-space equivalent to the objective of fitting an infinite, real-valued ReLU-network through data while controlling the $\ell_2$-norm of weight parameters (not including the networks biases). In particular, this functional contains a closed-form regularizing complexity term. In \cite{ongie2019function}, this result is lifted to the case of multidimensional input. For networks including a single unregularized linear unit the complexity cost (i.e., regularization term) in function space is derived depending on a certain semi-norm whose interpretation however is not as straightforward. We consider a different setting than both \cite{savarese2019infinite,ongie2019function}, i.e., one-dimensional  networks for which first-layer parameters remain untrained after random initialization. Moreover, we additionally investigate gradient dynamics and implicit regularization effects.
		    In our follow-up paper \cite{Part3Arxiv} we extended our theory to deep networks where the parameters of all layers are optimized as well.
	\item While working on the present manuscript, we learned about the highly related \cite{williams2019gradient} that studies gradient dynamics of infinitely wide networks with one-dimensional inputs in two different parametrizations. For wR\RSN s as we consider in this paper, \cite[Theorem 5]{williams2019gradient} establishes a function space equivalent of the $\ell_2$ regularized objective in parameter space, seemingly similar to this paper's main contribution \Cref{thm:ridgeToSpline}. We stress however that our result is original and in particular not covered by the previous proof. First, as mentioned above, \cite[Theorem 5]{williams2019gradient} is not concerned with the limit of minimizers as claimed and the interchange of limit and argmin is left out in the corresponding proof. In contrast, we precisely prove that the  trained networks and their weak first derivative converge uniformly to the \aregSpl.
	Second, we have proven \Cref{thm:ridgeToSpline} for a much more general loss function $\Ltr$ that is only required to satisfy \Cref{as:generalloss}.
    Third, the weighting function $\nu$ in \cite[Theorem 5]{williams2019gradient} is different to the function $g$ in our \Cref{thm:ridgeToSpline}. This leads us to believe that there is a mistake in their characterization of the learned function. %
 Moreover, in \Cref{thm:ridgeToSpline} we require $\supp(f^{''})\subset\supp(g)$. (If there are no neurons that correspond to kinks outside of $\supp(g)$, this network cannot represent any function which is not affine-linear outside of this area). By contrast, \cite[Theorem 5]{williams2019gradient} would allow for infinitely many solutions that behave arbitrarily crazy outside of $\supp(\nu)$.
Finally, we remark that crucial boundary conditions seem to get lost in the second half of the proof of \cite[Theorem 5]{williams2019gradient}: if all $y_{train}$ are set to one, the adapted spline precisely describes the u-shape outside the data that is resulting from ridge-regularized loss minimization of an \RSN{} (see \Cref{fig:GernealizationOfAdaptedSplineVsSpline} in \Cref{sec:IntuitionAdaptedSpline}), while \cite[Theorem 5]{williams2019gradient} would incorrectly result in a straight line.
	\end{itemize}
		\item \cite{jacot2018neural} gives an exact characterisation of the limiting function by proving an equivalence between neural networks and kernel methods (Gaussian Processes) under quite general assumptions.
		\cite{jacot2018neural} and its generalization \cite{YangFeatureLearningInfiniteWidth} are well-developed theories about the generalization properties and the training of (infinitely) wide neural networks where all parameters are trained, while for \RSN{}s only the last layer needs to be trained. For a more detailed comparison to our follow-up theory on networks where all layers are optimized please see \cite{Part3Arxiv}.
		\item Recently, there has been growing interest in analyzing the convergence behavior of the gradient descent algorithm in the training of infinitely wide (shallow and deep) neural networks \citep{jacot2018neural,chizat2018global,mei2018mean,pmlr-v89-ali19a}. Moreover, in these works, conditions for convergence to global optima are discussed. 
		\item%
		The relation between (possibly multivariate versions of) spline interpolation and neural network structures was analyzed as well in \cite{poggio1990networks}. Therein, a so-called regularization network is defined with activation functions given by certain basis functions based on Green functions corresponding to the optimization problem characterizing the spline interpolation. This network is then proven to implement the smooth spline interpolation. However, this result does not treat implicit regularization effects but rather explicitly implements the desired regularization in the form of a network structure. Moreover, there is no connection to ReLU-NNs established in \cite{poggio1990networks}. 
		\item The architecture of \RSN{}s has been used as well in \cite{doubleDescentbelkin2018reconciling} to describe the phenomenon of double descent. In this work, empirical studies suggest that an \RSN{} \enquote{appears to become smoother---with small norm---as the number of features [neurons] is increased}\citep[p.~6]{doubleDescentbelkin2018reconciling}. In the present work, we prove this phenomenon in a mathematically precise way and quantify this \enquote{smoothness} in terms of the exact regularizing term $\Pgpmm$.
  \item \citet{HUANG2006489} also study the \RSN{}-architecture, which they call \enquote{Extreme Learning Machines} but they focus on different aspects of the theory. They neither consider infinitely wide networks nor gradient descent.
  \item When updating this paper, we came across the related work of \citet{pmlr-v119-jacot20a}.%
  They also study the implicit regularization of \RSN{}s, which they call \enquote{Random Feature Models}. They however do not study the implicit regularization induced by gradient descent, but the implicit regularization induced by averaging over infinitely many ensemble members of which each ensemble member only has a small number of random neurons and the final layer $w$ is perfectly optimized with respect to $\ell_2$-regularization without considering gradient descent.
	\end{itemize}

	The remainder of this paper is structured as follows.
	In \Cref{sec:RSNN}, we begin by defining the specific type of neural network~\unimportant{$\RN$} considered in the subsequent analyses: \emph{1-dimensional \textbf{w}ide \textbf{R}eLU \hyperref[def:RSNN]{\textbf{r}andomized\footnote{The most striking property of this type of network is that the first layer is chosen randomly and not trained, i.e.\ after random initialization only the terminal layer is trained. One might expect that this randomness decreases the regularity of the learned function, but in fact the effect is quite the opposite: as we will thoroughly discuss, the learned function will be especially smooth because of this randomness, where smoothness is understood as minimizing the integrated squared second derivative; cp. \Cref{thm:ridgeToSpline})} \textbf{s}hallow neural \textbf{n}etworks}} (\wRRSN s)~\unimportant{\meqref{eq:RSN1d}}. Moreover, we discuss the expressiveness of the function class of such \RSN s and give further definitions that are central to the understanding of the main \Cref{thm:ridgeToSpline,thm:GDridge}.
	
	Thereafter, in \Cref{sec:RidgeToSpline,sec:GradientToRidge}, \Cref{thm:ridgeToSpline,thm:GDridge} are formulated and discussed. The corresponding proofs are to be found in \Cref{sec:proofs}.
	Finally, in \Cref{sec:conclusion} the implications of these results are summarized in \cref{eq:conclusion,eq:conclusionApprox}. Moreover, therein, we give a brief discussion on extensions of this work.

	\section{Randomized Shallow Neural Networks (\texorpdfstring{$\RSN$}{RSN}s)}\label{sec:RSNN}
	Within this section, we introduce the notion of \RSNlongAndShort, a specific kind of artificial neural network with one hidden layer, that we consider in this paper. 
	
		\begin{definition}[\RSN]\label{def:RSNN}
		Let $\left( \Om, \Sigma, \PP\right) $ be a probability space,
		and the activation function $\sigma:\mathbb{R}\to\mathbb{R}$ Lipschitz continuous and non-constant. Then a \textit{\RSNlongAndShort} is defined as $\RNwo:\,\R^d\to \mathbb{R}$ s.t.
		\begin{equation}\label{eq:RSN}
		\RNwo(x):=\sum_{k=1}^{n}w_k\,\sigma\left({b_k\omb}+\sum_{j=1}^{d}{v_{k,j}\omb}x_j\right) \quad\faog\ \faxdg
		\end{equation}
		with\footnote{\label{footnote:lastlayerbias} One could include an additional bias~$c\in\R$ to the last layer too, but in the limit $n\to\infty$ this last-layer bias $c$ does not change the behavior of the trained network-functions \RNw[\wt] or \RNR. In \Cref{fig:sine,fig:sinezoomedout,fig:secondderiv} this last layer bias $c$ was included in the training.}
		\begin{itemize}
			\item number of neurons~$n\in\N$ and input dimension~$d\in \mathbb{N}$,
			\item trainable weights~$w_k\in\mathbb{R}$, $k=1,\dots, n$,
			\item non-trainable random biases $b_k\begingroup\mediumimportantStart
			:(\Om, \Sigma) \to (\R, \B)\endgroup$ \iid\ real-valued random variables $k=1,\dots,n$,
			\item non-trainable random weights $v_k\begingroup\mediumimportantStart
			:(\Om, \Sigma) \to (\R^d, \B^d)\endgroup$ \iid~$\R^d$-valued random variables $k=1,\dots,n$.
		\end{itemize}
		
	\end{definition}
	\renewcommand{\meinPrefix}{\renewcommand{\thefootnote}{\fnsymbol{footnote}}\!\footnote[1]{\label{footnote:star}\Cref{rem:furtherNotationMeasureTheory}, \Cref{cor:universal_in_prob}, \Cref{le:almost_sure_interpolation}, \Cref{rem:AlmostSureInterpolation} and the text in-between (marked with a \enquote{\mfootref{footnote:star}}) form an independent storyline dealing with question~\ref{item:Expressive} about the expressiveness. If these results are skipped, one can still understand the main storyline and the main~\Cref{thm:ridgeToSpline,thm:GDridge}.}\unskip
	\renewcommand{\thefootnote}{\arabic{footnote}}}
	
	Before describing in detail the implicit regularization effects obtained by applying gradient descent methods to train the last layer of such an \RSN\ in \Cref{sec:theorem}, we elaborate on the expressiveness\mfootref{footnote:star} (question~\ref{item:Expressive}) of \RSN s.
	\begin{remark}[further notation]\label{rem:furtherNotationMeasureTheory}
		Throughout this paper, $\PP_\#f$ denotes the push-forward measure of $\PP$ under the map $f$. Moreover, we frequently use the notation $\mu:=\PP_\#(b_k,v_k)$ for denoting the distribution of a random first-layer parameter vectors $(b_k,v_k):\Omega\to\R^{d+1}$ corresponding to an \RSN~$\RNw$ and write $\lambda^d$ for the Lebesgue measure on $\R^d$. We further introduce the map $\psi_{(b,v)}:\Omega\times\R^{d}\to \R^n$, with $\psi_{(b,v)}:(\om,x)\mapsto\psi_{(b,v)\omb}(x)$ s.t. $\psi_{(b,v)\omb}(x)_k= \sigma\left({b_k\omb}+\sum_{j=1}^{d}{v_{k,j}\omb}x_j\right)$ for any $k=1,\ldots,n$, mapping the input to an \RSN's hidden layer. We call $\mathrm{range}(\psi_{(b,v)}):=\bigcup_{\omega\in\Omega}\mathrm{range}(\psi_{(b,v)(\omega)})\subseteq\R^n$ the latent space of an \RSN. 
	\end{remark}
	
	\renewcommand{\meinPrefix}{\mfootref{footnote:star}\unskip}
	\mfootref{footnote:star}\unskip The class of \RSN s might be interesting in supervised learning due to a number of reasons. 
	First, as a corollary to any of the much-cited universal approximation theorems, randomized shallow networks are what we call \emph{universal in probability}: Building on the results of \cite{HornikUniversalApprox1991251, CybenkoUniversalApprox1989} and later  \cite{leshno1993multilayer}, we obtain that any real-valued continuous function on a compact subset of $\R^d$ can be arbitrarily well approximated by an \RSN\ with arbitrarily high probability. This result holds under relatively weak assumptions on the activation function and probability distribution of first-layer weights and biases and is given below in \Cref{cor:universal_in_prob}.

	\mfootref{footnote:star}\unskip Second, given any set of observations $(x_i,y_i)\in\mathbb{R}^d\times\mathbb{R}$, $i=1,\ldots,N$, $N\in\N$ (with $\xtr_i$ distinct), if the induced measure on the latent space is zero on sets of lower dimension, then, almost surely, there exists an \RSN\ that precisely interpolates these data. In other words, for suitable choices of randomness in the first layer, with probability one the class of randomized shallow networks contains representatives whose parameters achieve zero training loss. More precisely, we have \Cref{le:almost_sure_interpolation}.

	\begin{corollary}[Universal in probability]\label{cor:universal_in_prob}
		Let $X\subset\R^d$ be compact and $f\in C(X,\mathbb{R})$. Furthermore, let $\RNw$ be as in \Cref{def:RSNN}, with weights $v_k$ and biases $b_k$, $k=1,\ldots,n$ \iid\ according to $\mu:=\PP_\#(b,v)$ with $\mu\gg \lambda^{d+1}$. Then, under mild conditions on the activation function (e.g. $\sigma$ non-polynomial \cite{leshno1993multilayer})

  \begin{displaymath}
		\forall \epsilon\in\mathbb{R}_+,\lim\limits_{n\to\infty} 
		\PP\left[\exists w\in\mathbb{R}^n:\|\RNw-f\|_\infty<\epsilon\right]=1.
		\end{displaymath}
	\end{corollary}
	\proofInSec{cor:universal_in_prob}{sec:Proof:expressiveness}

	\begin{lemma}[Almost sure interpolation]\label{le:almost_sure_interpolation}
		Let observations $(\xtr_i,\ytr_i)\in\mathbb{R}^d\times\mathbb{R}$, $i=1,\ldots,N$, with $\xtr_i$ distinct, be given. Then, any (perfectly trained\footnote{Since the optimization problem is convex in the last-layer weighs~$w$, the gradient descent actually converges to a global minimum. Hence, under the conditions of \Cref{le:almost_sure_interpolation}, the statement can be refined to:
		\[\PP[{\tlim\RNw[\wt](\xtr_i) = \ytr_i,\quad \forall i=1,\ldots, N}]=1.\]\vspace{-3ex}%
		}) \RSN~$\RNw$ with $n\ge N$ hidden nodes such that $\PP_\#(\psi_{(b,v)}(\xtr)_1)[A]=0$ for any $A\subseteq\mathrm{range}(\psi_{(b,v)})$ of dimension $N-1$, almost surely interpolates the data, i.e.
		\begin{displaymath}
		\PP[{\exists w^*\in\mathbb{R}^n:\forall i\in\fromto[1]{N}: \RNw[w^*](\xtr_i) = \ytr_i}]=1.
		\end{displaymath}
	\end{lemma}
	\proofInSec{le:almost_sure_interpolation}{sec:Proof:expressiveness}
	\begin{remark}\label{rem:AlmostSureInterpolation}
		In \Cref{le:almost_sure_interpolation} we required random features of the latent space $\psi_{(b,v)}(x_i)$, $i=1,\ldots,N$ to follow a distribution on $\R^n$ that puts zero mass on sets of dimension $N-1$. A setting which is rather usual in applications and for which this condition is satisfied would for instance consist in taking $\PP_\#(b,v)\ll \lambda^{d+1}$ and $\sigma:\R\to(0,1)$, $\sigma(x)=\exp(x)/(1+\exp(x))$.\footnote{For \ReLU\ activation functions almost sure interpolation is often not the case with finite $n<\infty$, but the probability of perfect interpolation converges to one when the number of neurons $n\to\infty$ tends to infinity.}
	\end{remark}
	\renewcommand{\meinPrefix}{}
	By \Cref{le:almost_sure_interpolation} and \Cref{cor:universal_in_prob}, the function class of \RSN s is expressive enough to qualify as a suitable architecture within the framework of supervised learning. At the same time, these results raise the question~\ref{item:Generalization}, if \RSN s generalize badly to unseen data, because of over-parametrization and over-fitting (see \Cref{paradox} \cref{itm:pardoxModel}.). Our main \Cref{thm:ridgeToSpline,thm:GDridge} are dealing with question~\ref{item:Generalization} by providing
	a certain understanding of the implicit regularization effects that occur when training a specific kind of \RSN: As we will show in the sequel, training the last layer of a \textbf{w}ide \unimportant{(i.e. $n\to\infty$)}, \RReLU-activated \textbf{\RSN} (\wRRSN) using gradient descent \unimportant{initialized close to zero} corresponds to solving a \SplReg.
	The main assumptions we require to hold are made precise in \Cref{as:mainAssumptions} below.
	\begin{assumption}\label{as:mainAssumptions} Using the notation from \Cref{def:RSNN}:	
		\begin{enumerate}[a)]
			\item\label{itm:as:ReLU} The activation function $\sigma(\cdot)=\relu$ is \ReLU.\footnote{In future work we want to derive other \Pfunc-functionals for other activation functions instead of the \textbf{re}ctified \textbf{l}inear \textbf{u}nits (\ReLU)}
			\item\label{item:densityExists} The distribution of the quotient~$\xi_k:=\frac{-b_k}{v_k}$ has a probability density function~$g_\xi$ \unimportant{with respect to the Lebesgue measure}.\footnote{\Cref{as:mainAssumptions}\ref{item:densityExists} holds for any distribution typically used in practice. Moreover, it implies that $\PP[v_k=0]=0 \quad\forall k \in \fromto{n}$. Note that \Cref{as:mainAssumptions}\ref{item:densityExists} is required in order to exclude certain degenerate cases of \RSN s such as those with constant weights and biases $w_k,b_k, k=1,\ldots,n$, and could in fact be weakened.}
			\item The input dimension~$d=1$.\footnote{In part~II~\cite{implReg2} we are going to generalize the result to arbitrary input dimension~$d\in\N$.}
		\end{enumerate}
		Under these assumptions, \cref{eq:RSN} simplifies to
		\begin{equation}\label{eq:RSN1d}
		\RNw(x)=\sum_{k=1}^{n}w_k\,\relu[{b_k}+v_{k} x] \quad \faxg\mydot
		\end{equation}
	\end{assumption}
	
	We henceforth require \Cref{as:mainAssumptions} to be in place. For later uses, we further introduce the notions of kink positions \veryunimportant{corresponding to individual neurons }of a one-dimensional \RSN\ with \ReLU\ activation and their density function.   
	
	\begin{definition}[kink positions~$\xi$]\label{def:kinkPos}
		The \textit{kink positions}~$\xi_k:=\frac{-b_k}{v_k}$ are defined using the notation of \Cref{def:RSNN} under the \Cref{as:mainAssumptions}.
	\end{definition}
	
	\begin{definition}[kink position density~\gxi]\label{def:kinkPosDens}
		The \textit{probability density function~$\gxi:\R\to\Rpz$ of the kink position}~$\xi_k:=\frac{-b_k}{v_k}$ is defined in the setting of \Cref{def:kinkPos}.
	\end{definition}

	\section{Main Theorems}\label{sec:theorem}

	We now proceed to show that a standard gradient descent method applied to optimize the (trainable) parameters~$w$ of an \wRRSNlongAndShort~\RN, implicitly minimizes the second derivative of the solution function~\unimportant{\RNw[\wt]}. That is, in the many particle (i.e. neurons) limit \unimportant{($n\to\infty$)} and as training time~$T\to\infty$ tends to infinity, the solution found by the gradient descent algorithm~\unimportant{\RNw[\wt]} converges to a \unimportant{slightly adapted} smooth spline interpolation~$\unimportant{\fLpm[0+]}\approx\fz$, if initialized \unimportant{$\wt[0]\approx 0$} close to zero.

	Our result follows by two separate observations. First, note that training a wide \RSN\ in essence reduces to solving a (random) kernelized linear regression in high dimensions \unimportant{(over-parameterized)}. We obtain in \Cref{thm:GDridge} that training an \RSN\ up to infinity \unimportant{(initialized at zero $\wt[0]=0$)} leads to the same solution as performing ridge regression \unimportant{(\Cref{def:ridgeNet})} with diminishing regularization to tune the parameters of the \RSN's terminal layer. Note, that the results in \Cref{sec:GradientToRidge} hold for a general input dimension $d\in\N$ and any fixed number of neurons in the hidden layer $n\in\N$.
	
	Second, in \Cref{sec:RidgeToSpline}, we relate the \RSN~\unimportant{\RNR} with optimal terminal-layer parameters~\unimportant{\wR} chosen according to a ridge regression \unimportant{(\Cref{def:ridgeNet})} to a smoothing spline~\fl\ \unimportant{(with certain regularization parameters $\lw:=\lambda n2g(0)$ and $\lambda\in\Rp$ respectively)}. More precisely, we show in \Cref{thm:ridgeToSpline} that as the number of hidden nodes~$n$ \unimportant{(i.e. the dimension of the hidden layer)} tends to infinity the \ridgeRSN~\unimportant{\RNR} converges to a slightly adapted smoothing spline~\unimportant{\fLpm} in probability with respect to a certain Sobolev norm. Recall that, by \Cref{as:mainAssumptions}, we prove this correspondence for \wRRSN s with one-dimensional input.

	\begin{remark}
		The implicit regularization effects we characterize within this paper are of an asymptotic nature. For applications, however, it is interesting to note that, even for finitely many hidden nodes and finite training time. In future work, we will give explicit bounds on the distance between the solution obtained by gradient descent and a certain smoothing spline (see also \Cref{sec:EarlyStopping} for further details). %
	\end{remark}

	In the following \Cref{sec:RidgeToSpline,sec:GradientToRidge} we discuss both observations separately, before combining them to formulate our main conclusion in \Cref{sec:conclusion}.  We start by introducing the notions of \ridgeRSN\ and minimum norm network.

	\begin{definition}[\ridgeRSN]\label{def:ridgeNet}
		Let $\RNwo$ be a randomized shallow network as introduced in \Cref{def:RSNN}, $\Ltr:L^{\infty}\to\Rpz$ a given loss functional and $\lw>0$. The \ridgeRSN\ is defined as
		\begin{equation}
		\RNRo := \RN_{\wRo,\om} \quad\faog\mycomma
		\end{equation}
		with $\wRo$ such that\footnote{If $\Ltr\ge 0$ is convex, $\wRo$ is uniquely defined by \meqref{eq:wR}, see \Cref{le:existenceuniquenessRidgeRSN}.} 
		\begin{equation}\label{eq:wR}
		\wRo \in \argmin_{w\in\R^n}\underbrace{
		\Ltr\left( \RNwo\right)+\lw\|w\|_2^2}_{\Fnb{\RNwo}} \quad\faog\mydot
		\end{equation}
		
	\end{definition}
	The ridge regularization is also known as \enquote{weight decay}, \enquote{ridge penalization}, \enquote{$L^2$ (parameter) regularization} or \enquote{Tikhonov regularization} (or \enquote{ridge regression}, \enquote{$\ell_2$ penalty}, \dots)\cite[section 7.1.1 on p. 227]{Goodfellow-et-al-DeepLearning-2016}.

	\begin{definition}[minimum norm \RSN]\label{def:minNormSol}
		Using the notation from \Cref{def:ridgeNet},
		the \emph{minimum norm\footnote{\label{footnote:existenceminimumnorm}Upon all global optima~$w^*\omb$ of the training loss~\Ltr, the minimum norm \RSN~$\RNwo[\wdago]$ has unique weights \wdago\ with minimal norm in the case of convex $\Ltr$ which fullfills \Cref{as:generalloss} and if an optimzer of $\Ltr$ exists. E.g., in the case of standard square loss~\meqref{eq:Ltr_squaredloss} in the over-parameterized setting ($n\gg N$) there are infinitely many global optima ~$\RNwo[w^*\omb]$ with arbitrary large norm~$\twonorm[w^*\omb]$, but \wdago\ is always unique. If the number of hidden neurons $n$ is large enough and the $\xtr_i$ are distinct \unimportant{(see \Cref{le:almost_sure_interpolation})}, \wdag\ could be equivalently defined as
		\[\wdago:\in\argmin_{w\in\R^n,\allIndi{N}{i}:\  \RNwo(\xtr_i)=\ytr_i}\twonorm[w]\quad\faog.
		\]
		} \RSN} is then defined as $\RNRz:=\RNw[\wdag]$ with weights%
		\begin{equation}\label{eq:minNormSol}
		\wdago:=\lim_{\lw\to 0+}\wRo \quad\faog\mydot
		\end{equation}
	\end{definition}

	\subsection{\texorpdfstring{\RIDGERSN}{Ridge Regularized RSN} \texorpdfstring{$\to$}{to} Spline Regularization \texorpdfstring{\color{hellgrau}($d=1,\lambda\in\Rp$)}{}}\label{sec:RidgeToSpline}%
	Throughout this section we rigorously derive the correspondence between the \regSplf\ respectively the \ridgeRSNRN\ with penalty parameters $\lambda>0$ and $\lw>0$. For giving a detailed description of the convergence behavior, we introduce an adapted version of the regression spline, for which we consider a weighted version of the spline penalization \unimportant{restricted to the support of the weighting function and introduce certain \enquote{boundary conditions}}. Depending on the distribution of the random weights~$w_k$ and biases~$w_b$, the \ridgeRSNRN\ will converge to such a (slightly) adapted version~\flpm\ of the classical \regSplf. 

	\begin{remark}
	For constant $g\equiv g(0)$, the following \Cref{def:adaptedSplineReg} recovers the original spline regression.\footnote{This statement holds in the limit $\frac{g}{g(0)}\to 1$ (see \Cref{prop:gTo1} in \Cref{subsec:SimilaritWithoutSkipConnections}). Formally \cref{eq:adaptedSplineReg} in \Cref{def:adaptedSplineReg} would not have a classical minimizer, if $g$ were constant (see \cref{footnote:first:as:truncated}), but one could reformulate the definition of $\Pgpmm$ in \Cref{def:adaptedSplineReg} by replacing the minimum by an infimum to extend \Cref{def:adaptedSplineReg} to arbitrary weighting functions $g$ that do not have finite second momentum or that even have infinite integral like constant $g\equiv g(0)\neq 0$. For typical choices of distribution for the first-layer weights $v_k$ and biases $b_k$, the corresponding weighting function~$g$ satisfies the finite second moment condition of \cref{foot:gFiniteSecondMoment} in \Cref{le:existenceuniquenessaregspl}.} As we will show in the sequel, the distribution chosen for the \nameref{def:kinkPos} of the \RNR\ %
	will determine the weighting function of the corresponding \flpm .
	Thus, $g$ reflects how the choice of randomness in the hidden layer affects the inductive bias.
	\end{remark}

	\begin{definition}[adapted spline regression]\label{def:adaptedSplineReg}
		Let $\Ltr:L^{\infty}\to\Rpz$ be a loss functional
		and $\lambda \in \Rp$.
		Then for a given function $g:\R\to\Rpz$ the \textit{\aregSpl}~$\flpm$ is defined%
		\footnote{\label{footnote:uniqueAdaRegSpline}The \aregSplf\ exists for $L$ satisfying \Cref{as:generalloss}\ref{item:continuousL}, if~$g$ is compactly supported and continuous on $\supp(g)$
		\unimportant{and $g(0)\neq 0$}. It is uniquely defined in case $L$ is convex (cp. \Cref{le:existenceuniquenessaregspl}).}\textsuperscript{,\,}\footnote{Note that $f_+$ and $f_-$ do \emph{not} denote the positive and negative part of $f$. The tuple of functions $\left(f_+,f_-\right)$ can be any arbitrary element of $\T$.%
  }
		as
		\begin{equation}\label{eq:adaptedSplineReg}
		\flpm \in \argmin_{f\in \WT(\R)}\underbrace{ \Ltrb{f}+\lambda \Pgpmm(f)  }_{=:\Flpmmb{f}},    
		\end{equation}
		with\footnote{Within this paper, we use the notation $\min \emptyset=\infty$.}
		\hypertarget{eq:Pgpmm}{\begin{equation*}
		\Pgpmm(f):=  2g(0)\underset{\underset{f=f_++f_-}{(f_+, f_-)\in\T}}{\min} \left(
		\int_{\supp (g)} \frac{\left( {f_+}^{''}(x) \right)^2}{g(x)} dx
		+\int_{\supp (g)} \frac{\left( {f_-}^{''}(x) \right)^2}{g(x)} dx
		\right),
		\end{equation*}}
		and
		\begin{align*}
		\T:=\bigg\{(f_+,f_-)\in \WT(\R)\times \WT(\R)\bigg|& \supp (f_+'')\subseteq \supp (g), \supp (f_-'')\subseteq \supp (g),\\
		&\lim_{x\to -\infty} f_+(x)=0, \lim_{x\to -\infty} f_+'(x)=0,\\
		&\lim_{x\to +\infty} f_-(x)=0, \lim_{x\to +\infty} f_-'(x)=0\bigg\}.
		\end{align*}
	\end{definition}
	
	\begin{remark}\label{rem:compactSupp}
		If for the weighting function $g$ it holds that $\supp(g)$ is compact (cp. \Cref{as:truncatedg}\ref{item:truncatedg}), we define
		\begin{equation}
		\Cgl:= \min (\supp(g)) \quad \text{and}\quad \Cgu:=\max(\supp(g)).
		\end{equation}
		Furthermore, in that case, the set $\T$ can be rewritten: From $\supp (f_+'')\subseteq \supp (g)$ it follows that $f_+'\in\C^1(\R)$ is constant on $(-\infty, \Cgl]$. With $\lim_{x\to -\infty} f_+'(x)=0$ we obtain that $f_+'(x)=0$ $\forall x\le\Cgl$. By the same argument we obtain $f_+(x)=0$ $\forall x\le\Cgl$.
		Moreover, we have that $\exists \,\cp\in\R: f_+'(x)\equiv \cp$ on $[\Cgu, \infty)$. Analogous derivations lead to $f_-'(x)\equiv \cm$ $\forall x\le\Cgl$  with $\cm\in\R$ and $f_-(x)=f_-'(x)=0$ on $[\Cgu, \infty)$. Hence, altogether, we have

		\begin{align*}
		\T=\bigg\{(f_+,f_-)\in \WT(\R)\times \WT(\R)\bigg|
		&\supp (f_+'')\subseteq \supp (g), \supp (f_-'')\subseteq \supp (g), &\\
		&\, \forall x\le\Cgl: f_+(x)=0=f_+'(x) ,  &\\
		&\, \forall x\ge\Cgu: f_-(x)=0=f_-'(x)  &\bigg\}.
		\end{align*}
		If we assume $\supp(g)=[\Cgl,\Cgu]$ we get:
		\begin{align*}
		\T=\bigg\{(f_+,f_-)\in \WT(\R)\times \WT(\R)\bigg|
		&\exists \cm,\cp\in\R:\\
		&\, \forall x\le\Cgl: \left( \vphantom{\Cgl} f_+(x)=0=f_+'(x) \ \text{ and } %
		\ f_-'(x)=\cm\right) ,&\\
		&\, \forall x\ge\Cgu: \left( \vphantom{\Cgl} f_-(x)=0=f_-'(x) \ \text{ and } %
		\ f_+'(x)=\cp\right)  &\bigg\}.
		\end{align*}
		
	\end{remark}
 For more intuition on the \aregSpl{} from \Cref{def:adaptedSplineReg} see \Cref{sec:IntuitionAdaptedSpline}. In \Cref{sec:IntuitionAdaptedSpline} we explain the (subtle) difference between the \aregSplf{} and the classical \regSplf{}.
 
	Building on \Cref{def:adaptedSplineReg}, we define an adapted version of the smooth spline interpolation.
	\begin{definition}[adapted spline interpolation]\label{def:adaptedSplineInterpolation}
		For \flpm\ as in \Cref{def:adaptedSplineReg}, the \textit{adapted spline interpolation}~$\fzpm: \R\to\R$ is defined\footnote{\label{footnote:uniquefzpm}Analogous to \cref{footnote:existenceminimumnorm} the spline interpolation~\fzpm\ is uniquely defined for convex $L$ satisfying \Cref{as:generalloss} if an optimizer of $L$ exists, and if~$g$ is compactly supported and continuous on $\supp(g)$
		\unimportant{and $g(0)\neq 0$} (cp. \Cref{le:existenceuniquenessaregspl}).}
		as
		\begin{equation}
		\fzpm :=  \lim_{\lambda\to 0+} \flpm
		.\end{equation}
	\end{definition}
	
	Before stating the core result of this paper's analyses in \Cref{thm:ridgeToSpline}, we like to discuss further assumptions we make therein. These requirements are technicalities that facilitate the \hyperlink{proof:thm:ridgeToSpline}{proof} of \Cref{thm:ridgeToSpline} and could be weakened (see \crefrange{footnote:first:as:truncated}{footnote:last:as:truncated}).
	\begin{assumption}[kink position density]\label{as:truncatedg} Using the notation from \Cref{def:RSNN,def:kinkPosDens} the following assumptions extend \Cref{as:mainAssumptions}:
		\begin{enumerate}[a)]
			\item\label{item:truncatedg} The probability density function~$\gxi$ of the kinks~$\xi_k$ has compact support~$\supp(\gxi)$.
			\footnote{\label{footnote:first:as:truncated}We believe that \Cref{as:truncatedg}\ref{item:truncatedg} can be weakened quite extensively. %
   This assumption facilitates our proofs and it assures that a minimum of~\meqref{eq:adaptedSplineRegTuple} exists. \Cref{subsec:SimilaritWithoutSkipConnections} suggests that the weaker assumption $\int_{\R}g(x)x^2dx<\infty$ should be sufficient to obtain the existence of a minimum. If %
   $\int_{\R}g(x)x^2dx=\infty$, it could happen that \meqref{eq:adaptedSplineRegTuple} does not have a classical minimum (e.g. $\PP[v_k=-1]=\frac{1}{2}=\PP[v_k=1]$ and $b_k\sim \text{Cauchy}$). As a remedy, one could define a weaker concept of minimum being the limit of minimizing sequences which converge to a unique function on every compact set. This also corresponds to the unique point-wise limit of minimizing sequences, which is not a classical minimum, because it does not satisfy all the boundary conditions~$\lim_{x\to -\infty} f_+(x)=0 =\lim_{x\to +\infty} f_-(x)$ anymore. For of this weaker minimum concept, \Cref{thm:ridgeToSpline} would need to be reformulated at least slightly, in case \Cref{as:truncatedg}\ref{item:truncatedg} were entirely skipped. This weaker minimum concept can also be seen as the limit of \aregSpl s~\flpm\ for truncated $g$ as the range of the truncation tends to~$(-\infty,\infty)$.%
   }
			\item\label{item:densityIsSmooth} The density~$\left.\gxi\right|_{\supp (\gxi)}$ is uniformly continuous on $\supp (\gxi)$.%
			\footnote{\label{footnote:densityIsSmooth}One could think of replacing \Cref{as:truncatedg}\ref{item:densityIsSmooth} by the weaker assumption that $\gxi$ is (improper) Riemann-integrable, however, almost all distributions which are typically used in practice satisfy \Cref{as:truncatedg}\ref{item:densityIsSmooth}.}
			\item\label{item:reziprokdensityIsSmooth} The reciprocal density~$\left.\frac{1}{\gxi}\right|_{\supp (\gxi)}$ is uniformly continuous on $\supp (\gxi)$.%
			\footnote{\Cref{as:truncatedg}\ref{item:reziprokdensityIsSmooth} implies that $\min_{x\in\supp (\gxi)}\gxi >0$. Similarly to \cref{footnote:densityIsSmooth}, this assumption might be weakened in a way allowing $\gxi$ to have finitely many jumps and  $\min_{x\in\supp (\gxi)}\gxi$ to be zero.}
			\item\label{item:vkDistrSmooth} The conditioned distribution~$\mathcal{L}(v_k|\xi_k=x)$ of $v_k$ is uniformly continuous in $x$ on $\supp(\gxi)$.%
			\footnote{\label{footnote:last:as:truncated}Similarly to \cref{footnote:densityIsSmooth}, \Cref{as:truncatedg}\ref{item:vkDistrSmooth} might be attenuated.}
			\item\label{item:vk:finiteSecondMoment}$\E[v_k^2]<\infty$.\footnote{\label{footnote:finiteSecondConditionalMomentOfvk}\Cref{as:truncatedg}\ref{item:vk:finiteSecondMoment} always holds in typical scenarios. \Cref{as:truncatedg}\ref{item:vk:finiteSecondMoment} together with \Cref{as:truncatedg}\ref{item:truncatedg} and \ref{item:vkDistrSmooth} implies that $\Eco{v_k^2}{\xi_k=x}$ is bounded on $\supp(\gxi)$.}
		\end{enumerate}
	\end{assumption}
	The following technical \Cref{as:easyReadable} makes the result of \Cref{thm:ridgeToSpline} more readable by referring to the easier \Cref{def:adaptedSplineReg}. Without \Cref{as:easyReadable}, the \Cref{cor:ridgeToSplineasym} would still hold, which is more general than \Cref{thm:ridgeToSpline}, but uses the heavier notation of \Cref{def:asymadaptedSplineReg}.
	\begin{assumption}[symmetry]\label{as:easyReadable} Using the notation from \Cref{def:RSNN,def:kinkPosDens} the following assumptions extend \Cref{as:mainAssumptions}:
		\begin{enumerate}[a)]
			\item\label{item:gVonNull} $\gxi(0)\neq 0$.%
			\footnote{\Cref{as:easyReadable}\ref{item:gVonNull} has to be satisfied due to the way \Cref{def:adaptedSplineReg} and \Cref{thm:ridgeToSpline} are formulated, although the theory could be easily reformulated (see for instance \Cref{cor:ridgeToSplineasym}) if \Cref{as:easyReadable}\ref{item:gVonNull} were not satisfied. The theorems presented would hold as well if $g(0)$ were replaced by a fixed value~$g(x_{\text{mid}})$ or by e.g. $\frac{1}{2}\int_{-1}^1 g(x)dx$, if this replacement is done consistently both for $\lw$ and $\Pgpmm$. However, the results are more easily interpreted if $x_{\text{mid}}$ is located somewhere \enquote{in the middle} of the training data. \Cref{thm:ridgeToSpline} would even hold true if $g(0)$ is replaced by $1$ (see \Cref{cor:ridgeToSplineasym} and \Cref{def:asymadaptedSplineReg}).}
			\item\label{item:symDistributions} The distributions of the random weights and biases $v_k$ respectively $b_k$ are symmetric w.r.t the sign, i.e.
			\begin{enumerate}[i)]
				\item $\PP[v_k\in E]=\PP[v_k\in -E] \quad\forall E\in\B$ and
				\item $\PP[b_k\in E]=\PP[b_k\in -E] \quad\forall E\in\B$.
			\end{enumerate}
		\end{enumerate}
	\end{assumption}
	\begin{assumption}[loss]\label{as:generalloss}
There exist $p\in[1,\infty)$ and a finite Borel-measure $\nu=\nuc+\nua$, where $\nuc$ is absolutely continuous w.r.t.\ the Lebesgue measure and $\nua$ is supported on a finite (possibly empty) subset $\{x_1,\ldots,x_N\}\subset\R$ such that the loss functional $\Ltr:L^\infty_\text{loc}\to\Rpz$ is
\begin{enumerate}[a)]
\item\label{item:continuousL} continuous w.r.t.\ $\sobnormop[\cdot]$ for some compact interval $K\subset\R$ and
\item\label{item:lipschitzL} 
Lipschitz\footnote{We think uniformly continuous should be sufficient, but would make the proof more cumbersome.} continuous w.r.t.\ $\sobnormop[\cdot]$ on $\{f: \Ltr(f)<\Ltr(0)+\epsilon\}$ for some $\epsilon>0$.
\end{enumerate}
	\end{assumption}
	
	\begin{example}\label{ex:lossfunctionalsum}
	\Cref{as:generalloss} allows for loss functionals of the form \begin{equation}\label{eq:Ltri}
	    \Ltr(f):=\sum_{i=1}^N \ltrib{ f(\xtr_i)}, 
	\end{equation}
	with convex and continuously differentiable loss functions $l_i:\R\to\Rpz$, $i=1,\ldots,N$ as we prove in \Cref{le:ex:lossfunctionalsum}. As a special case, this includes functionals of the kind
	\begin{equation}\label{eq:ltr}
 	    \Ltr(f):=\sum_{i=1}^N{\ltr\left( f(\xtr_i),\ytr_i\right), }
 	\end{equation}
 	used in classical supervised learning (e.g. MSE \meqref{eq:Ltr_squaredloss}, MAE or cross entropy loss, \dots).
 	In the literature, minimizing such a training loss is motivated sometimes as empirical cost minimization (or empirical risk minimization) if the cost function $C\propto\ltr$%
  . Alternatively it can be motivated as maximum %
 	likelihood method if the negative logarithm of the density of $Y|f(X)$ is proportional to $\ltr(f(\xtr_i),\cdot)$ (e.g., Gaussian noise in the case of squared loss).
 	However, \Cref{as:generalloss} also allows for much more general possibly non-convex loss functionals that can depend on the first derivative $f^{'}$ of $f$ and loss functionals that include integrals over $f$ or $f^{'}$ instead of finite sums over finitely many training points (e.g. the loss in \cite{NOMUICML}, if a Lipschitz-continuous version is used as in \cite{weissteiner2023BOCA}).
	\end{example}
	\begin{theorem}[ridge weight penalty corresponds to adapted spline]\label{thm:ridgeToSpline}
		Using the notation from \Cref{def:RSNN,def:kinkPosDens,def:ridgeNet,def:adaptedSplineReg} and let%
		\footnote{Since all $v_k$ are identically distributed and all $\xi_k$ are identically distributed as well, the conditioned expectation \Eco{v_k^2}{\xi_k=x} does not depend on the choice of $k\in\fromto{n}$.}
		$\fax: g(x):=\gxi(x) \Eco{v_k^2}{\xi_k=x}\frac{1}{2}$ and $\lw:=\lambda n 2g(0)$, then, under the \Cref{as:mainAssumptions,as:truncatedg,as:easyReadable,as:generalloss}, the following statement holds for every compact interval $K\subset\R$: for every $\left(\RNR\right)_{n\in\N}$ as in \Cref{def:ridgeNet}
		\begin{equation}\label{eq:ridgeToSplineNotUnique}
	\plim d_{\Wkp[K]{1}{\infty}}\left(\RNR,
	\argmin_{f\in\WT(\R)}\Flpmm(f)\right) =0,
	\end{equation}
	i.e., $\forall\left(\RNR\right)_{n\in\N}:\footnote{According to \Cref{def:ridgeNet}, \RNR does not necessarily have to be unique. By writing \enquote{$\forall\left(\RNR\right)_{n\in\N}$} we emphasize that the result holds for every possibly not unique sequence of optimizers, i.e., every sequence $\left(\RNR\right)_{n\in\N}$ with weights $\wR$ that satisfy $\faog:\left(\wRo\right)_{n\in\N}\in\prod_{n\in\N}\argmin\limits_{w\in\R^n}\Fn(\RNwo)$.}
 \forall\epsilon>0:\forall\rho\in(0,1):\exists n_0\in\N:\forall n>n_0$:
	\begin{equation}
	    \PP\left[\exists\flpm\in\argmin_{f\in\WT(\R)}\Flpmm(f):\sobnorm[{\RNR-\flpm }]<\epsilon\right]>\rho.
	\end{equation}
	\end{theorem}
	\proofInSec{thm:ridgeToSpline}{sec:proof:RidgeToSpline}
	\begin{corollary}[convex case of \Cref{thm:ridgeToSpline}]\label{cor:ridgeToSpline}
        Under the assumptions of \Cref{thm:ridgeToSpline}, if additionally $\Ltr$ is convex\footnote{\Cref{eq:uniqueRidgeToSpline} also holds also true for every $\RNR\in\argmin \Fn$ if $\Ltr$ is not convex as long as $\argmin \Flpmm$ is unique.}, then, $\RNR$ and $\flpm$ are the unique minimizers of $\Fn$ and $\Flpmm$ respectively, and
		the following statement holds for every compact interval $K\subset\R$:
		\begin{equation}\label{eq:uniqueRidgeToSpline}
		\plim
		\sobnorm[ \RNR -\flpm ]=0 .\protect\footnotemark
		\end{equation}\footnotetext{Using the definition of the $\Plim$, equation~\meqref{eq:uniqueRidgeToSpline} reads as: $\fae :\forall \rho \in (0,1): \exists n_0\in\N :\forall n\geq n_0:
			\PP[
			{\sobnorm[ \RNR -\flpm ]}
			< \epsilon ] > \rho.$}
	\end{corollary}
	\unimportant{\begin{proof}
	This result follows directly from \Cref{le:existenceuniquenessRidgeRSN,le:existenceuniquenessaregspl} and \Cref{thm:ridgeToSpline}.
	\end{proof}}
	\begin{remark}[The norm $\sobnorm$]\label{rem:sobnormoi}
	    Throughout this paper, we consider
	    \begin{equation}
	        \sobnorm[f]:=\max\{\sup_{x\in K}|f(x)|,\sup_{x\in K}|f^{'}(x)| \},
	    \end{equation}
	    for every $f\in\mathcal{C}(\R)$ with piece-wise continuous derivative $f^{'}$, where we assume w.l.o.g.~that $f^{'}$ is left continuous (i.e.~$\relu[x]^{'}=\ind_{(0,\infty)}(x)$). However, by $\|\cdot\|_{\Wkp[K]{2}{2}}$ we denote the Sobolev norm using weak (second) derivatives.
	\end{remark}
		Without \Cref{as:easyReadable}, the $P$-functional becomes slightly more complicated as we formulate and discuss in \Cref{sec:AssymetricDistribution}, where we also formulate the analogous statement to \Cref{thm:ridgeToSpline} in \Cref{cor:ridgeToSplineasym}.

		\subsection{\texorpdfstring{\RSN}{RSN} and Gradient Descent \texorpdfstring{$\to$}{to} Implicit Ridge Regularization \texorpdfstring{\color{hellgrau}($d\in\N$)}{}}\label{sec:GradientToRidge}
	
	We now move on to derive the relation between the \RSN~\unimportant{\RNw[\wt]} whose terminal-layer parameters are optimized performing gradient descent initialized at zero \unimportant{$\wt[0]=0$} up to a certain time~$T$ on the one hand, and the \ridgeRSNRN\ with penalization parameter $\lw$ on the other. In particular, we show that in the limit of infinite training time the solution~\unimportant{\RNw[{\wt[\infty]}]} obtained from the GD method corresponds to the one resulting by taking the limit $\lw\to 0$ in the ridge problem (This solution is also referred to as \hyperref[def:minNormSol]{minimum norm solution}~\RNRz.). Note again, that this result is well known thanks to the work of i.a. \cite{bishop1995regularizationEarlyStopping,friedman2003gradient,PoggioGernalizationDeepNN2018arXiv180611379P,GidelImplicitDiscreteRegularizationDeepLinearNN2019arXiv190413262G,pmlr-v89-ali19a}. Within the present section, we collect the most important findings relating these two solutions within our setting.
	
	Moreover, we will argue that, if suitably transformed, the ridge path mapping $\lw$ to the optimal parameter corresponds to the GD path mapping training time to the corresponding parameter. Again, this equivalence has been discussed in the existing literature (e.g. \cite{bishop1995regularizationEarlyStopping,friedman2003gradient,PoggioGernalizationDeepNN2018arXiv180611379P,GidelImplicitDiscreteRegularizationDeepLinearNN2019arXiv190413262G,pmlr-v89-ali19a}). In these works, it is frequently claimed that the GD solution at time $T$ approximately coincides with the ridge solution for $\lw=1/T$.
	We intend to make this relation more precise below (cp. \cref{eq:lambdaWelleT}). Within future work we will further analyze the errors arising from that approximate relation (see also \Cref{sec:conclusion} \Cref{itm:arbitraryT}).

	Throughout this section, we consider the setting of supervised learning with squared loss, i.e., we require \Cref{as:squaredloss} to hold true. We begin by defining the trained \RSN s~\unimportant{\RNw[\wt]} obtained by pursuing the gradient flow w.r.t.\ this choice of training loss starting in the origin~\unimportant{\wt[0]=0} in parameter space up to time $T$.
	
	\begin{assumption}[least squares loss]\label{as:squaredloss}
	Let $N\in\N$ be a finite number of arbitrary training data $\left( \xtr_i, \ytr_i\right)\in \R^d\times\R\, \allIndi{N}{i}$. The loss functional is given by 
	 \begin{equation}
	    \Ltrs(f):=\sum_{i=1}^N{\left( f(\xtr_i),\ytr_i\right)^2 }.
	\end{equation}
	\end{assumption}

	\begin{definition}[time-$T$ solution]\label{def:GDsolution}
		Let $\allIndi{N}{i}: (\xtr_i, \ytr_i)\in \R^{d+1}$ for some $N,d\in\mathbb{N}$ and $\RNw$ be a \RSNlongAndShort\ with $n\in\mathbb{N}$ hidden nodes. For any $\omega \in \Omega$ and $T>0$, the time-$T$ solution to the problem
		\begin{equation}\label{eq:GDproblem}
		\min_{w\in\R^n} \underbrace{\sum_{i=1}^N\left(\RNwo(\xtr_i)-\ytr_i\right)^2}_{\Ltrsb{\RNwo}}
		\end{equation}
		is defined as $\RNwo[\wto]$, with weights $\wto\in\R^n$ obtained by taking the gradient flow
		\begin{align*}\begin{split}
		d\wt[t]&=-\nabla_w \Ltrsb{\RNw[{\wt[t]}]}\, dt,\label{eq:GDflow}\\
		\wt[0]&=0,
		\end{split}\tag{GF}\end{align*}
		corresponding to \meqref{eq:GDproblem} up to time $T$.
	\end{definition}
	
	\begin{remark}\label{rem:discreteGD}
		In practice, the weights $\wt$ of the time-$T$ solution as introduced in \Cref{def:GDsolution} are approximated by taking 
		$\tau:=T/\gamma$ steps of size $\gamma>0$ according to the Euler discretization 
		\begin{align*}
		    \begin{split}
		\wth[t+\gamma]&=\wth[t]-\gamma\nabla_w \Ltrsb{\RNw[{\wth[t]}]},\label{eq:GDescent}\\
		\wth[0]&=0,
		\end{split}\tag{GD}
		\end{align*}
		corresponding to \meqref{eq:GDflow}. 
	\end{remark}
	
	Within our setting, which in essence corresponds to a kernelized linear regression with random features, the time-$T$ solution takes an explicit form, as shown in \Cref{le:GDsolution}.
	\begin{lemma}\label{le:GDsolution}
		Let $\allIndi{N}{i}: (\xtr_i, \ytr_i)\in \R^{d+1}$ for some $N,d\in\mathbb{N}$ and for any $\omega\in\Omega$, let $\RNwo$ be a randomized shallow network with $n\ge N$ hidden nodes. Define further $X\omb\in\R^{N\times n}$ via 
		\begin{displaymath}
		X_{i,k}\omb:= \sigma\left(b_k\omb+\sum_{j=1}^d v_{k,j}\omb\xtr_{i,j}\right)\quad \allIndi{N}{i}\ \allIndi{n}{k}\ \faog,
		\end{displaymath}
		where $\xtr_{i,j}$ denotes the j\textsuperscript{th} component of $\xtr_i$. For any $T\ge0$, the weights $\wto$ corresponding to the time-$T$ solution $\RNwo[\wto]$ satisfy
		\begin{equation}\label{eq:GDsolution}
		\wto=-\exp\left({-2T\Xt\omb X\omb}\right)\wdago+\wdago,
		\end{equation}
		with weights $\wdago$ corresponding to the minimum norm network (see \Cref{def:minNormSol}).
	\end{lemma}
	\proofInSec{le:GDsolution}{sec:Proof:GDridge}
	
	With the above, the asymptotic behavior of $\wto$ is easily analyzed. As \Cref{rem:GDlimit} shows, the time-$T$ parameters $\wto$ converge to the minimum norm parameters $\wdago$ (see \Cref{def:minNormSol}).
	Consequently, the time-$T$ solution converges to the ridge penalized network when choosing the penalization accordingly, as is stated in \Cref{thm:GDridge}. 
	\begin{remark}[limiting solution of gradient descent]\label{rem:GDlimit}
		By \Cref{le:GDsolution}, the weights~$\wt[T]$ corresponding to the time-$T$ solution converge to the minimum norm solution~\wdag\ as time tends to infinity---i.e.\ taking the limit $T\to\infty$ in \meqref{eq:GDsolution}, we have $\tlim \wto=\wdago$ $\faog$.
	\end{remark}
	\proofInSec{rem:GDlimit}{sec:Proof:GDridge}
	
	\begin{theorem}\label{thm:GDridge}
		Let $\RNw[{\wt}]$ be the time-$T$ solution and consider for $\lw=\frac{1}{T}$ the corresponding ridge solution $\RNR[\frac{1}{T}]$ (cp. \Cref{def:GDsolution} and \Cref{def:ridgeNet} with $\Ltrs$ as in \Cref{as:squaredloss}). We then have that
		\begin{equation}\label{eq:thm:GDridge}
		\fao:\quad \tlim\sobnorm[{\RNRo[\frac{1}{T}]-\RNwo[\wt\omb]}]=0.
		\end{equation}
	\end{theorem}

	\proofInSec{thm:GDridge}{sec:Proof:GDridge}
		\begin{remark}[Relaxed requirements on loss function]\label{rem:generalizeLoss}
	    Without \Cref{as:squaredloss}, \Cref{thm:GDridge} can still be proven, if instead one requires that $\ltr(\cdot, \ytr_i):\R\to\Rpz$ has a unique minimum for every $i=1\ldots,N$, where $\Ltr$ and $\ltr$ are given as in \Cref{ex:lossfunctionalsum}. This will be proven in future work. In the case of strongly convex sufficiently smooth loss functions, this convergence has already been proven on parameter space in \cite[Theorem~1]{NEURIPS2018_6459257d} and can be lifted to function space analogously as in our proof of \Cref{thm:GDridge}.
	\end{remark}
	\subsubsection{Early Stopping}\label{sec:EarlyStopping}
	Moreover, we may use the representation \meqref{eq:GDsolution} to derive an approximate relation between the weights~$\wt[T]$ corresponding to the time-$T$ solution and those obtained by performing a ridge regression with penalization parameter $\lw$. The idea is to first analyze which singular value is trained most at a given time $T$ in an infinitesimal step along the solution path of $\wt$. In other words, we seek to find $s\ge 0$ that maximizes the gradient w.r.t.\ time of the singular values corresponding to the matrix exponential characterizing the time-$T$ solution, i.e. we solve
	\begin{displaymath}
	\underset{s\ge 0}{\argmax}\nabla_T\exp(-2Ts)=\underset{s\ge 0}{\argmax}-2s\exp(-2Ts).
	\end{displaymath}
	The unique solution is given by
	\begin{equation}\label{eq:sMostImportantSV}
	    s^*=\frac{1}{2T}.
	\end{equation} %
	In a second step, we compare the closed-form solution of the parameters resulting from a $\lw$-ridge regression to the time-$T$ solution, which we now consider to be characterized by $s^*(T)$. To that end, we remark that using the singular value decomposition of the data matrix $X\in\R^{N\times n}$, i.e. $X=U\Sigma V^\top$ with $$\Sigma=\left(\begin{matrix}\mathrm{diag}(\sqrt{s_1},\ldots,\sqrt{s_r}) & 0 \\
	0 & 0\end{matrix}\right)\in\R^{N\times n},$$
	these solutions may be written as
	\begin{equation}\label{eq:timeTSVD}
	\wt=-V\left(\begin{matrix}\mathrm{diag}\left(\frac{\exp(-2Ts_1)-1}{\sqrt{s_1}},\ldots,\frac{exp(-2Ts_r)-1}{\sqrt{s_r}}\right) & 0 \\
	0 & 0\end{matrix}\right)U^\top y,
	\end{equation}
	\begin{equation}\label{eq:rdigeSVD}
	w_{\lw}=V\left(\begin{matrix}\mathrm{diag}\left(\frac{\sqrt{s_1}}{s_1+\lw},\ldots,
	\frac{\sqrt{s_r}}{s_r+\lw}\right) & 0 \\
	0 & 0\end{matrix}\right)U^\top y.
	\end{equation}
	We then arrive at the ridge estimate approximating the time-$T$ solution by comparing \cref{eq:timeTSVD,eq:rdigeSVD} for the singular value $s^*$, i.e., the one that is most affected by the training at time-point $T$. Hence, by solving $\frac{\exp(-2Ts^*)-1}{\sqrt{s^*}}=\frac{\sqrt{s^*}}{s^*+\lw}$, we relate the time-$T$ solution to the ridge solution obtained using the penalization parameter
	\begin{equation}\label{eq:lambdaWelleT}
	    \lw(T) =\unimportant{\frac{s^* e^{-2s^*T}}{1-e^{-2s^*T}}} =\frac{1}{2T(e-1)}.
	\end{equation}
	Note that, by the above relation $\lw(T)$ still is of order $1/T$ and hence the asymptotic behavior that we characterize in \Cref{thm:GDridge} below, is sufficiently captured taking the relation $\lw(T)=1/T$ (i.e., \Cref{thm:GDridge} holds true for $\lw(T)=1/T$ as well as for $\lw(T)=\frac{1}{2T(e-1)}$).
	However, for comparing the early-stopped time-$T$ solution~$\RNw[\wt]$ to a \ridgeRSNRN\ and, as a consequence, to a certain \regSpl~$\fl\approx\flpm$, we make use of the precise relation \meqref{eq:lambdaWelleT}. See also \Cref{sec:conclusion} (and \Cref{sec:IntuitionAdaptedSpline}) for empirical results, that underline the quality of the fit using \meqref{eq:lambdaWelleT}.

\section{Conclusion and Extensions}\label{sec:conclusion}
	Combining the main \Cref{thm:ridgeToSpline,thm:GDridge} finally yields our main result: for a large number of training epochs~$\tau=T/\gamma$, the \wlargeRRSNlongAndShort\ (obtained from \meqref{eq:GDescent} under \Cref{as:mainAssumptions,as:easyReadable,as:truncatedg,as:squaredloss})
	is very close to the \hyperref[def:splineInterpolation]{spline interpolation}~\fz. Formally,
 \begin{equation}\label{eq:conclusion}
	{\mediumimportantStart\RN_{\wth[{T,\wth[0]}]}
	\overset{\wth[0]\to 0}{\approx}}\RN_{\wth[T]}
	\overset{\gamma\to0}{\approx}\RN_{\wt[T]}
	\underset{\scriptscriptstyle\text{\tiny\Cref{thm:GDridge}}}{\overset{T\to\infty}{\approx}}\RNRz \!\!
	\underset{\scriptscriptstyle\text{\tiny\Cref{thm:ridgeToSpline}}}{\overset{\underset{n\to\infty}{\PP}}{\approx}}\fzpm
	\underset{\scriptscriptstyle\text{\tiny\Cref{prop:gTo1}}}{\overset{\frac{g}{g(0)}\to\ 1}{\approx}}\fz
	\end{equation}
 Here, the notation $\overset{\to}{\approx}$ corresponds to a mathematically proven exact limit in the very strong\footnote{Convergence in \sobnorm[\cdot] implies uniform convergence on $K$ or convergence in $\Wkp[K]{1}{p}$. Even stronger Sobolev convergenve, such as convergence w.r.t. $\WkpShort{2}{p}$, cannot be shown since $\RN_w\notin \Wkp[K]{2}{p}$.} Sobolev norm~\sobnorm[\cdot] {\color{hellgrau}(in probability in the case of $\scriptstyle \overset{\underset{n\to\infty}{\PP}}{\approx}$)}.\newline
	In applications, however, both the number of hidden nodes and training steps are finite. Hence, it is particularly interesting to note that in typical settings for \emph{arbitrary} training time $T\in\Rp$ (including early stopping, i.e. $T\ll\infty$) the same relation approximately holds true. In other words, by taking $T\overset{\text{\meqref{eq:lambdaWelleT}}}{=}\frac{1}{2\lw (e-1)}$ and $\lw\overset{\text{Th.\! \ref{thm:ridgeToSpline}}}{=}\lambda n 2g(0)$, we have

\begin{equation}\label{eq:conclusionApprox}
	{\mediumimportantStart
	\RN_{\wth[{T,\wth[0]}]}\overset{\wth[0]\approx 0}{\underset{\ref{itm:w0Klein}}\approx}\RN_{\wth[T]}\overset{\gamma\approx 0}{\underset{\ref{itm:gammaKlein}}\approx}}
	\RN_{\wt[T]}\overset{}{\underset{\ref{itm:arbitraryT}}\approx}\RNR[\lw]%
	\underset{\text{\ref{itm:nLarge}}}{\overset{\underset{n \text{ large}}{\PP}}{\approx}}\fLpm[\lambda]\overset{\substack{\text{standard distrib.}\\\text{for $v$ and $b$}\\ \text{and } K\subseteq\left[ -1,1\right] }}{\underset{\ref{itm:standardg}}\approx}\fl[\lambda],
	\end{equation}
	where \enquote{$\approx$} represents equality up to a (small) approximation error (that can be strictly larger than zero).

	Regarding the approximation~\meqref{eq:conclusionApprox}, we remark the following.
	\begin{enumerate}[1.{}]
		{\unimportantStart
			\item\label{itm:w0Klein} The first approximation should be quite simple but is not focused on within this work.\footnote{\Cref{le:smallw*} demonstrates, that with increasing $n$ the initial weights~$\wth[0]$ should be chosen closer to zero.} (As only the last layer of \RN\ is trained, one could just start with $w^0=0$)
			\item\label{itm:gammaKlein} It is of importance to choose the learning rate $\gamma$ rather small.\footnote{\label{footnote:gammaSmall}For finite values of $T$ a standard result on \href{https://en.wikipedia.org/w/index.php?title=Euler_method&oldid=907454399}{\oldnormalcolor Euler discretization} can be used. In the limit $T\to\infty$ one can formulate a direct argument that combines \cref{itm:gammaKlein,itm:arbitraryT}: $\tlim \wth[T]=\wdag$, if the learning rate~$\gamma<1/r(X^\top X)$ is smaller than 1 over the spectral radius (largest eigenvalue) of $X^\top X$ \cite[p.\ 4]{bishop1995regularizationEarlyStopping} \cite[p.\ 11]{GidelImplicitDiscreteRegularizationDeepLinearNN2019arXiv190413262G}.
			}
			When one decreases $\gamma$, one has to increase the number of steps $\tau$ to obtain the same $T$. Stochastic gradient descent is a good technique to get many small steps that are computationally cheap (cp.~\cref{footnote:stochasticGradient}). 
			Note, that by the \cref{footnote:gammaSmall} we have that for an \RSN~\RN\, the learning rate~$\gamma$ should typically be chosen approximately inversely proportional to the number of neurons~$n$ (if one keeps the scale of randomness of $v$ and $b$ fixed).
			}
		\item\label{itm:arbitraryT} Multiple papers assume that the third approximation is quite precise for arbitrary values of $T\in\Rp$ without rigorous proof \cite{bishop1995regularizationEarlyStopping,friedman2003gradient,PoggioGernalizationDeepNN2018arXiv180611379P}.%
        \footnote{We note, that while the literature chooses $\lw=\frac{1}{T}$, it might be more reasonable to chose $\lw=\frac{s e^{-2sT}}{1-e^{-2sT}}$, with an appropriate choice of $s$ (cp. \cref{eq:sMostImportantSV,eq:lambdaWelleT}) to get better approximation bounds. Nonetheless, in \Cref{sec:theorem,eq:conclusion} we work with the relation $\lw=\frac{1}{T}$, %
        since
        in the limit $T\to\infty$ these relations coincide.}
		\item\label{itm:nLarge} The mathematically precise asymptotic relation is the subject of \Cref{thm:ridgeToSpline}. 
		\item\label{itm:standardg} The \aregSplf\ is a macroscopically defined object that already is nice to interpret.  Intuitively, it is plausible that \flpm\ is very close to the classical \fl\ on the $[-1,1]$-cube, if one uses typical\footnote{For instance, one could choose $b_k,v_k\sim \text{Unif}(-\sScale,\sScale)$ \iid\ uniformly symmetrically distributed {\unimportantStart
		or $b_k,v_k\sim \mathcal{N}(0,\sScale)$ \iid\ normally distributed with zero mean} (see \Cref{sec:WeightingFunctionG:IntuitionAdaptedSpline}).} distributions for the first-layer weights and biases $v$ and $b$, and if the training data is scaled and shifted to fit into the $[-1,1]$-cube. Additionally, by that same intuition, it follows that if popular rules of thumb such as scaling and shifting the data to the $[-1,1]$-cube are broken, one can obtain rather poor approximations \flpm. Consequently, by providing these insights on the circumstances that would lead to undesirable results, \Cref{thm:ridgeToSpline} contributes to answering question~\ref{item:ProsAndCons} about best practices in machine learning. See \Cref{sec:IntuitionAdaptedSpline} for more details when $\flpm\approx\fl$ approximately holds true and when not.
	\end{enumerate}
	
	\subsection{Empirical results}\label{sec:EmpiricalResults}
	As a proof of concept we empirically verify the approximate relations~\meqref{eq:conclusionApprox} discussed above. The implementation of these experiments can be found at \url{https://github.com/JakobHeiss/NN_regularization1}. To that end, we consider the aim of approximating the function $f:\R\to\R, x\mapsto \sin{(\pi x)}$, given $N=16$ noisy data points $(x_i, f(x_i)+\epsilon_i)\in\R^2$, where $x_i$, $i=1,\ldots,N$ are equidistant points in the interval $[-1,1]$ and $\epsilon_i$ are realizations of a centred Gaussian random variable with standard deviation $\tt{scale}=1/8$. \Cref{fig:sine} shows a comparison of the solution functions obtained by
	\begin{enumerate}[a)]
		\item training an RSN with ReLU activation using a standard implementation of gradient descent with step size $\gamma=2^{-11}$ %
		for $\tau=2^{15}$ epochs \unimportant{(resulting in \RNw[{\wth[T]}]\ with $T=\tau\gamma$)},
		\item training that same RSN using a ridge penalty on the terminal weights with penalization parameter $\lw=\frac{1}{e -1}\frac{1}{2\tau\gamma}$ according to \cref{eq:lambdaWelleT} \unimportant{(resulting in \RNR)} and
		\item the \nameref{def:splineReg} with penalization parameter $\lambda=\frac{\lw}{n2g(0)}$ \unimportant{(resulting in \fl)}. (Here, $n$ represents the RSN's number of hidden nodes and the weighting function $g$ is defined in \Cref{thm:ridgeToSpline}.)
	\end{enumerate}
	The RSN was chosen to consist of $n=2^{12}$ hidden nodes with first-layer weights and biases sampled from a Uniform distribution on $[-0.05, 0.05]$. Moreover, a last-layer bias was included in the training (cp. \Cref{footnote:lastlayerbias}).

	Within this paper's setting, this experiment corresponds to comparing the time-$T$ solution~\RNw[{\wth[T]}] for $T=16$ to the \ridgeRSNRN\ with $\lw=\frac{1}{e -1}\frac{1}{2T}\approx 0.018$ %
	and the smooth \regSplf\ with penalization parameter $\lambda\approx 0.014$.\newline
	As \Cref{fig:sine} nicely shows, the three solution functions almost coincide on $[-1,1]$. This is of particular interest, since the training data typically is scaled to fit the interval $[-1,1]$.
	\begin{figure}[!h]
				\centering
				\scaledinset{l}{.2}{b}{.663}{\resizebox{0.155\hsize}{!}{\tiny \begin{tabular}{@{}l@{}}$\left( \xtr_i,\ytr_i\right)$\\[1.45ex]
				$\RNw[{\wth[T]}]$ \text{(implicit)}\\
				$\RNR$ \text{(Ridge)}\\
				\fl \text{(spline)} \end{tabular}}}{\includegraphics[width=0.7\linewidth]{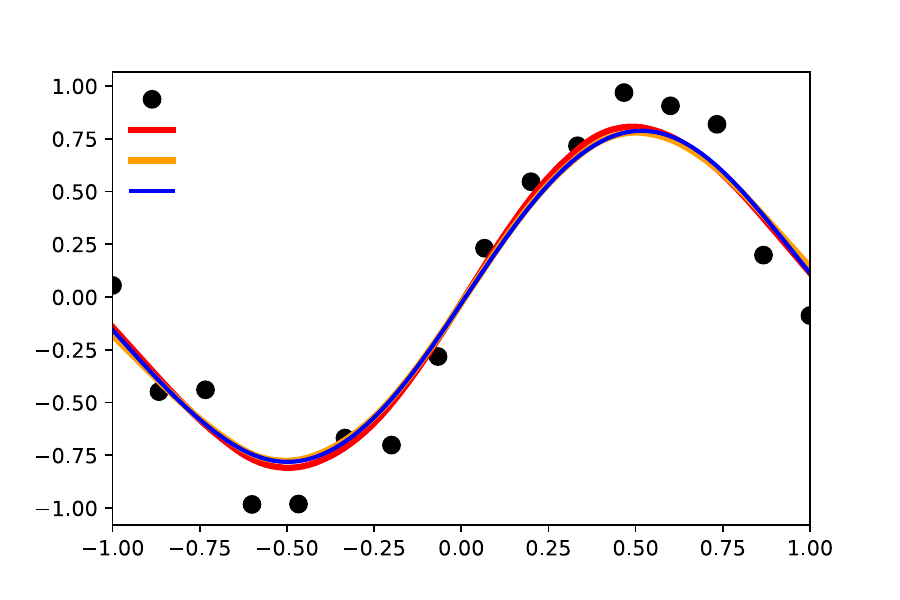}}
				\caption{Comparison of the solution functions obtained from performing gradient descent (red line) and ridge regularization (yellow line) to train an \RSN\ to the spline regression (blue line) with parameters chosen as suggested by \cref{eq:lambdaWelleT} and \Cref{thm:ridgeToSpline}.}
				\label{fig:sine}
			\end{figure}
			
	\unimportant{
	In certain situations %
	the \aregSpl~$\flpm\approx\RNR$ can deviate more from the classical \regSplf\
	as can be seen in \Cref{fig:sinezoomedout} far outside the training data. The \RSN's architecture can be extended to incorporate a direct affine-linear link onto the output, which, when included in the training process, can make up for the observed difference (see also \cref{item:direct_link} below and \Cref{sec:IntuitionAdaptedSpline}). %
	}
		
		\begin{figure}[!ht]
			\centering
			\scaledinset{l}{.2}{b}{.158}{\resizebox{0.155\hsize}{!}{\tiny \begin{tabular}{@{}l@{}}$\left( \xtr_i,\ytr_i\right)$\\[1.45ex]
				$\RNw[{\wth[T]}]$ \text{(implicit)}\\
				$\RNR$ \text{(Ridge)}\\
				\fl \text{(spline)} \end{tabular}}}{\includegraphics[width=0.7\linewidth]{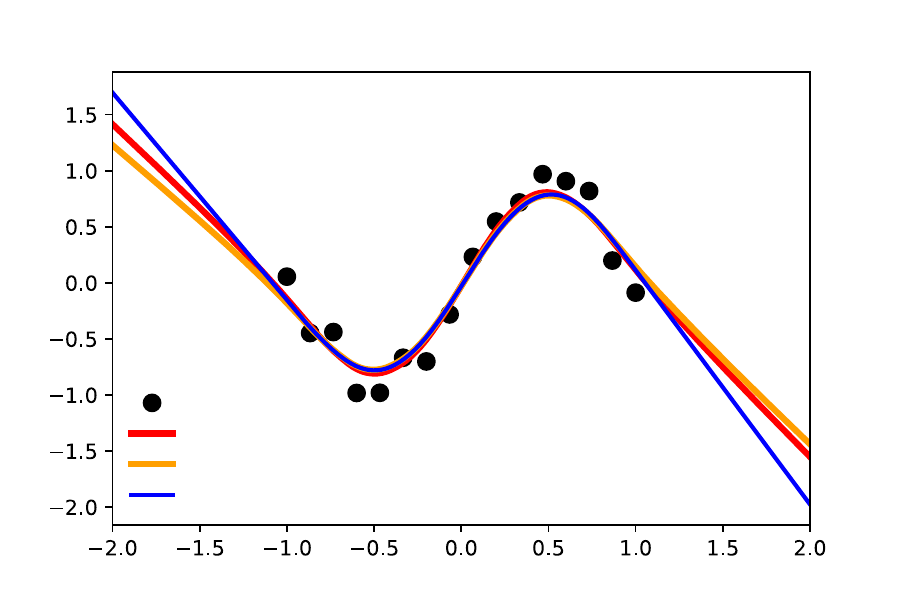}}
			\caption{Large scale comparison of the solution functions as in \Cref{fig:sine}. Outside the training data, the trained \RSN~$\RNw[\wth]$ ranges in between the \ridgeRSNRN\ and the \regSplf.}
			\label{fig:sinezoomedout}
		\end{figure}
		
	A more detailed view on the trained \RSN~$\RNw[\wt]$ is given in \Cref{fig:secondderiv}. Therein, we visualize the \RSN's (distributional) second derivative at the respective realized kink positions as well as a convoluted version of it using a Gaussian kernel. %
 \newline  
		\begin{figure}[h!]
			\centering
			\includegraphics[width=0.7\linewidth]{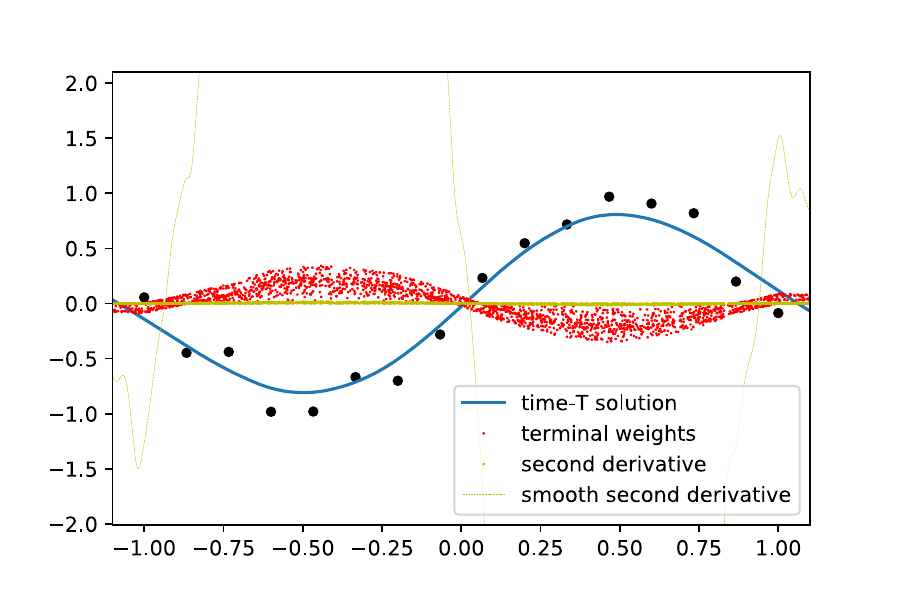}
			\caption{The trained \RSN~$\RN_{\wth[T]}$ (blue line) and its (distributional) second derivative~$\sum_{k=1}^n v_k w_k \delta_{\xi_k}$ (yellow dots) at the respective realized kink positions and a smoothed version of it. The smooth second derivative was obtained from a convolution using a Gaussian kernel. It nicely captures the trained \RSN's curvature. Moreover, the values of the terminal layer's weights~$w_k$ at the respective kink positions~$\xi_k$ are given (red dots).}
			\label{fig:secondderiv}
		\end{figure}

	\subsection{Extensions to this work}\label{sec:FutureWork}
	Besides discussing the correspondence of the spline interpolation and an RSN trained using gradient descent for a finite number of nodes and finite training time, we extend the theory in upcoming and recent work as follows:
	\begin{enumerate}[i., ref=\roman*]
		\item We generalize to multidimensional in- and output in $\X\subset\R^{d_{in}}, \Y\subset\R^{d_{out}}$ (\cite{implReg2}).

		\item\label{item:direct_link} In this paper, we gain insights on how adding skip connections can lead to models which are even more similar to the classical \regSpl{} (see \Cref{sec:IntuitionAdaptedSpline}). As a side effect, one can thereby save computational time and memory by removing neurons with kink positions $\xi_k$ outside of $C:=\ConvexHull\left(\supp(\nu)\right)$. If the skip connections are regularized accordingly, this is possible even without changing the learned function on $C$ (see \Cref{subsec:SimilaritWithoutSkipConnections}).

		{\unimportantStart
		\Cref{thm:ridgeToSpline} and its \hyperlink{proof:thm:ridgeToSpline}{proof} inspire to choose special types of randomness for the first-layer weights and biases. Naturally, we are interested to find out whether these choices provide advantages in the training of such \RN\ or other architectures.}%
		
		\item\label{item:fullytrained} We prove convergence to a differently regularized function (which is optimal with respect to another \Pfunc-functional) in the case of ordinary training of both layers of \NN\ instead of only training the last layer \citep{Part3Arxiv}.%
		
		\item We generalize to deep neural networks with more hidden layers \citep{Part3Arxiv}. %
		The long-term goal of this line of research is to find a \Pfunc-functional {\mediumimportantStart(or another easy to interpret macroscopic description)} for each type of neural network for each set of hyperparameters. (This could be extended to other Machine Learning methods like random forests too.)%
		
	\end{enumerate}

\bibliographystyle{plainnat}
 \newpage
	\appendix
	\section{Proofs}\label{sec:proofs}
	In the following, we rigorously prove the results presented within this paper.
	
	\subsection{Proof of \Cref{thm:ridgeToSpline} \texorpdfstring{($\RNR\to\flpm$)}{}}\label{sec:proof:RidgeToSpline}
	
	A number of lemmata are required for the \hyperlink{proof:thm:ridgeToSpline}{proof} of \Cref{thm:ridgeToSpline}. These will be presented and proved later in this section.
	We start by defining the objects that are central to the subsequent derivations. 
	
	\begin{center}
		\boxed{\text{
				Throughout this section, we henceforth require \Cref{as:mainAssumptions,as:truncatedg,as:easyReadable,as:generalloss} to be in place.}}
	\end{center}
	
	\begin{definition}\label{def:adaptedSplineRegTuple}
		Let 
		$\lambda \in \Rp$.
		Then for a given function $g:\R\to\Rpz$ a tuple $(\flp, \flm)$ is defined%
		\footnote{\label{footnote:uniqueflpmt}See \Cref{le:existenceuniquenessaregspl} for a proof of existence and uniqueness of the solutions of \meqref{eq:adaptedSplineRegTuple}. If the solution is not unique, we denote any element of the $\argmin$ as $(\flp, \flm)$. In that case, it will be clear from the context whether we want to say that a certain statement holds true for any $(\flp, \flm)$ or if there exists a $(\flp, \flm)$ in the $\argmin$ for which the statement holds.}
		as
		\begin{equation}\label{eq:adaptedSplineRegTuple}
		\flpmt \in \argmin_{(f_+,f_-)\in \T}\underbrace{ \Ltrb{f_+ + f_-}+\lambda \Pgpm(f_+,f_-)  }_{=:\Flpmb{f_+,f_-}},    
		\end{equation}
		with
		\hypertarget{eq:Pgpm}{\begin{equation*}
		\Pgpm(f_+,f_-):=  2g(0) \left(
		\int_{\supp (g)} \frac{\left( {f_+}^{''}(x) \right)^2}{g(x)} dx
		+\int_{\supp (g)} \frac{\left( {f_-}^{''}(x) \right)^2}{g(x)} dx
		\right).
		\end{equation*}}

	\end{definition}
	\begin{remark}\label{rem:adRegSplineSum}

	    Note that every solution to \meqref{eq:adaptedSplineRegTuple} $(\flp, \flm)$ yields a solution~$\flpm:=\flp+\flm$ to \meqref{eq:adaptedSplineReg}.
	    Conversely, for any solution~$\flpm$ to \meqref{eq:adaptedSplineReg} there exists a solution~$(\flp, \flm)$ to \meqref{eq:adaptedSplineRegTuple} such that $\flpm=\flp+\flm$.
	    Thus, $(\flp, \flm)$ is a solution to \meqref{eq:adaptedSplineRegTuple} if and only if  $\flpm:=\flp+\flm$ is a solution to \meqref{eq:adaptedSplineReg}
	    
	    \begin{proof} Assume there exists a tuple $(\flp,\flm)$ as in \meqref{eq:adaptedSplineRegTuple}. Then
	    \begin{equation}\label{eq:tupleineq1}
	    L(\flp+\flm)+\lambda\Pgpm(\flp,\flm)\le L(f_++f_-)+\lambda \Pgpm(f_+,f_-)    
	    \end{equation}
	    for all $(f_+,f_-)\in\T$.

	    Thus, since $\Pgpmm(f)=\underset{\underset{f=f_++f_-}{(f_+, f_-)\in\T}}\inf \Pgpm(f_+,f_-)$, we obtain from \meqref{eq:tupleineq1} that
	    \begin{subequations}
	    \begin{align}
	         L(\flp+\flm)+\lambda \Pgpmm(\flp+\flm)&\leq L(\flp+\flm)+\lambda\Pgpm(\flp,\flm)\\
	         &\le \inf_{f\in\WT(\R)} L(f)+\lambda \Pgpmm(f),    
	    \end{align}
	    \end{subequations}
	    i.e., $\flpm:=\flp+\flm$ is a minimizer of \meqref{eq:adaptedSplineReg}.
	    
	    For the reverse direction, assume there exists $\flpm$ solution to \meqref{eq:adaptedSplineReg}. 
	    Then by \Cref{le:uniqueness}, since $\left\{(f_+, f_-)\in\T\, \middle|\, f_+ + f_-=\flpm\right\}$ is convex and complete, there exists a unique minimizer $(f_+^*, f_-^*)$ of $\underset{\underset{\flpm=f_++f_-}{(f_+, f_-)\in\T}}\inf \left( 0+1\Pgpm(f_+,f_-)\right)=\Pgpmm(\flpm)$. The tuple $(f_+^*, f_-^*)$ is a solution to \meqref{eq:adaptedSplineRegTuple} as well (If there was a better tuple for \meqref{eq:adaptedSplineRegTuple} this would correspond to better function for \meqref{eq:adaptedSplineReg} contradicting the optimality of \flpm.). 
	    \end{proof}
	\end{remark}
	
	\begin{definition}[estimated kink distance~\hq\ w.r.t. $\sgnb{v}$]\label{def:estKinkDist}
		Let $\RN$ be a randomized shallow neural network with $n$ hidden nodes as introduced in \Cref{def:RSNN}. The \textit{estimated kink distance w.r.t. $\sgnb{v}$} at the $k$\textsuperscript{th} kink position $\xi_k$ corresponding to $\RN$ is defined as%
		\footnote{\label{footnote:analogous:def:estKinkDist}Without \Cref{as:easyReadable}\ref{item:symDistributions} one would define:
			\begin{subequations}
				\begin{align}
				\hqp_k&:=\frac{1}{n\PP[v_k>0]\, \gxip(\xi_k)}\label{subeq:hqp}\\
				\hqm_k&:=\frac{1}{n\PP[v_k<0]\, \gxim(\xi_k)}\label{subeq:hqm}.
				\end{align}
				Under \Cref{as:easyReadable}\ref{item:symDistributions} we have the equality:
				\begin{equation}
				\hq_k = \hqp_k=\hqm_k.
				\end{equation}
			\end{subequations}
		}
		\begin{equation}
		\hq_k:=\frac{2}{n\, \gxi(\xi_k)}.
		\end{equation}
		
	\end{definition}
	
	\begin{definition}[spline approximating RSN]\label{def:splineApproximatingRSN}
		Let $\RN$ be a real-valued randomized shallow neural network with $n$ hidden nodes (cp. \Cref{def:RSNN}) and 
		$\flpmt\in\T$
		be an adapted regression spline as introduced in \Cref{def:adaptedSplineReg,def:adaptedSplineRegTuple}. The corresponding \emph{spline approximating RSN}~\sRNw\ w.r.t. $\flpmt$ is given by
		\begin{equation}
		\sRNwo(x)=\sum_{k=1}^{n}\ww_k\omb\,\sigma\left({b_k\omb}+{v_{k}\omb}x\right) \quad\faog\ \faxg
		\end{equation}
		with weights~$\wwo$ defined as%
		\footnote{Since all $v_k$ are identically distributed and all $\xi_k$ are identically distributed as well, the conditioned expectation \Eco{v_k^2}{\xi_k=x} does not depend on the choice of $k\in\fromto{n}$. Therefore, we will sometimes use the following notation~$\Eco{v^2}{\xi=x}:=\Eco{v_k^2}{\xi_k=x}$.}\textsuperscript{, \!}\unskip
		\footnote{\label{footnote:analogous:def:splineApproximatingRSN}Note that under \Cref{as:mainAssumptions}\ref{item:densityExists}, the set $\{v_k=0\}$ is of zero measure for any $k\in\fromto{n}$ and hence is not included in the definition of the weights~$\wwo$. Without \Cref{as:easyReadable}\ref{item:symDistributions} (and with a weakened form of \Cref{as:mainAssumptions}\ref{item:densityExists}), \ww\ would need to be reformulated: $\allIndi{n}{k}: \faog:$
			\begin{equation*}\label{eq:wwAsym}
			\ww_k\omb:= w^{\flpmtasym, n}_k\omb:=
			\begin{cases}
			\frac{\hqp_k\omb v_k\omb}{\E[v^2|\xi=\xi_k\omb,v>0]}
			{\flpasym}^{''}(\xi_k\omb), & v_k\omb>0\\
			\frac{-\hqm_k\omb v_k\omb}{\E[v^2|\xi=\xi_k\omb,v<0]}{\flmasym}^{''}(\xi_k\omb) , & v_k\omb<0\\
			\frac{\relu[b_k\omb]}{n\PP[v=0]\E[{\relu[b]^2}]}\gammaAsym, & v_k\omb=0.
 			\end{cases}
 \qquad  \begin{matrix}\allIndi{n}{k}\\ \faog .\end{matrix}
			\end{equation*}
		}
		\begin{equation*}\label{eq:ww}
		\ww_k\omb:= w^{\flpmt, n}_k\omb:=
		\begin{cases}
		\frac{\hq_k\omb v_k\omb}{\E[v^2|\xi=\xi_k\omb]}
		{\flp}^{''}(\xi_k\omb), & v_k\omb>0\\
		\frac{-\hq_k\omb v_k\omb}{\E[v^2|\xi=\xi_k\omb]}{\flm}^{''}(\xi_k\omb) , & v_k\omb<0\\
		\end{cases} \qquad  \begin{matrix}\allIndi{n}{k}\\ \faog .\end{matrix}
		\end{equation*}
		We further define \faog:
		\begin{subequations}\label{eq:kpm}
			\begin{align}
			\kp\omb&:=\Set{k\in\fromto{n}|v_k\omb>0},\label{subeq:kp}\\
			\km\omb&:=\Set{k\in\fromto{n}|v_k\omb<0}\label{subeq:km}
			\end{align}
		\end{subequations}

		and \label{par:wwp}$\wwp:=\left( \ww_k\right)_{k\in\kp}$ respectively $\wwm:=\left( \ww_k\right)_{k\in\km}$.
		With the above, spline approximating RSNs can be alternatively represented as
		\begin{equation}\label{eq:splineApprRSNsplit}
		\sRNwo(x)=\underbrace{\sum_{k\in\kp\omb}\ww_k\omb\,\sigma\left({b_k\omb}+{v_{k}\omb}x\right)}_{=:\sRNwpo}
		+\underbrace{\sum_{k\in\km\omb}\ww_k\omb\,\sigma\left({b_k\omb}+{v_{k}\omb}x\right)}_{=:\sRNwmo}.
		\end{equation}
	\end{definition}
	
	\begin{remark}
		A spline approximating RSN introduced in \Cref{def:splineApproximatingRSN} is a particular randomized shallow neural network designed to be \enquote{close} to the corresponding adapted regression spline $\flpm$ in the sense that its curvature in between kinks is approximately captured by the size of corresponding weights $\ww$.
	\end{remark}
	\begin{definition}\label{def:partition}
	  Let $\nu=\nuc+\nua$ as in \Cref{as:generalloss}. The support $\{x_1,\ldots,x_N\}$ of $\nua$ (w.l.o.g. we assume it is sorted in ascending order) induces a partition
	  \begin{align*}
	   \mathcal{P}:=\left\{(-\infty, x_1),[x_1, x_2),\dots,[x_N,+\infty)\right\}   
	  \end{align*}
	  of the input space $\R$. 
	  For each input point $x\in\R$ we define its lower neighbour in $\supp(\nua)\cup \{-\infty\}$ as
	  \begin{align*}
	      \lip[x]&:=\max\Set{x_i\in \supp(\nua)\cup \{-\infty\}| x_i\le x},\\
	      \uip[x]&:=\min\Set{x_i\in \supp(\nua)\cup \{+\infty\}| x_i> x}.
	  \end{align*}
	  Furthermore, we define the operator that assigns to each input point $x\in\R$ the interval of the partition in which it is contained as follows
	  \begin{equation*}
	      \Ip[x]:=\begin{cases}
	      \left[\lip[x],\uip[x]\right), & -\infty<\lip[x],\\
	      \left(-\infty,\uip[x]\right), & \text{else}.\\
	      \end{cases}
	  \end{equation*}
	\end{definition}
	
	\begin{definition}[smooth RSN approximation]\label{def:smoothRSNappr}
		For $\wR$ and $\RNR$ as in \Cref{def:ridgeNet} with corresponding kink density $\gxi$ consider for every $x\in\R$ the kernel 
		\begin{displaymath}
		\kapx:\R\to\R, \quad \kapx(s):= \ind_{\Iwp[x]}(x-s)\frac{1}{|\Iwp[x]|}\quad {\color{hellgrau}\forall s\in\R}, 
		\end{displaymath}
		where
		\begin{align}
		   \label{eq:def:Iwp} \Iwp[x]&:=B_{\frac{1}{2\sqrt{n}\gxi(x)}}(x)\cap\Ip[x],\\
		    B_{r}(m)&:=\{\tau\in\R: |\tau-m|\le r \}.
		\end{align}
		The corresponding \emph{smooth RSN approximation} $\fwR$ w.r.t.~$\wR$ is then defined as the convolution\footnote{This \enquote{convolution} is a bit special, because the kernel~\kapx\ changes with $x\in\R$. Therefore, the notation $\RNR*\nolinebreak[4]\kappa$ would not be properly defined, but we could define $\RNR*\negthickspace*\ \kappa$ as: $ \left(\RNRo*\negthickspace*\ \kappa\right)(x):=\left(\RNRo*\kapx\right)(x)=\int_\R \RNRo(x-s)\kapx(s)\,ds=\int_{{\Iwp[x]}}\RNRo(t)\frac{1}{|\Iwp[x]|}\,dt$, \ \faog\ \fax.}
		\begin{equation}
		\fwRo(x):= \left(\RNRo*\kapx\right)(x) \qquad\faog\ \fax .
		\end{equation}
		Moreover, with the notation
		\begin{equation}\label{eq:RNRp}
		\RNR(x)=\underbrace{\sum_{k\in\kp}\wRk\,\sigma\left({b_k}+{v_{k}}x\right)}_{=:\RNRp(x)} + \underbrace{\sum_{k\in\km}\wRk\,\sigma\left({b_k}+{v_{k}}x\right)}_{=:\RNRm(x)} \quad\faxg,
		\end{equation}
		{\mediumimportantStart with $\wRp:=\left( \wRk \right)_{k\in\kp}$ and $\wRm$ analogously defined as \wwp\ and \wwm}, we have
		\begin{equation}\label{eq:smoothRSNapprsplit}
		\fwR(x)= \underbrace{\left(\RNRp*\kapx\right)(x)}_{=:\fwRp(x)} +\underbrace{\left(\RNRm*\kapx\right)(x)}_{=:\fwRm(x)} \qquad\faxg .
		\end{equation}
	\end{definition}
	{\footnotesize
		\begin{remark}\label{rem:kernel}
			For any $x\in\R$ the kernel $\kapx$ introduced in \Cref{def:smoothRSNappr} satisfies
			\begin{enumerate}
				\item $\int_\R\kapx(s)\,ds=1$ and
				\item $\lim_{n\to\infty} \kapx=\delta_0$, where $\delta_0$ denotes the Dirac distribution at zero.
			\end{enumerate}
		\end{remark}
	}
	
	\begin{proof}[\hypertarget{proof:thm:ridgeToSpline}{Proof of \Cref{thm:ridgeToSpline}}]%

		The auxiliary functions \sRNw\ and \fwR\ defined above in \Cref{def:splineApproximatingRSN,def:smoothRSNappr} will play an important role in this proof.
	    In the end, we want to show (probabilistic) convergence of \RNR\ to the $\argmin_{\WT(\R)}\Flpmm$ for every $\RNR\in\argmin\Fn$.
		Our strategy to achieve this goal is to first prove that every \RNR\ converges to a corresponding function~\fwR\ in the limit $n\to\infty$ (\Cref{le:3}).
		Second, we show that every $\fwR$ converges to the $	\argmin_{\WT(\R)}\Flpmm$. This requires more steps---first we show the convergence $\Flpm\fwRpmt\to\min_{\T}\Flpm$ (in multiple steps based on \Cref{le:2,le:4}) to further imply with the help of \Cref{le:7} the convergence \meqref{eq:ridgeToSplineNotUnique}.

		Following this strategy, we prove~\Cref{thm:ridgeToSpline} step by step:
		
		\begin{enumerate}[step 1, start=0]
			{\transparent{\meineTranzparenz} \item[step -0.5]Before starting with the proof, we need the auxiliary \Cref{le:poincare,le:deltastrip}}
			\item\label{itm:0} \Cref{le:0} shows
			\[\plim \sobnorm[{\sRNw-\flpm}]=0.\]
			\item\label{itm:1} It is directly clear that \[\Fnb{\RNR}\leq\Fnb{\sRNw},\] because of the optimality of~\RNR\ (see \Cref{def:ridgeNet}).
			{\transparent{\meineTranzparenz}\item[step 1.5] The auxiliary \Cref{le:lossconvergence} will be needed for \ref{itm:2} and \ref{itm:4}.}
			\item\label{itm:2} \Cref{le:2} shows \[ \plim\Fnb{\sRNw}=\Flpm\flpmt.\]
			{\transparent{\meineTranzparenz}\item[step 2.5] The auxiliary \Cref{le:help:wstetig,le:wstetig,le:smallw*} will be needed for \ref{itm:3} and \ref{itm:4}.}
			\item\label{itm:3} \Cref{le:3} shows \[\plim\sobnorm[\RNR-\fwR]=0.\]
			\item\label{itm:4} \Cref{le:4} shows \[ \plim \left|\Fnb{\RNR}-\Flpm\fwRpmt\right| =0.\]
			\item\label{itm:5}
			
			\[\Flpm\flpmt\leq\Flpm\fwRpmt\]
			holds, because of the optimality of $\flpmt\in\T$.
			\item\label{itm:6} Combining \labelcref{itm:4,itm:1,itm:2,itm:5} we directly get:\footnote{\label{footnote:pmEpsNotation}We are using the following notation:
				\[a_n\approx b_n\overset{\PP}{\pm}\epsilon_1 \logeq \fae[1]:\forall P_1\in (0,1):\exists n_0\in\N: \forall n\in\N_{>n_0}:\PP[{a_n\in b_n+\left[-\epsilon_1,\epsilon_1\right]}]>\rho,\]
				but a complete formalization of this notation would be quite long. This notation needs to be interpreted depending on the context---e.g.:
				\[b_n\overset{\PP}{\pm}\epsilon_1\approx b_n\overset{\PP}{\pm}\epsilon_1\overset{\PP}{\pm}\epsilon_2 \logeq \fae[2]:\forall P_{2}\in (0,1):\exists n_0\in\N: \forall n\in\N_{>n_0}:\PP[{b_n\in c_n+\left[-\epsilon_2,\epsilon_2\right]}]>P_2,\] or sometimes it makes sense to replace \enquote{$\in$} by \enquote{$\subseteq$} in a reasonable way. And in the proofs of some later lemmata $\overset{\PP}{\pm}\epsilon_2$ can have the meaning of $\overset{\underset{\delta,\epsilon_1\to 0}{\PP}}{\pm}\epsilon_2$ instead of $\overset{\underset{n\to 0}{\PP}}{\pm}\epsilon_2$ depending on the context.%

			}
			\begin{align*}
			\Flpm\fwRpmt
			&\overset{\text{\ref{itm:4}}}\approx \texttransparent{\meineTranzparenz}{\Fnb{\RNR}\overset{\PP}{\pm}\epsilon_1\leq}\\
			&\overset{\text{\ref{itm:1}}}\leq\texttransparent{\meineTranzparenz}{\Fnb{\sRNw}\overset{\PP}{\pm}\epsilon_1\approx}\\
			&\overset{\text{\ref{itm:2}}}\approx\Flpm\flpmt\overset{\PP}{\pm}\epsilon_1\overset{\PP}{\pm}\epsilon_2
			\overset{\text{\ref{itm:5}}}\leq\Flpm\fwRpmt\overset{\PP}{\pm}\epsilon_1\overset{\PP}{\pm}\epsilon_2,
			\end{align*}
			and thus:
			\begin{align*}
			\Flpm\fwRpmt
			&\overset{\substack{\text{\ref{itm:4}}\\
					\text{\ref{itm:2}}\\ \text{\ref{itm:1}}}}\lessapprox
			\Flpm\flpmt\overset{\PP}{\pm}\epsilon_3 \hphantom{\overset{\PP}{\pm}\epsilon_2}
			&\overset{\text{\ref{itm:5}}}\leq\Flpm\fwRpmt\overset{\PP}{\pm}\epsilon_3,\hphantom{\overset{\PP}{\pm}\epsilon_2}
			\end{align*}
			which directly implies
			\begin{equation}\label{eq:itm:6}
			\plim \Flpm\fwRpmt = \Flpm\flpmt.
			\end{equation}
			\item\label{itm:7} \Cref{le:7} shows
				\begin{equation*}
	\plim d_{\Wkp[K]{1}{\infty}}\left(\fwRpmt,
	\argmin_{\T}\Flpm\right) =0,
	\end{equation*}
			if one applies it on the result~\meqref{eq:itm:6} of \ref{itm:6}.
			\item\label{itm:8} Combining \labelcref{itm:3,itm:7} with the triangle inequality directly results in the statement~\meqref{eq:ridgeToSplineNotUnique} we want show.
		\end{enumerate}
		
	\end{proof}
	
	\begin{lemma}[Poincar\'e typed inequality]\label{le:poincare}
		Let $f:\R\to\R$ be continuously differentiable with $f':\R\to\R$ Lebesgue integrable. Then, for any interval $K=[a, b]\subset\R$ such that $f(a)=0=f'(a)$ there exists a $C_K^\infty\in\Rp$ such that
		\begin{equation}\label{eq:supPoincare}
		\sobnorm[f]\le C_K^\infty\supnorm[f'].
		\end{equation}
		Additionally, if $f$ is twice differentiable with $f'':\R\to\R$ Lebesgue integrable, there exists a $C_K^2 \in\Rp$ such that
		\begin{equation}\label{eq:2Poincare}
		\sobnorm[f]\le C_K^2\Ltwonorm[f''].
		\end{equation}
		
	\end{lemma}
	\begin{proof}
		By the fundamental theorem of calculus, if $ \supnorm[f']<\infty$, then
		\begin{align*}
		\supnorm[f]=\sup_{x\in K}\left|\int_a^x f'(y) \, dy\right|\le|b-a|\sup_{y\in K}|f'(y)|.
		\end{align*}
		Hence it follows that
		\begin{align*}
		\sobnorm[f]=\max\left\{\supnorm[f], \supnorm[f']\right\}\le\max\{|b-a|,1\}\supnorm[f']=C_K^\infty\supnorm[f'].
		\end{align*}
		Similarly, by the H\"older inequality we have 
		\begin{align*}
		\supnorm[f']=\sup_{x\in K}\left|\int_a^b f''(y)\ind_{[a, x]}(y) \, dy\right|\le\sup_{y\in K}\Ltwonorm[f'']\Ltwonorm[\ind_{[a,y]}]=\sqrt{|b-a|}\Ltwonorm[f''].
		\end{align*}
		Thus, \meqref{eq:2Poincare} follows from
		\begin{align*}
		\sobnorm[f]\le C_K^\infty\supnorm[f']\le C_K^\infty\sqrt{|b-a|}\Ltwonorm[f'']=C_K^2\Ltwonorm[f''].
		\end{align*}
	\end{proof}
	\begin{lemma}\label{le:deltastrip}
		Let $\RN$ be a real-valued randomized shallow network. 
		For $\varphi:\R^2\to \R$ uniformly continuous such that for all $x\in \supp(\gxi)$,
		$\E[\varphi(\xi, v)\frac{1}{n \gxi(\xi)}|\xi=x]<\infty$, it then holds that%
		\footnote{\label{footnote:analogous:le:deltastrip}The same statement as~\meqref{eq:deltastrip} is analogously true if one replaces $\kp$ by \km\ of course. Also
			\begin{equation*}
			\plim \sum_{k:\xi_k<T} \varphi(\xi_k, v_k)\frac{\hq_k}{2} = \int_{\Cgl[\gxi]\wedge T}^{\Cgu[\gxi]\wedge T}\E[\varphi(\xi, v)|\xi=x]\, dx
			\end{equation*} holds analogously. Without \Cref{as:easyReadable}\ref{item:symDistributions} the statement~\meqref{eq:deltastrip} needed to be reformulated as:
			\begin{align*}
			\plim \sum_{k\in\kp:\xi_k<T} \varphi(\xi_k, v_k)\hqp_k &= \int_{\Cgl[\gxip]\wedge T}^{\Cgu[\gxip]\wedge T}\E[\varphi(\xi, v)|\xi=x,v>0]\, dx \\
			\plim \sum_{k\in\km:\xi_k<T} \varphi(\xi_k, v_k)\hqm_k &= \int_{\Cgl[\gxim]\wedge T}^{\Cgu[\gxim]\wedge T}\E[\varphi(\xi, v)|\xi=x,v<0]\, dx
			\end{align*}}
		\begin{equation}\label{eq:deltastrip}
		\plim \sum_{k\in\kp:\xi_k<T} \varphi(\xi_k, v_k)\hq_k = \int_{\Cgl[\gxi]\wedge T}^{\Cgu[\gxi]\wedge T}\E[\varphi(\xi, v)|\xi=x]\, dx
		\end{equation}
		uniformly in $T\in K$.
	\end{lemma}
	\begin{proof}
		For $T\le\Cgl[\gxi]$ both sides of \meqref{eq:deltastrip} are zero, thus we restrict ourselves to $T>\Cgl[\gxi]$.
		By uniform continuity of $\varphi$ and $\frac{1}{\gxi}$ in $\xi$, for any $\epsilon>0$ there exists a $\delta(\epsilon)$ such that for every $|\xi'-\xi|<\delta(\epsilon)$ we have $|\varphi(\xi, v)\frac{1}{\gxi(\xi)}-\varphi(\xi', v)\frac{1}{\gxi(\xi')}|<\epsilon$ uniformly in $v$.
		W.l.o.g. assume $\supp(\gxi)$ is an interval.
		Thus, by splitting the interval $[\Cgl[\gxi], \Cgu[\gxi]\wedge T]$ into disjoint strips\footnote{\label{footnote:deltaStripsPartition}Assume $\exists\ell_1,\ell_2 \in \Z: \Cgl[\gxi]=\delta\ell_1,\Cgu[\gxi]=\delta\ell_2$ to make the notation simpler. For a cleaner proof, one should choose a suitable partition of $\supp(\gxi)$.} of equal length $\delta\le\delta(\epsilon)$, we have\footnote{The notation ${\pm}\epsilon$ from \cref{footnote:pmEpsNotation} on \cpageref{footnote:pmEpsNotation} and slight adaptions of it will be used in this proof a lot. The relations of all the epsilons will be explicitly described in \meqref{eq:deltaStripsEpsilonConnections}.}
		\begin{align*}
		&\sum_{\stackrel{k\in\kp}{\xi_k<T}} \varphi(\xi_k, v_k)\hq_k=\\
		&\overset{\ref{footnote:deltaStripsPartition}}{=} \sum_{\stackrel{\ell\in\Z}{[\delta\ell, \delta(\ell+1))\subseteq[\Cgl[\gxi], \Cgu[\gxi]\wedge T]}}\left(\sum_{\stackrel{k\in\kp}{\xi_k\in[\delta\ell, \delta(\ell+1))}}\varphi(\xi_k, v_k)\hq_k\right)\\
		& \approx \sum_{\stackrel{\ell\in\Z}{[\delta\ell, \delta(\ell+1))\subseteq[\Cgl[\gxi], \Cgu[\gxi]\wedge T]}}\left(\sum_{\stackrel{k\in\kp}{\xi_k\in[\delta\ell, \delta(\ell+1))}}\left(\varphi(\ell\delta, v_k)\frac{2}{n\gxi(\ell\delta)}\pm \frac{\epsilon}{n}\right)\frac{|\{m\in\kp:\xi_m\in[\delta\ell, \delta(\ell+1))\}|}{|\{m\in\kp:\xi_m\in[\delta\ell, \delta(\ell+1))\}|}\right)\\
		&\approx \sum_{\stackrel{\ell\in\Z}{[\delta\ell, \delta(\ell+1))\subseteq[\Cgl[\gxi], \Cgu[\gxi]\wedge T]}}\left(\frac{\sum_{\stackrel{k\in\kp}{\xi_k\in[\delta\ell, \delta(\ell+1))}}\varphi(\ell\delta, v_k)}{|\{m\in\kp:\xi_m\in[\delta\ell, \delta(\ell+1))\}|}\frac{2|\{m\in\kp:\xi_m\in[\delta\ell, \delta(\ell+1))\}|}{n\gxi(\ell\delta)}\right) \pm \epsilon.
		\end{align*}
		The number of nodes within a $\delta$-strip follows a binomial distribution with 
		\begin{align*}
		\E[|\{m\in\kp:\xi_m\in[\delta\ell, \delta(\ell+1))\}|]&=\PP[v_k>0]n\int_{[\delta\ell, \delta(\ell+1))}\gxi(x)\, dx\approx \frac{1}{2}n(\delta \gxi(\ell\delta)\pm \delta\tilde{\epsilon}),
		\end{align*}
		for any $\delta\le\delta(\epsilon, \tilde{\epsilon})$, since $\gxi$ is uniformly continuous on $\supp(\gxi)$ by \Cref{as:truncatedg}\ref{item:densityIsSmooth}.
		For $\delta\le\delta(\epsilon, \tilde{\epsilon})$ small enough, we have $\mathcal{L}(v_k)\approx\mathcal{L}(v|\xi=\ell\delta)$ $\forall k\in\kp:\xi_k\in[\delta\ell, \delta(\ell+1))$ and we may apply the law of large numbers to further obtain
		\begin{align*}
		\sum_{k\in\kp:\xi_k<T} \varphi(\xi_k, v_k)\hq_k &\approx \sum_{\stackrel{\ell\in\Z}{[\delta\ell, \delta(\ell+1))\subseteq[\Cgl[\gxi], \Cgu[\gxi]\wedge T]}}\left(\E[\varphi(\xi, v)|\xi=\ell\delta]\overset{\PP}{\pm}\tilde{\tilde{\epsilon}}\right)\delta\left(1\pm \frac{\tilde{\epsilon}}{\gxi(\ell\delta)}\right)\pm\epsilon\\
		&\approx\left(\sum_{\stackrel{\ell\in\Z}{[\delta\ell, \delta(\ell+1))\subseteq[\Cgl[\gxi], \Cgu[\gxi]\wedge T]}}\Bigg(\E[\varphi(\xi, v)|\xi=\ell\delta]\delta\Bigg)\overset{\PP}{\pm}\tilde{\tilde{\epsilon}}|\Cgu[\gxi]-\Cgl[\gxi]|\right)\\
		&\quad \cdot \left(1\pm \frac{\tilde{\epsilon}}{\gxi(\ell\delta)}\right)\pm\epsilon
		\end{align*}
		Since $1/\gxi(\cdot)$ and $\E[\varphi(\xi, v)|\xi=\cdot]$ are bounded on $\supp(\gxi)$, and $\epsilon, \tilde{\epsilon}$ depend on $\delta$ only, we may for some $\epsilon^*, \rho^*\in(0,1)$ define
		\begin{subequations}\label{eq:deltaStripsEpsilonConnections}
			\begin{align}
			\epsilon&:=\frac{\epsilon^*}{3},\\
			\tilde{\epsilon}&:=\frac{\epsilon^*\min_{x\in\supp(\gxi)}\gxi(x)}{3|\Cgu[\gxi]-\Cgl[\gxi]|\left(\max_{x\in\supp(\gxi)}\E[\varphi(\xi, v)|\xi=x]+1\right)},\\
			\tilde{\tilde{\epsilon}}&:=\frac{\epsilon^*}{3 |\Cgu[\gxi]-\Cgl[\gxi]|},\\
			\tilde{\tilde{\rho}}&:=\left(\rho^*\right)^{\frac{\delta}{|\Cgu[\gxi]-\Cgl[\gxi]|}},\\
			n_0^*&:=\tilde{\tilde{n}}_0(\tilde{\tilde{\epsilon}}, \tilde{\tilde{\rho}}).
			\end{align}\end{subequations}
		With the above, it follows that for any $\epsilon^*,\rho^*\in(0,1)$ there exists a $n_0^*$ such that $\forall n>n_0^*:$
		\begin{align*}
		\PP[{\left| \sum_{k\in\kp:\xi_k<T}  \varphi(\xi_k, v_k)\hq_k -\sum_{\stackrel{\ell\in\Z}{[\delta\ell, \delta(\ell+1))\subseteq[\Cgl[\gxi], \Cgu[\gxi]\wedge T]}}\E[\varphi(\xi, v)|\xi=\ell\delta]\delta\right|\le \epsilon^*}]>\rho^*.
		\end{align*}
		For $\delta$ small enough, the above Riemann sum converges uniformly in $T$ to yield the desired result.
	\end{proof}
	
	\begin{lemma}[\ref{itm:0}]\label{le:0}
		For any choice of the penalty parameter $\lambda>0$ and $K\subset\R$ a compact interval, the spline approximating RSN $\sRNw$ converges to the adapted regression spline $\flpm$ in probability w.r.t. $\sobnorm$ with increasing number of nodes, i.e.
		\begin{displaymath}
		\plim \sobnorm[{\sRNw-\flpm}]=0.
		\footnote{ Using the definition of the $\Plim$, we get:
			\begin{displaymath}
			\fae :\forall \rho \in (0,1): \exists n_0\in\N :\forall n\geq n_0:
			\PP[
			{\sobnorm[{ \sRNw -\flpm}]}
			< \epsilon ] > \rho.
			\end{displaymath}}
	\end{displaymath}
\end{lemma}
\begin{proof}
	Let $\lambda>0$ and $K\subset\R$ be a compact interval with $[\Cgl, \Cgu]\subset K$. Directly from the definition~\meqref{eq:splineApprRSNsplit} of \sRNwp\ and \sRNwp\ and the \Cref{def:adaptedSplineReg,def:adaptedSplineRegTuple} of \flpm, it follows that it is sufficient to show:
	\begin{subequations}
	\begin{align}
	\plim\sobnorm[{\sRNwp-\flp}] &= 0 \quad \text{and}\label{eq:limp} \\
	\quad \plim\sobnorm[{\sRNwm-\flm}] &= 0\, .
	\end{align}
	\end{subequations}
	W.l.o.g. we restrict ourselves to proving \meqref{eq:limp}, as the latter limit follows analogously. By \Cref{le:poincare} it suffices to show that
	\begin{equation}\label{eq:supnorm0}
	\plim\supnorm[{{\sRNwp}^{'}-{\flp}^{'}}]=0.
	\end{equation}
	Since for any $x\in K$
	\begin{displaymath}
	{\sRNwp}^{'}(x)=\sum_{k\in\kp}\ww_k v_k= \sum_{k\in\kp}{\flp}^{''}(\xi_k)\frac{v^2_k}{\E[v^2|\xi=\xi_k]}\hq_k,
	\end{displaymath}
	we may employ \Cref{le:deltastrip}\footnote{Note that $\varphi(x,y)$ is uniformly continuous on $\supp(\gxi)$ since, by definition, $\flp\in\C^2(\R)$ and $\supp(\gxi)$ is compact by \Cref{as:truncatedg}.} with $\varphi(z, y)={\flp}^{''}(z)\frac{y^2}{\E[v^2|\xi=z]}$ to obtain
	\begin{displaymath}
	\plim {\sRNwp}^{'}(x)=\int_{\Cgl[\gxi]\wedge x}^{\Cgu[\gxi]\wedge x}\E[{{\flp}^{''}(\xi)\frac{v^2}{\E[v^2|\xi=z]}|\xi=z}]\, dz=\int_{\Cgl[\gxi]\wedge x}^{\Cgu[\gxi]\wedge x}{\flp}^{''}(z)\, dz
	\end{displaymath}
	uniformly in $x\in K$. Employing the fundamental theorem of calculus we further obtain
	\begin{displaymath}
	\plim {\sRNwp}^{'}(x) = {\flp}^{'}(\Cgu[\gxi]\wedge x)-{\flp}^{'}(\Cgl[\gxi]\wedge x) \qquad \fax .
	\end{displaymath}
	By \Cref{rem:compactSupp}, we have that ${\flp}^{'}(\Cgl[\gxi]\wedge x)=0$ for any $x\in\R$. Since by the same remark, ${\flp}^{'}$ is constant on $[\Cgu[\gxi], \infty)$, we finally obtain
	\begin{displaymath}
	\plim {\sRNwp}^{'}(x) = {\flp}^{'}(x) \qquad \text{uniformly in } x\in K.
	\end{displaymath}
	Hence \meqref{eq:supnorm0} follows.
	
\end{proof}

\begin{lemma}[$\Ltr(f_n)\to\Ltr (f)$]\label{le:lossconvergence}

Let $(f_n)_{n\in\N}$ be a sequence of continuous functions with piece-wise continuous derivatives which converges in probability w.r.t.\ $\Woi$\footnote{
Hence \Cref{le:lossconvergence} can be used together with \Cref{le:0} to show $\plim\Ltr(\sRNw)=\Ltr(\flpm)$ or together with \Cref{le:3} to show $\plim\Ltrb{\RNR}=\Ltrb{\fwR}$ since a ReLU-neural network is always continuous and piece-wise continuously differentiable.} to a function $f:\R\to\R$, then the training loss $\Ltr$ of $f_n$ converges in probability to $\Ltrb{f}$ as $n$ tends to infinity, i.e.
	\begin{equation}
	\plim\Ltr(f_n)=\Ltr(f).
	\end{equation}
	
\end{lemma}

\begin{proof}
By \Cref{as:generalloss} there exists a finite Borel-measure $\nu$ and some $p>1$ s.t. \Ltr\ is continuous w.r.t.\ $\Wkp[K,\nu]{1}{p}$. Since 
\begin{align*}
    \sobnormop[f]&=\left(\int_K (f)^p\,d\nu\right)^{\frac{1}{p}}+\left(\int_K (f^{'})^p\,d\nu\right)^{\frac{1}{p}}\\
    &\le\left(\sup_{x\in K}|f(x)|+\sup_{x\in K}|f^{'}(x)|\right)\nu(K)^{\frac{1}{p}}\\
    &\le 2\nu(K)^{\frac{1}{p}}\sobnorm[f],
\end{align*}
where the last inequality follows from \Cref{rem:sobnormoi}. Therefore, \Ltr\ is continuous w.r.t.\ $\Woi$ as well and the result follows.
\end{proof}

\begin{lemma}[\ref{itm:2}]\label{le:2}
	For any $\lambda>0$, we have
	\begin{equation}\label{eq:statement2}
	\plim\Fnb{\sRNw}=\Flpm\flpmt,
	\end{equation}
	with $\lw$ and $g$ as defined in \Cref{thm:ridgeToSpline}.
\end{lemma}
\begin{proof}
	We start by showing
	\begin{equation}\label{eq:penalty}
	\plim\lw \twonorm[\ww]^2=\lambda 2g(0) \left(
	\int_{\supp (g)} \frac{\left( {\flp}^{''}(x) \right)^2}{g(x)} dx
	+\int_{\supp (g)} \frac{\left( {\flm}^{''}(x) \right)^2}{g(x)} dx
	\right).
	\end{equation}
	Since $\twonorm[\ww]^2=\twonorm[\wwp]^2+\twonorm[\wwm]^2$,
	we restrict ourselves to proving
	\begin{equation}\label{eq:penaltyp}
	\plim\lw \twonorm[\wwp]^2=\lambda 2g(0)\int_{\supp(\gxi)}\frac{\left({\flp}^{''}(x)\right)^2}{g(x)}\, dx.
	\end{equation}
	With the definitions of $\wwp$, $\lw$ and \hq\ we have
	\begin{align*}
	\lw\twonorm[\wwp]^2&=\lw\sum_{k\in\kp}\left({\flp}^{''}(\xi_k)\frac{\hq_k v_k}{\E[v^2|\xi=\xi_k]}\right)^2\\
	&=\lw\sum_{k\in\kp}\left(\left({\flp}^{''}\right)^2(\xi_k)\frac{\hq_k v_k^2}{\E[v^2|\xi=\xi_k]^2}\right)\hq_k\\
	&=\lambda 2g(0) \sum_{k\in\kp}\left(\left({\flp}^{''}\right)^2(\xi_k)\frac{2v_k^2}{\gxi(\xi_k)\E[v^2|\xi=\xi_k]^2}\right)\hq_k.\\
	\end{align*}
	An application of \Cref{le:deltastrip} with $\varphi(x,y)=\left({\flp}^{''}\right)^2(x)\frac{2y^2}{\gxi(x)\E[v^2|\xi=y]^2}$ further yields \meqref{eq:penaltyp} via
	\begin{align*}
	\plim\lw\twonorm[\wwp]^2&= \lambda 2g(0)\int_{\supp(\gxi)}\E[{\left({\flp}^{''}\right)^2(\xi)\frac{2v^2}{\gxi(\xi)\E[v^2|\xi=x]^2}\bigg|\xi=x}]\, dx\\
	&=\lambda 2g(0)\int_{\supp(\gxi)}\frac{2\left({\flp}^{''}\right)^2(x)}{\gxi(x)\E[v^2|\xi=x]}\, dx\\
	&=\lambda 2g(0)\int_{\supp(\gxi)}\frac{\left({\flp}^{''}(x)\right)^2}{g(x)}\, dx.
	\end{align*}
	Thus we have proved the convergence of the penalization terms \meqref{eq:penalty}. Together with \Cref{le:0,le:lossconvergence}, \meqref{eq:statement2} follows.
	
\end{proof}

\begin{lemma}\label{le:wstetig}%
	Using the notation of \Cref{def:kinkPos,def:ridgeNet}, the following statement holds:
	\begin{equation*}\small
	\begin{split}
	&\forall \epsilon\in\Rp:\exists\delta\in\Rp:\forall n\in\N:\fao:\allIndi{n}{\grave k,\acute{k}}:\\
	&\left( \underbrace{\left( \big|\underbrace{\xi_{\grave k}\omb-\xi_{\acute{k}} \omb}_{=:\Delta \xi\omb}\big|<\delta \ \wedge \ \sgnb{v_{\grave k}\omb}=\sgnb{v_{\acute{k}}\omb}\wedge \Ip[\xi_{\grave k}\omb]=\Ip[\xi_{\acute{k}}\omb]\right)}_{
	(*)}
	\Rightarrow \left| \frac{\wRk[{\grave k}]\omb}{v_{\grave k}\omb}-\frac{\wRk[\acute{k}]\omb}{v_{\acute{k}}\omb}\right|<\frac{\epsilon}{n} \right),
	\end{split}
	\end{equation*}
	if we assume that $v_k$ is never zero.
\end{lemma}
\begin{proof}
	For all $\tilde{\epsilon}>0$ and all sufficiently small $\delta$, assuming (*) we will prove the following inequality :

\begin{subequations}\label{eq:strongweightIneq}
		\begin{align}		
		\label{subeq:strongweightIneqMain}\left| \frac{\wRk[{\grave k}]}{v_{\grave k}}-\frac{\wRk[\acute{k}]}{v_{\acute{k}}}\right|
		&%
		\overset{\ref{item:strongweightIneqMain}}{\leq}
		\frac{\left(\delta\nu(K)^{\frac{1}{p}}+\tilde{\epsilon}\right)}{2\lw} \limsup_{\sobnormop[f]\to 0} \frac{\Ltr\left( \RNR+f\right)-\Ltr\left( \RNR\right)}{\sobnormop[f]}
		\overset{\ref{item:strongweightIneqpNorms}}{\leq}\\
		\label{subeq:strongweightIneqSimple}&\overset{\ref{item:strongweightIneqpNorms}}{\leq}\frac{\left(\delta\nu(K)^{\frac{1}{p}}+\tilde{\epsilon}\right)}{2\lw}\LLip.
		\end{align}
	\end{subequations}
	From inequality~\meqref{eq:strongweightIneq}, the statement of \Cref{le:wstetig} follows for $\delta:=\min\Set{\tilde{\delta}_{\nuc,\epsilon},\underbrace{(\frac{\epsilon\lambda 4g(0)}{\LLip}-\tilde{\epsilon})}_{>0 \text{ for $\tilde{\epsilon}$ small}}\nu(K)^{\frac{1}{p}}}$.
	\begin{enumerate}[1.]
		\item\label{item:strongweightIneqMain} Proof of \meqref{subeq:strongweightIneqMain}:
		First we define the disturbed weight vector~$w^{\Delta s}$ such that
		\begin{equation*}
		w_k^{\Delta s}:=\wRk+\begin{cases}
		+\frac{\Delta s}{\left| v_{\grave k}\right|} & k={\grave k} \\
		-\frac{\Delta s}{\left| v_{\acute{k}}\right|} & k=\acute{k} \\
		0 & \, \text{else-wise}
		\end{cases}
		\end{equation*}
		by shifting a little bit of the distributional second derivative~$\Delta s$ from the $\acute{k}$th kink to the ${\grave k}$th kink.
		By a case analysis (or by drawing a sketch) one can easily show that conditioned on $\sgnb{v_{\grave k}}=\sgnb{v_{\acute{k}}}$
		\begin{align}\label{eq:leqDxiDs}
		\forall x\in\R: \left| \RNR(x) - \left( \RN_{w^{\Delta s}}(x) %
		\right) \right| &\leq|\Delta \xi \Delta s| \\
		\forall x\in\R: \left| {\RNR}^{'}(x) - \left( {\RN_{w^{\Delta s}}}^{'}(x)\right) \right| &\leq |\Delta s|\ind_{\left\lbrack\xi_{\grave k},\xi_{\acute{k}}\right\rbrack}(x),\label{eq:leqIndDs}
		\end{align}
		where we use the notation $\left\lbrack\xi_{\grave k},\xi_{\acute{k}}\right\rbrack:=\left\lbrack\min\{\xi_{\grave{ k}},\xi_{\acute{k}}\},\max\{\xi_{\grave k},\xi_{\acute{k}}\}\right\rbrack$.
		Let $\nu, p$ be as in \Cref{as:generalloss}. We then have
		\begin{align*}\small
		    \sobnormop[\RNR-\RN_{w^{\Delta s}}]&=\left(\int_K (\RNR-\RN_{w^{\Delta s}})^p\,d\nu\right)^{\frac{1}{p}}+\left(\int_K ({\RNR}^{'}- {\RN_{w^{\Delta s}}}^{'})^p\,d\nu\right)^{\frac{1}{p}}\\
		    &\overset{\meqref{eq:leqDxiDs},\meqref{eq:leqIndDs}}{\leq}\left(\int_{K} \left(|\Delta\xi\Delta s|\right)^p\,d\nu\right)^{\frac{1}{p}}+\left(\int_{\left\lbrack\xi_{\grave k},\xi_{\acute{k}}\right\rbrack} (|\Delta s|)^p\,d\nu\right)^{\frac{1}{p}}\\
		    &\le |\Delta \xi\Delta s|\nu(K)^{\frac{1}{p}}+|\Delta s|\nu(\left\lbrack\xi_{\grave k},\xi_{\acute{k}}\right\rbrack)^{\frac{1}{p}}\\
		    &\leq|\Delta s|\left(\delta\nu(K)^{\frac{1}{p}}+\nuc(\left\lbrack\xi_{\grave k},\xi_{\acute{k}}\right\rbrack)^{\frac{1}{p}}\right),
		\end{align*}
		where the last inequality holds since first, $|\Delta \xi|<\delta$ by assumption and second, $\left\lbrack\xi_{\grave k},\xi_{\acute{k}}\right\rbrack\subset\Ip[\xi_{\grave k}]$ (note that $\PP$-a.s. $\xi_{\grave k},\xi_{\acute{k}}\notin \supp(\nua)$).
		$|\Delta \xi|<\delta$ further implies that if we choose $\delta>0$ small enough (i.e., if $\delta<\tilde{\delta}_{\nuc,\epsilon}$ for a $\tilde{\delta}_{\nuc,\epsilon}>0$), then $\nuc(\left\lbrack\xi_{\grave k},\xi_{\acute{k}}\right\rbrack)^{\frac{1}{p}}<\tilde{\epsilon}$, because $\nuc$ is absolutely continuous (and thus uniformly continuous) w.r.t.\ the Lebesgue-measure. Thus, we get
		\begin{equation}\label{eq:sobnormopleqDs}
		     \sobnormop[\RNR-\RN_{w^{\Delta s}}]\le |\Delta s|\left(\delta\nu(K)^{\frac{1}{p}}+\tilde{\epsilon}\right).
		\end{equation}
		Furthermore, as \RNR\ is optimal, the following inequality has to hold:
		\begin{subequations}\label{eq:subderivative}
		\begin{align}
		0
		&\leq\liminf_{\Delta s\to 0} \frac{\Fn\left( \RN_{w^{\Delta s}}\right)-\Fn\left( \RNR\right)}{|\Delta s|}=\\
		&=-\lw 2 \left| \frac{\wRk[{\grave k}]}{v_{\grave k}}-\frac{\wRk[\acute{k}]}{v_{\acute{k}}}\right|
		+\liminf_{\Delta s\to 0} \frac{\Ltr\left( \RN_{w^{\Delta s}}\right)-\Ltr\left( \RNR\right)}{|\Delta s|}.
		\end{align}
		\end{subequations}
		Transforming this equation gives:
		\begin{subequations}
		\begin{align}
		\lw 2\left|  \frac{\wRk[{\grave k}]}{v_{\grave k}}-\frac{\wRk[\acute{k}]}{v_{\acute{k}}}\right|
		&\overset{\meqref{eq:subderivative}}{\leq}\liminf_{\Delta s\to 0} \frac{\Ltr\left( \RN_{w^{\Delta s}}\right)-\Ltr\left( \RNR\right)}{|\Delta s|}\\
		&\overset{\meqref{eq:sobnormopleqDs}}{\leq}
	    \left(\delta\nu(K)^{\frac{1}{p}}+\tilde{\epsilon}\right) \limsup_{\sobnormop[f]\to 0} \frac{\Ltr\left( \RNR+f\right)-\Ltr\left( \RNR\right)}{\sobnormop[f]}.
		\end{align}
		\end{subequations}
		Dividing both sides by $2\lw$ results in \meqref{subeq:strongweightIneqMain}.
		
		\item\label{item:strongweightIneqpNorms} \meqref{subeq:strongweightIneqMain}$\leq$\meqref{subeq:strongweightIneqSimple} holds because \Ltr\ is assumed to be Lipschitz continuous on a sub-levelset of $\Ltr$ (cp. \Cref{as:generalloss}\ref{item:lipschitzL}) and $\Ltrb{\RNR}\leq\Ltrb{0}$ because of optimality and non-negativity of $\lw\twonorm^2$:
		\begin{equation}\label{eq:RNRbetterzero}
    \Ltrb{\RNR} + \underbrace{\lw \twonorm[\wR]^2}_{\geq 0}
    \overset{\text{optimality}}{\leq} \Ltrb{\smash[b]{\underbrace{0}_{\RNw[0]}}} \unimportant{+\twonorm[{\smash[b]{\underbrace{0}_{\in\R^n}}}]^2}
    \vphantom{\underbrace{0}_{\RNw[0]}}.
\end{equation}
	\end{enumerate}
	
\end{proof}

\begin{lemma}[$\frac{\wR}{v}\approx\BigO(\frac{1}{n})$]\label{le:smallw*}
	For any $\lambda>0$, we have
	\begin{equation}\label{eq:smallw*}
	\max_{k\in\fromto{n}}\frac{\wRk}{v_k}=\pBigO{\frac{1}{n}}.
	\protect\footnotemark
	\end{equation}\footnotetext{Using the definition of $\PBigO$, \cref{eq:smallw*} reads as:
		\begin{equation*}\label{footnote:smallw*}
		\forall \rho\in (0,1):\exists C\in \Rp: \exists n_0\in \N: \forall n>n_0:\PP[\max_{k\in\fromto{n}}<C\frac{1}{n}]>\rho.
		\end{equation*}
	}
\end{lemma}
\begin{proof}
	Let $k^*\in\argmax_{k\in\fromto{n}}\frac{\wRk}{v_k}$ and thus $\frac{\wRk[k^*]}{v_{k^*}}=\max_{k\in\fromto{n}}\frac{\wRk}{v_k}$.
	W.l.o.g. assume $k^*\in\kp$. Now we consider the case $\frac{\wRk[k^*]}{v_{k^*}}>\frac{1}{n}$.
	\begin{subequations}\label{eq:long:smallw*}
		\begin{align}
		\frac{\Flpm\flpmt}{\lw}%
		&\overset{\text{\Cref{le:2}}}{\overset{\PP}{\geq}}\frac{1}{2\lw}\Fnb{\RNR}\\
		&\underset{\hphantom{\text{\Cref{le:wstetig}}}}\geq\frac{1}{2}\sum_{k\in\kp:\xi_k\in (\xi_{k^*}-\delta,\xi_{k^*}+\delta)\cap\Ip[\xi_{k^*}]}{\wRk}^2\\
		&\underset{\hphantom{\text{\Cref{le:wstetig}}}}=\frac{1}{2}\sum_{k\in\kp:\xi_k\in (\xi_{k^*}-\delta,\xi_{k^*}+\delta)\cap\Ip[\xi_{k^*}]}\frac{{\wRk}^2}{v_k^2}v_k^2\\
		&\overset{\text{\Cref{le:wstetig}}}{\geq}\frac{1}{8}\frac{{\wRk[k^*]}^2}{v_{k^*}^2}\sum_{k\in\kp:\xi_k\in (\xi_{k^*}-\delta,\xi_{k^*}+\delta)\cap\Ip[\xi_{k^*}]}v_k^2\\
		&\underset{\hphantom{\text{\Cref{le:wstetig}}}}{\overset{\PP}{\geq}} \frac{1}{8}\frac{{\wRk[k^*]}^2}{v_{k^*}^2}\left|(\xi_{k^*}-\delta,\xi_{k^*}+\delta)\cap\Ip[\xi_{k^*}]\right|\frac{n\gxi(\xi_{k^*})}{2}\Eco{v_k^2}{\xi_k=\xi_{k^*}}
		\\
		&\underset{\hphantom{\text{\Cref{le:wstetig}}}}{\geq} \frac{1}{8}\frac{{\wRk[k^*]}^2}{v_{k^*}^2}\min\left\lbrace \delta,\min_{I\in\mathcal{P}}\left|I\right|\right\rbrace\frac{n\gxi(\xi_{k^*})}{2}\Eco{v_k^2}{\xi_k=\xi_{k^*}}.
		\end{align}
	\end{subequations}
	Transforming inequality~\meqref{eq:long:smallw*} and using the definition $\lw:=\lambda n 2g(0)$ gives:
	\begin{equation}
	\frac{{\wRk[k^*]}^2}{v_{k^*}^2}
	\overset{\PP}{\leq} \frac{16}{n^2} \frac{\Flpm\flpmt}{\min\left\lbrace \delta,\min_{I\in\mathcal{P}}\left|I\right|\right\rbrace\gxi(\xi_{k^*})\Eco{v_k^2}{\xi_k=\xi_{k^*}}\lambda 2g(0)}.
	\end{equation}
	Taking the square root of both sides, bounding \gxi\ with its minimum\footnote{\Cref{as:truncatedg}\ref{item:truncatedg} and \ref{item:reziprokdensityIsSmooth} guarantee that $\min_{x\in\supp(g)}\gxi(x)>0$.} and taking $\Eco{v_k^2}{\xi_k=\xi_{k^*}}
	\geq\min_{x\in\supp(\gxi)}\Eco{v_k^2}{\xi_k=x}
	=C_{\text{mom}}>0$ (which holds because of \Cref{as:generalloss}\ \ref{item:vkDistrSmooth}), we get:
	\begin{equation}
	\frac{{\wRk[k^*]}}{v_{k^*}}
	\overset{\PP}{\leq} \frac{4}{n} \left(\frac{\Flpm\flpmt}{\min\left\lbrace \delta,\min_{I\in\mathcal{P}}\left|I\right|\right\rbrace \min_{x\in\supp(g)}\gxi(x)C_{\text{mom}}\lambda 2g(0)}\right)^{\frac{1}{2}}.
	\end{equation}
	This proves statement~\meqref{eq:smallw*} by choosing $C$ from \cref{footnote:smallw*} as:
	\begin{equation}
	C:= 4 \left(\frac{\Flpm\flpmt}{\min\left\lbrace \delta,\min_{I\in\mathcal{P}}\left|I\right|\right\rbrace \min_{x\in\supp(g)}\gxi(x)C_{\text{mom}}\lambda 2g(0)}\right)^{\frac{1}{2}}.
	\end{equation}
\end{proof}

\begin{lemma}[\ref{itm:3}]\label{le:3}
	For any $\lambda>0$, we have
	\begin{equation}\label{eq:statement3}
	\plim\sobnorm[\RNR-\fwR]=0,
	\end{equation}
	with $\lw$ as defined in \Cref{thm:ridgeToSpline}.
\end{lemma}
\begin{proof}
	By \Cref{le:poincare} (as $\RNR, \fwR$ are zero outside of $\supp(g)+\supp(\kapx)$ like described in \Cref{rem:compactSupp}), we only need to show that for all $\epsilon>0$:
	\begin{displaymath}
	\lim_{n\to\infty}\PP[{\supnorm[{\RNR}^{'}-{\fwR}^{'}]<\epsilon}]=1.
	\end{displaymath}
	W.l.o.g. it is sufficient to prove:
	\begin{displaymath}
	\lim_{n\to\infty}\PP[{\supnorm[{\RNRp}^{'}-{\fwRp}^{'}]<\epsilon}]=1.
	\end{displaymath}
	For every $x\in K$ and $\omega\in\Omega$, using the \Cref{def:smoothRSNappr} of \fwRp\ we have
	\begin{align*}
	{\RNRp}^{'}(x)-{\fwRp}^{'}(x)&={\RNRp}^{'}(x)-\left({\RNRp}^{'} * \kapx\right) (x) \\
	&={\RNRp}^{'}(x)-\int_{{\Iwp[x]}}{\RNRp}^{'}(t)\frac{1}{|\Iwp[x]|}\,dt
	\\
	&=\int_{{\Iwp[x]}}{\RNRp}^{'}(x)\frac{1}{|\Iwp[x]|}\,dt-\int_{{\Iwp[x]}}{\RNRp}^{'}(t)\frac{1}{|\Iwp[x]|}\,dt
	\\
	&=\int_{{\Iwp[x]}}\left({\RNRp}^{'}(x)-{\RNRp}^{'}(t)\right)\frac{1}{|\Iwp[x]|}\,dt.
	\end{align*}
	Using the definition of $\RNRp$ we get:
	\begin{equation}
	{\RNRp}^{'}(x)=\sum_{k\in\kp:\xi_k<x}\wRk v_k
	\end{equation}
	and hence with $r_n:=\frac{1}{2\sqrt{n}\gxi(x)}$ and $\lip[x],\uip[x]$ as in \Cref{def:partition} we get after some algebraic calculations that
	\begin{align*}
	{\RNRp}^{'}(x)-{\fwRp}^{'}(x)
	=&\hphantom{-}\sum_{k\in\kp:\max\set{x-r_n, \lip[x]}<\xi_k<x}\wRk v_k \int_{\max\set{x-r_n,\lip[x]}}^{\xi_k}\frac{1}{|\Iwp[x]|} ds\\
	\hphantom{=}&-\sum_{k\in\kp:x<\xi_k<\min\set{x+r_n,\uip[x]}}\wRk v_k \int_{\xi_k}^{\min\set{x+r_n,\uip[x]}}\frac{1}{|\Iwp[x]|}ds=\\
	=&\hphantom{-}\sum_{k\in\kp:\max\set{x-r_n, \lip[x]}<\xi_k<x}\frac{\wRk}{v_k}v_k^2 \frac{|\xi_k -\max\set{x-r_n, \lip[x]}|}{|\Iwp[x]|}\\
	\hphantom{=}&-\sum_{k\in\kp:x<\xi_k<\min\set{x+r_n,\uip[x]}} \frac{\wRk}{v_k}v_k^2 \frac{|\min\set{x+r_n,\uip[x]}-\xi_k|}{|\Iwp[x]|}%
	\end{align*}
	Thus, we can use the triangle inequality\footnote{Actually, one could use a much tighter bound than the triangle inequality used in inequality \meqref{subeq:3:triangleIneq}, because in asymptotic expectation the positive and negative summands would cancel each other instead of adding up.} to get:
	\begin{subequations}
		\begin{align}
		\left|{\RNRp}^{'}(x)-{\fwRp}^{'}(x)\right|
		&\underset{\hphantom{\text{\Cref{le:smallw*}}}}\leq%
		\sum_{k\in\kp:\max\set{x-r_n, \lip[x]}<\xi_k<\min\set{x+r_n,\uip[x]}} \left|\frac{\wRk}{v_k}v_k^2\right|\label{subeq:3:triangleIneq}\\
		&\underset{\hphantom{\text{\Cref{le:smallw*}}}}\leq
		\max_{k\in\kp}\left|\frac{\wRk}{v_k}\right|\sum_{k\in\kp:\max\set{x-r_n,\lip[x]}<\xi_k<\min\set{x+r_n,\uip[x]}} v_k^2\\
		&\overset{\text{\Cref{le:smallw*}}}{\leq}\pBigO{\frac{1}{n}}\pBigO{\sqrt{n}}
		=\pBigO{\frac{1}{\sqrt{n}}}
		\end{align}
		uniformly in $x$ on $\supp(\gxi)$ and thus on $K$ (since outside of $\supp(\gxi)+(-r_n,r_n)$ both functions and there derivatives are zero). 
	\end{subequations}

\end{proof}

\begin{lemma}[\ref{itm:4}]\label{le:4}
	For any $\lambda>0$, we have
	\begin{equation}\label{eq:statement4}
	\plim \left|\Fnb{\RNR}-\Flpm\fwRpmt\right| =0,
	\end{equation}
	with $\lw$ as defined in \Cref{thm:ridgeToSpline}.
\end{lemma}
\begin{proof}
	\Cref{le:lossconvergence,le:3} combined show that
	\begin{equation*}
	\plim\left|\Ltrb{\RNR}-\Ltr(\fwRp+\fwRm)\right|=0.
	\end{equation*}
	So it is sufficient to show:
	\begin{equation}\label{eq:4:penalty}
	\plim\left|\lw \twonorm[\wR]^2-\lambda 2g(0) \left(
	\int_{\supp (g)} \frac{\left( {\fwRp}^{''}(x) \right)^2}{g(x)} dx
	+\int_{\supp (g)} \frac{\left( {\fwRm}^{''}(x) \right)^2}{g(x)} dx
	\right)\right|=0.
	\end{equation}
	Since $\twonorm[\wR]^2=\sum_{k\in\kp}{\wRk}^2 + \sum_{k\in\km}{\wRk}^2$, %
	we  restrict ourselves to proving
	\begin{equation}\label{eq:4:penaltyp}
	\plim\left|\lw\sum_{k\in\kp}{\wRk}^2-\lambda 2g(0)\int_{\supp(\gxi)}\frac{\left({\fwRp}^{''}(x)\right)^2}{g(x)}\, dx\right|=0.
	\end{equation}
	Using the \Cref{def:smoothRSNappr} of $\fwRp$ we get that
	\begin{subequations}\label{eq:4:fwRp''}
		\begin{align}
		{\fwRp}^{''}(x)&\overset{\text{\Cref{def:smoothRSNappr}}}{=} \sum_{k\in\kp:\xi_k\in \Iwp[x]}\frac{1}{|\Iwp[x]|}\wRk v_k\\
		&\underset{\hphantom{\text{\Cref{def:smoothRSNappr}}}}=\sum_{k\in\kp:\xi_k\in \Iwp[x]}\frac{1}{|\Iwp[x]|}\frac{\wRk}{v_k} v_k^2\\
		&\underset{\hphantom{\text{\Cref{def:smoothRSNappr}}}}{\overset{\text{\Cref{le:wstetig}}}{\approx}}\left(\frac{\wRk[l_x]}{v_{l_x}}\pm\frac{\epsilon}{n}\right)\sum_{k\in\kp:\xi_k\in \Iwp[x]}\frac{1}{|\Iwp[x]|} v_k^2\\
		&\underset{\hphantom{\text{\Cref{def:smoothRSNappr}}}}{{\approx}}\left(\frac{\wRk[l_x]}{v_{l_x}}\pm\frac{\epsilon}{n}\right) \left( 1\overset{\PP}{\pm}\epsilon_1 \right) \PP[v_k>0] n\gxi(x)\left( \Eco{v_k^2}{\xi_k=x} \overset{\PP}{\pm}\epsilon_2 \right)\\
		&\underset{\hphantom{\text{\Cref{def:smoothRSNappr}}}}{\overset{\text{\Cref{le:smallw*}}}{\approx}}\frac{\wRk[l_x]}{v_{l_x}}\PP[v_k>0] n\gxi(x) \Eco{v_k^2}{\xi_k=x} \overset{\PP}{\pm}\epsilon_3 
		\end{align}
	\end{subequations}
	uniformly in $x$ on $K$ for any $l_x$ satisfying $l_x \in\kp:\xi_l\in\Iwp[x]\ \forall x\in \supp(\gxi)$. %
	Therefore we can plug this into the right-hand term of \cref{eq:4:penaltyp}:
	\begin{align*}
	\lambda 2g(0)\int_{\supp(\gxi)}\frac{\left({\fwRp}^{''}(x)\right)^2}{g(x)}\, dx
	&\approx \lambda 2g(0)\int_{\supp(\gxi)}\dfrac{\left(  \frac{\wRk[l_x]}{v_{l_x}}\PP[v_k>0] n\gxi(x) \Eco{v_k^2}{\xi_k=x} \overset{\PP}{\pm}\epsilon_3    \right)^2}{g(x)}\, dx\\
	&\approx \underbrace{\lambda 2g(0)\int_{\supp(\gxi)}\dfrac{\left(  \frac{\wRk[l_x]}{v_{l_x}}\PP[v_k>0] n\gxi(x) \Eco{v_k^2}{\xi_k=x}     \right)^2}{g(x)}\, dx}_{\displaystyle = \frac{\lw n}{2}\int_{\supp(\gxi)}\left(  \frac{\wRk[l_x]}{v_{l_x}}\right)^2 \gxi(x) \Eco{v_k^2}{\xi_k=x} \, dx } \overset{\PP}{\pm}\epsilon_4
	\end{align*}
	by uniformity of approximation~\meqref{eq:4:fwRp''} and by using the definitions of $\lw:=\lambda n 2g(0)$ and $g(x):=\gxi(x)\Eco{v_k^2}{\xi_k=x}\frac{1}{2}$.

	In the next steps we show that the left-hand term of \cref{eq:4:penaltyp} converges to the same term as the right-hand side did.
	Therefore, we choose a suitable partition $\mathcal{P}_\delta:=\set{[\ell_1,\ell_2),\dots,[\ell_{m-1},\ell_m)}$ of $\supp(\gxi)$,
	i.e., $\mathcal{P}_\delta$ is a refinement of $\mathcal{P}$ from \Cref{def:partition} and no interval in $\mathcal{P}_\delta$ is longer than $\delta$ or shorter than $\frac{\delta}{2}$.
	By setting $\delta=\frac{1}{\sqrt{n}}$ and letting $n$ tend to infinity, we get:
	
	\begin{align*}
	\lw\sum_{k\in\kp}{\wRk}^2
	&\underset{\hphantom{\text{\Cref{le:smallw*}}}}{=}
	\lw\sum_{j=1}^m\left(\sum_{\stackrel{k\in\kp}{\xi_k\in[\ell_j, \ell_{j+1})}}\left(\frac{\wRk}{v_k}\right)^2 v_k^2\right)\\
	&\overset{\text{\Cref{le:wstetig}}}{\approx}\lw\sum_{j=1}^m\left(\left(\frac{\wRk[l_{\ell_j}]}{v_{l_{\ell_j}}}\pm\frac{\epsilon_5}{n}\right)^2
	\vphantom{\sum_{\stackrel{k\in\kp}{\xi_k\in[\ell_j, \ell_{j+1})}}}\right.
	\underbrace{\sum_{\stackrel{k\in\kp}{\xi_k\in[\ell_j, \ell_{j+1})}} v_k^2}_{\mathclap{\approx\left(1\overset{\PP}{\pm}\epsilon_6\right) \frac{n}{2}|\ell_{j+1}-\ell_j|\gxi(\ell_j)\left( \Eco{v_k^2}{\xi_k=\ell_j} \overset{\PP}{\pm}\epsilon_7 \right)}}
	\left.\vphantom{\sum_{\stackrel{k\in\kp}{\xi_k\in[\ell_j, \ell_{j+1})}}}\right)\\
	&\overset{\text{\Cref{le:smallw*}}}{\approx}\frac{\lw n}{2}\sum_{j=1}^m\left(\left(\frac{\wRk[l_{\ell_j}]}{v_{l_{\ell_j}}}\right)^2
	|\ell_{j+1}-\ell_j|\gxi(\ell_j)\left( \Eco{v_k^2}{\xi_k=\ell_j}  \right)\overset{\PP}{\pm}\epsilon_8
	\right)\\
	&\underset{\hphantom{\text{\Cref{le:smallw*}}}}{\overset{\text{Riemann}}{\approx}}\frac{\lw n}{2}\int_{\supp(\gxi)}\left(  \frac{\wRk[l_x]}{v_{l_x}}\right)^2 \gxi(x) \Eco{v_k^2}{\xi_k=x} \, dx \overset{\PP}{\pm}\epsilon_9
	\end{align*}
	This proves \cref{eq:statement4}.
\end{proof}

\begin{lemma}\label{le:PfuncStronglyConvex}
    Let $K\subset\R$ be a compact interval and consider the Banach space \newline$(\T,\|\cdot\|_{\Wkp[K]{2}{2}})$, with norm $\|(f_+, f_-)\|_{\Wkp[K]{2}{2}}:=\|f_+\|_{\Wkp[K]{2}{2}}+\|f_-\|_{\Wkp[K]{2}{2}}$ for any $(f_+, f_-)\in\T$.The penalty term of \Flpm, given by $\Pgpm:\T\to\Rpz$,
	\begin{equation}\label{eq:Plpm}
	\Pgpm\left(f_+,f_-\right):= 2g(0) \left(
	\int_{\supp (g)} \frac{\left( {f_+}^{''}(x) \right)^2}{g(x)} dx
	+\int_{\supp (g)} \frac{\left( {f_-}^{''}(x) \right)^2}{g(x)} dx
	\right)
	\end{equation}
	is strongly convex w.r.t. $\|\cdot\|_{\Wkp[K]{2}{2}}$. Moreover, if $\supp(g)\subset K$, then $\Pgpm$ is continuous w.r.t. $\|\cdot\|_{\Wkp[K]{2}{2}}$.
\end{lemma}
\begin{proof}
    Let $K\subset\R$ be a compact interval with diameter $|K|$, $\fonepmt,\ftwopmt\in\T$ and define the tuple
	\begin{equation}
	\upmt:=\ftwopmt- \fonepmt\in\T
	\end{equation} as the component-wise difference. Note that $t\fonepmt+(1-t)\ftwopmt\in\T$ for every $t\in[0,1]$. Since $\Pgpm$ is a quadratic form, we get for any $t\in[0,1]$ with the help of some algebraic calculations
	\begin{equation}\label{eq:stronglyconvexP}
	\Pgpmb{t\fonepmt+(1-t)\ftwopmt}=
	t\Pgpm{\fonepmt}+(1-t)\Pgpm{\ftwopmt}-t(1-t)\Pgpm{\upmt}.
	\end{equation}
	Moreover, we have
	\begin{align}
	\Ltwonorm[{\up}^{''}]^2\leq\frac{max_{x\in\supp(g)}g(x)}{ 2g(0)}\Pgpm\upmt,
	\end{align}
	since $\upmt\in\T$ has zero second derivative outside $\supp(g)$.
	Applying the Poincar\'e-typed \Cref{le:poincare} twice (first on ${\up}^{''}$ then on ${\up}^{'}$) yields 
	\begin{align}
	\Ltwonorm[\up^{'}]^2&\le |K|^2\frac{max_{x\in\supp(g)}g(x)}{ 2g(0)}\Pgpm\upmt,\\
	\Ltwonorm[\up]^2&\le |K|^4\frac{max_{x\in\supp(g)}g(x)}{ 2g(0)}\Pgpm\upmt,
	\end{align}
	as $\unpmt\in\T$ satisfies the boundary conditions at $\Cgl$ (cp. \Cref{rem:compactSupp}) because of the compact support of $g$.
	Analogously,
		\begin{align}
	\Ltwonorm[\um^{'}]^2&\le |K|^2\frac{max_{x\in\supp(g)}g(x)}{ 2g(0)}\Pgpm\upmt,\\
	\Ltwonorm[\um]^2&\le |K|^4\frac{max_{x\in\supp(g)}g(x)}{ 2g(0)}\Pgpm\upmt,
	\end{align}
    and hence
	\begin{align}
	\Ltwonorm[(\up,\um)]^2&\le |K|^4\frac{max_{x\in\supp(g)}g(x)}{ g(0)}\Pgpm\upmt,\\
		\left\|(\up,\um)\right\|_{\Wkp[K]{1}{2}}^2&\le \left(|K|^4+|K|^2\right)\frac{max_{x\in\supp(g)}g(x)}{ g(0)}\Pgpm\upmt,\\
	\left\|(\up,\um)\right\|_{\Wkp[K]{2}{2}}^2&\le \left(|K|^4+|K|^2+1\right)\frac{max_{x\in\supp(g)}g(x)}{ g(0)}\Pgpm\upmt.\label{eq:W22leP}
	\end{align}
	Combining this with \meqref{eq:stronglyconvexP} results in
	\begin{subequations}\label{eq:proofStrongyConvexFInalConstant}
		\begin{align}
	&\Pgpmb{t\fonepmt+(1-t)\ftwopmt}\\
	\leq &\,
	t\Pgpm{\fonepmt}+(1-t)\Pgpm{\ftwopmt}\\
	&\qquad\qquad-\frac{t(1-t) g(0)}{\left(|K|^4+|K|^2+1\right)\max_{x\in\supp(g)}g(x)}\left\|(\up,\um)\right\|_{\Wkp[K]{2}{2}}^2.
	\end{align}
	\end{subequations}
	Thus, $\Pgpm$ is strongly convex with parameter $2g(0)/\left(\left(|K|^4+|K|^2+1\right)\max_{x\in\supp(g)}g(x)\right)$.
	To show continuity, note that we have
	\begin{subequations}\label{eq:7secondDerivativeunpIsHigh}
	\begin{align}
	\Pgpm\upmt&\leq\frac{ 2g(0)}{\min_{x\in\supp(g)}g(x)}\left(\Ltwonorm[{\up}^{''}]^2+\Ltwonorm[{\um}^{''}]^2\right)\\
	&\leq \frac{ 2g(0)}{\min_{x\in\supp(g)}g(x)}\|\upmt\|_{\Wkp[K]{2}{2}}^2,
	\end{align}
	\end{subequations}
	since $\upmt\in\T$ has zero second derivative outside $\supp(g)$ and $\supp(g)\subset K$. Let $\fonepmt\in\T$. Then, for any $\ftwopmt\in\T$
	\begin{subequations}
	\begin{align}
	    |\Pgpm\ftwopmt-\Pgpm\fonepmt|&=2g(0) \Bigg(
	\int_{\supp (g)} \frac{\left( {f_+^{2^{''}}-f_+^{1^{''}}} \right)\left( {f_+^{2^{''}}+f_+^{1^{''}}} \right)(x)}{g(x)} dx\\
	& 
	\qquad\qquad+\int_{\supp (g)} \frac{\left( {f_-^{2^{''}}-f_-^{1^{''}}}\right)\left( {f_-^{2^{''}}+f_-^{1^{''}}}\right)(x) }{g(x)} dx
	\Bigg)\\
	&\le \sqrt{\Pgpm\left(\ftwopmt-\fonepmt\right)}\sqrt{\Pgpm\left(\ftwopmt+\fonepmt\right)}\\
	&\le \sqrt{\Pgpm\upmt}\sqrt{8\Pgpm\fonepmt+2\Pgpm\upmt}\\
	&\le \sqrt{\frac{ 16g(0)\Pgpm\fonepmt}{\min_{x\in\supp(g)}g(x)}}\|\upmt\|_{\Wkp[K]{2}{2}}\\
	&\quad+\frac{ 2^{\frac{3}{2}}g(0)}{\min_{x\in\supp(g)}g(x)}\|\upmt\|_{\Wkp[K]{2}{2}}^2\\
	&=c_{\fonepmt}\|\upmt\|_{\Wkp[K]{2}{2}}+c\|\upmt\|_{\Wkp[K]{2}{2}}^2
	\end{align}
	\end{subequations}
	with positive constants $c_{\fonepmt}$ and $ c$, where we employed the Cauchy-Schwarz inequality and \meqref{eq:7secondDerivativeunpIsHigh}.
	Thus, for any $\epsilon>0$ we achieve $	|\Pgpm\ftwopmt-\Pgpm\fonepmt|<\epsilon$ for every $\ftwopmt$ such that 
	\[\|\ftwopmt-\fonepmt\|_{\Wkp[K]{2}{2}}<\delta\]
	with $\delta := \sqrt{\frac{c_{\fonepmt}^2}{4c^2}+\frac{\epsilon}{c}}-\frac{c_{\fonepmt}}{2c}$. Hence we have shown continuity at an arbitrary $\fonepmt\in\T$.

\end{proof}

\begin{lemma}\label{le:compactnessoflevelset}
    For any $c>0$, the sub-level set $\mathcal{K}:=\left\{(f^+,f^-)\in\T:\Pgpm(f^+,f^-)\leq c\right\}$ is sequentially compact w.r.t. $\Wkp[K]{1}{\infty}$ for any compact interval $K\subset\R$.
\end{lemma}
\begin{proof}
    Note that $\Wkp[K]{1}{2}$ can be compactly embedded into $L^{\infty}(K)$.
    This holds, since by Morrey's inequality (see \citep{EvansPDE}) $\Wkp[K]{1}{2}$ can be continuously embedded into the Hölder-space $C^{0, 0.5}$, $C^{0, 0.5}$ embeds compactly in $C^{0, \alpha}$, for $\alpha<\frac{1}{2}$, and $C^{0, \alpha}$ embeds continuously into $L^{\infty}(K)$. In total we thus have the compact embedding\footnote{Note that the composition of compact and continuous embeddings is itself compact.}
    \begin{align*}
    \Wkp[K]{1}{2}\subset C^{0, 1/2}{\subset\subset}C^{0, \alpha}\subset L^{\infty}(K).
    \end{align*}
    To show then that $\Wkp[K]{2}{2}$ can be compactly embedded into $\Wkp[K]{1}{\infty}$, we first note that $\Wkp[K]{2}{2}$ can be continuously embedded into $\Wkp[K]{1}{\infty}$ by Sobolev embedding theorem.
    Then we show that for every bounded sequence~$f^n$ in $\Wkp[K]{2}{2}$, there exists a sub-sequence $f^{\tilde{\tilde{n}}}$ that converges in $\Wkp[K]{1}{\infty}$:
    Since $f^{n}$ is bounded in $\Wkp[K]{2}{2}\subset\Wkp[K]{1}{2}$, there exists a convergent sub-sequence $f^{\tilde{n}}$ that converges in $L^{\infty}(K)$ (because of $\Wkp[K]{1}{2}{\subset\subset} L^{\infty}(K)$).
    And since $(f^{\tilde{n}})^{'}$ is bounded in $\Wkp[K]{1}{2}$, there exists a convergent sub-sub-sequence $(f^{\tilde{\tilde{n}}})^{'}$ that converges in $L^{\infty}(K)$ (because of $\Wkp[K]{1}{2} {\subset\subset} L^{\infty}(K)$).
    Thus, $f^{\tilde{\tilde{n}}}$ converges in $\Wkp[K]{1}{\infty}$.
    Since $\Pgpm$ is strongly convex (\Cref{le:PfuncStronglyConvex}), $\mathcal{K}$ is bounded in $\Wkp[K]{2}{2}$ (or this be can be seen from \meqref{eq:W22leP}).
    Thus, the $\Wkp[K]{2}{2}$-bounded set $\mathcal{K}$ is sequentially compact w.r.t.\ $\Wkp[K]{1}{\infty}$.
    The same holds true for finite products of the spaces.

\end{proof}
\begin{lemma}[\ref{itm:7}]\label{le:7}
	For any sequence of tuples of functions\footnote{Precisely speaking $\fnpmt$ is a random variable $\fnpmt:\Om\to\T$.} $\fnpmt\in \T$ satisfying
	\begin{equation}\label{eq:condition7}
	\plim \Flpm\fnpmt =\min_{\T}\Flpm
	\end{equation}
	it holds that
	\begin{equation}\label{eq:statement7}
	\plim d_{\Wkp[K]{1}{\infty}}\left(\left(\fnp,\fnm\right),
	\argmin_{\T}\Flpm\right) =0,
	\end{equation}
	i.e., $\forall\epsilon>0:\forall\rho\in(0,1):\exists n_0\in\N:\forall n>n_0$
	such that
	\begin{equation}
	    \PP\left[\exists\flpmt\in\argmin_{\T}\Flpm:\sobnorm[{\left(\fnp,\fnm\right)-\flpmt }]<\epsilon\right]>\rho.
	\end{equation}
	
\end{lemma}
\begin{proof}
Assume on the contrary that
$\exists\epsilon>0:\exists\rho\in(0,1):\forall n_0\in\N:\exists n>n_0$
	such that
	\begin{equation}\label{eq:farAwayFromArgMinINProbability}
	    \PP\left[\forall\flpmt\in\argmin_{\T}\Flpm:\sobnorm[{\left(\fnp,\fnm\right)-\flpmt }]\ge\epsilon\right]>(1-\rho).
	\end{equation}
This implies the existence of a sub-sequence $\left(f_+^{\tilde{n}},f_-^{\tilde{n}}\right)$ of which every element satisfies \meqref{eq:farAwayFromArgMinINProbability}. We further denote by $\Omega_\text{\meqref{eq:farAwayFromArgMinINProbability}}(n)$ the set 
\[\Omega_\text{\meqref{eq:farAwayFromArgMinINProbability}}(n):=\Set{\omega\in\Omega|\forall\flpmt\in\argmin_{\T}\Flpm:\sobnorm[{\left(\fnp,\fnm\right)-\flpmt }]\ge\epsilon}.\]
Moreover, we define \[\Om_\text{\meqref{eq:condition7}}(n):=\Set{\omega\in\Omega| \left|\Flpm\fnpmt -\min_{\T}\Flpm\right|<\epsilon_\text{\meqref{eq:condition7}}}.\]
\Cref{eq:condition7} implies that for every $\epsilon_\text{\meqref{eq:condition7}}>0$ and $\rho_\text{\meqref{eq:condition7}}\in(0,1)$, $\PP\left[\Om_\text{\meqref{eq:condition7}}(n)\right]>\rho_\text{\meqref{eq:condition7}}$ for all $n$ large enough. 
Since $\rho_\text{\meqref{eq:condition7}}$ can be chosen arbitrarily close to one and thus $\PP[\Om_\text{\meqref{eq:condition7}}(\tilde{n})]+\PP[\Omega_\text{\meqref{eq:farAwayFromArgMinINProbability}}(\tilde{n})]=\rho_\text{\meqref{eq:condition7}}+\rho>1$ for $n$ large enough. Therefore, by the axiom of choice, we may consider in what follows an element $\omega(\tilde{n})\in\Om_\text{\meqref{eq:condition7}}(\tilde{n})\cap \Omega_\text{\meqref{eq:farAwayFromArgMinINProbability}}(\tilde{n})\neq\emptyset$.
Define the set $\mathcal{K}:=\{(f^+,f^-)\in\T:\Pgpm(f^+,f^-)\le \frac{L(0)+{\epsilon}_\text{\meqref{eq:condition7}}}{\lambda}\}$. It then holds that
     \[\left(f_+^{\tilde{n}},f_-^{\tilde{n}}\right)(\omega(\tilde{n}))\in\mathcal{K}.\]
     
     Since $\mathcal{K}$ is sequentially compact w.r.t.~$\Wkp[K]{1}{\infty}$ by \Cref{le:compactnessoflevelset}, there exists a sub-sequence $\left(f_+^{\tilde{\tilde{n}}},f_-^{\tilde{\tilde{n}}}\right)(\omega(\tilde{\tilde{n}}))$ of $\left(f_+^{\tilde{n}},f_-^{\tilde{n}}\right)(\omega(\tilde{n}))$ which converges w.r.t.~$\Wkp[K]{1}{\infty}$ to a limit $\left(\tilde{f}_+,\tilde{f}_-\right)$ in $\mathcal{K}$. 
     By \Cref{as:generalloss} $\Ltr$ is continuous w.r.t.\ $\Wkp[K]{1}{\infty}$ and we define
     \[\lim_{\tilde{\tilde{n}}\to\infty}\Ltr\left(f_+^{\tilde{\tilde{n}}}+f_-^{\tilde{\tilde{n}}}\right)(\omega(\tilde{\tilde{n}}))=\Ltr\left(\tilde{f}_++\tilde{f}_-\right)=:L^*\]
     and
     \[\lim_{\tilde{\tilde{n}}\to\infty}\Pgpm\left(f_+^{\tilde{\tilde{n}}},f_-^{\tilde{\tilde{n}}}\right)(\omega(\tilde{\tilde{n}}))=\min_{\T} \Flpm-L^*=:P^*.\]
     We now proceed to show that $\left(f_+^{\tilde{\tilde{n}}},f_-^{\tilde{\tilde{n}}}\right)(\omega(\tilde{\tilde{n}}))$ is a $\Wkp[K]{2}{2}$-Cauchy sequence.
     By continuity of $\Ltr$ there exists a $\delta>0$ such that
     \[\Ltr\left(\frac{\left(f_+^{\tilde{\tilde{n}}},f_-^{\tilde{\tilde{n}}}\right)(\omega(\tilde{\tilde{n}}))+\left(f_+^{\tilde{\tilde{m}}},f_-^{\tilde{\tilde{m}}}\right)(\omega(\tilde{\tilde{m}}))}{2}\right)<L^*+\epsilon_L,\]
     for $\tilde{\tilde{n}},\tilde{\tilde{m}}$ large enough such that 
     \[\left\|\frac{\left(f_+^{\tilde{\tilde{n}}},f_-^{\tilde{\tilde{n}}}\right)(\omega(\tilde{\tilde{n}}))+\left(f_+^{\tilde{\tilde{m}}},f_-^{\tilde{\tilde{m}}}\right)(\omega(\tilde{\tilde{m}}))}{2}-\left(\tilde{f}_+,\tilde{f}_-\right) \right\|_{\Wkp[K]{1}{\infty}}<\delta_L.\]
     By $M$-strong convexity\footnote{In \meqref{eq:proofStrongyConvexFInalConstant}, one can see the explicit form of $M:=2g(0)/\left(\left(|K|^4+|K|^2+1\right)\max_{x\in\supp(g)}g(x)\right)$.} of $\Pgpm$ (see \Cref{le:PfuncStronglyConvex}) we have
     \begin{align*}
     \Pgpm\left(\frac{\left(f_+^{\tilde{\tilde{n}}},f_-^{\tilde{\tilde{n}}}\right)(\omega(\tilde{\tilde{n}}))+\left(f_+^{\tilde{\tilde{m}}},f_-^{\tilde{\tilde{m}}}\right)(\omega(\tilde{\tilde{m}}))}{2}\right)
     &\leq\frac{1}{2}\Pgpm\left(\left(f_+^{\tilde{\tilde{n}}},f_-^{\tilde{\tilde{n}}}\right)(\omega(\tilde{\tilde{n}}))\right)\\
     &+\frac{1}{2}\Pgpm\left(\left(f_+^{\tilde{\tilde{m}}},f_-^{\tilde{\tilde{m}}}\right)(\omega(\tilde{\tilde{m}}))\right)\\
     &-\frac{M}{8}\left\|\left(f_+^{\tilde{\tilde{n}}},f_-^{\tilde{\tilde{n}}}\right)(\omega(\tilde{\tilde{n}}))-\left(f_+^{\tilde{\tilde{m}}},f_-^{\tilde{\tilde{m}}}\right)(\omega(\tilde{\tilde{m}}))\right\|_{\Wkp[K]{2}{2}}\\
     &<\frac{P^*+\epsilon_P}{2}+\frac{P^*+\epsilon_P}{2}\\
     &-\frac{M}{8}\left\|\left(f_+^{\tilde{\tilde{n}}},f_-^{\tilde{\tilde{n}}}\right)(\omega(\tilde{\tilde{n}}))-\left(f_+^{\tilde{\tilde{m}}},f_-^{\tilde{\tilde{m}}}\right)(\omega(\tilde{\tilde{m}}))\right\|_{\Wkp[K]{2}{2}},
     \end{align*}
     where in the last inequality we used $\Wkp[K]{2}{2}$-continuity of $\Pgpm$ (and $\tilde{\tilde{n}},\tilde{\tilde{m}}$ again large enough)
.
Combining the inequalities above, we obtain
\begin{align*}
    &\Ltr\left(\frac{\left(f_+^{\tilde{\tilde{n}}},f_-^{\tilde{\tilde{n}}}\right)(\omega(\tilde{\tilde{n}}))+\left(f_+^{\tilde{\tilde{m}}},f_-^{\tilde{\tilde{m}}}\right)(\omega(\tilde{\tilde{m}}))}{2}\right)+\lambda \Pgpm\left(\frac{\left(f_+^{\tilde{\tilde{n}}},f_-^{\tilde{\tilde{n}}}\right)(\omega(\tilde{\tilde{n}}))+\left(f_+^{\tilde{\tilde{m}}},f_-^{\tilde{\tilde{m}}}\right)(\omega(\tilde{\tilde{m}}))}{2}\right)\\
    &<L^*+\lambda P^*+\epsilon_L+\lambda\frac{\epsilon_P+\epsilon_P}{2}-\lambda\frac{M}{8}\left\|\left(f_+^{\tilde{\tilde{n}}},f_-^{\tilde{\tilde{n}}}\right)(\omega(\tilde{\tilde{n}}))-\left(f_+^{\tilde{\tilde{m}}},f_-^{\tilde{\tilde{m}}}\right)(\omega(\tilde{\tilde{m}}))\right\|_{\Wkp[K]{2}{2}}.
\end{align*}
Thus, 
\[\left\|\left(f_+^{\tilde{\tilde{n}}},f_-^{\tilde{\tilde{n}}}\right)(\omega(\tilde{\tilde{n}}))-\left(f_+^{\tilde{\tilde{m}}},f_-^{\tilde{\tilde{m}}}\right)(\omega(\tilde{\tilde{m}}))\right\|_{\Wkp[K]{2}{2}}<\frac{8}{M}\left(\frac{1}{\lambda}\epsilon_L+\epsilon_P\right)\]
for $\tilde{\tilde{n}},\tilde{\tilde{m}}$ large enough in order not to violate optimality. Thus, $\left(f_+^{\tilde{\tilde{n}}},f_-^{\tilde{\tilde{n}}}\right)(\omega(\tilde{\tilde{n}}))$ is a $\Wkp[K]{2}{2}$-Cauchy sequence, since we can choose $\epsilon_L$ and $\epsilon_P$ arbitrarily small.

Since $\left(f_+^{\tilde{\tilde{n}}},f_-^{\tilde{\tilde{n}}}\right)(\omega(\tilde{\tilde{n}}))$ is a $\Wkp[K]{2}{2}$-Cauchy sequence, it converges to $\left(\tilde{f}_+,\tilde{f}_-\right)$ in $\Wkp[K]{2}{2}$.
Note that $\left(\tilde{f}_+,\tilde{f}_-\right)\in\T$, since $\T$ is closed w.r.t.~$\Wkp[K]{2}{2}$.
This, together with $\Wkp[K]{2}{2}$-continuity of $\Ltr$ and $\Pgpm$ implies
\[\min_{\T}\Flpm =\lim_{\tilde{\tilde{n}}\to\infty}\Ltr\left(f_+^{\tilde{\tilde{n}}}+f_-^{\tilde{\tilde{n}}}\right)(\omega(\tilde{\tilde{n}}))+\lambda\Pgpm\left(f_+^{\tilde{\tilde{n}}},f_-^{\tilde{\tilde{n}}}\right)(\omega(\tilde{\tilde{n}}))=\Flpm\left(\tilde{f}_+,\tilde{f}_-\right),\]
i.e., $\left(\tilde{f}_+,\tilde{f}_-\right)\in\argmin_{\T}\Flpm$.
This however is a contradiction to $\omega(\tilde{\tilde{n}})\in\Omega_\text{\meqref{eq:farAwayFromArgMinINProbability}}(\tilde{\tilde{n}})$ for $\tilde{\tilde{n}}$ large enough.
\end{proof}

\begin{lemma}[step 7 convex]
	Assume \Ltr\ is convex. Then, for any sequence of tuples of functions $\fnpmt\in \T$ such that
	\begin{equation}
	\plim \Flpm\fnpmt =\Flpm\flpmt,
	\end{equation}
	it follows that %
	\begin{equation}\label{eq:statement7convex}
	\plim \sobnorm[{\left(\fnp+\fnm\right)-\smash[b]{\underbrace{\flpm}_{\flp +\flm}} }] =0\vphantom{\underbrace{\flpm}_{\flp +\flm}}.
	\end{equation}
\end{lemma}
\begin{proof}
	Define the tuple of functions
	\begin{equation}\label{eq:7defu}
	\unpmt:=\flpmt - \fnpmt
	\end{equation} as the difference.
	\unimportant{Recall that the penalty term of \Flpm\ is given by
	\begin{equation}
	\Pgpm\left(f_+,f_-\right):= 2g(0) \left(
	\int_{\supp (g)} \frac{\left( {f_+}^{''}(x) \right)^2}{g(x)} dx
	+\int_{\supp (g)} \frac{\left( {f_-}^{''}(x) \right)^2}{g(x)} dx
	\right).
	\end{equation}}
	This penalty~$\Pgpm$ is obviously a quadratic form.
	Note that $\frac{\fnpmt+\flpmt}{2}\in\T$.
	Since the training loss~\Ltr\ is convex, we get the inequality
	\begin{equation}\label{eq:konvexL}
	\Ltrb{\frac{\fnp+\fnm+\flp+\flm }{2}}\leq
	\frac{\Ltrb{\fnp+\fnm }}{2}+\frac{\Ltrb{\flp+\flm }}{2}.
	\end{equation}
	Since by \Cref{le:PfuncStronglyConvex} the penalty~$\Pgpm$ is strongly convex we obtain the inequality
	\begin{equation}\label{eq:quadraticP}
	\Pgpmb{\frac{\fnpmt+\flpmt}{2}}\leq
	\frac{\Pgpm{\fnpmt}}{2}+\frac{\Pgpm{\flpmt}}{2}-\frac{\Pgpm{\unpmt}}{4}.
	\end{equation}
	Combining the inequalities~\meqref{eq:konvexL} and~\meqref{eq:quadraticP} results in
	\begin{equation}\label{eq:7importantIneq}
	\Flpmb{\frac{\fnpmt+\flpmt}{2}}\leq
	\underbrace{\frac{\Flpm{\fnpmt}+\Flpm{\flpmt}}{2}}_{%
		\textstyle\overset{\meqref{eq:condition7}}{\approx}\Flpm\flpmt\overset{\PP}{\pm}\epsilon}
	-\lambda\frac{\Pgpm{\unpmt}}{4}.
	\end{equation}
	Together with the optimality of \flpmt\ this result leads directly to
	\begin{subequations}\label{eq:7:secondlastIneq}
		\begin{align}
		\Flpm\flpmt
		&\overset{\substack{\text{optimality}\\
				\text{\Cref{def:adaptedSplineRegTuple}}}}{\leq}
		\Flpmb{\frac{\fnpmt+\flpmt}{2}}\\
		&\underset{\hphantom{\substack{\text{optimality}\\
					\text{\Cref{def:adaptedSplineRegTuple}}}}}{\overset{\text{\meqref{eq:7importantIneq}}}{\lessapprox}}
		\Flpm\flpmt\overset{\PP}{\pm}\epsilon
		-\lambda\frac{\Pgpm{\unpmt}}{4}.
		\end{align}
	\end{subequations}
	By subtracting $\left(\Flpm\flpmt-\lambda\frac{\Pgpm{\unpmt}}{4}\right)$ from both sides of ineq.~\meqref{eq:7:secondlastIneq} and multiplying by $4$ we get
	\begin{equation*}
	\lambda\Pgpm{\unpmt}\overset{\text{\meqref{eq:7:secondlastIneq}}}{\lessapprox}\overset{\PP}{\pm}4\epsilon,
	\end{equation*}
	which implies that
	\begin{equation}\label{eq:7plimunpmtIs0}
	\plim\Pgpm{\unpmt}=0.
	\end{equation}
	First, we will show that the weak second derivative~${\unp}^{''}$ converges to zero. We have
	\begin{align}\label{eq:7secondDerivativeunpIsLow}
	\Ltwonorm[{\unp}^{''}]\leq\frac{max_{x\in\supp(g)}g(x)}{ 2g(0)}\Pgpm\unpmt \quad\forall K\subseteq\R,
	\end{align}
	because $\unpmt\in\T$ has zero second derivative outside $\supp(g)$.
	Thus, $$\plim\Ltwonorm[{\unp}^{''}]=0$$ (by combining \cref{eq:7plimunpmtIs0,eq:7secondDerivativeunpIsLow}).
	This can be used to apply two times the Poincar\'e-typed \Cref{le:poincare} (first on ${\unp}^{''}$ then on ${\unp}^{'}$) to get for every compact set $K\subset\R$
	\begin{equation}
	\plim\sobnorm[\unp]=0,
	\end{equation}
	as $\unpmt\in\T$ satisfies the boundary conditions at $\Cgl$ (cp. \Cref{rem:compactSupp}) because of the compact support of $g$.
	Analogously, $\plim\sobnorm[\unm]=0$ for every compact set $K\subset\R$ and hence
	\begin{equation}\label{eq:7:limUisZero}
	\plim\sobnorm[\unp+\unm]=0.
	\end{equation}
	Thus, by the definition~\meqref{eq:7defu} of~\unpmt\ we get
	\begin{equation*}
	\plim \sobnorm[{\left(\fnp+\fnm\right)-\smash[b]{\underbrace{\flpm}_{\flp +\flm}} }]
	\overset{\text{\meqref{eq:7defu}}}{=} \plim\sobnorm[{\unp+\unm}]
	\overset{\text{\meqref{eq:7:limUisZero}}}{=}0\vphantom{\underbrace{\flpm}_{\flp +\flm}},
	\end{equation*}
	which shows \meqref{eq:statement7convex}.
\end{proof}

\subsection{Proof of \Cref{thm:GDridge}\texorpdfstring{ ($\RNwo[\wt]\to\RNRo[\frac{1}{T}]$)}{}}\label{sec:Proof:GDridge}

In this section we prove all the results (\Cref{le:GDsolution}, \Cref{rem:GDlimit} and \Cref{thm:GDridge}) presented in \Cref{sec:GradientToRidge}. These results are analogous to the results presented in \cite{bishop1995regularizationEarlyStopping,friedman2003gradient,PoggioGernalizationDeepNN2018arXiv180611379P,GidelImplicitDiscreteRegularizationDeepLinearNN2019arXiv190413262G}, but we will repeat the proofs briefly in this \namecref{sec:Proof:GDridge}.

\begin{proof}[\hypertarget{proof:le:GDsolution}{Proof of \Cref{le:GDsolution}}]
	
	We need to show that for any $\omega\in\Omega$, 
	\begin{equation}\tag{\enquote{\meqref{eq:GDsolution}}}
	\wto=-\exp\left({-2T\Xt\omb X\omb}\right)\wdago+\wdago,
	\end{equation}
	satisfies \meqref{eq:GDflow}. Let $\omega\in\Omega$ be fixed and set $y:=(\ytr_1,\ldots,\ytr_N)^\top$. Clearly, $\wt[0]=0$. Since
	\begin{displaymath}
	\nabla_w \Ltrs(\RNw)=2\Xt (Xw-y),
	\end{displaymath}
	\meqref{eq:GDflow} reads as
	\begin{equation}
	d\wt[t]=-2(\Xt X\wt[t]-\Xt y)\, dt.
	\end{equation}
	Differentiating \meqref{eq:GDsolution} we obtain
	\begin{align}
	\frac{d}{dt}\wt[t]=2X^\top X\exp\left(-2tX^\top X\right)\wdag.
	\end{align}
	Moreover, since
	\begin{align*}
	-2(\Xt X\wt[t]-\Xt y)&=2X^\top X \exp\left({-2t\Xt X}\right)\wdag-2X^\top y\wdag+2X^\top y\wdag\\
	&=2X^\top X \exp\left({-2t\Xt X}\right)\wdag
	\end{align*}
	the result follows {\transparent{\meineTranzparenz}(by the \href{https://en.wikipedia.org/w/index.php?title=Picard\%E2\%80\%93Lindel\%C3\%B6f_theorem&oldid=906035907}{Picard–-Lindel\"of theorem} the solution of linear \href{https://en.wikipedia.org/w/index.php?title=Ordinary_differential_equation&oldid=906735137}{ODE}s is unique)}.
			
		\end{proof}
		
		\begin{proof}[\hypertarget{proof:rem:GDlimit}{Proof of \Cref{rem:GDlimit}}]
			Using \href{https://en.wikipedia.org/w/index.php?title=Moore–Penrose_inverse&oldid=906563099#Singular_value_decomposition_(SVD)}{basic results} on the \href{https://en.wikipedia.org/w/index.php?title=Moore–Penrose_inverse&oldid=906563099}{Moore-Penrose pseudoinverse} \cite{ben2003generalizedPseudoInverse} and \href{https://en.wikipedia.org/w/index.php?title=Singular_value_decomposition&oldid=911708147}{singular value decomposition} it directly follows that the \href{https://en.wikipedia.org/w/index.php?title=Moore–Penrose_inverse&oldid=906563099#Minimum_norm_solution_to_a_linear_system}{minimum norm solution}~\wdag\ does not have any {\transparent{\meineTranzparenz}singular-value-}components in the \href{https://en.wikipedia.org/w/index.php?title=Singular_value_decomposition&oldid=911708147#Range,_null_space_and_rank}{null-space} of the matrix~$X$. Combining this with \href{https://en.wikipedia.org/w/index.php?title=Matrix_exponential&oldid=908188529#Diagonalizable_case}{basic knowledge} about the \href{https://en.wikipedia.org/w/index.php?title=Matrix_exponential&oldid=908188529}{matrix exponential} of diagonalizable matrices, the result follows. Since the matrix-exponential in \cref{eq:GDsolution} only preserves the null-space of $X$, every {\transparent{\meineTranzparenz}singular-value-}component outside the null-space is scaled down to zero as $T\to\infty$.
		\end{proof}
		
		\begin{proof}[\hypertarget{proof:thm:GDridge}{Proof of \Cref{thm:GDridge}}]
			First, we note that obviously
			\begin{equation}\label{eq:weightsToMinNorm}
			\tlim \wRl[\frac{1}{T}]\omb=\wdago \quad \faog
			\end{equation}
			holds by \Cref{def:minNormSol}.
			
			Secondly, the continuity of the map $(\R^n ,\twonorm )\to\Woi: w\mapsto\RNwo$ implies: \faog:
			\begin{subequations}\label{eq:twoLimits}
				\begin{align}
				\tlim\sobnorm[{ \RNRo[\frac{1}{T}] - \RNwo[\wdago] }] &= 0 \text{, because of \cref{eq:weightsToMinNorm}}\\
				\tlim\sobnorm[{ \RNwo[{\wt[T]\omb}] - \RNwo[\wdago] }] &= 0 \text{, because of \Cref{rem:GDlimit}}.
				\end{align}
			\end{subequations}
			Thirdly, by applying the triangle inequality on eqs.~\meqref{eq:twoLimits} the result~\meqref{eq:thm:GDridge} follows.
		\end{proof}

		\subsection{Proof of \Cref{cor:universal_in_prob} and \Cref{le:almost_sure_interpolation}}\label{sec:Proof:expressiveness}
		\begin{lemma}[Uniform continuity w.r.t. first-layer weights]\label{le:uniform_cont_weights}
			Let $\NN_\theta$ be a shallow neural network as introduced in \Cref{def:NN} and define $(b,v)\in \R^{n\times(d+1)}$ to be the collection of the network's first layer parameters. Then, for every $\epsilon>0$ and for any compact $K\subset \R^d$ there exists a $\delta>0$ such that,
			\begin{displaymath}
			\forall (\tilde{b}, \tilde{v})\in U_\delta(b,v):\supnorm[\sum_{k=1}^{n} w_k\sigma\left(\tilde{b}_k+\sum_{j=1}^{d}\tilde{v}_{k,j} x_j\right)-\NN_\theta]<\frac{\epsilon}{2},
			\end{displaymath}
			with
			\begin{displaymath}
			U_\delta(b,v):=\left\{ (\tilde{b}, \tilde{v})\in\R^{n\times(d+1)}\bigg|\underset{k\in\{1,\ldots,n\}}{\max}\twonorm[(b_k,v_k)-(\tilde{b}_k,\tilde{v}_k)]<\delta\right\}.
			\end{displaymath}
		\end{lemma}
		\begin{proof}
			For any $x\in K$, we have
			\begin{align*}
				\frac{\partial \NN_\theta(x)}{\partial b_k}&=w_k\sigma'\left(b_k+\sum_{j=1}^{d}v_{k,j} x_j\right),\\
				\frac{\partial \NN_\theta(x)}{\partial v_{k, i}}&=w_k\sigma'\left(b_k+\sum_{j=1}^{d}v_{k,j} x_j\right)x_i.
			\end{align*}
			Both derivatives can be bounded by above by $L:=\underset{k\in\{1,\ldots,n\}}{\max}|w_k| L_\sigma c_K$, with $L_\sigma$ the Lipschitz constant corresponding to $\sigma$ and $c_K>0$ s.t. $\twonorm[x]\le c_K \forall x\in K$ as $K$ was assumed to be compact. Since the bound $L$ is independent of $x$ and $w$, the statement follows. 
		\end{proof}
		\begin{proof}[\hypertarget{proof:cor:universal_in_prob}{Proof of \Cref{cor:universal_in_prob}}]
  The equation
    \begin{displaymath}
		\forall \epsilon\in\mathbb{R}_+:\lim\limits_{n\to\infty} 
		\PP\left[\exists w\in\mathbb{R}^n:\|\RNw-f\|_\infty<\epsilon\right]=1.
		\end{displaymath}
  can be formulated more concretely as
  \begin{displaymath}
		\forall \epsilon\in\mathbb{R}_+:\lim\limits_{n\to\infty} 
		\PP\left[\Set{\omega\in\Om  |  \exists w\in\mathbb{R}^n:\|\RNwo-f\|_\infty<\epsilon}\right]=1,
		\end{displaymath} or alternatively as
  {\small\begin{displaymath}
		\forall \epsilon\in\mathbb{R}_+:\lim\limits_{n\to\infty} 
		\mu^n\left(\Set{(\tilde{b}_k,\tilde{v}_k)_{k\in \fromto[1]{n}}\in(\R\times\R^d)^n  |  \exists w\in\mathbb{R}^n:\left\|\sum_{k=1}^{n}w_k\sigma\left(\tilde{b}_k+\langle \tilde{v}_k,\cdot\rangle\right)-f\right\|_\infty<\epsilon}\right)=1,
		\end{displaymath}}%
where $\mu^n$ denotes the $n$-fold product measure of $\mu$.

			By uniform approximation in the sense of \cite{leshno1993multilayer}, we have for any $\epsilon>0$, that there exists an $N^{\epsilon/2}\in\N$,  $\NN^{\epsilon/2}:\R^d\to\R$ with
			\begin{displaymath}
			\NN^{\epsilon/2}(x):=\sum_{k=1}^{N^{\epsilon/2}} \theta_k\sigma\left(b_k+\sum_{j=1}^{d}v_{k,j} x_j\right)
			\end{displaymath}
			with $\theta_k, b_k, v_{k,j}\in\R$ such that
			\begin{equation}\label{eq:UAT}
			\supnorm[\NN^{\epsilon/2}-f]<\frac{\epsilon}{2}.
			\end{equation}
			We now like to consider the probability that a randomly chosen vector of weights $(\tilde{b}_k, \tilde{v}_k)$ corresponding to the $k$\textsuperscript{th} neuron in the hidden layer is close to a specific weight vector $(b_i, v_i)$ of $\NN^{\epsilon/2}$. Since $\lambda^{d+1}(U_\delta(b_i,v_i))>0$ it follows from $\mu\gg\lambda^{d+1}$ that $\mu(U_\delta(b_i,v_i))>0$. Therefore,
			\begin{displaymath}
			0<p:=\underset{i\in\{1,\ldots,N^{\epsilon/2}\}}{\min}\mu(U_\delta(b_i,v_i))\le 1.
			\end{displaymath}
			The probability that none of the sampled weights $(\tilde{b}_k, \tilde{v}_k)$, $k=1,\ldots,n$ is in the $\delta$-neighborhood of a specific vector $(b_i,v_i)$ can be bounded as follows:
			\begin{align*}
			\mu^n\left(\left[\forall k\in\{1,\ldots,n\}:(\tilde{b}_k, \tilde{v}_k)\notin U_\delta(b_i,v_i)\right]\right)=\left(1-\mu(U_\delta(b_i,v_i))\right)^n\le (1-p)^n.
			\end{align*}
			This implies
			\begin{align*}
			&\mu^n\left(\underbrace{\left[\exists i\in\{1,\ldots,N^{\epsilon/2}\}:\forall k\in\{1,\ldots,n\}:(\tilde{b}_k, \tilde{v}_k)\notin U_\delta(b_i,v_i)\right]}_{=:B}\right)\\
			&=\mu^n\left(\bigcup_{i=1}^{N^{\epsilon/2}}\left[\forall k\in\{1,\ldots,n\}:(\tilde{b}_k, \tilde{v}_k)\notin U_\delta(b_i,v_i)\right]\right)\\
			&\le\sum_{i=1}^{N^{\epsilon/2}}\mu^n\left(\left[\forall k\in\{1,\ldots,n\}:(\tilde{b}_k, \tilde{v}_k)\notin U_\delta(b_i,v_i)\right]
			\right)\\
			&\le\sum_{i=1}^{N^{\epsilon/2}}(1-p)^n=(1-p)^n\cdot N^{\epsilon/2}\underset{n\to\infty}{\longrightarrow} 0.
			\end{align*}
			For every $\omega\in B^c$ define
			\begin{align*}
			\iota :\left\{1,\ldots,N^{\epsilon/2}\right\}&\to\{1,\ldots,n\},\\
			i&\mapsto\iota(i),
			\end{align*}
			with $(\tilde{b}_{\iota(i)},\tilde{v}_{\iota(i)})(\omega)\in U_\delta(b_i,v_i)$. Without loss of generality, $\iota$ is injective (choose $\delta$ small enough s.t. $U_\delta(b_i,v_i)$, $i=1,\ldots,N^{\epsilon/2}$ are disjoint). For those $\omega\in B^c$ we further define $\RNw$ as in the statement of the corollary, with trainable last layer weights
			\begin{displaymath}
			w_k:=\begin{cases}
			\theta_k, & \exists i\in\{1,\ldots,N^{\epsilon/2}\}:\iota(i)=k,\\
			0, &\nexists i\in\{1,\ldots,N^{\epsilon/2}\}:\iota(i)=k.
			\end{cases}
			\end{displaymath}
			By \Cref{le:uniform_cont_weights}, it follows that $\supnorm[\RNw-\NN^{\epsilon/2}]<\epsilon/2$ on $B^c$. Hence, an application of the triangle inequality, together with \meqref{eq:UAT} yield that
			\begin{displaymath}
			\forall \omega \in B^c:\supnorm[\RNw-f]<\epsilon.
			\end{displaymath} 
		\end{proof}
		
		\begin{proof}[\hypertarget{proof:le:almost_sure_interpolation}{Proof of \Cref{le:almost_sure_interpolation}}]
			We use the notation $\psi_{(b,v)}(\xtr)_k:=\left(\psi_{(b,v)}(\xtr_1)_k,\dots,\psi_{(b,v)}(\xtr_N)_k\right)=
			\left(
			\sigma\left({b_k}+\sum_{j=1}^{d}{v_{k,j}}\xtr_{1,j}\right),\dots,
			\sigma\left({b_k}+\sum_{j=1}^{d}{v_{k,j}}\xtr_{N,j}\right)
			\right)
			$ and note that $\psi_{(b,v)}(\xtr)_k$ are \iid{} for $k=1,\ldots,n$.
			We want to show that $\PP$-almost surely, $\{\psi_{(b,v)}(\xtr_1),\ldots, \psi_{(b,v)}(\xtr_N)\}$ are linearly independent, for then the terminal linear regression can be (uniquely in case $N=n$) solved.
			Since the column-rank of a matrix equals its row-rank, we can equivalently show that $\{\psi_{(b,v)}(\xtr)_1,\ldots, \psi_{(b,v)}(\xtr)_N\}$ are linearly independent $\PP$-almost surely.
			We show this by induction.
			First, the vector $\psi_{(b,v)}(\xtr)_1$ is almost surely non-zero by assumption and therefore almost surely linearly independent. 
			Assume then that \[\{\psi_{(b,v)}(\xtr)_1,\ldots, \psi_{(b,v)}(\xtr)_{N-1}\}\] are almost surely linearly independent. 
			Then almost surely, the linear hull \[L_{N-1}:=[\psi_{(b,v)}(\xtr)_1, \ldots,\psi_{(b,v)}(\xtr)_{N-1}]\subseteq\mathrm{range}(\psi_{(b,v)})\]
			constitutes an ${(N{-}1)}$-dimensional subspace of $\mathrm{range}(\psi_{(b,v)})$ for which \[\PP_\#(\psi_{(b,v)}(\xtr)_N)[L_{N-1}]=0,\]
			and thus $\psi_{(b,v)}(\xtr)_N\notin L_{N-1}$ $\PP$-almost surely. Thus, almost surely there exists $w\in\R^n$ such that $\sum_{k=1}^{n}w_k\psi_{(b,v)}(x_i)_k=y_i$ for all $i=1,\ldots,N$ and $n\ge N$. 
		\end{proof}
	\subsection{Existence and uniqueness of \texorpdfstring{$\flpm$ and $\RNR$}{the arg min f* and RN*}}\label{sec:proof:ExistenceAdaptedSpline}
	\begin{lemma}\label{le:existence}
	    Let $(\X,\tau)$ be a convex subset of a locally convex topological vector space and $\|\cdot\|$ be a norm on $\X$ whose induced topology is at least as fine as $\tau$ and such that $(\X, \|\cdot\|)$ is complete. Let further $L:\X\to\Rpz$ be continuous w.r.t. $\tau$, and $P:\X\to\R$ strongly convex and continuous w.r.t. $\|\cdot\|$. Moreover, for $\lambda>0$ let $\mathcal{K}:=\{f:P(f)\le \frac{L(0)}{\lambda}+P(0)\}$ be sequentially compact w.r.t. $\tau$. Then, the set
	    \begin{equation}\label{eq:argmin}
	        \argmin_{f\in\X} L(f)+ \lambda P(f)
	    \end{equation}
	    is non-empty. 
	\end{lemma}
	\begin{proof}
	    Consider a minimizing sequence of \begin{equation*}
	        \inf_{f\in\X} L(f)+ \lambda P(f)=\inf_{f\in K} L(f)+ \lambda P(f)
	    \end{equation*}
	    and denote by $(f_n)_{n\in\N}$ the $\tau$-convergent subsequence we know exists by $\tau$-compactness of $\mathcal{K}$ w.r.t.\ $\tau$ and $f^\infty%
	    $ is defined as a\footnote{If $\tau$ is not Hausdorff, the limit does not have to be unique.} limit of the subsequence~$(f_n)_{n\in\N}$. By $\tau$-continuity of $L$, for any $\epsilon>0$ there exists a $\tau$-neighbourhood $U_\epsilon\in\tau$ of $f^\infty$  s.t.\ $L(f)< L^*+\epsilon$ for every $f\in U_\epsilon$, where $L^*:=L(f^\infty)=\lim_{n\to\infty} L(f_n)$. Furthermore, there exists a convex $\tau$-neighbourhood $B\subset U_\epsilon$  of $f^\infty$ and (by $\tau$-convergence of $(f_n)_{n\in\N}$) an index $N\in\N$ s.t.\ $f_n\in B$ for all $n>N$. Note that for 
	    \[P^*:=\frac{\inf_{f\in K} L(f)+ \lambda P(f)-L^*}{\lambda},\]
	    also $L(f_n)+\lambda P(f_n)< L^*+\lambda P^*+\epsilon$ for $f_n\in B$ with $n$ large enough since $(f_n)_{n\in\N}$ is a minimizing sequence.\footnote{Note that we do not assume $P^*=P(f^\infty)$ at this point. If $\tau$ is not Hausdorff, $P(f^\infty)$ can be arbitrarily much larger than $P^*$}
	    
	    We now proceed to show that $(f_n)_{n\in\N}$ is a Cauchy sequence w.r.t.\ $\|\cdot\|$. Assume on the contrary that there exists a $\delta>0$ s.t.\ for all $M\in\N$ there exist $n,m>M$ for which $\|f_n-f_m\|>\delta$. First, by convexity of $B$, $\frac{f_n+f_m}{2}\in B$ and hence $L(\frac{f_n+f_m}{2})< L^*+\epsilon$. Second, by strong convexity of $P$ with parameter $c>0$
	    \begin{align}
	    P\left(\frac{f_n+f_m}{2}\right)&\le \frac{1}{2}P(f_m)+\frac{1}{2}P(f_n)-\frac{c}{8}\|f_n-f_m\| \\
	    &< P^*+\frac{2\epsilon}{\lambda}-\frac{c}{8}\delta,
	    \end{align}
	    and thus
	    \begin{align}
	        L\left(\frac{f_n+f_m}{2}\right)+\lambda P\left(\frac{f_n+f_m}{2}\right)&< L^*+\lambda P^*+(3\epsilon-\frac{c\lambda}{8}\delta)\\
	        &=\inf_{f\in K} L(f)+ \lambda P(f)+(3\epsilon-\frac{c\lambda}{8}\delta).
	    \end{align}
	    This however is a contradiction since $\epsilon$ was arbitrary.
	    
	    Since $(\X, \|\cdot\|)$ is assumed to be complete, there exists a $\|\cdot\|$-limit $f^*$ of $(f_n)_{n\in\N}$. Furthermore, $\tau$-continuity of $L$ implies its $\|\cdot\|$-continuity and hence
	    \begin{align}
	       L(f^{*})+\lambda P(f^{*})= \lim_{n\to\infty}L(f_{n})+\lambda P(f_{n})=\inf_{f\in\X} L(f)+ \lambda P(f).
	    \end{align}
	    Therefore, the limit $f^*\in K$ is a minimizer and $\argmin_{f\in\X} L(f)+ \lambda P(f)\neq \emptyset$.
	    
	\end{proof}
	\begin{lemma}\label{le:uniqueness}
	    Let $(\X,\tau)$ be a convex subset of a locally convex topological vector space and $\|\cdot\|$ be a norm on $\X$ whose induced topology is at least as fine as $\tau$ and such that $(\X, \|\cdot\|)$ is complete. Let further $L:\X\to\Rpz$ be convex, continuous w.r.t. $\tau$, and $P:\X\to\R$ strongly convex and continuous w.r.t. $\|\cdot\|$. Then, for any $\lambda>0$ the set
	    \begin{equation}\label{eq:argminConvexLoss}
	        \argmin_{f\in\X} L(f)+ \lambda P(f)
	    \end{equation}
	    is a singleton, i.e., there exists a unique minimizer.
	\end{lemma}
	\begin{proof}
	     Consider a minimizing sequence $(f_n)_{n\in\N}$ of \begin{equation*}
	        \inf_{f\in\X} L(f)+ \lambda P(f).
	    \end{equation*}

	    We now proceed to show that $(f_n)_{n\in\N}$ is a Cauchy sequence w.r.t.\ $\|\cdot\|$. Assume on the contrary that there exists a $\delta>0$ s.t.\ for all $M\in\N$ there exist $n,m>M$ for which $\|f_n-f_m\|>\delta$. First, by convexity of $\X$, $\frac{f_n+f_m}{2}\in \X$. Second, since the functionals $L$ and $P$ are assumed to be convex respectively strongly convex, the objective $L(\cdot)+\lambda P(\cdot)$ is strongly convex with parameter $c>0$ too. Note hat since $(f_n)_{n\in\N}$ is a minimizing sequence, for any $\epsilon>0$ \[L\left(f_n\right)+\lambda P\left(f_n\right)< \inf_{f\in \X}L(f)+\lambda P(f)+\epsilon\] for $n,m$ large enough. Thus we obtain
	    
	    \begin{align}
	        L\left(\frac{f_n+f_m}{2}\right)+\lambda P\left(\frac{f_n+f_m}{2}\right)&<\inf_{f\in \X}L(f)+\lambda P(f)+\epsilon-\frac{c}{8}\|f_n-f_m\|\\
	        &<\inf_{f\in \X}L(f)+\lambda P(f)+\epsilon-\frac{c}{8}\delta.
	    \end{align}
	    
	    This however is a contradiction since $\epsilon$ was arbitrary.
	    
	    Since $(\X, \|\cdot\|)$ is assumed to be complete, there exists a $\|\cdot\|$-limit $f^*$ of $(f_n)_{n\in\N}$.
	    
	    Moreover, by strict convexity of the objective, the minimizer is unique. 
	\end{proof}
	\begin{lemma}[Existence and Uniqueness of \textit{\aregSpl}~$\flpm$]\label{le:existenceuniquenessaregspl}
	    Let $L:L^\infty_\text{loc}\to\Rpz$ be a loss functional satisfying \Cref{as:generalloss} \ref{item:continuousL}. Furthermore, let $g$ satisfy \Cref{as:truncatedg} \ref{item:truncatedg}-\ref{item:densityIsSmooth} (replacing \gxi\ by $g$)\footnote{\label{foot:gFiniteSecondMoment}We expect that instead of assuming compact support of $g$, one could require that $g$ is the scaled probability density function of a distribution with finite second moment (cp.~\cref{eq:regCorrespondingToG} in \Cref{subsec:SimilaritWithoutSkipConnections}). This hypothesis is not proven in the present paper, but \Cref{subsec:SimilaritWithoutSkipConnections} provides some intuition on this (and probably the tools that would facilitate such a proof).}. Then, an \aregSpl~$\flpm$ exists. If in addition, $L$ is convex\footnote{\label{footnote:convexlossexistence} If we assume convexity of $L$, we could weaken \Cref{as:generalloss} \ref{item:continuousL} such that we only need to assume continuity of $L$ with respect to \Wkp[K]{2}{2} for some compact set $K$ by using \Cref{le:uniqueness} instead of \Cref{le:existence}.}, $\flpm$ is uniquely defined.
	\end{lemma}
		\begin{proof}
	Note that by \Cref{rem:adRegSplineSum} it suffices to show existence and uniqueness of the optimization problem formulated for tuples in \Cref{def:adaptedSplineRegTuple}. We set $\X:=\T$ and define $\tau$ as the product topology induced by $\sobnorm[\cdot]$ for some compact $K\subset\R$ with $\supp(g)\subset K$. Moreover, we consider the Banach space $(\T, \|\cdot\|_{\Wkp[K]{2}{2}})$, with norm $\|(f_+, f_-)\|_{\Wkp[K]{2}{2}}:=\|f_+\|_{\Wkp[K]{2}{2}}+\|f_-\|_{\Wkp[K]{2}{2}}$. By \Cref{as:generalloss} \ref{item:continuousL} $L$ is continuous w.r.t. $\tau$. By \Cref{le:PfuncStronglyConvex}, $P:=\Pgpm$ is strongly convex and continuous w.r.t. $\|\cdot\|_{\Wkp[K]{2}{2}}$. An application of \Cref{le:existence} proves the statement, since $\mathcal{K}$ is sequentially compact w.r.t.\ $\Wkp[K]{1}{\infty}$ by \Cref{le:compactnessoflevelset}.
	
	For convex loss functional $L$, the statement follows by an application of \Cref{le:uniqueness} with $\X,\tau,\|\cdot\|$ defined as before.
	\end{proof}
	\begin{lemma}[Existence and Uniqueness of \textit{\ridgeRSN}~$\RNR$]\label{le:existenceuniquenessRidgeRSN}
	    Let $L:L^\infty_\text{loc}\to\Rpz$ be a loss functional such that 
	    \begin{enumerate}[a)]
	        \item\label{itm:case:Lcontionous} $L$ satisfies \Cref{as:generalloss}\ref{item:continuousL} or 
	        \item\label{itm:case:Lconvex} $L$ is convex.
	    \end{enumerate} Then, there exists a solution to \meqref{eq:wR}, i.e., the {\ridgeRSN}~$\RNR$ is  well-defined. If $L$ is convex, $\RNR$ is uniquely defined.
	\end{lemma}
	\begin{proof}
	    \ref{itm:case:Lcontionous} If $L$ satisfies \Cref{as:generalloss}\ref{item:continuousL}, we set $\X:=\R^n$, $\tau=\|\cdot\|:=\|\cdot\|_2$ and apply \Cref{le:existence} to get that there exists a minimizer of \meqref{eq:wR}. Note that the set $\mathcal{K}$ is bounded, closed and therefore compact w.r.t.\ $\|\cdot\|_2$. Moreover, by a proof analogous to the proof of \Cref{le:uniform_cont_weights}, the map $\theta\mapsto\RN_\theta\mapsto L(\RN_\theta)$ is continuous as concatenation of continuous functions.
	    \ref{itm:case:Lconvex} In case $L$ is convex, $\theta\mapsto L(\RN_\theta)$ is convex too and thus $\theta\mapsto L(\RN_\theta)$ is continuous by \cite[Theorem 7.1.1]{MR2467621}. An application of \Cref{le:uniqueness} yields the result.
	\end{proof}	

\subsection{Proof of \Cref{ex:lossfunctionalsum}}

Before we prove \Cref{le:ex:lossfunctionalsum}, we need an auxiliary \Cref{le:help:wstetig} that would be quite easy to prove in the case of square loss \unimportant{(i.e. $\ltri(y):=(y-\ytr_i)^2$)}, but gets a bit more evolved in the case of more general forms of training losses \ltri. 

\begin{lemma}[$\ltri^{'}$ -bound]\label{le:help:wstetig}
Let $\epsilon>0$. Under the assumptions of \Cref{ex:lossfunctionalsum}, i.e., for convex and continuously differentiable loss functions $l_i:\R\to\Rpz$, $i=1,\ldots,N$, and $\Ltr$ as in \meqref{eq:Ltri}, there exists an upper bound:
$\exists \ClPrime\in\Rp:\forall n\in\N:\fao:\allIndi{N}{i}:\forall f\in\Set{f|\Ltr(f)<\Ltr(0)+\epsilon}$
\begin{equation}\label{eq:l_prime_bound}
    \left|\ltri^{'}\left(f(\xtr_i)\right)\right|\leq \ClPrime.
\end{equation}
    
\end{lemma}
\begin{proof}

Let $f\in\Set{f|\Ltr(f)<\Ltr(0)+\epsilon}$. Then $\allIndi{N}{i}:$
    \begin{equation}\label{eq:ltriRSN_leq_L0}
    \ltri \left(f(\xtr_i)\right)
    \overset{\ltri[\iota]\geq 0}{\leq} \sum_{\iota=1}^N \ltri[\iota] \left(f(\xtr_\iota)\right)
    \overset{\meqref{eq:Ltri}}{=}\Ltrb{f}
    \le \Ltrb{0}+\epsilon.
\end{equation}
In other words, $\allIndi{N}{i}$ 
\begin{equation}\label{eq:RNRinSublevelSet}
    f(\xtr_i)\overset{\meqref{eq:ltriRSN_leq_L0}}{\in}
    \ltri^{-1}\left( (-\infty,\Ltrb{0}+\epsilon\rbrack\right)
    :=\Set{y\in\R | \ltrib{y}\leq \Ltr(0)+\epsilon}
\end{equation}
lies in a certain sublevel set of \ltri.

This implies that $\allIndi{N}{i}$:
\begin{equation}\label{eq:ltriRNR_leq_sup}
         \left|\ltri^{'}\left(f(\xtr_i)\right)\right|\overset{\text{\meqref{eq:RNRinSublevelSet}}}{\leq}\sup_{y\in\ltri^{-1}\left( (-\infty,\Ltrb{0}+\epsilon\rbrack\right)}\left| \ltri^{'}(y)\right|=:c_i\in\bar{\R}.
\end{equation}
So we want to show that the right-hand side~$c_i$ of \meqref{eq:ltriRNR_leq_sup} is finite. For this, we need a better understanding of the sublevel set~$\ltri^{-1}\left( (-\infty,\Ltrb{0}+\epsilon\rbrack\right)$.

The sublevel sets of convex functions are convex (\ltri\ is convex by assumption). Convex subsets of \R\ are always intervals. Since $\ltri\in\C^1$ is continuous by assumption, the preimage of the closed set $(-\infty,\Ltr(0)+\epsilon\rbrack$ is closed.
    
    Hence, $\ltri^{-1}\left( (-\infty,\Ltrb{0}+\epsilon\rbrack\right)$ is a closed interval. There are only four types of closed intervals: $\lbrack \alpha, \beta \rbrack$, $\lbrack \alpha, \infty )$, $(-\infty, \beta \rbrack$ and $(-\infty, \infty )$\unimportant{, where $\alpha,\beta\in\R$}. As the domain of \ltri\ is unbounded and as \ltri\ is continuous, we know that $\alpha,\beta\in\ltri^{-1}\left( \Ltrb{0}+\epsilon\right)$. 

Consider these four cases for each $i\in\fromto{N}$ separately:
\begin{enumerate}[{case} 1:,ref={case} \arabic*]
    \item\label{itm:case:alpha_beta} $\ltri^{-1}\left( (-\infty,\Ltrb{0}+\epsilon\rbrack\right)=\lbrack \alpha, \beta \rbrack$ \unimportant{is compact}:\newline
    Since \ltri\ is convex, $\ltri^{'}$ is monotonically increasing. Hence, the minimum of $\ltri^{'}$ must be attained at the left boundary~$\alpha$ and the maximum at the right border~$\beta$. So, we can bound
    \begin{equation}
        c_i\unimportant{\overset{\meqref{eq:ltriRNR_leq_sup}}{:=}\sup_{y\in\ltri^{-1}\left( (-\infty,\Ltrb{0}+\epsilon\rbrack\right)}\left| \ltri^{'}(y)\right|}=\max\left\{\left| \ltri^{'}(\alpha)\right|,\left| \ltri^{'}(\beta)\right|\right\}\in\R
    \end{equation}
    as a finite number (i.e. the maximum of two finite numbers).
    
    \item\label{itm:case:alpha_infty} $\ltri^{-1}\left( (-\infty,\Ltrb{0}+\epsilon\rbrack\right)=\lbrack \alpha, \infty )$ \unimportant{is not compact}:\newline
    This case allows to imply that
    \begin{equation}\label{eq:ltri_leq_0}
      \ltri^{'}(y)\leq 0\quad \forall y \in\R\unimportant{\supseteq\ltri^{-1}\left( (-\infty,\Ltrb{0}+\epsilon\rbrack\right)},  
    \end{equation}
    because of the following contraposition:\newline
    Assume $\exists y^+\in\R: \ltri^{'}(y^+)>0$ then, $\forall y \in \lbrack y^+,\infty): \ltri^{'}(y)\geq\ltri^{'}(y^+)$, because of mo\-no\-to\-nicity. Then $\forall y \in \lbrack y^+,\infty): \ltri(y)\geq\ltri(y^+)+(y-y^+)\ltri^{'}(y^+)$, and further $$\left(y^+ + \frac{\Ltr(0)+\epsilon-\ltri(y^+)}{\ltri^{'}(y^+)},\infty\right)\cap \ltri^{-1}\left( (-\infty,\Ltrb{0}+\epsilon\rbrack\right)=\emptyset,$$ which would contradict the assumption of \ref{itm:case:alpha_infty}. This contraposition has proven ineq.~\meqref{eq:ltri_leq_0}.\hfill\newline
    With the help of ineq.~\meqref{eq:ltri_leq_0} we can bound
    \begin{equation}\label{eq:ci:alpha_infty}
        c_i\unimportant{\overset{\meqref{eq:ltriRNR_leq_sup}}{:=}\sup_{y\in\ltri^{-1}\left( (-\infty,\Ltrb{0}+\epsilon\rbrack\right)}\left| \ltri^{'}(y)\right|}\overset{\meqref{eq:ltri_leq_0}}{=}\left|\inf_{y\in\ltri^{-1}\left( (-\infty,\Ltrb{0}+\epsilon\rbrack\right)} \ltri^{'}(y)\right|\overset{\text{monotonicity}}{=}\left| \ltri^{'}(\alpha)\right|\in\R.
    \end{equation}
    
    \item $\ltri^{-1}\left( (-\infty,\Ltrb{0}+\epsilon\rbrack\right)=(-\infty, \beta \rbrack$ \unimportant{is not compact}:\newline
    Analogously to \meqref{eq:ltri_leq_0} we get
    \begin{equation}\label{eq:ltri_geq_0}\tag{\ref*{eq:ltri_leq_0}\textsubscript{-}}
      \ltri^{'}(y)\geq 0\quad \forall y \in\R\unimportant{\supseteq\ltri^{-1}\left( (-\infty,\Ltrb{0}+\epsilon\rbrack\right)},  
    \end{equation}
    which implies analogously to \meqref{eq:ci:alpha_infty} that we can bound
    \begin{equation}\label{eq:ci:infty_beta}\tag{\ref*{eq:ci:alpha_infty}\textsubscript{-}}
        c_i\unimportant{\overset{\meqref{eq:ltriRNR_leq_sup}}{:=}\sup_{y\in\ltri^{-1}\left( (-\infty,\Ltrb{0}+\epsilon\rbrack\right)}\left| \ltri^{'}(y)\right|}\overset{\meqref{eq:ltri_geq_0}}{=}\left|\sup_{y\in\ltri^{-1}\left( (-\infty,\Ltrb{0}+\epsilon\rbrack\right)} \ltri^{'}(y)\right|\overset{\text{monotonicity}}{=}\left| \ltri^{'}(\beta)\right|\in\R.
    \end{equation}
    
    \item\label{itm:case:infty_infty} $\ltri^{-1}\left( (-\infty,\Ltrb{0}+\epsilon\rbrack\right)=(-\infty, \infty )$ \unimportant{is not compact}:\newline
    Analogously to \meqref{eq:ltri_leq_0} and \meqref{eq:ltri_geq_0} we get
    \begin{equation}\label{eq:ltri_is_0}\tag{\ref*{eq:ltri_leq_0}\textsubscript{0}}
      \ltri^{'}(y) = 0\quad \forall y \in\R\unimportant{\supseteq\ltri^{-1}\left( (-\infty,\Ltrb{0}+\epsilon\rbrack\right)},  
    \end{equation}
    which directly implies that we can bound
    \begin{equation}\label{eq:ci:infty_infty}\tag{\ref*{eq:ci:alpha_infty}\textsubscript{0}}
        c_i\unimportant{\overset{\meqref{eq:ltriRNR_leq_sup}}{:=}\sup_{y\in\ltri^{-1}\left( (-\infty,\Ltrb{0}+\epsilon\rbrack\right)}\left| \ltri^{'}(y)\right|}\overset{\meqref{eq:ltri_is_0}}{=}0\in\R.
    \end{equation}

\end{enumerate}

 Since this case analysis showed that in each case $c_i\in\R$ is finite, we can use \meqref{eq:ltriRNR_leq_sup} to conclude \unimportant{$\forall n\in\N:\fao:\allIndi{N}{i}:$}
 \begin{equation}\label{eq:l_prime_bound:proof}
    \left|\ltri^{'}\left(f(\xtr_i)\right)\right|\unimportant{\overset{\meqref{eq:ltriRNR_leq_sup}}{\leq} c_i} \leq \max_{i\in\fromto{N}}c_i =:\ClPrime \overset{\text{\hyperref[itm:case:alpha_beta]{cases 1}\hyperref[itm:case:alpha_infty]{\crefrangeconjunction}\hyperref[itm:case:infty_infty]{4}}}{<}\infty.
\end{equation}

\unimportant{An equivalent more explicit definition would be:
\begin{equation}
    \ClPrime:=\max \left( \{0\} \cup %
    \Set{\ltri^{'}(y_i)|i\in\fromto{N}, y_i\in\ltri^{-1}(\Ltr(0)+\epsilon)} \right).
\end{equation}}
 \end{proof}

\begin{lemma}[\Cref{ex:lossfunctionalsum}]\label{le:ex:lossfunctionalsum}
Under the assumptions of \Cref{ex:lossfunctionalsum}, the loss~$\Ltr$ from \meqref{eq:Ltri} satisfies \Cref{as:generalloss}.
\end{lemma}
\begin{proof}
    Since by Morrey's inequality point evaluations are Lipschitz-continuous w.r.t.~$\sobnormop[\cdot]$ and $\ltri$ are continuous by assumption, $\Ltr$ of \cref{eq:Ltri} is continuous w.r.t.~$\sobnormop[\cdot]$ as concatenation of continuous functions. Thus, $\Ltr$ satisfies \Cref{as:generalloss}\ref{item:continuousL}.
    
    Since by Morrey's inequality point evaluations are Lipschitz-continuous w.r.t.~$\sobnormop[\cdot]$ and $\ltri$ are Lipschitz-continuous on $\Set{f(\xtr) | \Ltrb{f}\leq\Ltrb{0}+\epsilon}$ for every $\epsilon>0$ by \Cref{le:help:wstetig}, $\Ltr$ of \cref{eq:Ltri} is Lipschitz-continuous w.r.t.~$\sobnormop[\cdot]$ as concatenation of Lipschitz-continuous functions (since finite sums are also Lipschitz-continuous). Thus, $\Ltr$ satisfies \Cref{as:generalloss}\ref{item:lipschitzL}.
\end{proof}

\newpage
\section{Leaving \cref{as:easyReadable}}\label{sec:AssymetricDistribution}
Without \Cref{as:easyReadable}, \Cref{thm:ridgeToSpline} has to be reformulated to \Cref{cor:ridgeToSplineasym}. This is done in the rest of this \namecref{sec:RidgeToSpline}.
		\begin{definition}[asymmetric adapted spline regression]\label{def:asymadaptedSplineReg}
			Let %
			$\Ltr:L^{\infty}\to\Rpz$ be a loss functional and $\lambda \in \Rp$.
			Then for given functions $g_+:\R\to\Rpz$, $g_-:\R\to\Rpz$ the \textit{\asymaregSpl}~$\flpmasym:=\flpasym +\flmasym +\gammaAsym$ is defined%
			\footnote{\label{footnote:uniqueflpmtasym}The optimization problem~\meqref{eq:asymadaptedSplineReg} should be interpreted such that $\frac{0}{0}$ is replaced by zero (For example, if $\PP[v=0]=0$ the last fraction should be ignored.). The triple~\flpmtasym\ and thus an \asymaregSpl\ is defined  for $L$ satisfying \Cref{as:generalloss}\ref{item:continuousL}, if~$g_+$,$g_-$ are compactly supported and continuous on $\supp(g_+)$ respectively $\supp(g_-)$. It is uniquely defined in case $L$ is convex (analogously to \Cref{le:existenceuniquenessaregspl}).}
			for 
			\begin{equation}\label{eq:asymadaptedSplineReg}
			\flpmtasym \in
			\argmin_{(f_+,f_-,\gamma )\in \Tgpgm}\underbrace{\left( \Ltrb{f_+ + f_- + \gamma }+\lambda P^{g_+, g_-}(f_+,f_-, \gamma)\right)}_{=:\Flpmasymb{f_+,f_-,\gamma }},  
			\end{equation}
			
			with
			\begin{equation*}
			P^{g_+, g_-}(f_+,f_-,\gamma):=
			\int_{\supp (g_+)}\!\!\!\! \frac{\left( {f_+}^{''}(x) \right)^2}{g_+(x)} dx
			+\int_{\supp (g_-)}\!\!\!\! \frac{\left( {f_-}^{''}(x) \right)^2}{g_-(x)} dx
			+\frac{\gamma^2}{\PP[v_k=0]\E[{\relu[b]^2}]},
			\end{equation*}
			
			and
			\hypertarget{eq:Tgpgm}{\begin{align*}
				\Tgpgm:=\bigg\{(f_+,f_-,\gamma )\in \WT(\R)\times \WT(\R)\times\R \bigg|& \supp (f_+'')\subseteq \supp (g_+), \supp (f_-'')\subseteq \supp (g_-),\\
				&\lim_{x\to -\infty} f_+(x)=0, \lim_{x\to -\infty} f_+'(x)=0,\\
				&\lim_{x\to +\infty} f_-(x)=0, \lim_{x\to +\infty} f_-'(x)=0,\\
				&\PP[v=0]=0\Rightarrow \gamma = 0
				\bigg\}.\\
				\end{align*}}
		\end{definition}
		
		\begin{remark}[connection to \Cref{def:adaptedSplineReg}]If \Cref{as:easyReadable} holds, then 
		    \begin{equation}
		    2g(0)P^{g_+, g_-}(f_+,f_-,0)=\Pgpm (f_+,f_-)
		    \end{equation}
		holds with $g=g_+=g_-$ and connects \Cref{def:asymadaptedSplineReg} with \Cref{def:adaptedSplineReg,def:adaptedSplineRegTuple}.\footnote{This factor $2g(0)$ explains the difference between $\lw:=\lambda n 2g(0)$ in \Cref{thm:ridgeToSpline} and $\Tilde{\Tilde{\lambda}}:=\lambda n$ in \Cref{cor:ridgeToSplineasym}.}
		\end{remark}
		
		\begin{definition}[conditioned kink position density~\gxip, \gxim]\label{def:conKinkPosDens}
			The \emph{conditioned kink position density}~$\gxip:\R\to\R$ of $\xi_k$ conditioned on $v_k>0$ is defined such that $\int_E \gxip(x) dx =\PPco{\xi_k \in E}{v_k>0}\quad\forall E\in\B$. Analogously, $\int_E \gxim(x) dx =\PPco{\xi_k \in E}{v_k<0},\forall E\in\B$.
		\end{definition}
		
		\begin{corollary}[generalized \Cref{thm:ridgeToSpline}]\label{cor:ridgeToSplineasym}
        Using the notation from \Cref{def:RSNN,def:conKinkPosDens,def:ridgeNet,def:asymadaptedSplineReg} and let
			$\fax:$
			\begin{align*} g_+(x)&:=\gxip(x) \Eco{v_k^2}{\xi_k=x, v_k>0}\PP[v_k>0],\\g_-(x)&:=\gxim(x) \Eco{v_k^2}{\xi_k=x, v_k<0}\PP[v_k<0],\end{align*} and $\Tilde{\Tilde{\lambda}}:=\lambda n $. Then, under the \Cref{as:mainAssumptions,as:truncatedg,as:generalloss}, the following statement holds for every compact set $K\subset\R$: for every $\left(\RNR[\Tilde{\Tilde{\lambda}}]\right)_{n\in\N}$ as in \Cref{def:ridgeNet} and
			\[F^{\text{asym}}_{\text{min}}:=\Set{f=f_++f_-+\gamma|(f_+,f_-,\gamma )\in\argmin_{ \Tgpgm} \Flpmasym },\]
		\begin{equation}\label{eq:ridgeToSplineasymNotUnique}
	\plim d_{\Wkp[K]{1}{\infty}}\left(\RNR[\Tilde{\Tilde{\lambda}}],F^{\text{asym}}_{\text{min}}\right) =0,
	\end{equation}
	i.e., $\forall\left(\RNR[\Tilde{\Tilde{\lambda}}]\right)_{n\in\N}\in\prod_{n\in\N}\argmin \hyperref[eq:wR]{F_n^{\Tilde{\Tilde{\lambda}}}}:\forall\epsilon>0:\forall\rho\in(0,1):\exists n_0\in\N:\forall n>n_0$:
	\begin{equation}
	    \PP\left[\exists\flpmasym\in F^{\text{asym}}_{\text{min}}:\sobnorm[{\RNR[\Tilde{\Tilde{\lambda}}]-\flpmasym }]<\epsilon\right]>\rho.
	\end{equation}
		\end{corollary}
		{\mediumimportantStart
			\begin{proof}
				The proof of \Cref*{cor:ridgeToSplineasym} is analogous to the \hyperlink{proof:thm:ridgeToSpline}{proof of \Cref*{thm:ridgeToSpline}} in \Cref{sec:proof:RidgeToSpline}. (The \cref{footnote:analogous:def:estKinkDist,footnote:analogous:def:splineApproximatingRSN,footnote:analogous:le:deltastrip} on \cpageref{footnote:analogous:def:estKinkDist,footnote:analogous:def:splineApproximatingRSN,footnote:analogous:le:deltastrip} in \Cref{sec:proof:RidgeToSpline} help to understand this analogy.)
		\end{proof}}

\newpage
\section{Intuition about \texorpdfstring{\aregSplf}{adapted regression spline}}\label{sec:IntuitionAdaptedSpline}
In this section, we provide further intuition for the \aregSplf{} by comparing it to the \wregSplf{} and to the classical \regSplf{}. The main difference is that the first (and zeroth) derivative of \aregSplf{} are slightly regularized, whereas the first (and zeroth) derivative of \wregSplf{} and \regSplf{} are not regularized at all. In the following, we discuss further (more subtle) differences.

For all experiments in \Cref{sec:IntuitionAdaptedSpline} we use $\sScale=0.05$ and $n=1024$ and the squared loss~$\Ltr$ from \cref{eq:Ltr_squaredloss}. See \url{https://github.com/JakobHeiss/NN_regularization1} for the implementation of these experiments and further hyper-parameters.
\subsection{Linear skip connections}\label{subsec:linearSkip}
A slight modification of the RSN architecture can lead to a simpler \Pfunc{}-functional.
If we add a trainable linear skip connection with parameter $a$, and a trainable bias $c$ to our \RSN{}~$\RNaffine_{\waffine,\om}(x):=\RNaffine_{(w,a,c),\om}(x):=\RN_{w,\om}(x)+a x + c$, we obtain a result analogous to \Cref{thm:ridgeToSpline}:
\begin{corollary}\label{thm:ridgeToSplineSkipConnection} Let $\text{reg}:\R^2\to\Rpz$ be any convex locally Lipschitz function and \begin{equation}\label{eq:RidgeProblem With SKipConnections}
\wRaffine\omb\in\argmin_{(w,a,c)\in\R^n\times\R\times\R}
		\Ltr\left( \RNaffine_{(w,a,c),\om}(x)\right)+\lw\left(\|w\|_2^2+\frac{1}{n}\text{reg}(a,c)\right).
\end{equation}
Then, under the assumptions of \Cref{thm:ridgeToSpline},
\begin{equation}\label{eq:ridgeToSplineNotUniqueSKipCOnnections}
	\plim d_{\Wkp[K]{1}{\infty}}\left(\RNRaffine,
	\argmin_{f\in\WT(\R)}\left(\Ltrb{f}+\lambda\Pgpmmaffine(f)\right)\right) =0,
	\end{equation}
	where
\hypertarget{eq:Pgpmaffine}{\begin{equation*}
		\Pgpmmaffine(f):=  2g(0)\underset{\underset{f=f_++f_-+a(\cdot)+c}{((f_+, f_-),a,b)\in\T\times\R\times\R}}{\min} \left(
		\int_{\supp (g)} \frac{\left( {f_+}^{''}(x) \right)^2}{g(x)} dx
		+\int_{\supp (g)} \frac{\left( {f_-}^{''}(x) \right)^2}{g(x)} dx
		+\text{reg}(a,c)
		\right).
		\end{equation*}}
\end{corollary}
\begin{proof}
The proof is analogous to the \hyperlink{proof:thm:ridgeToSpline}{proof} of \Cref{thm:ridgeToSpline} formulated in \Cref{sec:proof:RidgeToSpline}.
\end{proof}

\begin{corollary}\label{cor:simplifyingPfuncSkipConnections}
        In the case of unregularized skip connection and bias (i.e., $\text{reg}(a,c)\equiv0$) $\Pgpmmaffine$ from \Cref{thm:ridgeToSplineSkipConnection} simplifies to the regularization functional of the \wregSpl{}, i.e, for every function $f\in \WT(\R)$ that satisfies $\supp(f'')\subseteq\supp(g)$, we obtain
        \[\Pgpmmaffine(f)%
        =\Pg(f)=g(0) \int_{\supp (g)} \dfrac{\left( f''(x) \right)^2}{g(x)} dx \text{ from \Cref{def:splineReg}.}\]
\end{corollary}
\begin{proof}
First, we reformulate $\Pgpmmaffine$ for $\text{reg}\equiv0$ as
\begin{equation}\label{eq:PfuncSkipSimplified}
		\Pgpmmaffine(f)=  2g(0)\underset{\underset{f=f_++f_-+a(\cdot)+c}{((f_+, f_-),a,b)\in\T\times\R\times\R}}{\min} \left(
		\int_{\supp (g)} \frac{\left( {f_+}^{''}(x) \right)^2+\left( {f_-}^{''}(x)\right)^2}{g(x)} dx
		\right).
\end{equation}
Then, for any given function $f\in \WT(\R)$ that satisfies $\supp(f'')\subseteq\supp(g)$ (cp. \Cref{as:truncatedg}\ref{item:truncatedg}), we can explicitly formulate the solution to \meqref{eq:PfuncSkipSimplified} $((f_+^*, f_-^*),a^*,b^*)\in\T\times\R\times\R$ as
\begin{subequations}\label{subeqs:SolutionFUnctionSpaceUnregularizedSkipCOnnections}
    \begin{align}
f_+^*(x)&=\frac{f(x)}{2}-\frac{f'(-\infty) x + \lim_{r\to-\infty}\left(f(r)-f'(-\infty) r\right)}{2},\\
f_-^*(x)&=\frac{f(x)}{2}-\frac{f'(\infty) x + \lim_{r\to\infty}\left(f(r)-f'(\infty) r\right)}{2},\\
a^*&=\frac{f'(-\infty) + f'(\infty)}{2},\\
c^*&=\frac{\lim_{r\to-\infty}\left(f(r)-f'(-\infty) r\right) + \lim_{r\to+\infty}\left(f(r)-f'(\infty) r\right)}{2},
    \end{align}
\end{subequations}
where we use short notations such as $f'(-\infty):=\lim_{s\to-\infty}f(s)$.
In this case the second derivatives ${f_+^*}^{''}={f_-^*}^{''}=\frac{f''}{2}$ are equal.%
\footnote{For visualization of $f_+^*$, $f_-^*$, $a^*$ and $c^*$ from~\eqref{subeqs:SolutionFUnctionSpaceUnregularizedSkipCOnnections}, see \Cref{fig:strongerRegLinear,fig:jump,fig:roleOfWeighting,fig:sinUnderfit,fig:sinFit,fig:sinOverfit}(b), where $\RNRpaffine\approx f_+^*$, $\RNRmaffine\approx f_-^*$, $a^{*{\color{hellgrau},n,\lwnl,\text{reg}}}\approx a^*$ and $c^{*{\color{hellgrau},n,\lwnl,\text{reg}}}\approx c^*$ are very good approximations. (Strictly speaking, this approximations are not exact, since we use only $n=1024<\infty$ neurons and a very weak nut nonzero regularization such as $\frac{1}{n}\text{reg}(a,c)= 2^{-20}(a^2+c^2)>0$ for the implementation of these plots, but the results are visually practically not distinguishable from $n\to\infty$ neurons and $\text{reg}\equiv0$.)}
\end{proof}
\begin{remark}\label{rem:implicitRegAffineSkip}
    The implicit regularization induced by gradient descent
    		\hypertarget{eq:GDescentSkip}{\begin{align*}
		    \begin{split}
		\wthaffine[t+\gamma]&=\wthaffine[t]-\gamma\nabla_w \Ltrsb{\RNaffine_{\wthaffine[t]}},\label{eq:GDescentSkip}\\
		\wthaffine[0]&=0,
		\end{split}\tag{GD\textsubscript{skip}}
		\end{align*}}
    corresponds to the explicitly regularized solution of \meqref{eq:RidgeProblem With SKipConnections} iff $\text{reg}(a,c)=(a^2+c^2)$. (This implicit regularization could be scaled by using different learning rates for $a$ and $c$ than for $w$ or by multiplying $a$ and $c$ by some scaling factors.)
\end{remark}
\begin{proof}
In the case of linear skip connection, one has to add two further columns 
    \begin{align*}
		X_{i,n+1}\omb&:= \xtr_{i}%
  \text{ and}\\
		X_{i,n+2}\omb&:= 1 \quad \allIndi{N}{i}\ \faog
	\end{align*}
  to the matrix $X$ introduced in \Cref{le:GDsolution} in order to obtain analogous results as \Cref{le:GDsolution} and \Cref{thm:GDridge}.
\end{proof}

\Cref{thm:ridgeToSplineSkipConnection,cor:simplifyingPfuncSkipConnections} show how the \Pfunc{}-functional simplifies for RSN architectures with \emph{unregularized} linear skip connections. The P-functional $\Pg$ corresponding to an architecture with unregularized, linear skip connection is easier to interpret than $\Pgpmm$ or $\Pgpmmaffine$, because one can directly plug in a function $f$ into $\Pg$ without solving an optimization problem over $\T$ and without any boundary conditions imposed by \T.

\begin{figure}[hptb]
				\centering
				\scaledinset{l}{.2}{b}{.1475}{\resizebox{0.18\hsize}{!}{\tiny \begin{tabular}{@{}l@{}}%
				$\RNw[{\wth[T]}]$ \text{(implicit)}\vphantom{\RNRp}\\
				$\RNR\approx\flpm$ \text{(Ridge)}\vphantom{\RNRp}\\
				$\fl=\flg$ \text{(spline)}\vphantom{\RNRp}\\
    $\RNRp\approx\flp$\\
    $\RNRm\approx\flm$ \end{tabular}}}{\includegraphics[width=0.7\linewidth]{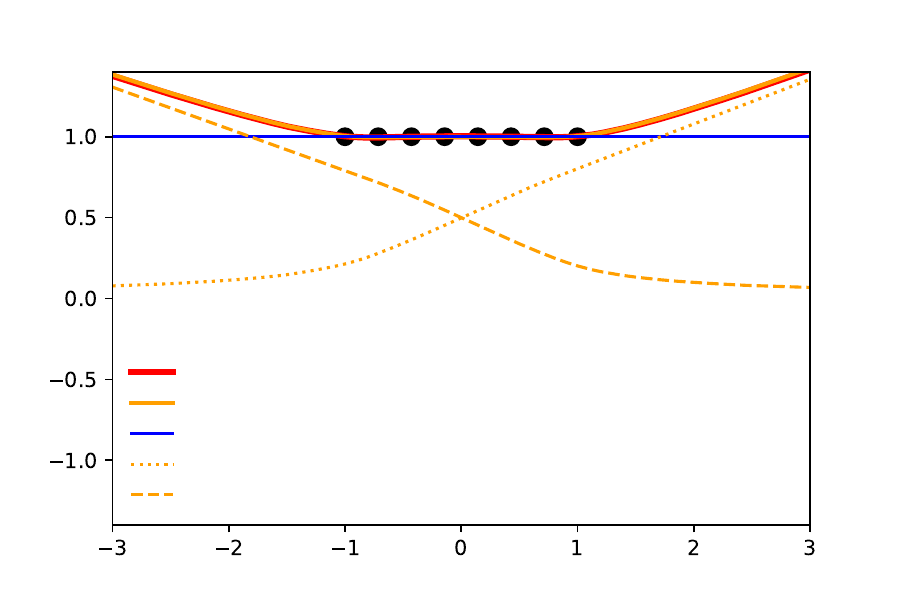}}
				\caption{In the specific example of $\ytr_i=1$ $\forall i$, we obtain that the weighted spline $\flg=\fl\equiv1$ is constant for any weighting function $g$ and for any $\lambda\in\Rp$. However, we see that in this specific example the \RSN~$\RNw[{\wth[T]}]\approx\RNR\approx\flpm$ extrapolates the data differently than $\flg=\fl\equiv1$. This difference is due to the boundary conditions in $\T$. %
    To visualize why these boundary conditions on $\flpmt\in\T$ result in the u-shape of of $\flpm$, we split $\RNR\approx\flpm$ into $\RNRp\approx\flp$ and $\RNRm\approx\flm$. %
    (Recall that $\RNRp+\RNRm=\RNR\approx\flpm=\flp+\flm$.)
    If one adds an unregularized bias $c$ to the architecture, one obtains $\RNRaffine=c^{*{\color{hellgrau},n,\lwnl,\text{reg}}}=\flg=\fl\equiv1$. If regularization~$\text{reg}\neq0$ was added to the loss, $\RNRaffine$ would continuously transform from $\flg$ into $\flpm$ as the strength of the regularization $\text{reg}$ of the bias increases.
    $T=128 \overset{\text{\eqref{eq:lambdaWelleT}}}{\implies}\lw=\frac{1}{2T(e-1)}\approx2.3\cdot 10^{-3}\overset{\text{\Cref{thm:ridgeToSpline}}}{\implies}\lambda=\frac{\lw}{2ng(0)}\approx7.1\cdot 10^{-3}$. %
                }
				\label{fig:GernealizationOfAdaptedSplineVsSpline}
			\end{figure}

\begin{figure}[htbp]
				\centering
				\scaledinset{l}{.2}{b}{.603}{\resizebox{0.18\hsize}{!}{\tiny \begin{tabular}{@{}l@{}}%
				$\RNw[{\wth[T]}]$ \text{(implicit)}\vphantom{\RNRp}\\
				$\RNR\approx\flpm$ \text{(Ridge)}\vphantom{\RNRp}\\
				$\fl=\flg$ \text{(spline)}\vphantom{\RNRp}\\
    $\RNRp\approx\flp$\\
    $\RNRm\approx\flm$\end{tabular}}}{\includegraphics[width=0.7\linewidth]{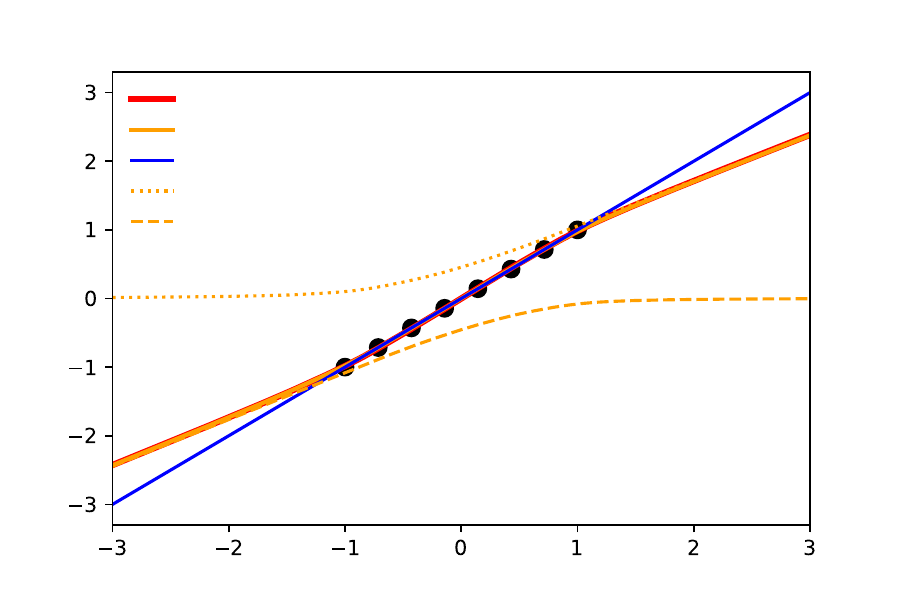}}
				\caption{In the specific example of $\ytr_i=\xtr_i$ $\forall i$, we obtain that the weighted spline $\flg=\fl=\text{id}_\R$ is the identity for any weighting function $g$ and for any $\lambda\in\Rp$. We see that in this specific example the \RSN~$\RNw[{\wth[T]}]\approx\RNR\approx\flpm$ extrapolates the data slightly differently than $\flg=\fl=\text{id}_\R$. This difference is due to the boundary conditions in $\T$. %
    To visualize these boundary conditions on $\flpmt\in\T$, we split $\RNR\approx\flpm$ into $\RNRp\approx\flp$ and $\RNRm\approx\flm$. %
    (Recall again that $\RNRp+\RNRm=\RNR\approx\flpm=\flp+\flm$.) If one adds an unregularized linear skip connection $ax$ to the architecture, one obtains $\RNRaffine=\flg=\fl=\text{id}_\R$. If regularization~$\text{reg}\neq0$ was added to the loss, $\RNRaffine$ would continuously transform from $\flg$ into $\flpm$ as the strength of the regularization $\text{reg}$ of $a$ increases.
    Note that the difference is here not as extreme as in \Cref{fig:GernealizationOfAdaptedSplineVsSpline}. In \Cref{fig:GernealizationOfAdaptedSplineVsSpline}, the extrapolation behavior of $\flpm$ seems to be rather undesirable in most applications. Here, however, the reduced slope for extrapolation of $\flpm$ can be seen as a reasonable additional regularization of the first derivative.
    $T=32 \overset{\text{\eqref{eq:lambdaWelleT}}}{\implies}\lw=\frac{1}{2T(e-1)}\approx9.1\cdot 10^{-3}\overset{\text{\Cref{thm:ridgeToSpline}}}{\implies}\lambda=\frac{\lw}{2ng(0)}\approx2.8\cdot 10^{-2}$. 
                }			\label{fig:LinearGernealizationOfAdaptedSplineVsSpline}
			\end{figure}

However, since the \aregSplf{}$\approx\RNR$ does not have skip connections, it can deviate from \wregSplf{} as can be seen in \Cref{fig:GernealizationOfAdaptedSplineVsSpline,fig:LinearGernealizationOfAdaptedSplineVsSpline}.\footnote{We never plot the exact functions $\flp,\flm$ or $\flpm$, but instead we plot $\RNRp,\RNRm$ and $\RNR$. \Cref{thm:ridgeToSpline} tells us that $\RNR$ is a good numerical approximation of $\flpm$, since in the case of a convex loss $\Ltr$ the theorem tells us that $\plim\sobnorm[\RNR-\flpm]=0$. In the proof of \Cref{thm:ridgeToSpline}, one can also see that $\plim\sobnorm[\RNRp-\flp]=0$ and that $\plim\sobnorm[\RNRm-\flm]=0$ for a convex loss $\Ltr$. For the calculation of $\fl$ we use a completely different implmentation to empirically check the validity of our theory.} In practice, when $\ytr_i$-values are centered around 0 in a pre-processing step, this difference between the two models usually becomes much smaller.\footnote{In \Cref{fig:GernealizationOfAdaptedSplineVsSpline} we see an example where the $\ytr_i$-values are not centered around 0 which leads to a quite substantial difference in the extrapolation behavior of the two models. In \Cref{fig:LinearGernealizationOfAdaptedSplineVsSpline} this difference is already smaller, because the $\ytr_i$-values are centered around 0. In other examples where data does not come from an (affine-)linear map, this difference is usually even smaller (see \Cref{fig:jump,fig:sinUnderfit,fig:sinFit,fig:sinOverfit,fig:DifferentGernalizationDifferentModels,fig:RegDifferentGernalizationDifferentModels}).}

In the case of much stronger regularization, the \aregSplf{} usually deviates substantially from the \wregSplf{}, since $\Pgpmm$ also regularizes the first derivative, also for centered $\ytr_i$-values (see \Cref{subfig:strongerRegLinearWithoutSkip}). $\Pgpmm$ regularizes the first derivative only indirectly via the boundary conditions in $\T$. $\Pgpmm$ regularizes the second derivative much more strongly than the first derivative, but for high values of $\lambda$, the regularization of the first derivative becomes less negligible. By contrast, $\Pg$ does not regularize the first derivative \emph{at all} no matter how large $\lambda$ is. This can be seen in \Cref{fig:strongerRegLinear}, where $\RNRaffine$ (solid yellow line in \Cref{subfig:strongerRegLinearWithSkip}) is much steeper compared to $\RNR$ (solid yellow line in \Cref{subfig:strongerRegLinearWithoutSkip}).  In \Cref{subfig:strongerRegLinearWithSkip}, we split $\RNRaffine$ according to \Cref{def:RNRpmaffine}.  
\begin{definition}\label{def:RNRpmaffine} Analogous to the the definition of $\RNRp$ and $\RNRm$ in \cref{eq:RNRp} we define $\RNRpaffine$ and $\RNRmaffine$ by splitting $\RNRaffine$ into sub-networks according to the signs of $v_k$ (via $\kp$ and $\km$ from \cref{eq:kpm}), i.e.,
    \begin{multline}\label{eq:RNRpmaffine}
        \RNRaffine(x)=%
        \underbrace{\sum_{k\in\kp}\wRkaffine\,\sigma\left({b_k}+{v_{k}}x\right)}_{=:\RNRpaffine(x)}+\underbrace{\sum_{k\in\km}\wRkaffine\,\sigma\left({b_k}+{v_{k}}x\right)}_{=:\RNRmaffine(x)}\\
        + a^{*{\color{hellgrau},n,\lwnl,\text{reg}}}x +c^{*{\color{hellgrau},n,\lwnl,\text{reg}}} \quad\faxg.\footnotemark
    \end{multline}\footnotetext{%
    We use the notation $\wRaffine=\left(w^{*{\color{hellgrau},n,\lwnl,\text{reg}}},a^{*{\color{hellgrau},n,\lwnl,\text{reg}}},c^{*{\color{hellgrau},n,\lwnl,\text{reg}}}\right)$.%
    }

\end{definition}

Moreover, according to \Cref{thm:ridgeToSplineSkipConnection,cor:simplifyingPfuncSkipConnections}, $\RNRaffine$ (solid yellow line in \Cref{subfig:strongerRegLinearWithSkip}) is a very precise approximation of $\flg$ for $\text{reg}\equiv 0$.
In the next subsection, we will show that the \wregSplf{} is actually exactly equal to the classical \regSplf{} in settings such as the setting of \Cref{fig:strongerRegLinear}, which explains why $\RNRaffine$ is practically identical to $\fl$ in \Cref{subfig:strongerRegLinearWithSkip}.

\begin{figure}
     \centering
     \begin{subfigure}[b]{0.45\textwidth}
         \centering
          \scaledinset{l}{.195}{b}{.647}{
    \resizebox{0.245\hsize}{!}{\tiny \begin{tabular}{@{}l@{}}%
				$\RNw[{\wth[T]}]$ \text{(implicit)}\vphantom{\RNRp}\\
				$\RNR\approx\flpm$ \text{(Ridge)}\vphantom{\RNRp}\\
				$\fl=\flg$ \text{(spline)}\vphantom{\RNRp}\\
    $\RNRp\approx\flp$\\
    $\RNRm\approx\flm$\end{tabular}}}{
    \includegraphics[width=\textwidth]{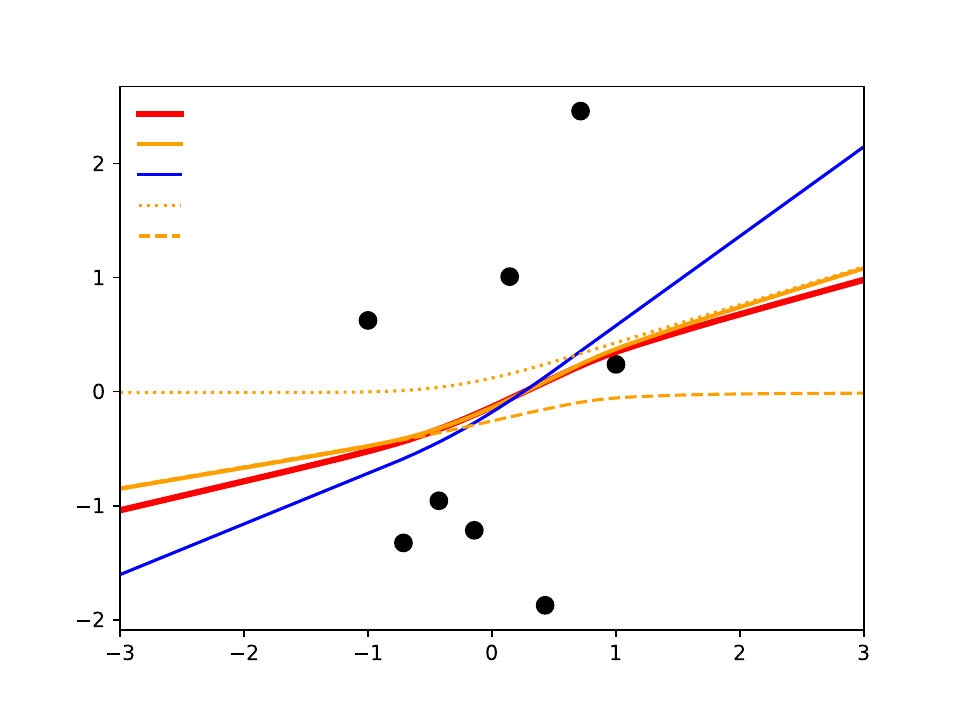}%
    }
         \caption{Without skip connections $a$ and $c$.}
         \label{subfig:strongerRegLinearWithoutSkip}
     \end{subfigure}
     \begin{subfigure}[b]{0.44\textwidth}
         \centering
         \scaledinset{l}{.195}{b}{.602}{
    \resizebox{0.245\hsize}{!}{\tiny \begin{tabular}{@{}l@{}}%
				$\RNaffine_{\wthaffine[T]}$ \text{(implicit)}\vphantom{\RNRp}\\
				\rlap{$\RNRaffine\approx\flg$ \text{(Ridge)}}%
    \hphantom{$\RNR\approx\flpm$ \text{(Ridge)}}\vphantom{\RNRp}\\
				$\fl=\flg$ \text{(spline)}\vphantom{\RNRp}\\
    $\RNRpaffine$\\
    $\RNRmaffine$\\
    \rlap{$a^{*{\color{hellgrau},n,\lwnl,\text{reg}}}x +c^{*{\color{hellgrau},n,\lwnl,\text{reg}}}$}\vphantom{\RNRp}\end{tabular}}}{
         \includegraphics[width=\textwidth]{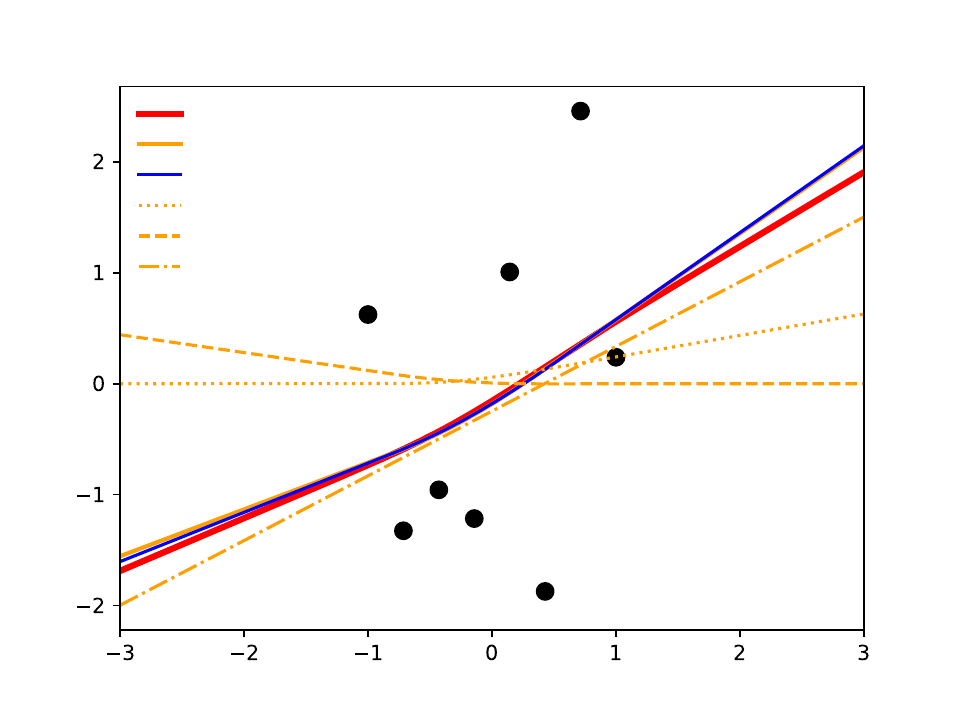}}
         \caption{With skip connections $ax+c$%
         .}
         \label{subfig:strongerRegLinearWithSkip}
     \end{subfigure}
        \caption{The \RSN~$\RNw[{\wth[T]}]\approx\RNR\approx\flpm$ has a clearly smaller slope than the much steeper spline $\RNRaffine\approx\flg=\fl$ for large values of $\lambda$, since $\Pgpmm$ also regularized the first derivative%
        . The yellow line in \Cref{subfig:strongerRegLinearWithSkip} is barely visible since it is almost identical to the blue line (see \Cref{thm:ridgeToSplineSkipConnection,cor:simplifyingPfuncSkipConnections,cor:moresimplifyingPfuncSkipConnections}).
        $\RNaffine_{\wthaffine[T]}$ is in-between $\RNw[{\wth[T]}]$ and $\fl$, since $a$ and $c$ are subject to implicit regularization (see \Cref{rem:implicitRegAffineSkip}).
        $T=0.5 \overset{\text{\eqref{eq:lambdaWelleT}}}{\implies}\lw=\frac{1}{2T(e-1)}\approx0.58 \overset{\text{\Cref{thm:ridgeToSpline}}}{\implies}\lambda=\frac{\lw}{2ng(0)}\approx1.8$. (For the implementation we use $\frac{1}{n}\text{reg}(a,c)= 2^{-20}(a^2+c^2)$ instead of $\text{reg}\equiv0$.)
        }
        \label{fig:strongerRegLinear}
\end{figure}

   \subsection{Weighting function \texorpdfstring{$g$}{g}}\label{sec:WeightingFunctionG:IntuitionAdaptedSpline}
   In the case of $b_k,v_k\sim \text{Unif}(-\sScale,\sScale)$, we obtain
   \begin{align}\label{eq:g}
       g(x):=\gxi(x) \Eco{v_k^2}{\xi_k=x}\frac{1}{2}
       =\begin{cases}
       \frac{\sScale^2}{16} &\text{, if } x\in [-1,1], \\
       \frac{\sScale^2}{16 x^4} &\text{, else,} 
       \end{cases}
   \end{align}
   for the weighting function $g$ given in \Cref{thm:ridgeToSpline} (see \Cref{fig:g}).
   \begin{figure}[!h]
				\centering
				\scaledinset{l}{.208}{b}{.80}{\resizebox{0.43\hsize}{!}{\tiny \rlap{$g$}\phantom{$\RNR\approx\flpm$ \text{(Ridge)}}}}{
                \includegraphics[width=0.7\linewidth]{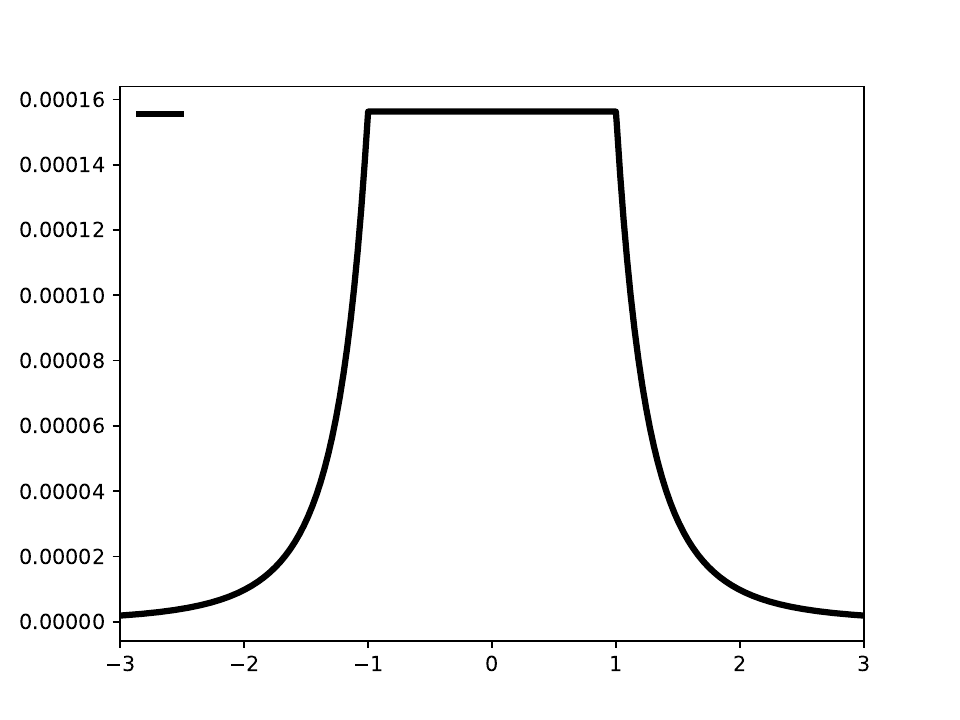}}
				\caption{$g(x):=\gxi(x) \Eco{v_k^2}{\xi_k=x}\frac{1}{2}$ given in \cref{eq:g} in the case of $b_k,v_k\sim \text{Unif}(-\sScale,\sScale)$ \iid{} uniformly distributed with $\sScale=0.05$.}
				\label{fig:g}
			\end{figure}
   The graph of a classical \regSplf{} fitted to data points can be seen as a flat\footnote{The deformation energy of an elastic stick is proportional to the integral over its squared curvature along its length. This integral is exactly proportional to $\Pg[1]$ in the limit of horizontal scales being infinitely times larger than vertical scales. Even without taking any limit over the scales, this picture can also provide basic intuition on how \regSplf{} approximately behaves if the function is not too steep or if the $y$-scale is not too large compared to the $x$-scale.} elastic stick that is pulled to the data points by elastic springs, where high values of $\lambda$ correspond to a stiff stick and low values of $\lambda$ correspond to a soft stick. The \wregSplf{} can also be seen as such an elastic stick whose elasticity is not constant along the stick, i.e., the stick is softer in input regions of high $g(x)$ and stiffer in input regions of low $g(x)$.
   \Cref{eq:g} and \Cref{fig:g} show that for uniform initialization of first-layer parameters, $g$ is constant on the interval [-1,1]. This implies that for any training data set satisfying $\xtr_i\in[-1,1] \ \forall i$, it holds that $\flg=\fl$ for $g$ from \cref{eq:g} (see \Cref{prop:constantg}).
   In practice this usually is the case, since it is very common to re-scale the training data to fit into $[-1,1]$ as a pre-processing step.
   \begin{proposition}\label{prop:constantg}
       Using the notation of \Cref{def:splineReg}, under the \Cref{as:generalloss} with $\nu$ satisfying $\supp(\nu)\subseteq[-1,1]$,\footnote{For instance, \Cref{as:generalloss} with $\supp(\nu)\subseteq[-1,1]$ is satisfied if \Cref{as:squaredloss} is satisfied with $\xtr_i\in[-1,1]\ \forall i \in \{1,\dots,N\}$.
       In case of \Cref{as:squaredloss}, $\supp(\nu)=\Set{\xtr_i : i\in\fromto{N}}$.} let $g:\R\to\Rpz$ be any non-negative function that satisfies $g(x)=g(0)\neq0\ \forall x\in[-1,1]$. Then it holds that
       \[\flg(x)=\fl(x) \forall x \in \R.\]
   \end{proposition}
   \begin{proof}
   As $\supp(\nu)$ (e.g., the training data in the case of \Cref{as:squaredloss}) is contained in [-1,1] (and as $\Ltr(f)$ does not depend on the behaviour of $f$ outside of $\supp(\nu)$ according to \Cref{as:generalloss}), ${\fl}''(x)=0={\flg}''(x)\ \forall x\not\in[-1,1]$ by optimality. Thus, we can w.l.o.g.{} restrict the optimization problems to functions $f$ whose second derivative is zero outside of [-1,1] and for such functions $f$ we see that
       \begin{align*}
           \Pg(f)
       &=g(0) \int_{\supp (g)} \dfrac{\left( {f}^{''}(x) \right)^2}{g(x)} dx \\
       &= \int_{[-1,1]} \dfrac{g(0)\left( {f}^{''}(x) \right)^2}{g(x)} dx \\
       &= \int_{[-1,1]} \dfrac{g(0)\left( {f}^{''}(x) \right)^2}{g(0)} dx \\
       &= \int_{\R} \left( {f}^{''}(x) \right)^2 dx
       =\Pg[1](f).
       \end{align*}
   \end{proof}

   If one does not scale the training data to fit into [-1,1] the \wregSplf{} (and thus \RNR) deviate significantly from the classical \regSplf{}  especially for input values far outside the interval [-1,1], and in particular for large values of $\lambda$ (see \Cref{fig:roleOfWeighting}).

\begin{figure}[htbp]
     \centering
     \begin{subfigure}[b]{0.45\textwidth}
         \centering
         \scaledinset{l}{.725}{b}{.602}{
    \resizebox{0.254\hsize}{!}{\tiny \begin{tabular}{@{}l@{}}%
				$\RNw[{\wth[T]}]$ \text{(implicit)}\vphantom{\RNRp}\\
				$\RNR\approx\flpm$ \text{(Ridge)}\vphantom{\RNRp}\\
				$\fl$ \text{(spline)}\vphantom{\RNRp}\\
    $\RNRp\approx\flp$\\
    $\RNRm\approx\flm$\end{tabular}}}{
         \includegraphics[width=\textwidth]{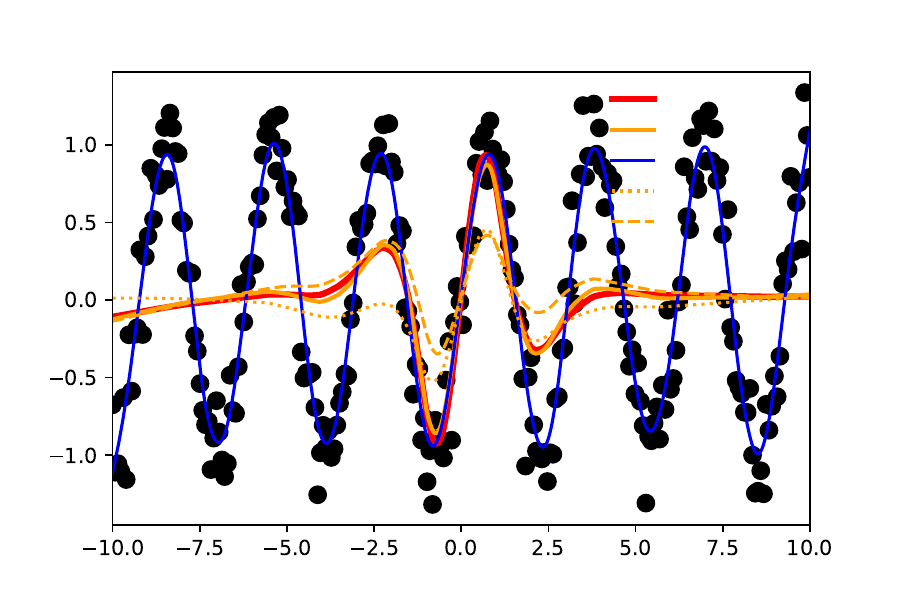}}         \caption{Without skip connections $a$ and $c$.}
         \label{subfig:roleOfWeightingWithoutSkip}
     \end{subfigure}
     \begin{subfigure}[b]{0.45\textwidth}
         \centering
         \scaledinset{l}{.725}{b}{.552}{
    \resizebox{0.25\hsize}{!}{\tiny \begin{tabular}{@{}l@{}}%
				$\RNaffine_{\wthaffine[T]}$ \text{(implicit)}\vphantom{\RNRp}\\
				\rlap{$\RNRaffine\approx\flg$ \text{(Ridge)}}%
    \hphantom{$\RNR\approx\flpm$ \text{(Ridge)}}\vphantom{\RNRp}\\
				$\fl$ \text{(spline)}\vphantom{\RNRp}\\
    $\RNRpaffine$\\
    $\RNRmaffine$\\
    \rlap{$a^{*{\color{hellgrau},n,\lwnl,\text{reg}}}x +c^{*{\color{hellgrau},n,\lwnl,\text{reg}}}$}\vphantom{\RNRp}\end{tabular}}}{
         \includegraphics[width=\textwidth]{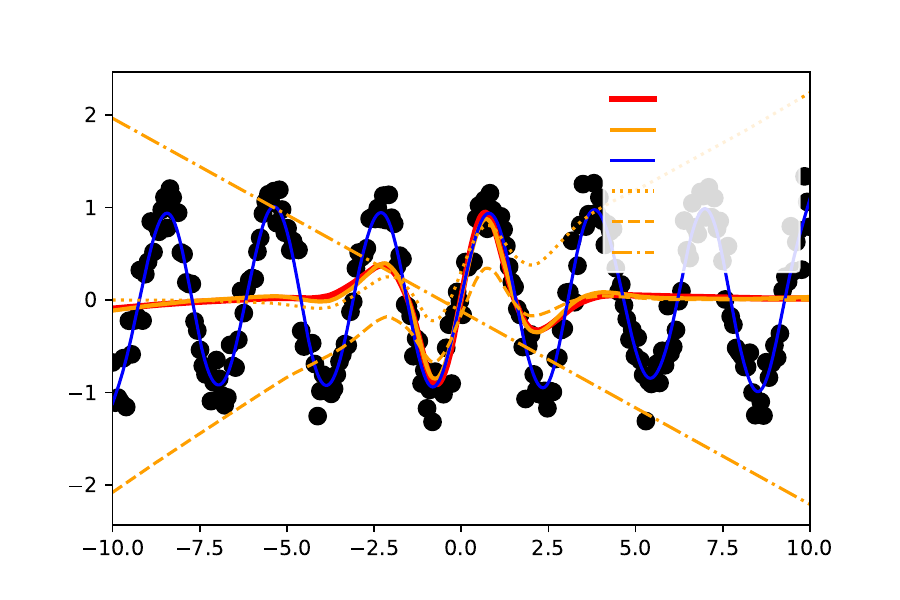}}
         \caption{With skip connections $ax+c$%
         .}
         \label{subfig:roleOfWeightingWithSkip}
     \end{subfigure}
        \caption{%
        The further the input lies from the interval $[-1,1]$ the less flexible the red and yellow curves become, since $g(x)$ gets very low for $x$ far away from $[-1,1]$. This is not the case for the \regSplf{} (blue curve).
        $T=16 \overset{\text{\eqref{eq:lambdaWelleT}}}{\implies}\lw=\frac{1}{2T(e-1)}\approx1.8\cdot 10^{-2} \overset{\text{\Cref{thm:ridgeToSpline}}}{\implies}\lambda=\frac{\lw}{2ng(0)}\approx5.7\cdot10^{-2}$. (For the implementation we use $\frac{1}{n}\text{reg}(a,c)= 2^{-20}(a^2+c^2)$ instead of $\text{reg}\equiv0$.)
        }
        \label{fig:roleOfWeighting}
\end{figure}

\subsection{Summarizing the intuition about \texorpdfstring{\aregSplf}{adapted regression spline}}\label{subsec:SummarizingTheIntuitionAboutARegSpl}
By combining the results from \Cref{thm:ridgeToSplineSkipConnection,cor:simplifyingPfuncSkipConnections,prop:constantg}, we get that if the training input data lies within $[-1,1]$ and if we use \emph{unregularized} linear skip connections parametrized by $a,c$, then $\RNRaffine$ converges to the classical \regSplf{} as we state in the following \Cref{cor:moresimplifyingPfuncSkipConnections}.
\begin{corollary}\label{cor:moresimplifyingPfuncSkipConnections}
        Using the assumptions\footnote{The main assumptions are $\text{reg}\equiv0$, $\xtr_i\in[-1,1]\ \forall i \in \{1,\dots,N\}$ and $g(x)=g(0)\neq0\ \forall x\in[-1,1]$.} and the notation of \Cref{thm:ridgeToSplineSkipConnection,cor:simplifyingPfuncSkipConnections,prop:constantg}  and under \Cref{as:squaredloss}\footnote{\Cref{as:squaredloss} can be weakened. We do not need \Cref{as:squaredloss} and the existence of two different input training points $\xtr_i\neq\xtr_j$, if we directly assume that $\flg$ is unique.}, if there exist two different input training points $\xtr_i\neq\xtr_j$, we get
        \begin{equation}\label{eq:ridgeToSplineSKipCOnnections}
	\plim \RNRaffine=\fl.
	\end{equation}
\end{corollary}
\begin{proof}
    This follows directly from combining \Cref{thm:ridgeToSplineSkipConnection,cor:simplifyingPfuncSkipConnections,prop:constantg}. We need $\xtr_i\neq\xtr_j$ to guarantee that the \regSplf{} and $\wRaffine$ are unique. 
\end{proof}

We have seen in \Cref{fig:GernealizationOfAdaptedSplineVsSpline,subfig:strongerRegLinearWithoutSkip,fig:roleOfWeighting} that \RSN{}s and the classical \regSplf{} can differ significantly, if the assumptions of \Cref{cor:moresimplifyingPfuncSkipConnections} are violated. However, \Cref{fig:GernealizationOfAdaptedSplineVsSpline,fig:roleOfWeighting} show only examples where typical rules of thumbs used in pre-processing of the data are heavily violated.

In a more typical setting where the usual rules of thumb are respected (i.e., if $\xtr_i\in[-1,1]\ \forall i$ and $\ytr_i$ are centered around 0 and $\lambda$ is not too large), intuitively our theory tells that even if the assumptions of \Cref{cor:moresimplifyingPfuncSkipConnections} are violated, both early stopped and $\ell_2$-regularized \RSN{}s, both with linear skip connections and without them, and with different regularizations of the linear skip connections all behave similarly to the classical \regSplf{} (always under the premise that we use \cref{eq:lambdaWelleT} to translate a stopping time $T$ into a weight regularization $\lw$ and that we use the translation from $\lw$ to $\lambda$ given in \Cref{thm:ridgeToSpline}).

\begin{figure}[htbp]
     \centering
     \begin{subfigure}[b]{0.45\textwidth}
         \centering
         \scaledinset{l}{.195}{b}{.607}{
    \resizebox{0.254\hsize}{!}{\tiny \begin{tabular}{@{}l@{}}%
				$\RNw[{\wth[T]}]$ \text{(implicit)}\vphantom{\RNRp}\\
				$\RNR\approx\flpm$ \text{(Ridge)}\vphantom{\RNRp}\\
				$\fl=\flg$ \text{(spline)}\vphantom{\RNRp}\\
    $\RNRp\approx\flp$\\
    $\RNRm\approx\flm$\end{tabular}}}{
         \includegraphics[width=\textwidth]{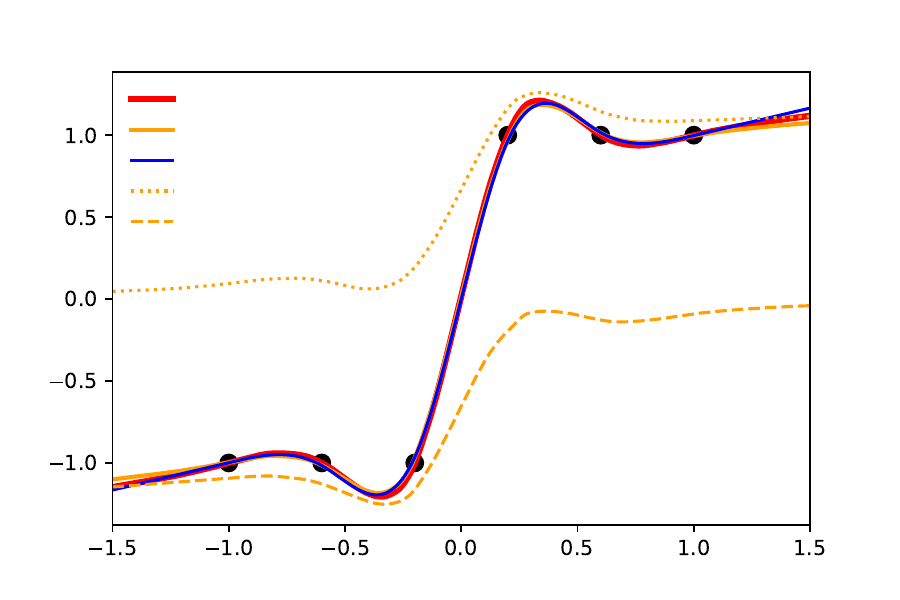}}         \caption{Without skip connections $a$ and $c$.}
         \label{subfig:jumpWithoutSkip}
     \end{subfigure}
     \begin{subfigure}[b]{0.45\textwidth}
         \centering
         \scaledinset{l}{.195}{b}{.557}{
    \resizebox{0.25\hsize}{!}{\tiny \begin{tabular}{@{}l@{}}%
				$\RNaffine_{\wthaffine[T]}$ \text{(implicit)}\vphantom{\RNRp}\\
				\rlap{$\RNRaffine\approx\flg$ \text{(Ridge)}}%
    \hphantom{$\RNR\approx\flpm$ \text{(Ridge)}}\vphantom{\RNRp}\\
				$\fl=\flg$ \text{(spline)}\vphantom{\RNRp}\\
    $\RNRpaffine$\\
    $\RNRmaffine$\\
    \rlap{$a^{*{\color{hellgrau},n,\lwnl,\text{reg}}}x +c^{*{\color{hellgrau},n,\lwnl,\text{reg}}}$}\vphantom{\RNRp}\end{tabular}}}{
         \includegraphics[width=\textwidth]{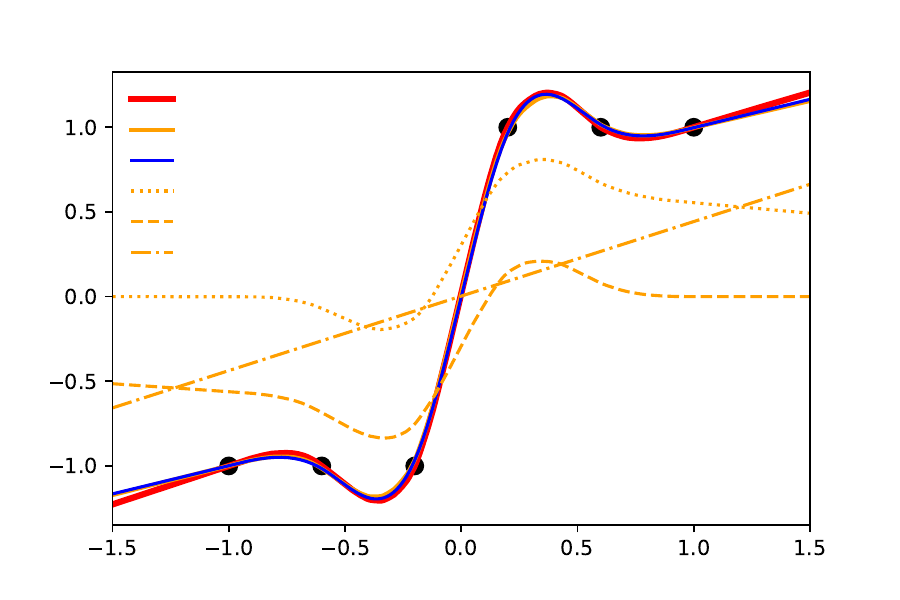}}
         \caption{With skip connections $ax+c$%
         .}
         \label{subfig:jumpWithSkip}
     \end{subfigure}
        \caption{The \RSN~$\RNw[{\wth[T]}]\approx\RNR\approx\flpm$ is in most common settings quite similar to $\RNRaffine\approx\flg=\fl$ for small values of $\lambda$. %
        $T=5461.\dot{3} \overset{\text{\eqref{eq:lambdaWelleT}}}{\implies}\lw=\frac{1}{2T(e-1)}\approx5.1\cdot 10^{-5} \overset{\text{\Cref{thm:ridgeToSpline}}}{\implies}\lambda=\frac{\lw}{2ng(0)}\approx1.7\cdot10^{-4}$. (For the implementation we use $\frac{1}{n}\text{reg}(a,c)= 2^{-30}(a^2+c^2)$ instead of $\text{reg}\equiv0$.)
        }
        \label{fig:jump}
\end{figure}

\Cref{fig:jump,fig:sinUnderfit,fig:sinFit,fig:sinOverfit,fig:DifferentGernalizationDifferentModels,fig:RegDifferentGernalizationDifferentModels} show that this intuition is typically correct. This motivates why it is popular to re-scale the training data to the $[-1,1]$-cube as a pre-processing step.

If all the standard pre-processing steps are in place the most relevant remaining difference among various versions of \RSN{}s and splines occurs for large values of $\lambda$:
 $\RN_{\wth[{T,\wth[0]}]},$
	 $\RN_{\wth[T]},$
	$\RN_{\wt[T]},$
	$\RNR,$
	 $\flpm,$ and various versions of $\RNaffine$ with implicit regularization or $\text{reg}(a,c)=(a^2+c^2)$ have \emph{regularized} first (and zeroth) derivative (see \Cref{subfig:strongerRegLinearWithoutSkip}), while the \regSplf{}, the \wregSplf{} and $\RNaffine$ with $\text{reg}\equiv0$  have \emph{unregularized} first (and zeroth) derivative %
  (see \Cref{subfig:strongerRegLinearWithSkip}).

\Cref{fig:jump} shows an example where the \regSplf{} displays a very characteristic behavior which differs a lot from what many other popular regression techniques would predict. Even this quite particular behavior is almost indistinguishably mimicked by $\RNw[{\wth[T]}],\RNR,\flpm,\RNaffine_{\wthaffine},\RNRaffine$ and $\flg$.

\begin{figure}[htbp]
     \centering
     \begin{subfigure}[b]{0.45\textwidth}
         \centering
         \scaledinset{l}{.195}{b}{.607}{
    \resizebox{0.254\hsize}{!}{\tiny \begin{tabular}{@{}l@{}}%
				$\RNw[{\wth[T]}]$ \text{(implicit)}\vphantom{\RNRp}\\
				$\RNR\approx\flpm$ \text{(Ridge)}\vphantom{\RNRp}\\
				$\fl=\flg$ \text{(spline)}\vphantom{\RNRp}\\
    $\RNRp\approx\flp$\\
    $\RNRm\approx\flm$\end{tabular}}}{
         \includegraphics[width=\textwidth]{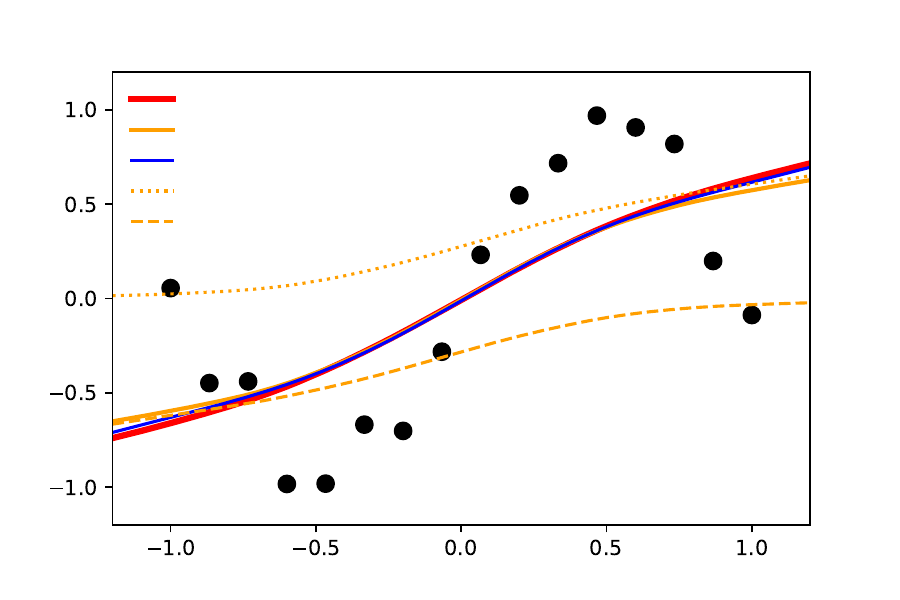}}
         \caption{Without skip connections $a$ and $c$.}
         \label{subfig:sinUnderfitWithoutSkip}
     \end{subfigure}
     \begin{subfigure}[b]{0.45\textwidth}
         \centering
         \scaledinset{l}{.195}{b}{.557}{
    \resizebox{0.25\hsize}{!}{\tiny \begin{tabular}{@{}l@{}}%
				$\RNaffine_{\wthaffine[T]}$ \text{(implicit)}\vphantom{\RNRp}\\
				\rlap{$\RNRaffine\approx\flg$ \text{(Ridge)}}%
    \hphantom{$\RNR\approx\flpm$ \text{(Ridge)}}\vphantom{\RNRp}\\
				$\fl=\flg$ \text{(spline)}\vphantom{\RNRp}\\
    $\RNRpaffine$\\
    $\RNRmaffine$\\
    \rlap{$a^{*{\color{hellgrau},n,\lwnl,\text{reg}}}x +c^{*{\color{hellgrau},n,\lwnl,\text{reg}}}$}\vphantom{\RNRp}\end{tabular}}}{
         \includegraphics[width=\textwidth]{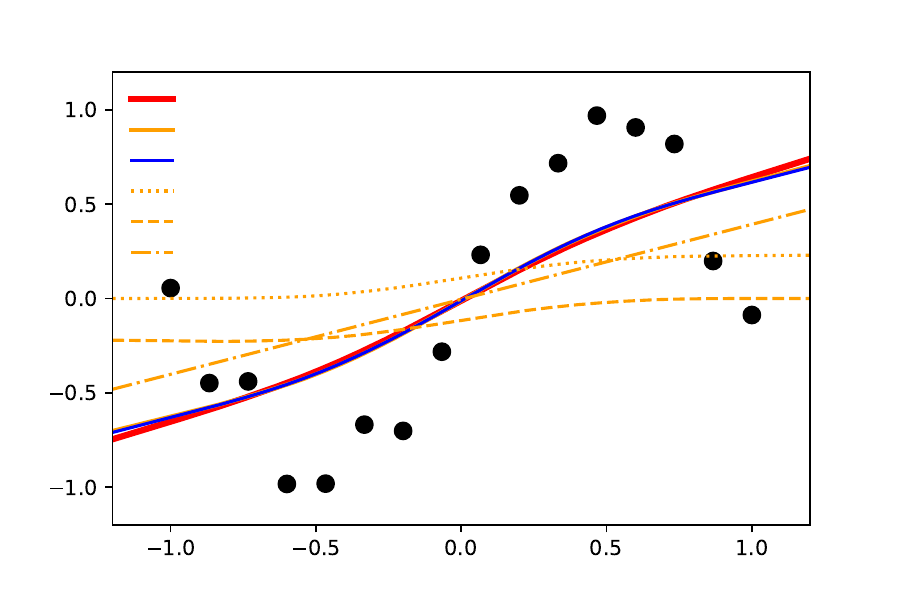}}
         \caption{With skip connections $ax+c$.}
         \label{subfig:sinUnderfitWithSkip}
     \end{subfigure}
        \caption{%
        $T=2 \overset{\text{\eqref{eq:lambdaWelleT}}}{\implies}\lw=\frac{1}{2T(e-1)}\approx 0.15 \overset{\text{\Cref{thm:ridgeToSpline}}}{\implies}\lambda=\frac{\lw}{2ng(0)}\approx0.45$. (For the implementation we use $\frac{1}{n}\text{reg}(a,c)= 2^{-30}(a^2+c^2)$ instead of $\text{reg}\equiv0$.)
        }
        \label{fig:sinUnderfit}
\end{figure}

\begin{figure}[htbp]
     \centering
     \begin{subfigure}[b]{0.45\textwidth}
         \centering
         \scaledinset{l}{.195}{b}{.607}{
    \resizebox{0.254\hsize}{!}{\tiny \begin{tabular}{@{}l@{}}%
				$\RNw[{\wth[T]}]$ \text{(implicit)}\vphantom{\RNRp}\\
				$\RNR\approx\flpm$ \text{(Ridge)}\vphantom{\RNRp}\\
				$\fl=\flg$ \text{(spline)}\vphantom{\RNRp}\\
    $\RNRp\approx\flp$\\
    $\RNRm\approx\flm$\end{tabular}}}{
         \includegraphics[width=\textwidth]{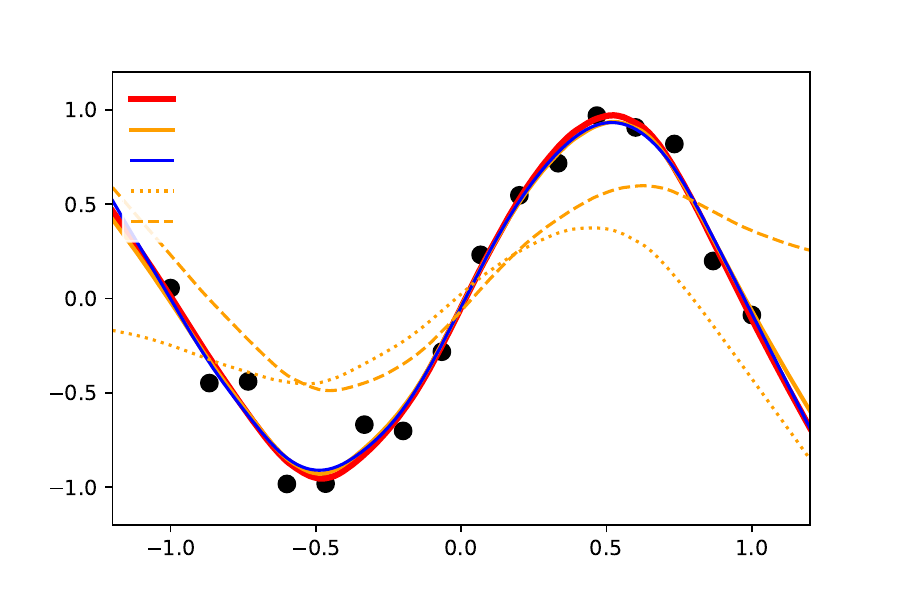}}
         \caption{Without skip connections $a$ and $c$.}
         \label{subfig:sinFitWithoutSkip}
     \end{subfigure}
     \begin{subfigure}[b]{0.45\textwidth}
         \centering
         \scaledinset{l}{.195}{b}{.557}{
    \resizebox{0.25\hsize}{!}{\tiny \begin{tabular}{@{}l@{}}%
				$\RNaffine_{\wthaffine[T]}$ \text{(implicit)}\vphantom{\RNRp}\\
				\rlap{$\RNRaffine\approx\flg$ \text{(Ridge)}}%
    \hphantom{$\RNR\approx\flpm$ \text{(Ridge)}}\vphantom{\RNRp}\\
				$\fl=\flg$ \text{(spline)}\vphantom{\RNRp}\\
    $\RNRpaffine$\\
    $\RNRmaffine$\\
    \rlap{$a^{*{\color{hellgrau},n,\lwnl,\text{reg}}}x +c^{*{\color{hellgrau},n,\lwnl,\text{reg}}}$}\vphantom{\RNRp}\end{tabular}}}{         \includegraphics[width=\textwidth]{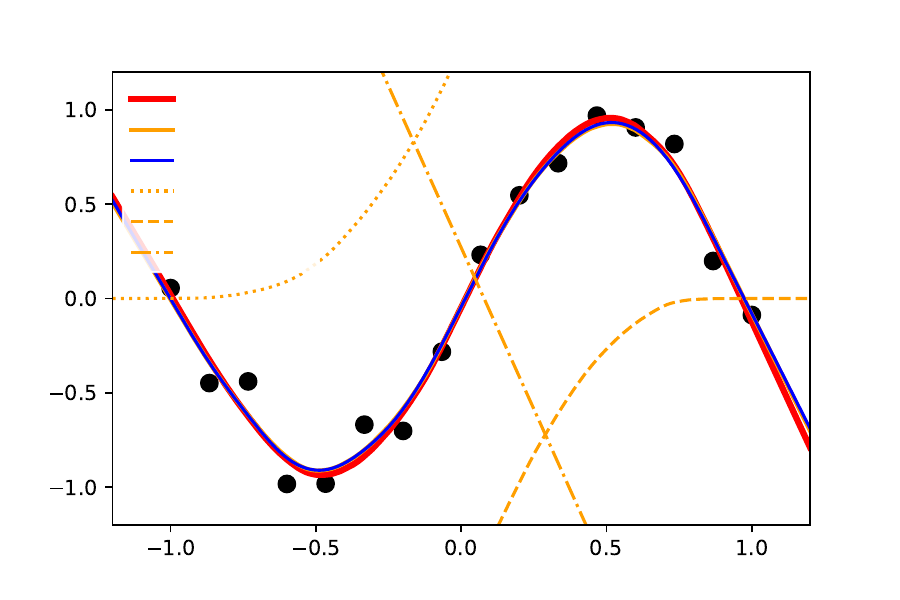}}
         \caption{With skip connections $ax+c$.}
         \label{subfig:sinFitWithSkip}
     \end{subfigure}
        \caption{%
        $T=512 \overset{\text{\eqref{eq:lambdaWelleT}}}{\implies}\lw=\frac{1}{2T(e-1)}\approx 5.7\cdot10^{-4} \overset{\text{\Cref{thm:ridgeToSpline}}}{\implies}\lambda=\frac{\lw}{2ng(0)}\approx1.7\cdot10^{-3} $. (For the implementation we use $\frac{1}{n}\text{reg}(a,c)= 2^{-20}(a^2+c^2)$ instead of $\text{reg}\equiv0$.)
        }
        \label{fig:sinFit}
\end{figure}

\begin{figure}[htbp]
     \centering
     \begin{subfigure}[b]{0.45\textwidth}
         \centering
         \scaledinset{l}{.715}{b}{.15}{
    \resizebox{0.254\hsize}{!}{\tiny \begin{tabular}{@{}l@{}}%
				$\RNw[{\wth[T]}]$ \text{(implicit)}\vphantom{\RNRp}\\
				$\RNR\approx\flpm$ \text{(Ridge)}\vphantom{\RNRp}\\
				$\fl=\flg$ \text{(spline)}\vphantom{\RNRp}\\
    $\RNRp\approx\flp$\\
    $\RNRm\approx\flm$\end{tabular}}}{
         \includegraphics[width=\textwidth]{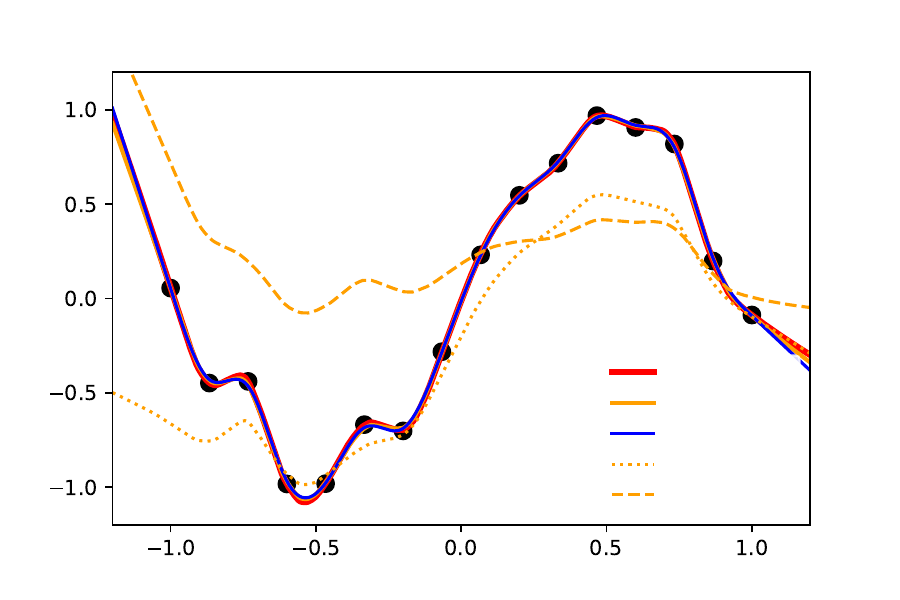}}
         \caption{Without skip connections $a$ and $c$.}
         \label{subfig:sinOverfitWithoutSkip}
     \end{subfigure}
     \begin{subfigure}[b]{0.45\textwidth}
         \centering
         \scaledinset{l}{.715}{b}{.153}{
    \resizebox{0.25\hsize}{!}{\tiny \begin{tabular}{@{}l@{}}%
				$\RNaffine_{\wthaffine[T]}$ \text{(implicit)}\vphantom{\RNRp}\\
				\rlap{$\RNRaffine\approx\flg$ \text{(Ridge)}}%
    \hphantom{$\RNR\approx\flpm$ \text{(Ridge)}}\vphantom{\RNRp}\\
				$\fl=\flg$ \text{(spline)}\vphantom{\RNRp}\\
    $\RNRpaffine$\\
    $\RNRmaffine$\\
    \rlap{$a^{*{\color{hellgrau},n,\lwnl,\text{reg}}}x +c^{*{\color{hellgrau},n,\lwnl,\text{reg}}}$}\vphantom{\RNRp}\end{tabular}}}{
         \includegraphics[width=\textwidth]{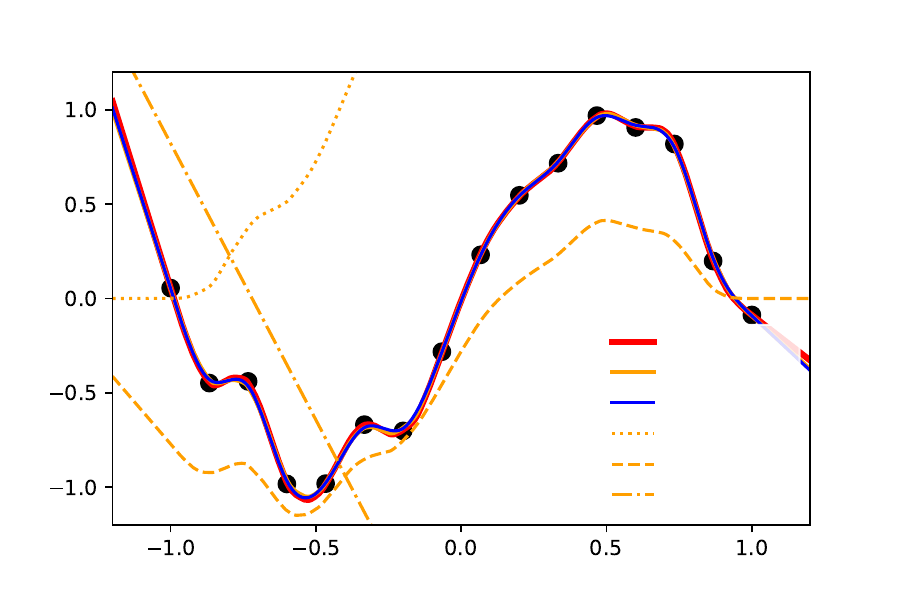}}
         \caption{With skip connections $ax+c$%
         .}
         \label{subfig:sinOverfitWithSkip}
     \end{subfigure}
        \caption{%
        $T=65536 \overset{\text{\eqref{eq:lambdaWelleT}}}{\implies} \lw=\frac{1}{2T(e-1)}\approx 4.4\cdot10^{-6} \overset{\text{\Cref{thm:ridgeToSpline}}}{\implies}\lambda=\frac{\lw}{2ng(0)}\approx1.4\cdot10^{-5} $. (For the implementation we use $\frac{1}{n}\text{reg}(a,c)= 2^{-20}(a^2+c^2)$ instead of $\text{reg}\equiv0$.)
        }
        \label{fig:sinOverfit}
\end{figure}

\Cref{fig:sinUnderfit,fig:sinFit,fig:sinOverfit} all show exactly the same $N=16$ training data points, but with different regularization parameters $\lambda$ and $\lw$ for $\fl,\RNR$ and $\RNRaffine$ corresponding to different stopping times $T\in\Set{2,512,65536}$ of the unregulated \RSN{}s~$\RNw[\wth]$ and $\RNaffine_{\wthaffine}$.
These figures visualize how short training times $T$ (corresponding to strong regularization $\lambda$ and $\lw$) leads to under-fitting (see \Cref{fig:sinUnderfit}), while on the other side extremely long training times $T$ (corresponding to extremely weak regularization $\lambda$ and $\lw$) leads to over-fitting (see \Cref{fig:sinOverfit}), whereas moderate training times $T$ (corresponding to moderate regularization $\lambda$ and $\lw$) nicely filter out the noise (see \Cref{fig:sinFit}).
While strictly speaking, \Cref{thm:GDridge} only tells us for $T\to\infty$ that $\RNw[\wth]\approx\RNR$%
, we can empirically see in \Cref{fig:sinUnderfit,fig:sinFit} that these approximations usually already hold quite well for low values of $T$ (corresponding to high values of $\lambda$ and $\lw$) as we have already discussed in \cref{sec:EarlyStopping} from a theoretical point of view.
Furthermore, we can see in \Cref{fig:sinUnderfit} that even for quite short training times~$T$ (corresponding to strong regularization $\lambda$ and $\lw$ where we are already in the under-fitting regime) the \RSN{}s~$\RNw[{\wth[T]}]\approx\RNR\approx\flpm$ behave very similar to $\RNRaffine\approx\flg=\fl$. The counter example in \Cref{subfig:strongerRegLinearWithoutSkip} is a quite extreme edge case of extremely large regularization, which is only needed in scenarios where linear regression would already lead to over-fitting. By contrast, \Cref{fig:sinUnderfit} shows an example were linear regression is still in the under-fitting regime. We see in this figure that at training time $T=2$, the model has already learned the linear component quite well and the second derivative gets equally regularised by $\Pgpmm$ and $\Pg[1]$.

While $\RNw[{\wth[T]}],\RNR\approx\flpm$ and $\fl=\flg$ have their subtle differences, we want to emphasize how similar they behave compared to other (non-linear) regression methods (see \Cref{fig:DifferentGernalizationDifferentModels,fig:RegDifferentGernalizationDifferentModels} for a visualization of the diversity of different machine-learning models).
In the following \Cref{subsec:SimilaritWithoutSkipConnections} we explain why the different \RSN{} models and different spline models are usually so similar.

\begin{figure}[htbp]
				\centering
				\scaledinset{l}{.2}{b}{.505}{\resizebox{0.18\hsize}{!}{\tiny \begin{tabular}{@{}l@{}}%
				$\RNw[{\wth[T]}]$ \text{(implicit)}\vphantom{\RNRp}\\
				$\RNR\approx\flpm$ \text{(Ridge)}\vphantom{\RNRp}\\
				$\fl=\flg$ \text{(spline)}\vphantom{\RNRp}\\
    linear regression\vphantom{\RNRp}\\
    polynomial interpolation\vphantom{\RNRp}\\
    nearest neighbor\vphantom{\RNRp}\\
    $\NN_{\theta^{*{\color{hellgrau},\lambda}}}$\vphantom{\RNRp}%
    \end{tabular}}}{\includegraphics[width=0.7\linewidth]{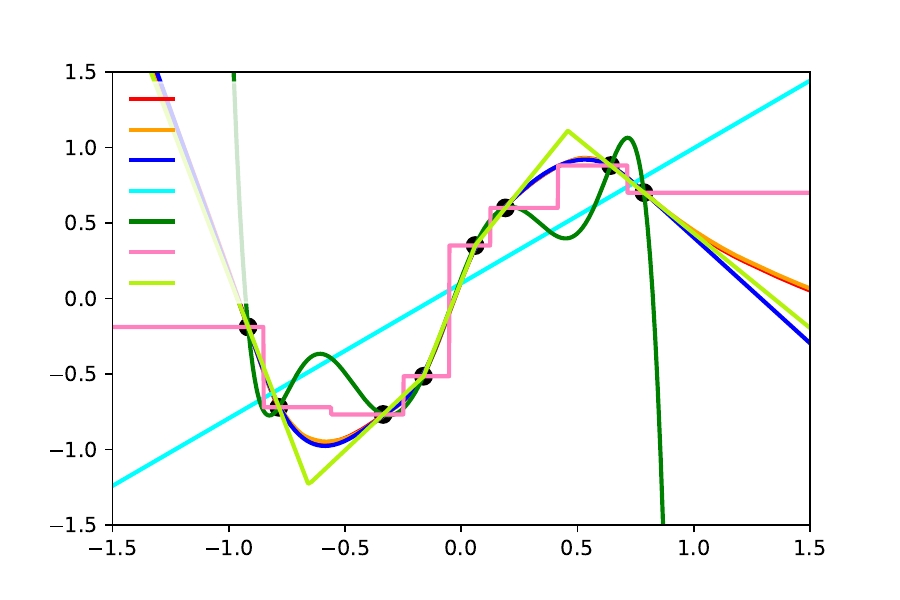}}
				\caption{Different models interpolate and extrapolate the same given 8 data points $(\xtr_i,\ytr_i)$ differently. While $\RNw[{\wth[T]}],\RNR\approx\flpm$ and $\fl=\flg$ can differ from each other in theory, they are almost indistinguishable in this plot as the 3 lines almost cover each other. Linear regression fits and (affine) linear function $ax+c$ by minimizing the squared loss. (In this case linear regression is not able to fit the nonlinear nature of this data-set.) Polynomial regression interpolates the data with a polynomial of minimal degree. (For most real-world data-sets the generalization behavior of polynomial regression displayed in this example would be very undesirable.) Nearest neighbor predicts for any input point $x$ the output value $\ytr_{i^{*,x}}$ of the nearest training input point $\xtr_{i^{*,x}}$, i.e., $i^{*,x}\in\argmin_{i\in\fromto{N}}\left|\xtr_i-x\right|$. For $\NN_{\theta^{*{\color{hellgrau},\lambda}}}$ both layers of a very wide shallow $\NN$ were trained with very small $\ell_2$-regularization (see \cite{Part3Arxiv} for the corresponding \Pfunc{}-Functional).}			\label{fig:DifferentGernalizationDifferentModels}
			\end{figure}

\begin{figure}[htbp]
				\centering
				\scaledinset{l}{.2}{b}{.505}{\resizebox{0.18\hsize}{!}{\tiny \begin{tabular}{@{}l@{}}%
				$\RNw[{\wth[T]}]$ \text{(implicit)}\vphantom{\RNRp}\\
				$\RNR\approx\flpm$ \text{(Ridge)}\vphantom{\RNRp}\\
				$\fl=\flg$ \text{(spline)}\vphantom{\RNRp}\\
    linear regression\vphantom{\RNRp}\\
    polynomial regression\vphantom{\RNRp}\\
    k-nearest neighbor\vphantom{\RNRp}\\
    $\NN_{\theta^{*{\color{hellgrau},\lambda}}}$\vphantom{\RNRp}%
    \end{tabular}}}{\includegraphics[width=0.7\linewidth]{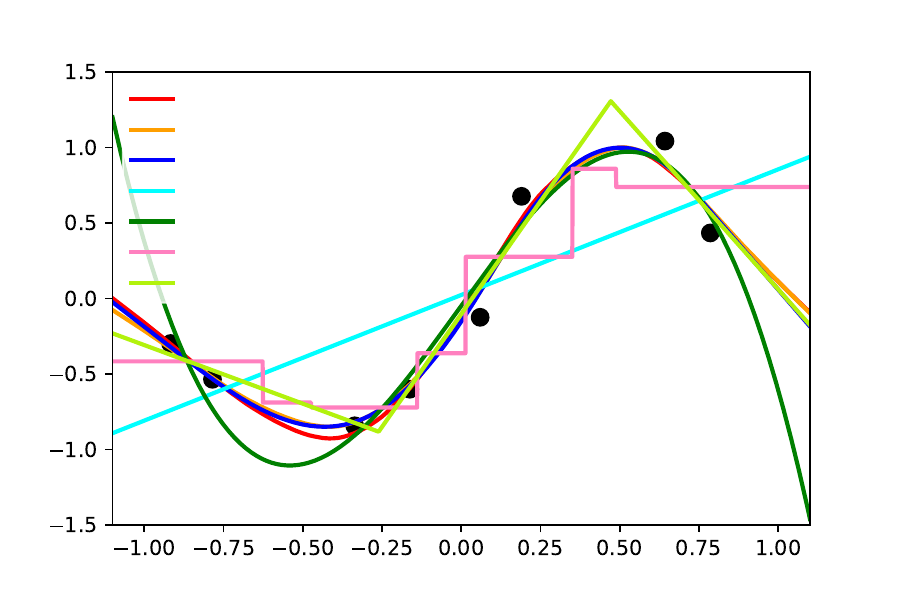}}
				\caption{Different models fit different functions to the same given 8 data points $(\xtr_i,\ytr_i)$. Here we have more noise and more regularization than in \Cref{fig:DifferentGernalizationDifferentModels}. While $\RNw[{\wth[T]}],\RNR\approx\flpm$ and $\fl=\flg$ can differ from each other in theory (especially for stronger regularization), they are very similar in this plot as the 3 lines almost cover each other ($T=512 \overset{\text{\eqref{eq:lambdaWelleT}}}{\implies} \lw=\frac{1}{2T(e-1)}\approx 5.6\cdot10^{-4} \overset{\text{\Cref{thm:ridgeToSpline}}}{\implies}\lambda=\frac{\lw}{2ng(0)}\approx1.7\cdot10^{-3} $). For polynomial regression we fitted a polonium of degree 3 to the data. K-nearest neighbor predicts for any input point $x$ the average of the output-values $\ytr_i$ of the $k=2$ nearest training input point $\xtr_i$. For $\NN_{\theta^{*{\color{hellgrau},\lambda}}}$ both layers of a very wide shallow $\NN$ were trained with $\ell_2$-regularization $\lambda\twonorm[\theta]^2$ with $\lambda=2.5\cdot10^{-4}$ (see \cite{Part3Arxiv} for the corresponding \Pfunc{}-Functional).}			\label{fig:RegDifferentGernalizationDifferentModels}
			\end{figure}

   \subsection{Why is the \texorpdfstring{\aregSplf{}}{adapted regression spline} so similar to the \texorpdfstring{\regSpl{}}{regression spline}?}\label{subsec:SimilaritWithoutSkipConnections}
   In the previous subsections, \Crefrange{thm:ridgeToSplineSkipConnection}{cor:moresimplifyingPfuncSkipConnections} show that different \RSN{}s and different spline models are equivalent if the linear skip connections from \Cref{subsec:linearSkip} are introduced. However, the empirical results displayed in \Cref{fig:jump,fig:sinUnderfit,fig:sinFit,fig:sinOverfit,fig:DifferentGernalizationDifferentModels,fig:RegDifferentGernalizationDifferentModels} %
   suggest that all these models $\RNw[{\wth[T]}],\RNR,\flpm,\RNaffine_{\wthaffine},\RNRaffine$ and $\flg$ are usually very similar for moderate values of $\lambda$ independent of whether or not they contain skip connections. This similarity was already discussed in \Cref{subsec:SummarizingTheIntuitionAboutARegSpl}, but in the following, we want to provide a more mathematical perspective on this.
   \begin{lemma}\label{le:OutSideGEquivToSkip}
       Under the assumptions of \Cref{thm:ridgeToSpline}, let $\alpha\in\Rp$ %
       and $\tilde{f}\in C^2(\R)$, then
       for
       \begin{equation}\label{eq:regCorrespondingToG}
       \text{reg}(a,c)=\frac{a^2}{\int_{\R\setminus[-\alpha,\alpha]}g(x)dx}+
       \frac{c^2}{\int_{\R\setminus[-\alpha,\alpha]}x^2g(x)dx}
       \end{equation}
       the following equations holds:
       \begin{equation}\label{eq:OutSideGEquivToSkip}
           \min_{f\in C^2: f|_{[-\alpha,\alpha]}=\tilde{f}|_{[-\alpha,\alpha]}}\Pgpmm(f)= \min_{f\in C^2: f|_{[-\alpha,\alpha]}=\tilde{f}|_{[-\alpha,\alpha]}}\Pgpmmaffine[g\ind_{[-\alpha,\alpha]}](f).
       \end{equation}
   \end{lemma}
   \begin{proof}
       Let $\left(f_+^{*\ggreg[,g\ind_{[-\alpha,\alpha]}]},
       f_-^{*\ggreg[,g\ind_{[-\alpha,\alpha]}]},
       a^{*\ggreg[,g\ind_{[-\alpha,\alpha]}]},
       c^{*\ggreg[,g\ind_{[-\alpha,\alpha]}]}\right)$ be the solution to the inner optimization problem for the solution of the outer optimization problem of the right-hand side of \cref{eq:OutSideGEquivToSkip}. Then with the help of some variational calculus and some symmetry arguments (see \Cref{as:easyReadable}), one can derive a solution to the inner optimization problem for the solution of the outer optimization problem of the left-hand side of \cref{eq:OutSideGEquivToSkip}\footnote{The inner optimization problem for the solution of the outer optimization problem of the left-hand side of \cref{eq:OutSideGEquivToSkip} is simply
       \[
       \min_{(f_+,f_-)\in\T:(f_++f_-)|_{[-\alpha,\alpha]}=\tilde{f}|_{[-\alpha,\alpha]}}\Pgpm(f_+,f_-).
       \]
       }:
       \begin{subequations}\label{eqs:fpmSuplementingSKipConnectionsSecondDerivative}
       \begin{align}
           {f_+^{*{\color{hellgrau},g}}}^{''}\!(x)&=\begin{cases}
               \frac{a^{*\ggreg[,g\ind_{[-\alpha,\alpha]}]} g(x)}{\int_{\R\setminus[-\alpha,\alpha]}g(\xi)d\xi}+
       \frac{c^{*\ggreg[,g\ind_{[-\alpha,\alpha]}]} |x| g(x)}{\int_{\R\setminus[-\alpha,\alpha]}\xi^2g(\xi)d\xi} & \text{if } x<-\alpha\\
       {f_+^{*\ggreg[,g\ind_{[-\alpha,\alpha]}]}}^{''}\!(x) & \text{if } x\in[-\alpha,\alpha]\\
       0 & \text{if } x>\alpha
           \end{cases}
           \\
           {f_-^{*{\color{hellgrau},g}}}^{''}\!(x)&=\begin{cases}
               \frac{-a^{*\ggreg[,g\ind_{[-\alpha,\alpha]}]} g(x)}{\int_{\R\setminus[-\alpha,\alpha]}g(\xi)d\xi}+
       \frac{c^{*\ggreg[,g\ind_{[-\alpha,\alpha]}]} |x| g(x)}{\int_{\R\setminus[-\alpha,\alpha]}\xi^2g(\xi)d\xi} & \text{if } x>\alpha\\
       {f_-^{*\ggreg[,g\ind_{[-\alpha,\alpha]}]}}^{''}\!(x) & \text{if } x\in[-\alpha,\alpha]\\
       0 & \text{if } x<-\alpha.
           \end{cases}
       \end{align}
              \end{subequations}
    Since $\left(f_+^{*{\color{hellgrau},g}},f_-^{*{\color{hellgrau},g}}\right)\in\T$, \cref{eqs:fpmSuplementingSKipConnectionsSecondDerivative} define
\begin{subequations}\label{eqs:fpmSuplementingSKipConnections}
       \begin{align}
           f_+^{*{\color{hellgrau},g}}(x)&=\int_{-\infty}^x\int_{-\infty}^t{f_+^{*{\color{hellgrau},g}}}^{''}\!(s)\,ds\,dt
           \\
           f_-^{*{\color{hellgrau},g}}(x)&=\int_x^{\infty}\int_t^{\infty}{f_-^{*{\color{hellgrau},g}}}^{''}\!(s)\,ds\,dt.
       \end{align}
              \end{subequations}
    From \cref{eqs:fpmSuplementingSKipConnectionsSecondDerivative}, we can directly compute
    \begin{multline*}
    \Pgpm\left(f_+^{*{\color{hellgrau},g}},f_-^{*{\color{hellgrau},g}}\right)
    =\\
    2g(0) \left(
		\int_{\supp (g\ind_{[-\alpha,\alpha]})} \frac{\left( {f_+^{*\ggreg[,g\ind_{[-\alpha,\alpha]}]}}^{''}\!(x) \right)^2}{g(x)} dx
		+\int_{\supp (g\ind_{[-\alpha,\alpha]})} \frac{\left( {f_-^{*\ggreg[,g\ind_{[-\alpha,\alpha]}]}}^{''}\!(x) \right)^2}{g(x)} dx\right.
  \\
		\left.\vphantom{\int_{\supp (g)} \frac{\left( {f_+^{*\ggreg[,g\ind_{[-\alpha,\alpha]}]}}^{''}\!(x) \right)^2}{g(x)} dx}
  +\frac{\left(a^{*\ggreg[,g\ind_{[-\alpha,\alpha]}]}\right)^2}{\int_{\R\setminus[-\alpha,\alpha]}g(x)dx}+
       \frac{\left(c^{*\ggreg[,g\ind_{[-\alpha,\alpha]}]}\right)^2}{\int_{\R\setminus[-\alpha,\alpha]}x^2g(x)dx}
		\right),
    \end{multline*}
    which concludes the proof.
   \end{proof}

   \begin{lemma}\label{le:OutSideGEquivToSkipWithLoss}
       Under the assumptions of \Cref{le:OutSideGEquivToSkip}, if $\supp(\nu)\subseteq[-\alpha,\alpha]$, then
       for every
       \[\flpm \in \argmin_{f\in \WT(\R)} \Ltrb{f}+\lambda \Pgpmm(f),\]
       there exists a
       \[f_{\pm}^{*\ggreg[,g\ind_{[-\alpha,\alpha]}]} \in \argmin_{f\in \WT(\R)} \Ltrb{f}+\lambda \Pgpmmaffine[g\ind_{[-\alpha,\alpha]}](f)\]
       with $\flpm|_{[-\alpha,\alpha]}=f_{\pm}^{*\ggreg[,g\ind_{[-\alpha,\alpha]}]}|_{[-\alpha,\alpha]}$ and vice versa.\footnote{In the case of convex $\Ltr$, the solutions $\flpm$ and$f_{\pm}^{*\ggreg[,g\ind_{[-\alpha,\alpha]}]}$ are unique, so we can directly say: \enquote{Under the assumptions of \Cref{le:OutSideGEquivToSkipWithLoss}, $\flpm|_{[-\alpha,\alpha]}=f_{\pm}^{*\ggreg[,g\ind_{[-\alpha,\alpha]}]}|_{[-\alpha,\alpha]}$ holds for convex $\Ltr{}$.}%
       }
   \end{lemma}
   \begin{proof}
       The proof directly follows from \Cref{le:OutSideGEquivToSkip}.
   \end{proof}
   \Cref{le:OutSideGEquivToSkip,le:OutSideGEquivToSkipWithLoss} show that every \aregSplf{} without skip connections behaves on every compact interval $[-\alpha,\alpha]$ exactly equivalently to $f_{\pm}^{*\ggreg[,g\ind_{[-\alpha,\alpha]}]}$ with skip connections whose regularization $\text{reg}$ depends on $g$ outside of $[-\alpha,\alpha]$ as specified in \cref{eq:regCorrespondingToG}. For large values of $\int_{\R\setminus[-\alpha,\alpha]}g(x)dx$ the regularization of $a$ becomes small and for large values of $\int_{\R\setminus[-\alpha,\alpha]}x^2g(x)dx$ the regularization of $c$ becomes small, because of \cref{eq:regCorrespondingToG}. For our experiments these two quantities were quite large\footnote{$\int_{\R\setminus[-1,1]}g(x)dx\href{https://www.wolframalpha.com/input?i2d=true&i=2Integrate\%5BDivide\%5Bs\%C2\%B2\%2C16Power\%5Bx\%2C4\%5D\%5D\%2C\%7Bx\%2C1\%2C\%E2\%88\%9E\%7D\%5D}{=}\frac{\sScale^2}{24}$ and $\int_{\R\setminus[-1,1]}x^2g(x)dx\href{https://www.wolframalpha.com/input?i2d=true&i=2Integrate\%5BDivide\%5Bs\%C2\%B2\%2C16Power\%5Bx\%2C2\%5D\%5D\%2C\%7Bx\%2C1\%2C\%E2\%88\%9E\%7D\%5D}{=}\frac{\sScale^2}{8}$ are rather large compared to the other quantities involved.}, which made the regularization of $a$ and $c$ quite negligible (unless you scale them up by a very large $\lambda$ as done in \Cref{subfig:strongerRegLinearWithoutSkip}) which explains the similarity of all these models $\RNw[{\wth[T]}],\RNR,\flpm,\RNaffine_{\wthaffine},\RNRaffine$ and $\flg$ in \Cref{fig:jump,fig:sinUnderfit,fig:sinFit,fig:sinOverfit,fig:DifferentGernalizationDifferentModels,fig:RegDifferentGernalizationDifferentModels}.%

\begin{proposition}[$\frac{g}{g(0)}\to1$]\label{prop:gTo1}
    Let $\Ltr$ be the square loss (i.e., $\Ltr$ satisfies \Cref{as:squaredloss}) with $\forall i \in \fromto[1]{N}: \xtr_i\in[-1,1]$ and $\exists (i,j) \in \fromto[1]{N}^2: \xtr_i\neq\xtr_j$, and let $(g_m)_{m\in\N}$ be a sequence of functions $g_m:\R\to\Rpz$ that have compact support and are continuous on their support.\footnote{All assumptions up to this footnote are only technicalities to simplify the proof and could be substantially weakened. The relevant assumptions follow after this footnote.}
    We further assume that $\forall m\in\N:\forall x\in[-1,1]:g_m(x)=g_m(0)\neq0$ and $\forall\alpha\in\Rp:\lim_{m\to\infty}\int_{\R\setminus[-\alpha,\alpha]}\frac{g_m}{g_m(0)}(x)dx=\infty$.\footnote{E.g., $g_m=\ind_{[-m,m]}$ would satisfy all these constraints on $g_m$.} Then under these assumptions, for every $\lambda\in\Rp$ and every compact interval $K$,
    \[\lim_{m\to\infty}\sobnorm[{\flpm[g_m]-\fl}]=0\]
    holds without any need for skip connections, where is $\flpm[g_m]$ is the \aregSpl{} from \Cref{def:adaptedSplineReg} and $\fl$ is the classical \regSpl{} from \Cref{def:splineReg}.
\end{proposition}
\begin{proof}
We define $\alpha=\max_{x\in K\cup\{1\}}|x|$.
\Cref{prop:constantg} tells us that $\forall m \in \N: \fl=\flgg[g_m]=\flgg[g_m\ind_{[-\alpha,\alpha]}]$, since the training input data is within $[-1,1]$ and $\alpha>1$.
\Cref{cor:simplifyingPfuncSkipConnections} tells us that $\forall m \in \N: \Pg[{g_m\ind_{[-\alpha,\alpha]}}]=\PgpmmaffineReg[g_m\ind_{[-\alpha,\alpha]}]{0}$.
We define
\begin{equation}\label{eq:regCorrespondingToGm}
       \text{reg}_m(a,c)=\frac{a^2}{\int_{\R\setminus[-\alpha,\alpha]}g_m(x)dx}+
       \frac{c^2}{\int_{\R\setminus[-\alpha,\alpha]}x^2g_m(x)dx}
       \end{equation}
       corresponding to \cref{eq:regCorrespondingToG}. 
       Note that
\begin{equation*}
		\Pgpmmaffine(f):=  2g(0)\underset{\underset{f=f_++f_-+a(\cdot)+c}{((f_+, f_-),a,b)\in\T\times\R\times\R}}{\min} \left(
		\Pgpm\left(f_+,f_-\right)
		+\text{reg}(a,c)
		\right).
		\end{equation*}
  We use the notation
   \[f_{\pm}^{*\ggreg[,g]} :\in \argmin_{f\in \WT(\R)} \Ltrb{f}+\lambda \Pgpmmaffine[g](f).\]
   Obviously, the assumption $\forall\alpha\in\Rp:\lim_{m\to\infty}\int_{\R\setminus[-\alpha,\alpha]}\frac{g_m}{g_m(0)}(x)dx=\infty$ implies that $\forall\alpha\in\Rp: \lim_{m\to\infty}\int_{\R\setminus[-\alpha,\alpha]}x^2\frac{g_m}{g_m(0)}(x)dx = \infty$, which means that $2g_m(0)\text{reg}_m$ vanishes in the limit $m\to\infty$.
    By using the strong convexity of $\Pgpm[g_m\ind_{[-\alpha,\alpha]}]$ (see \Cref{le:PfuncStronglyConvex}) one can prove that $\lim_{m\to\infty}\sobnorm[{f_{\pm}^{*{\color{hellgrau},g_m\ind_{[-\alpha,\alpha]},\text{reg}_m}}
    -f_{\pm}^{*{\color{hellgrau},g_m\ind_{[-\alpha,\alpha]},0}} }]=0.$
    And \Cref{le:OutSideGEquivToSkipWithLoss} shows that $\flpm[g_m]|_{[-\alpha,\alpha]}=f_{\pm}^{*{\color{hellgrau},g_m\ind_{[-\alpha,\alpha]},\text{reg}_m}}|_{[-\alpha,\alpha]}$. Combining all these steps concludes the proof.
\end{proof}

	\end{document}